\documentclass[final]{colt2020}

\usepackage{tikz} 
\usepackage{amsmath}
\usepackage{amssymb}
\usepackage{amsfonts}
\usepackage{mathtools}
\usepackage{hyperref}
\usepackage{color}
\usepackage{comment}
\usepackage{algpseudocode}
\usepackage{algorithm}
\usepackage[shortlabels]{enumitem}
\usepackage{caption}
\usepackage{subcaption}
\usepackage{float}
\usepackage{minitoc}
\usepackage{stackengine}
\setstackEOL{\\}
\makeatletter
\def\BState{\State\hskip-\ALG@thistlm}
\makeatother

\allowdisplaybreaks

\renewcommand \thepart{}
\renewcommand \partname{}

\newcommand{\todo}[1]{}
\newcommand{\com}[1]{}
\newcommand{\kevin}[1]{}

\DeclareMathOperator*{\argmin}{arg\,min}
\DeclareMathOperator*{\argmax}{arg\,max}

\numberwithin{theorem}{section}

\title[Active Learning for Identification of Linear Dynamical Systems]{Active Learning for Identification of Linear Dynamical Systems}
\usepackage{times}

\coltauthor{%
 \Name{Andrew Wagenmaker} \Email{ajwagen@cs.washington.edu} \\
 \Name{Kevin Jamieson} \Email{jamieson@cs.washington.edu}\\
 \addr Paul G. Allen School of Computer Science \& Engineering, University of Washington, Seattle, WA 98195%
}

\begin{document}

\maketitle

\begin{abstract}%
We propose an algorithm to actively estimate the parameters of a linear dynamical system. Given complete control over the system's input, our algorithm adaptively chooses the inputs to accelerate estimation. We show a finite time bound quantifying the estimation rate our algorithm attains and prove matching upper and lower bounds which guarantee its asymptotic optimality, up to constants. In addition, we show that this optimal rate is unattainable when using Gaussian noise to excite the system, even with optimally tuned covariance, and analyze several examples where our algorithm provably improves over rates obtained by playing noise. Our analysis critically relies on a novel result quantifying the error in estimating the parameters of a dynamical system when arbitrary periodic inputs are being played. We conclude with numerical examples that illustrate the effectiveness of our algorithm in practice.
\end{abstract}


\section{Introduction}
System identification is a fundamental problem in control theory, reinforcement learning, econometrics, and time-series modeling. Given observations of the input-output behavior of a dynamical system, system identification seeks to estimate the parameters of the system. When the governing dynamics cannot be derived from first principles, this is an important tool for modeling the behavior of a system, allowing for downstream analysis and engineering. In this work we focus on the simplest possible dynamical system model---discrete-time, linear dynamical systems. Several recent works  \cite{simchowitz2018learning, sarkar2018how} have shown sharp rates for estimating the parameters of such systems in the \textit{passive} case---where the system is driven by random noise. Here we seek to understand \textit{active} system identification---given complete control over the inputs, how can we best excite the system to accelerate estimation?  Dating back to the 1970s, significant attention has been given to the problem of how to best excite systems for estimation \cite{mehra1976synthesis,goodwin1977dynamic,bombois2011optimal} yet these works typically lack theoretical guarantees. To the best of our knowledge, we present the first provably correct method for active system identification. We show finite time and asymptotic sample complexity guarantees and characterize settings in which active input design yields performance improvements.

Formally, we consider linear dynamical systems (LDS) of the form:
\begin{equation}\label{eq:lds}
x_{t+1} = A_* x_t + B_* u_t + \eta_t
\end{equation}
where $A_* \in \mathbb{R}^{d \times d}$ is unknown, $B_* \in \mathbb{R}^{d \times p}$, and $\eta_t$ is unobserved process noise. We choose the input $u_t$ sequentially, observe the state $x_t$, and wish to estimate $A_*$ from this data. For simplicity and ease of exposition, we assume $B_*$ is known, though all our results can be extended to the case where $B_*$ is unknown. From an engineering perspective, assuming $B_*$ is known is a reasonable assumption as one may have knowledge of $B_*$ from the design of the system actuation.  Throughout, we assume that $\rho(A_*) < 1$ where $\rho(A_*)$ is the spectral radius of $A_*$. We are interested in estimating $A_*$ in the spectral norm, in the case where our input is constrained to have bounded energy, that is: $\mathbb{E} \left [ \frac{1}{T} \sum_{t=1}^T u_t^\top u_t \right ] \leq \gamma^2$ for some constant $\gamma^2$.  

As we will show, the fundamental quantity that determines the sample complexity of estimation is the minimum eigenvalue of the covariates: $\lambda_{\min} \left ( \sum_{t=1}^T x_t x_t^\top \right )$. Optimally exciting the system is then equivalent to maximizing this quantity subject to the input power constraints. This quantity, however, depends on $A_*$, the parameter we wish to estimate, so cannot be optimized in practice.

Our main contribution is an algorithm which balances this tradeoff---progressively updating the inputs as the estimates of $A_*$ improve---and finite time bounds quantifying the estimation rate it achieves, as well as the number of samples necessary to guarantee the optimally exciting inputs are being played. In addition, we present a lower bound and asymptotic upper bound guaranteeing the asymptotic optimality of our algorithm. We show that playing Gaussian noise, even with an optimally tuned covariance, is insufficient to achieve this optimal rate. Our algorithm can be seen as an instance of adaptive E-optimal design \cite{pronzato2013design}. 

An important piece in our analysis is a new finite-time bound on the estimation error $\| A_* - \hat{A} \|_2$ that holds when arbitrary periodic inputs are being played. Previous works  \cite{simchowitz2018learning,sarkar2018how,dean2018regret} only consider inputs that are Gaussian or state feedback. These works emphasize obtaining bounds that scale properly with the spectral radius of the system. Following this, we develop bounds that avoid a poor scaling with the spectral radius. To the best of our knowledge, this is a novel result and may be of independent interest.

\subsection{Related Works}
A significant body of work exists on how to optimally excite dynamical systems for identification \cite{mehra1976synthesis,goodwin1977dynamic,jansson2005input,gevers2009identification,manchester2010input,hagg2013robust}. An excellent survey of classical results can be found in \cite{mehra1974optimal} and a more recent survey in \cite{bombois2011optimal}. Broadly speaking, earlier works tended to focus on designing inputs so as to be optimal with respect to traditional experimental design objectives. More recent works \cite{hjalmarsson1996model,hildebrand2002identification,katselis2012application} have focused on designing inputs to meet certain task-specific objectives---for instance, identifying a system for the purpose of control. 

A primary difficulty in designing inputs for identification is that the design criteria, often some function of the Fisher Information Matrix, depend on the unknown parameters of the system. Several different approaches have been proposed to overcome this challenge. One line of work  \cite{rojas2007robust,rojas2011robustness,larsson2012robust,hagg2013robust} performs robust experimental design and optimizes a minimax objective. More comparable to our approach are works which perform adaptive experimental design \cite{lindqvist2001identification,gerencser2005adaptive,barenthin2005applications,gerencser2007adaptive,gerencser2009identification}---alternating between estimating the unknown parameters and designing inputs based on the current estimates.

Existing works in active system identification lack sound theoretical guarantees and too often specialize results to single-input single-output systems. While several results guarantee asymptotic consistency  \cite{gerencser2007adaptive,gerencser2009identification}, most proposed approaches are heuristic and are validated only through examples. To our knowledge, no finite-time performance bounds exist. In addition, many works seek to optimize quantities that only describe the asymptotic behavior of the system---for instance minimizing the asymptotic variance---and it is unclear and unjustified if these are the correct quantities to optimize for over a finite time interval. Finally, existing works do not give precise, explicit algorithms.

Recently, considerable interest has been shown in the machine learning community towards obtaining finite-time performance guarantees for system identification and control problems. The latter category has primarily centered around developing finite time regret bounds for the LQR problem with unknown dynamics \cite{abbasi2011regret,dean2017sample,dean2018regret,mania2019certainty,dean2019safely,cohen2019learning}. Recent results in system identification have focused on obtaining finite time high probability bounds on the estimation error of the system's parameters when observing the evolution over time \cite{tu2017non,faradonbeh2018finite,hazan2018spectral,hardt2018gradient,simchowitz2018learning,sarkar2018how,oymak2019non,simchowitz2019learning,sarkar2019finite,tsiamis2019finite}. Existing results rely on excitation from random noise to guarantee learning and do not consider the problem of learning with arbitrary sequences of inputs or optimally choosing inputs for excitation. 

In the context of the existing literature, this work can be seen as the first rigorous treatment of active system identification and the first work to provide finite-time performance guarantees for the problem---bridging the gap between classical approaches  and modern machine learning techniques. Indeed, our algorithm is similar to the adaptive input design approach in \cite{lindqvist2001identification}; our work can be seen as making their algorithm more precise and providing finite-time performance and asymptotic optimality guarantees. Our analysis framework is general enough it could be extended to different experimental design criteria proposed in the existing literature.

\subsection{Notation}
We will let $\rho(A)$ denote the spectral radius of $A$. $\| \ \cdot \ \|_2$ denotes the spectral norm of a matrix. $\tilde{\mathcal{O}}( \ \cdot \ )$ hides log factors. We assume throughout that $\eta_t \sim \mathcal{N}(0,\sigma^2 I)$ though all results can be extended to more general noise distributions. Let:
$$\Gamma_t(A) = {\textstyle \sum}_{s=0}^{t-1} (A^{s}) (A^{s})^\top, \ \ \ \  \Gamma_{t}^{B}(A) = {\textstyle \sum}_{s=0}^{t-1} (A^{s} B) (A^{s} B)^\top  $$
and $\Gamma_t := \Gamma_t(A_*)$, $\Gamma_t^{B_*} := \Gamma_t^{B_*}(A_*)$. $\Gamma_t$ is the expected value of $x_t x_t^\top$ when $u_t = 0, \forall t$, and $\Gamma_t^{B_*}$ is the expected value of $x_t x_t^\top$ when $u_t \sim \mathcal{N}(0,I), \eta_t = 0, \forall t$. In the case when the input is a deterministic, periodic signal of period $k$ and $\frac{1}{k} \sum_{t = 1}^k u_t^\top u_t = \gamma^2$, then setting $\eta_t = 0$ and applying this input on the system with parameters $A$ and $B$ for all $t$, we denote the steady state covariates as:
$$ \Gamma_{k}^u(A,B) = \lim_{T \rightarrow \infty} \frac{1}{\gamma^2 T} {\textstyle \sum}_{t=1}^T x_t x_t^\top \overset{(a)}{=}  \frac{1}{\gamma^2 k^2} {\textstyle \sum}_{\ell = 0}^{k-1} (e^{j \frac{2 \pi \ell}{k}} I - A)^{-1} B_* U_\ell U_\ell^H B_*^H (e^{j \frac{2 \pi \ell}{k}} I - A)^{-H}$$
where $U_\ell$ denotes the Discrete Fourier Transform of $\{ u_t \}_{t=1}^k$. Here $(a)$ holds by Parseval's Theorem. Let $\Gamma_k^u := \Gamma_k^u (A_*,B_*)$. $\bar{\Gamma}_T$ will denote an upper bound on the covariates: $\sum_{t=1}^T x_t x_t^\top \preceq T \bar{\Gamma}_T$. We will specify its precise form as needed. To aid in analyzing the transient behavior of a system, let:
$$ \beta(A) := \sup \{ \| A^k \|_2 \left ( 1/2 + \rho(A)/2 \right )^{-k} \ : \ k \geq 0 \} $$
$\beta(A)$ is then the smallest value such that $\| A^k \|_2 \leq \beta(A) (1/2 + \rho(A)/2)^k$ for all $k \geq 0$, and is always finite. We give a more thorough discussion of this parameter in Appendix \ref{sec:notation}. To determine the optimal inputs, we will solve the following optimization problem. As we make clear in Section \ref{sec:asympt_opt}, the fundamental quantity that controls the sample complexity of estimation is the minimum eigenvalue of the covariates, the quantity \texttt{OptInput} maximizes:
\begin{align*}
\texttt{OptInput}_k(A,B,\gamma^2,\mathcal{I},\{ x_t \}_{t=1}^T) := \begin{matrix*}[l]  \argmax_{u_1,...,u_k \in \mathbb{R}^p} \ \lambda_{\min} \left ( \gamma^2 \bar{T} \Gamma_k^u(A,B) +  \sum_{t=1}^T x_t x_t^\top \right ) \\
 \text{s.t.} \ \ u_1,...,u_k \in \bar{\mathcal{U}}_{\gamma^2}, U_\ell = 0, \forall \ell \not\in \mathcal{I}
 \end{matrix*}
\end{align*}
Here $\mathcal{I} \subseteq [k]$ is the set of frequencies we are optimizing over,  $\bar{T}$ is the time horizon we will play the inputs for, $U_\ell$ is the DFT of $u_1,...,u_k$, and $\bar{\mathcal{U}}_{\gamma^2}$ is the set of mean-zero signals of length $k$ with average power bounded by $\gamma^2$. The constraint that the signal be mean zero is for technical reasons and does not affect the results. We let $\mathcal{U}_{\gamma^2}$ denote the same set without the constraint that the signal be mean 0. In some cases we will overload notation, letting $\texttt{OptInput}_k(A,B_*,\gamma^2,\mathcal{I},M)$ denote $\texttt{OptInput}_k(A,B_*,\gamma^2,\mathcal{I},\{ x_t \}_{t=1}^T)$ but with the $\sum_{t=1}^T x_t x_t^\top$ term in the optimization replaced by $M$. In addition, we will sometimes use $\texttt{OptInput}$ to refer to the maximum value of the optimization, and sometimes to refer to the inputs attaining that maximum---it will be clear from context which we are referring to.

\section{Main Results}
Algorithm \ref{alg:active_lds_noise} proceeds in epochs, successively improving its input design as its estimate of $A_*$ improves. At each epoch, the input computed in the previous epoch is played (line \ref{line:play_inputs}), and $A_*$ estimated from the data collected (line \ref{line:estimateA}). Using this estimate, a set of inputs are designed to excite the estimated system (line \ref{line:design_input}), and these inputs are played on the real system in the subsequent epoch, yielding a new estimate of $A_*$. This procedure continues with exponentially growing epoch length.

\begin{algorithm}[H] 
\begin{algorithmic}[1]
\State \textbf{Input:} Confidence $\delta$, input power $\gamma^2$, $T_0$ (Default: $T_0 = 100$), $k_0$ (Default: $k_0 = 20$), \\ \hspace{2cm} $FT$ (Default: \texttt{True})
\State $T \leftarrow T_0$  
\State Run LDS for $T_0$ steps with $u_t  \sim \mathcal{N}(0,\frac{\gamma^2}{p} I )$  
\State $\hat{A}_0 \leftarrow \argmin_{A} \ \sum_{t=1}^{T_0} \| x_{t+1} - A x_t - B_* u_t \|_2^2$ 
\State $\epsilon_0 \leftarrow \sigma \left \| \left ( \sum_{t=1}^{T_0} x_t x_t^\top \right )^{-1/2} \right \|_2 \sqrt{16 \log \frac{5^d}{\delta}  + 8 \log \det (\bar{\Gamma}_{T_0} ( \sigma^2 \Gamma_{k_0} + \gamma^2  \Gamma_{k_0}^{B_*} /p )^{-1} + I)}$ 
\State $ k_1 \leftarrow 2 k_{0}$ 
\State $\tilde{u}^1, \sigma_u^2 \leftarrow$ \texttt{UpdateInputs}$(\hat{A}_0,B_*,\{ x_t \}_{t=1}^T,\gamma^2,k_1,\epsilon_0,FT)$ 
\For{$i=1,2,3,...$}
 	 \State $T_i \leftarrow 3 T_{i-1}$, $T \leftarrow T + T_i$ 
	 \State Run LDS for $T_i$ steps with $u_t = \tilde{u}_t^i + \eta_t^u$, $\eta_t^u \sim \mathcal{N}(0, \frac{\gamma^2}{2p} I )$ \label{line:play_inputs}
	 \State $\hat{A}_i \leftarrow \argmin_{A} \ \sum_{t=1}^{T} \| x_{t+1} - A x_t - B_* u_t \|_2^2$ \label{line:estimateA}
 	 \State $\epsilon_i \leftarrow \sigma \left \| \left ( \sum_{t=1}^{T} x_t x_t^\top \right )^{-1/2} \right \|_2 \sqrt{16 \log \frac{5^d}{\delta}  + 8 \log \det (\bar{\Gamma}_{T} ( \sigma^2 \Gamma_{k_i} + \sigma_u^2 \Gamma_{k_i}^{B_*} + \gamma^2 \Gamma_{k_i}^{u_i})^{-1} + I)}$  
	 \State $k_{i+1} \leftarrow 2k_i$ 
 	 \State $\tilde{u}^{i+1} \leftarrow$ \texttt{UpdateInputs}$(\hat{A}_i,B_*,\{ x_t \}_{t=1}^T,\gamma^2,k_{i+1},\epsilon_i,FT)$ \label{line:design_input}
\EndFor
\end{algorithmic}
\caption{Active Estimation of LDS}\label{alg:active_lds_noise}
\end{algorithm}

\begin{algorithm}[H]
\begin{algorithmic}[1]
\caption*{\texttt{UpdateInputs} pseudocode (full definition in Appendix \ref{sec:notation})}
\Function{\texttt{UpdateInputs}}{$A$,$B$,$\{ x_t \}_{t=1}^T$,$\gamma^2$,$k$,$\epsilon$,$FT$}
    \State Check if $\epsilon$ small enough to plan with all frequencies, if so set $\mathcal{I} = [k]$ 
    \State Otherwise set $\mathcal{I}$ to include frequencies we can guarantee will sufficiently excite the system
    \State \textbf{if} $FT == $ \texttt{True}: \Return $\texttt{OptInput}_k(A,B, \frac{\gamma^2}{2}, \mathcal{I},\{ x_t \}_{t=1}^T)$ 
    \State \textbf{else}: \Return $\texttt{OptInput}_k  (A,B, \frac{\gamma^2}{2} , \mathcal{I}, (2T + T_0) \sigma^2 \Gamma_k(A)  )$ 
\EndFunction
\end{algorithmic}
\end{algorithm}

The $FT$ flag in \texttt{UpdateInputs} controls how the inputs are designed. With $FT = $ \texttt{True} (the finite time case), the algorithm does not take into account the expected future contribution due to noise when designing the inputs. Results for this case are outlined in Section \ref{sec:ft_alg_performance}. With $FT = $ \texttt{False} (the asymptotic case), the algorithm does take into account the estimated future contribution due to noise when designing the inputs. Results for this case are outlined in Section \ref{sec:asympt_opt}.

\subsection{Asymptotic Optimality of Algorithm \ref{alg:active_lds_noise}}\label{sec:asympt_opt}
We show that our algorithm is asymptotically optimal---up to constants, no algorithm can estimate $A_*$ more quickly as $\delta \rightarrow 0$. We first present a lower bound for estimating linear dynamical systems actively. We call an algorithm $(\epsilon,\delta)$-locally-stable in $A$ if there exists a finite time $\tau$ such that for all $t \geq \tau$ and all $A' \in \mathcal{B}(A, 3 \epsilon)$: $\mathbb{P}_{A'}( \| \hat{A}_t - A' \|_2 \leq \epsilon ) \geq 1 - \delta$. Here $\mathbb{P}_{A'}$ is the measure induced when the true matrix is $A'$, $\mathcal{B}(A, 3 \epsilon) := \{ A' \in \mathbb{R}^{d \times d} \ : \ \| A - A' \|_2 \leq 3 \epsilon \}$, and $\hat{A}_t$ is the estimate obtained by the algorithm after $t$ observations. The sample complexity $\tau_{\epsilon \delta}$ is the infimum of all times $\tau$ satisfying the above definition. This condition was introduced in \cite{jedra2019sample} and allows us to avoid trivial algorithms that simply return $\hat{A}_t = A_*$ for all time. Also define:
\begin{equation*}
\max_{u \in \mathcal{U}_{\gamma^2}} \lambda_{\min} (\sigma^2 \Gamma_{\infty} + \gamma^2 \Gamma_{\infty}^{u}) := \lim_{i \rightarrow \infty} \max_{u \in \mathcal{U}_{\gamma^2}} \lambda_{\min} ( \sigma^2 \Gamma_{2^i} + \gamma^2 \Gamma_{2^i}^{u} ) 
\end{equation*}
Note that, by Lemma \ref{lem:sln_limit} and Lemma \ref{lem:cov_limit_existence}, this limit exists and is equal to the limit obtained by replacing $2^i$ with any other sequence $n_i \rightarrow \infty$ as $i \rightarrow \infty$.

\begin{theorem}\label{thm:lb} 
Assume there exists finite $k$ such that the input $u_t$ satisfies $\frac{1}{k} \sum_{t=1}^k u_{s+t}^\top u_{s+t}$ $\leq \gamma^2$ for any $s \geq 0$. Then for $(\epsilon, \delta)$ small enough, any $(\epsilon,\delta)$-locally-stable in $A_*$ algorithm will have:
\begin{equation*}
\tau_{\epsilon \delta} \geq \frac{\sigma^2 \epsilon^{-2} / 8}{ \max_{u \in \mathcal{U}_{\gamma^2}}   \lambda_{\min} \left ( \sigma^2 \Gamma_{\infty}  + \gamma^2 \Gamma_{\infty}^{u} \right )  } \log \frac{1}{2.4 \delta}.
\end{equation*}
\end{theorem}

\begin{theorem}\label{thm:asymp} 
Assume we are running Algorithm \ref{alg:active_lds_noise} with $FT = $ \texttt{False}. Then for any $\delta, \epsilon \in (0,1)$, there exists a deterministic $\tau_{\epsilon \delta}$ such that, for any $T \geq \tau_{\epsilon \delta}$ where $T$ is at an epoch boundary, we have: $\mathbb{P} \left [ \| \hat{A} - A_* \|_2 > \epsilon \right ] \leq \delta$, and, for small enough $\epsilon$ and some universal constant $C$:
\begin{equation*}
\lim_{\delta \rightarrow 0} \ \frac{\tau_{\epsilon \delta}}{\log(1/\delta)} \leq \frac{C \sigma^2 \epsilon^{-2} }{\max_{u \in \mathcal{U}_{\gamma^2}} \lambda_{\min} (\sigma^2 \Gamma_{\infty} + \gamma^2 \Gamma_{\infty}^{u})}.
\end{equation*}
\end{theorem}

The proof of Theorem \ref{thm:lb} is given in Section \ref{sec:lower_bound} and the proof of Theorem \ref{thm:asymp} is given in Section \ref{sec:proof_asymp}. It follows that up to constant factors, Algorithm \ref{alg:active_lds_noise} is asymptotically optimal. The fundamental value present in both the upper and lower bound controlling the sample complexity of estimation is $ \lambda_{\min} (\sigma^2 \Gamma_{\infty} + \gamma^2 \Gamma_{\infty}^{u})$, the minimum eigenvalue of the expected covariates when the input $u$ is being played. Optimally exciting the system for identification is then equivalent to choosing $u$ so as to maximize $\lambda_{\min} (\sigma^2 \Gamma_{\infty} + \gamma^2 \Gamma_{\infty}^{u})$.

\subsection{Suboptimality of Colored Noise}\label{sec:subopt_noise} 
While Theorem \ref{thm:lb} and Theorem \ref{thm:asymp} together show that the optimal performance can be attained in the limit by periodic inputs, it may seem reasonable that one could attain a similar rate by playing the optimal noise---setting $u_t \sim \mathcal{N}(0,\Sigma^*)$ for the optimal choice of $\Sigma^*$ that satisfies the expected power constraint. We show this is false. Consider the following example. Let $A_*$ be PSD with eigenvalues $\lambda = [\lambda_1,\ldots,\lambda_d]$, $B_* = I$,  and assume that $\gamma^2 \gg \sigma^2$. We show in the proof of Corollary \ref{cor:symmetric_a_informal} that $ \max_{u \in \mathcal{U}_{\gamma^2}} \lambda_{\min} (\sigma^2 \Gamma_{\infty} + \gamma^2 \Gamma_{\infty}^{u}) = \Theta \left ( \gamma^2 / \| \mathbf{1} - \lambda \|_2^2 \right )$. In contrast, when playing $u_t \sim \mathcal{N}(0,\Sigma^*)$, as we show in Appendix \ref{sec:subopt_noise_proof}, we will have that $\lambda_{\min}(\sigma^2 \Gamma_\infty + \sum_{s=0}^\infty A^s \Sigma^* (A^s)^\top) = \Theta(\gamma^2 / \| \mathbf{1} - \lambda \|_1)$. Note here that $\lambda_{\min}(\sigma^2 \Gamma_\infty + \sum_{s=0}^\infty A^s \Sigma^* (A^s)^\top)$ upper bounds the minimum eigenvalue of the expected covariates when  $u_t \sim \mathcal{N}(0,\Sigma^*)$. Depending on the values of $\lambda$, there is clear gap between these quantities. For example, if $\lambda_i = 1 - 1/d$ for $i = 1,...,d$, the upper bound on the sample complexity of our algorithm is $\Theta( \sigma^2 \epsilon^{-2} / ( d \gamma^2))$ while the lower bound on the sample complexity when playing optimal noise is $\Theta(\sigma^2 \epsilon^{-2} / \gamma^2)$, a gap of $\Theta(d)$. Note that existing works on system identification \cite{simchowitz2018learning,sarkar2018how} only apply to the case when the input is zero-mean noise and are thus insufficient to guarantee optimal rates.

\subsection{Finite Time Performance of Algorithm \ref{alg:active_lds_noise}}\label{sec:ft_alg_performance}
We next present our main result quantifying the finite time performance of Algorithm \ref{alg:active_lds_noise}. Throughout, we let $T = \sum_{j=0}^i T_j$, the total time elapsed after $i$ epochs, and $k(T)$ the value of $k_i$ after $T$ steps. If $T$ is at an epoch boundary, $k(T) = k_0 2^{\log(2T/T_0 + 1)/\log 3 - 1}  \approx \mathcal{O}((T/T_0)^{0.63})$. 
\begin{theorem}\label{cor:largek_alg_opt_informal} 
(\textbf{Informal}) Assume that $T_0$ is chosen sufficiently large relative to $k_0$. Then for $T$ large enough, with $FT = $ \texttt{True}, Algorithm \ref{alg:active_lds_noise} will achieve the following rate:
$$ \mathbb{P} \left [ \| \hat{A}- A_* \|_2 \leq C \sigma \sqrt{\frac{\log \frac{1}{\delta} + d + \log \det \left ( \bar{\Gamma}_{T } \left ( \sigma^2 \Gamma_{k(T)} + \frac{\gamma^2}{p} \Gamma_{k(T)}^{B_*} \right )^{-1} + I \right )}{T \lambda_{\min} \left (  \sigma^2 \Gamma_{k(T)}  + \gamma^2 \Gamma_{k(T)}^{u^*} \right )}} \right ] \geq 1 - 9 \delta$$
and will produce inputs satisfying $\mathbb{E} \left [ 1/T \sum_{t=1}^T u_t^\top u_t \right ] \leq \gamma^2$. Here $C$ is a universal constant, $u^*$ is the solution to {\normalfont \texttt{OptInput}}$_{k(T)}(A_*,B_*,\gamma^2,k(T),0)$, and $\bar{\Gamma}_T = I \cdot \mathcal{O}(\beta(A_*)^2 \gamma^2 T /(1 - \rho(A_*))^2 )$.
\end{theorem}

Note that our finite time rate critically depends on the minimum eigenvalue of the expected covariates. At a high level, Theorem \ref{cor:largek_alg_opt_informal} provides a finite sample bound on the error in the estimates produced by Algorithm \ref{alg:active_lds_noise} and states that once $T$ is large enough, despite lacking knowledge of the true system parameters, Algorithm \ref{alg:active_lds_noise} will play inputs that maximize $\lambda_{\min}(\gamma^2 \Gamma_k^u)$. As was shown in Section \ref{sec:asympt_opt}, the fundamental quantity that controls the estimation rate is $ \lambda_{\min} (\sigma^2 \Gamma_{\infty} + \gamma^2 \Gamma_{\infty}^{u})$ which, in finite time, can be thought of as $ \lambda_{\min} (\sigma^2 \Gamma_{k} + \gamma^2 \Gamma_{k}^{u})$. When $\gamma^2 \gg \sigma^2$, maximizing $\lambda_{\min}(\gamma^2 \Gamma_k^u)$ is essentially equivalent to maximizing $ \lambda_{\min} (\sigma^2 \Gamma_{k} + \gamma^2 \Gamma_{k}^{u})$. Theorem \ref{cor:largek_alg_opt_informal} then guarantees in this case that Algorithm \ref{alg:active_lds_noise} plays the inputs that best excite the system for estimation.

The proof of this theorem is sketched in Section \ref{sec:proof_sketch} and formally proved in Section \ref{sec:alg_ft_full_proof}. A full version of this result is presented as Theorem \ref{cor:largek_alg_opt} in Appendix \ref{sec:alg_perf}, where we quantify formally how large $T$ must be for the rate given in Theorem \ref{cor:largek_alg_opt_informal} to apply. Corollary \ref{cor:symmetric_a_informal} works this out explicitly in a simplified setting. Intuitively, $T$ must be large enough for the transient effects of the last input to have dissipated, and for $\epsilon_{i-1}$ to be small enough to guarantee we are playing inputs that achieve nearly optimal performance. The former quantity scales as $\tilde{\mathcal{O}} \left ( 1/(1 - \rho(A_*)) \right )$. The latter depends on the system parameters in a complicated fashion. In the case where $A_*$ is diagonalizable with largest and smallest magnitude eigenvalues $\lambda_1$ and $\lambda_d$, respectively, and $B_*$ allows for sufficient excitation of all modes, then when $\frac{1}{1 - | \lambda_1 |} \gg \frac{1}{1 - |\lambda_d|}$, it will behave like $\tilde{\mathcal{O}}\left ((1 - |\lambda_d|)^4/(1 - |\lambda_1|)^4 \right )$. If $|\lambda_1| \approx |\lambda_d|$ it will behave like $\tilde{\mathcal{O}}(1/(1 - |\lambda_1|)^2)$.

\begin{remark}
If $B_*$ is also unknown, it is still possible to run a procedure similar to Algorithm \ref{alg:active_lds_noise}, choosing the inputs to improve estimation of both $A_*$ and $B_*$ simultaneously. In this case, we minimize the same least squares objective but now over both $A$ and $B$. Theorem \ref{thm:concentration2} can be modified to bound the error $ \left \| \begin{bmatrix} A_* & B_* \end{bmatrix} - \begin{bmatrix} \hat{A} & \hat{B} \end{bmatrix} \right \|_2$, but the error scales instead with:
$$ \lambda_{\min} \left ( \sum_{t=1}^T \begin{bmatrix} x_t x_t^\top & x_t u_t^\top \\ u_t x_t^\top & u_t u_t^\top \end{bmatrix} \right )$$
In this setting, the optimal design is one that maximizes this minimum eigenvalue. To obtain a result similar to Theorem \ref{cor:largek_alg_opt_informal}, a version of Theorem \ref{thm:opt_sln_perturbation_informal} is needed to quantify how suboptimal our choice of input may be given only estimates of $A_*$ and $B_*$. A fairly straightforward extension of the argument used to obtain Theorem \ref{thm:opt_sln_perturbation_informal} can be used to argue such a bound, allowing a version of Theorem \ref{cor:largek_alg_opt_informal} to be proved. 
\end{remark}

\begin{remark}
The update of $\epsilon_i$ in Algorithm \ref{alg:active_lds_noise} requires knowledge of the true system parameters to compute $\bar{\Gamma}_T, \Gamma_{k_i}, \Gamma_{k_i}^{B_*}$. In practice, bootstrapped estimates of these quantities could be used. Further, these terms only appear logarithmically and will not be the dominant terms in the expression. Experimentally, we found that greedily designing our inputs with respect to $\hat{A}_i$, equivalent to solving \texttt{UpdateInputs}$(\hat{A}_i,B_*,\{ x_t \}_{t=1}^T,\gamma^2,2k(T),0,FT)$, yielded better performance and did not require any estimate of $\epsilon_{i}$. 
\end{remark}

\subsection{Estimating Dynamical Systems With Periodic Inputs}
As was shown in Section \ref{sec:subopt_noise}, exciting a system with random noise is insufficient to obtain optimal estimation rates. Relying on carefully designed periodic inputs, Algorithm \ref{alg:active_lds_noise} is able to attain this optimal rate. Showing this critically requires bounding the estimation error when arbitrary periodic inputs are being played. The following result quantifies this and can be thought of as a novel extension of \cite{simchowitz2018learning,sarkar2018how} to non-noise inputs. This result may be of independent interest and is proved in Section \ref{sec:concentration}.

\begin{theorem}\label{thm:concentration2}
Assume that we start from initial state $x_0$ and play input $u_t = \tilde{u}_t + \eta_t^u$ where $\tilde{u}_t$ is deterministic with period $k$ and average power $\gamma^2 \geq 0$, and $\eta_t^u \sim \mathcal{N}(0,\sigma_u^2 I)$ with $\sigma_u^2 \geq 0$. Let $T_{ss}$ be some value satisfying $T_{ss} = \tilde{\mathcal{O}}(1 / (1 - \rho(A_*)))$. Then as long as:
\begin{equation*}\label{eq:thm_error_burnin2_informal}
T \geq T_{ss} + c k \left (  d + \max \left \{ \log \det ( \bar{\Gamma}_T{ \Gamma_{k}^{u}}^{-1} / \gamma^2 ), \log \det (\bar{\Gamma}_T (\sigma^2 \Gamma_k + \sigma_u^2 \Gamma_k^{B_*})^{-1}) \right \}+ \log \frac{1}{\delta} \right )
\end{equation*}
we have:
\begin{equation*}\label{eq:thm_error2_informal}
\mathbb{P} \left [ \| \hat{A} - A_* \|_2 > C \sigma \sqrt{\frac{\log \frac{1}{\delta} +  \log \det (\bar{\Gamma}_T (\sigma^2 \Gamma_k + \sigma_u^2 \Gamma_k^{B_*} + \gamma^2 \Gamma_k^u)^{-1} + I) +  d }{T \lambda_{\min} ( \sigma^2 \Gamma_k + \sigma_u^2 \Gamma_k^{B_*} + \gamma^2 \Gamma_k^u)}} \right ] \leq 3 \delta
\end{equation*}
where $\bar{\Gamma}_T = 4 \left ( \frac{1}{T} \sum_{t=0}^T x_t^{\tilde{u}} {x_t^{\tilde{u}}}^\top + Tr(\sigma^2 \Gamma_T + \sigma_u^2 \Gamma_T^{B_*} ) ( 1 + \log \frac{2}{\delta} ) I \right )$, $c,C$ are universal constants, and $x_t^{\tilde{u}}$ is the (deterministic) response of the system to $u_t = \tilde{u}_t$.
\end{theorem}

Note, critically, the $\Gamma_k^u$ term in the denominator. This term quantifies how the estimation error scales in terms of the interaction between the input and the system.

\section{Interpreting the Results}\label{sec:interpreting}
We next present several corollaries to Theorem \ref{cor:largek_alg_opt_informal}. Let $A_* = V \Lambda V^\top$ for orthogonal $V$, real, diagonal $\Lambda \succeq 0$, and $B_* = I$. Denote the eigenvalues of $A_*$ as $\lambda_1 \geq \lambda_2 \geq \ldots \geq \lambda_d$ and $\lambda = [\lambda_1,\lambda_2,...,\lambda_d]$. To aid in interpretability, assume that $\frac{1}{1 - \lambda_1} \gg \frac{1}{1 - \lambda_2}$, $\frac{1}{1 - \lambda_d}$ is small enough to be thought of as a small constant factor, $\gamma^2 > \sigma^2$, and, $\log \frac{1}{\delta} > 1$. We then have the following.
\begin{corollary}\label{cor:symmetric_a_informal}
\textbf{(Symmetric $A_*$)}  Let $T_{init}$ be some value satisfying:
\begin{align*}
T_{init} = \tilde{\mathcal{O}} \left ( \max \left \{ \frac{T_0^2}{k_0^2} \max_{i=1,...,d} \frac{i^2}{(1 - \lambda_i)^2}, \frac{d^2 \sigma^2 \| \mathbf{1} - \lambda \|_2^4 }{(d \sigma^2 + \gamma^2 ) (1 - \lambda_1)^4} \right \} \right ) 
\end{align*}
then after $T \geq T_{init}$ steps, running Algorithm \ref{alg:active_lds_noise} with $FT = $ \texttt{True} will produce an estimate satisfying, with high probability:
\begin{equation*}
\| \hat{A} - A_* \|_2 = \tilde{\mathcal{O}} \left ( \sqrt{\frac{\sigma^2 \| \mathbf{1} - \lambda \|_2^2}{\gamma^2 + \sigma^2 \| \mathbf{1} - \lambda \|_2^2}}  \sqrt{ \frac{ d }{T }} \right )
\end{equation*}
while instead playing $u_t \sim \mathcal{N}(0, \frac{\gamma^2}{d} I)$ for all time, our estimate will satisfy, with high probability:
\begin{equation*}
\| \hat{A} - A_* \|_2 = \tilde{\mathcal{O}} \left ( \sqrt{\frac{\sigma^2 d}{\gamma^2 +  \sigma^2 d}}  \sqrt{\frac{d}{T}} \right )
\end{equation*}
\end{corollary}
In the high SNR regime of $\gamma^2 \gg d \sigma^2$, the leading constant for the rate attained by Algorithm \ref{alg:active_lds_noise} behaves as $\frac{\sigma \| \mathbf{1} - \lambda \|_2}{\gamma}$ compared to a leading constant of $\frac{\sigma \sqrt{d}}{\gamma}$ when playing $u_t \sim \mathcal{N}(0, \frac{\gamma^2}{d} I)$. Note that in both cases the expected average power is $\gamma^2$.

Now let $A_*$ and $B_*$ be block diagonal matrices where $A_j \in \mathbb{R}^{d_j \times d_j}$ and $B_j \in \mathbb{R}^{d_j \times p_j}$ denote their $j$th blocks. Assume that it is known that $A_*$ has this structure. For simplicity, assume $\gamma^2 \gg \sigma^2$ so that $\lambda_{\min} \left ( \sigma^2 \Gamma_{k}^{j}  + \gamma^2 \Gamma_k^{u^*,j} \right ) \approx \lambda_{\min} \left (  \gamma^2 \Gamma_k^{u^*,j} \right )$. Here $\Gamma_{k}^{j}$ and $\Gamma_k^{u^*,j}$ denote the expected noise and input covariates of the $j$th subsystem.

\begin{corollary}\label{cor:block_diagonal_a}
\textbf{(Block Diagonal $A_*$)} For $T$ large enough, a version of Algorithm \ref{alg:active_lds_noise} slightly modified to account for the block structure, will have, with high probability, when $FT = $ \texttt{True}:
\begin{equation*}
\| \hat{A} - A_* \|_2 = \tilde{\mathcal{O}} \left ( \sqrt{\sum_{j=1}^m \frac{d_j}{ \lambda_{\min} \left ( \Gamma_{k(T)}^{u^*,j} \right )}} \sqrt{\frac{1}{\gamma^2 T}} \right )
\end{equation*}
In contrast, simply playing $u_t \sim \mathcal{N}(0, \frac{\gamma^2}{p} I)$ will, with high probability, achieve the following rate:
\begin{equation*}
\| \hat{A} - A_* \|_2 = \tilde{\mathcal{O}} \left ( \sqrt{ \max_{j=1,...,m} \frac{m d_j}{  \lambda_{\min} \left ( \Gamma_{k(T)}^{B_*,j} \right ) }} \sqrt{\frac{1}{\gamma^2 T}} \right )
\end{equation*}
\end{corollary}
Intuitively, the rate obtained by Algorithm \ref{alg:active_lds_noise} scales as the average error in estimating each block, while the rate obtained by playing $u_t \sim \mathcal{N}(0, \frac{\gamma^2}{p} I)$ scales as the error of the worst case block. Note that while in both Corollary \ref{cor:symmetric_a_informal} and Corollary \ref{cor:block_diagonal_a} we are comparing upper bounds, the leading constants in these bounds are identical to those obtained in the asymptotic lower bound, Theorem \ref{thm:lb}, and are thus unimprovable---the improvement in upper bounds we see in performing active estimation compared to playing noise are matched by the lower bound. Both corollaries are proved in Section \ref{sec:main_thm_cors}. 

It is difficult to work out analytically what the performance will be when $A_*$ is a Jordan block. However, at an intuitive level, our algorithm should yield a large improvement over isotropic noise as the proper excitation of a Jordan block focuses nearly all the energy on the last coordinate in the block. This conjecture is supported by our experiments in Section \ref{sec:exp}.

\section{Proof Sketch of Theorem \ref{cor:largek_alg_opt_informal} }\label{sec:proof_sketch}
To prove Theorem \ref{cor:largek_alg_opt_informal}, our primary upper bound on the error in the estimates of $A_*$ produced by Algorithm \ref{alg:active_lds_noise}, we first bound the error in the estimate of $A_*$ obtained at the $(i-1)$th epoch, then bound the suboptimality of the inputs computed from this estimate, and finally bound the estimation error at the $i$th epoch in terms of these inputs. 

\noindent \textbf{Controlling the estimation error $\| \hat{A}_{i-1} - A_* \|_2$ at the $(i-1)$th epoch.} We rely on excitation due to noise to guarantee learning and  bound $\| \hat{A}_{i-1} - A_* \|_2$. This proof is similar to those given in \cite{simchowitz2018learning,sarkar2018how} and is outlined in the appendix.

\noindent \textbf{Bounding the suboptimality of the inputs.} Given the estimate $\hat{A}_{i-1}$ and past data $\{ x_t \}_{t=1}^{T-T_i}$, and letting $\hat{u}_i$ denote the optimal inputs on the estimated system and $u_i^*$ the optimal inputs on the true system, we wish to bound:
\begin{equation}\label{eq:sketch_perturbation}
\left |  \lambda_{\min}  ( {\textstyle\sum}_{t=1}^{T - T_i} x_t x_t^\top   + \Gamma_{k_i}^{\hat{u}_i}  ) - \lambda_{\min}  ( {\textstyle\sum}_{t=1}^{T - T_i} x_t x_t^\top + \Gamma_{k_i}^{u_i^*}  ) \right | 
\end{equation}
in terms of $\epsilon_{i-1}$, as this will quantify how suboptimal our input's response on the true system is. Theorem \ref{thm:opt_sln_perturbation_informal} provides such a bound in terms of $\epsilon_{i-1}$. 
\begin{theorem}\label{thm:opt_sln_perturbation_informal} 
(\textbf{Informal}) Assuming that $\| A_* - \hat{A}_{i-1} \|_2 \leq \epsilon$, then:
\begin{align*}
\left |  \lambda_{\min}  ( {\textstyle\sum}_{t=1}^{T - T_i} x_t x_t^\top + \Gamma_{k_i}^{\hat{u}_i}  ) - \lambda_{\min}  ( {\textstyle\sum}_{t=1}^{T - T_i} x_t x_t^\top + \Gamma_{k_i}^{u_i^*}  ) \right | \leq \max_{U \in \mathcal{U}_{\gamma^2}, w \in \mathcal{M}} 2 \epsilon L(A_*,B_*,U,\epsilon,\mathcal{I},w)
\end{align*}
where $L(A_*,B_*,U,\epsilon,\mathcal{I},w)$ is a measure of the smoothness of $\Gamma_{k_i}^u$ with respect to $A_*$.
\end{theorem}
The full version of Theorem \ref{thm:opt_sln_perturbation_informal} is stated and proved in Appendix \ref{sec:design_perturbation}. At a high level, the proof follows by upper bounding (\ref{eq:sketch_perturbation}) in terms of the difference between $\Gamma_{k_i}^u$ and $\hat{\Gamma}_{k_i}^u := \Gamma_{k_i}^u(\hat{A}_{i-1},B_*)$. This difference can be quantified in terms of the sensitivity of  $\Gamma_{k_i}^u$ to changes in $A_*$ and, critically, does not require bounding the difference between $\hat{u}_i$ and $u_i^*$. The primary challenge in proving Theorem \ref{thm:opt_sln_perturbation_informal} is in avoiding standard matrix perturbation bounds of the form:
\begin{equation}\label{eq:sketch_perturbation2}
\left |  \lambda_{\min} ( {\textstyle\sum}_{t=1}^{T - T_i} x_t x_t^\top + \hat{\Gamma}_{k_i}^{u}  ) - \lambda_{\min}  ( {\textstyle\sum}_{t=1}^{T - T_i} x_t x_t^\top + \Gamma_{k_i}^{u}  ) \right | \leq  \| \hat{\Gamma}_{k_i}^{u} - \Gamma_{k_i}^{u}  \|_2
\end{equation}
Depending on the structure of $A_*$, $\Gamma_{k_i}^u$ could be very ill-conditioned and (\ref{eq:sketch_perturbation2}) could be very loose. We instead show that it is sufficient to bound:
\begin{equation}\label{eq:sketch_perturbation3}
\max_{w \in \mathcal{M}} \left |  w^\top  ( {\textstyle\sum}_{t=1}^{T - T_i} x_t x_t^\top + \hat{\Gamma}_{k_i}^{u}  ) w - w^\top ( {\textstyle\sum}_{t=1}^{T - T_i} x_t x_t^\top + \Gamma_{k_i}^{u}  ) w \right | = \max_{w \in \mathcal{M}} \left | w^\top ( \hat{\Gamma}_{k_i}^{u} - \Gamma_{k_i}^{u}  ) w  \right |
\end{equation}
for a set $ \mathcal{M} $ guaranteed to include the eigenvectors corresponding to the minimum eigenvalues of  ${\textstyle\sum}_{t=1}^{T - T_i} x_t x_t^\top + \Gamma_{k_i}^{\hat{u}_i}$ and ${\textstyle\sum}_{t=1}^{T - T_i} x_t x_t^\top + \Gamma_{k_i}^{u_i^*}$. Applying (\ref{eq:sketch_perturbation3}) instead of (\ref{eq:sketch_perturbation2}) with this $\mathcal{M}$ can save a factor of as much as $1 / (1 - \rho(A_*))$ in the final perturbation bound.

Given this perturbation bound, we can quantify how suboptimal the inputs computed by solving $\texttt{OptInput}$ on our estimated system are. As we make precise in  Appendix \ref{sec:design_perturbation}, the suboptimality depends on the frequencies our input signal contains. \texttt{UpdateInputs} carefully takes this into account, only playing inputs for which it can guarantee the system will be sufficiently excited. Ultimately, we are interested in exciting the system optimally, which requires that we have learned the system well enough to guarantee the performance at every frequency. We quantify this in Lemma \ref{lem:ft_sufficient} and show that for sufficiently large $T$, we will be playing inputs that attain the optimal response.

\noindent \textbf{Controlling the estimation error $\| \hat{A}_i - A_* \|_2$ in terms of the inputs.} The final piece in the proof involves showing that, for the inputs being played, $\hat{u}_{i}$, the estimation error will scale in accordance with how these inputs excite the true system. We can decompose the error in our estimate of $A_*$ as:
\begin{align*}
\| \hat{A}_i - A_* \|_2 & = \| ({\textstyle\sum}_{t=1}^T x_t x_t^\top )^{-1} {\textstyle\sum}_{t=1}^T x_t \eta_t^\top \|_2 \\
& \leq \lambda_{\min}({\textstyle\sum}_{t=1}^T x_t x_t^\top)^{-1/2} \| ({\textstyle\sum}_{t=1}^T x_t x_t^\top )^{-1/2} {\textstyle\sum}_{t=1}^T x_t \eta_t^\top \|_2
\end{align*}
$\| ({\textstyle\sum}_{t=1}^T x_t x_t^\top )^{-1/2} {\textstyle\sum}_{t=1}^T x_t \eta_t^\top \|_2$ scales like $\mathcal{O}(\sqrt{d + \log 1 / \delta + \log T})$ and can be handled using a self-normalized bound \cite{abbasi2011improved,sarkar2018how}. The primary difficulty is obtaining a lower bound on $\lambda_{\min}({\textstyle\sum}_{t=1}^T x_t x_t^\top)$ in terms of the inputs being played. We in fact want to show something even stronger, that $\sum_{t=T - T_i}^{T} x_t x_t^\top \succeq c (T - T_i) \Gamma_{k_i}^{\hat{u}_i}$, as this allows us to quantify precisely how an input affects the covariates, and how we can adjust the input to increase $\lambda_{\min}({\textstyle\sum}_{t=1}^T x_t x_t^\top)$. The following proposition is the key piece in proving such a lower bound. 
\begin{proposition}\label{prob:covbound_informal}\textbf{\textit{(Informal)}}
Consider $w \in \mathcal{S}^{d-1}$ and let $u_t$ be a deterministic signal with period $k$. Assuming that $T_{ss}$ is large enough that the transient effects of the input have dissipated, we have:
\begin{equation}\label{eq:prop_cov_bound_informal}
\mathbb{P} \left [ {\textstyle \sum}_{t=T_{ss}+1}^{T_{ss}+T} (w^\top x_t)^2 \geq \frac{2}{81} k \lfloor T / k \rfloor  \gamma^2 w^\top \Gamma_{k}^u w  \right ] \geq 1 -  e^{-\frac{2}{81} \lfloor T / k \rfloor }
\end{equation}
\end{proposition}
The proof of this proposition is given in Section \ref{sec:concentration}. The main technical challenge comes in handling the interactions between the inputs and the noise. To avoid directly bounding these cross terms, we prove that the covariates over one period of the input are, with constant probability, lower bounded by the covariates obtained if running the system with no process noise. After enough periods, we show that with high probability the bound (\ref{eq:prop_cov_bound_informal}) holds. Given this pointwise lower bound, we can apply a similar argument to that in \cite{simchowitz2018learning,sarkar2018how} to show the estimation error bound given in Theorem \ref{thm:concentration2}.

To complete the proof of Theorem \ref{cor:largek_alg_opt_informal}, we effectively apply Theorem \ref{thm:concentration2} to bound the estimation error in the $i$th epoch in terms of $\Gamma_{k_i}^{\hat{u}_i}$, and using the fact that $\hat{u}_i$ excites the system nearly optimally, conclude that we attain the optimal estimation rate.

\begin{figure}
\centering
\begin{minipage}{.5\textwidth}
  \centering
  \captionsetup{justification=centering}
  \includegraphics[width=\linewidth]{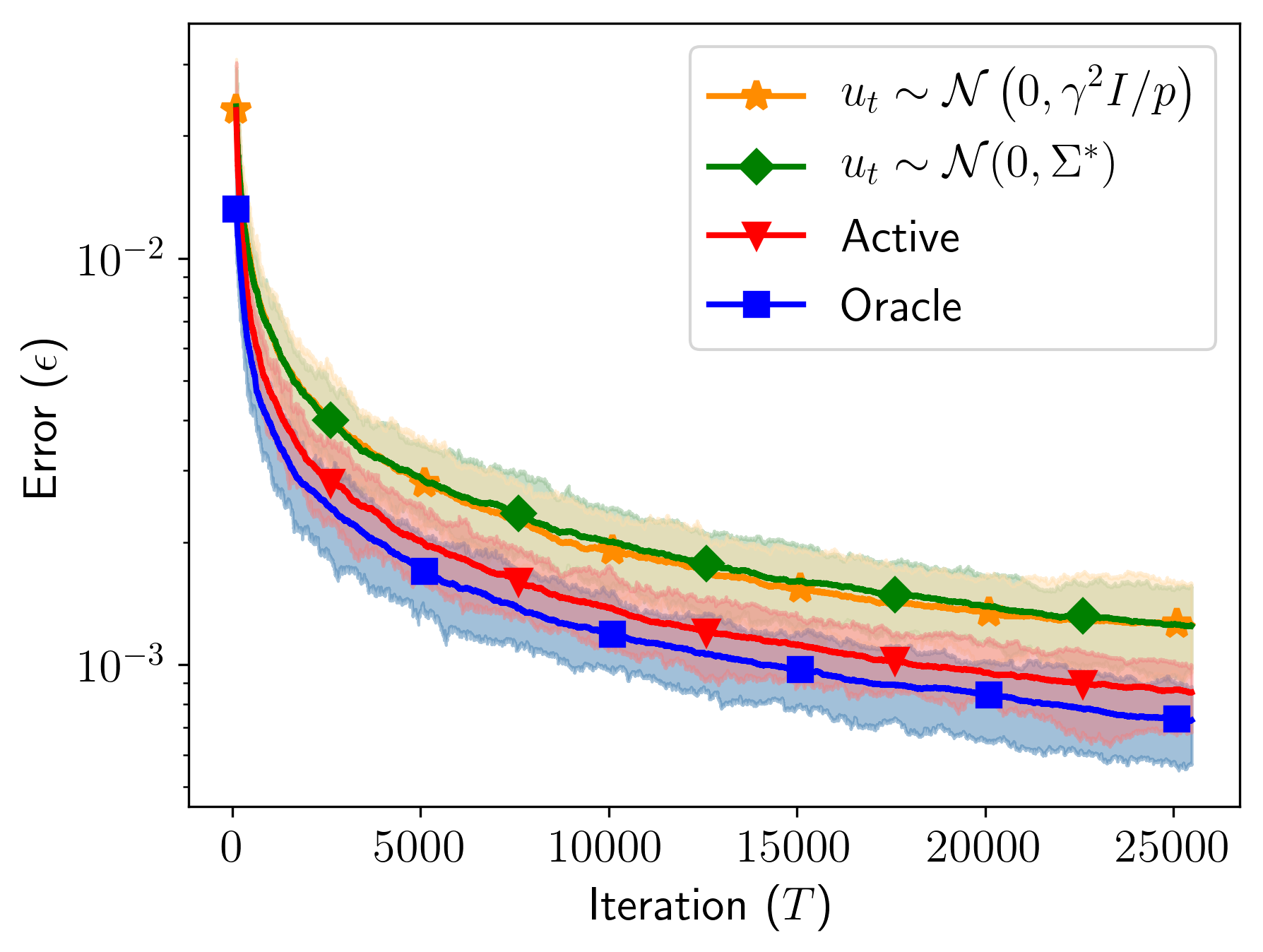}
  \captionof{figure}{$A_*$ diagonalizable by unitary matrix, $d = 6, p = 4$, $B_*$ randomly generated}
  \label{fig:diagA}
\end{minipage}%
\begin{minipage}{.5\textwidth}
  \centering
  \captionsetup{justification=centering}
  \includegraphics[width=\linewidth]{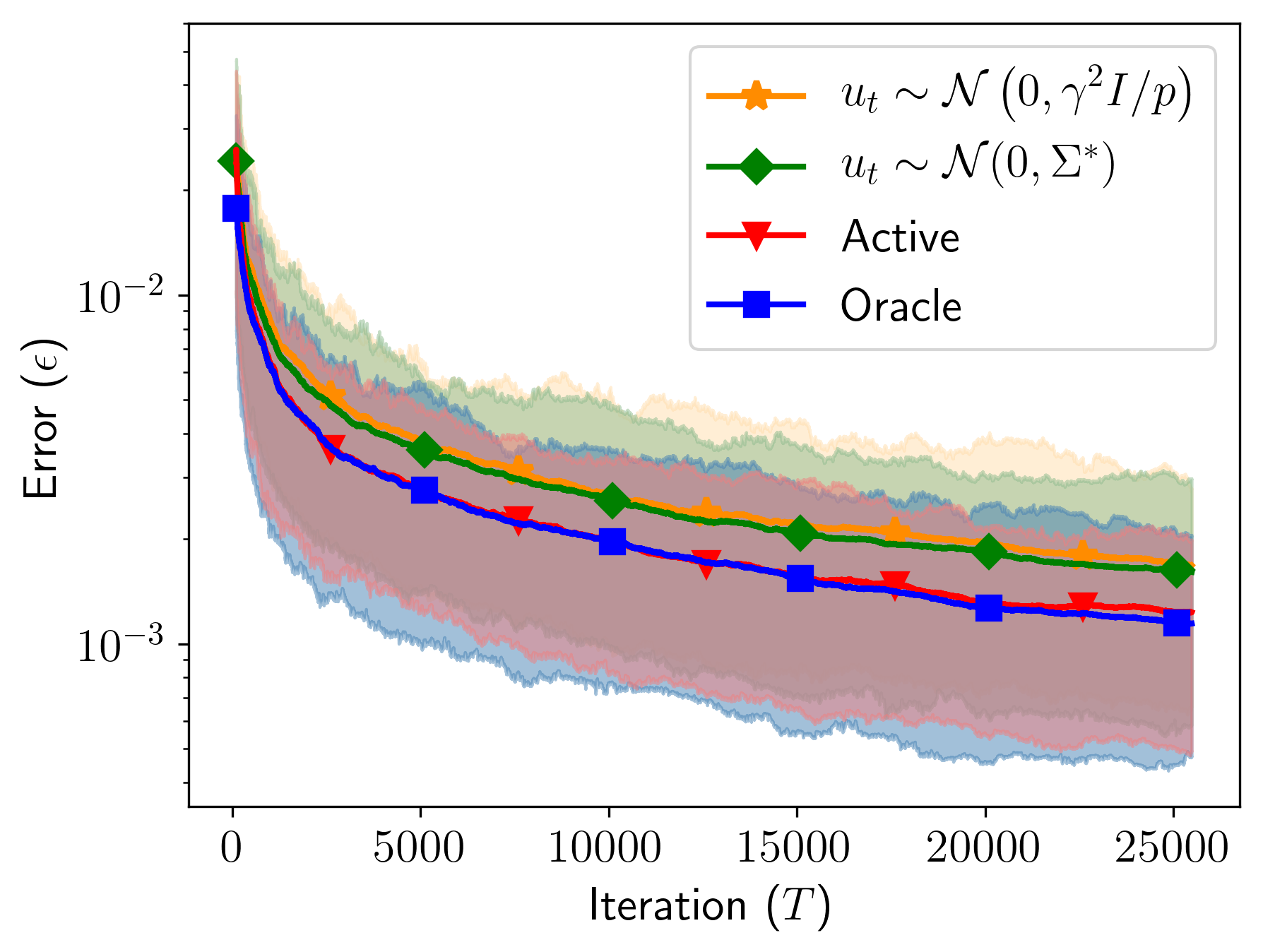}
  \captionof{figure}{$A_*$ and $B_*$ randomly generated, $d = 5$, $p = 3$}
  \label{fig:randomA}
\end{minipage}
\end{figure}

\section{Experimental Results}\label{sec:exp}
We next validate our algorithm on several examples. Additional trials are included in Section \ref{sec:experiment_additional}. We compare Algorithm \ref{alg:active_lds_noise} against three baselines: playing $u_t \sim \mathcal{N}(0,\gamma^2 I / p)$, playing $u_t \sim \mathcal{N}(0,\Sigma^*)$, and playing the oracle set of inputs as computed by solving $\texttt{OptInput}$ on the true system parameters. $\Sigma^*$ is the covariance yielding the optimal noise excitation and can be computed via an SDP. We do not compare against existing works in active system identification as these works typically either require knowledge of $A_*$ to implement, and so are not directly comparable, or propose approaches similar enough to ours (\cite{lindqvist2001identification}) a comparison is not relevant. 

\begin{figure}
\centering
\begin{minipage}{.5\textwidth}
  \centering
  \captionsetup{justification=centering}
  \includegraphics[width=\linewidth]{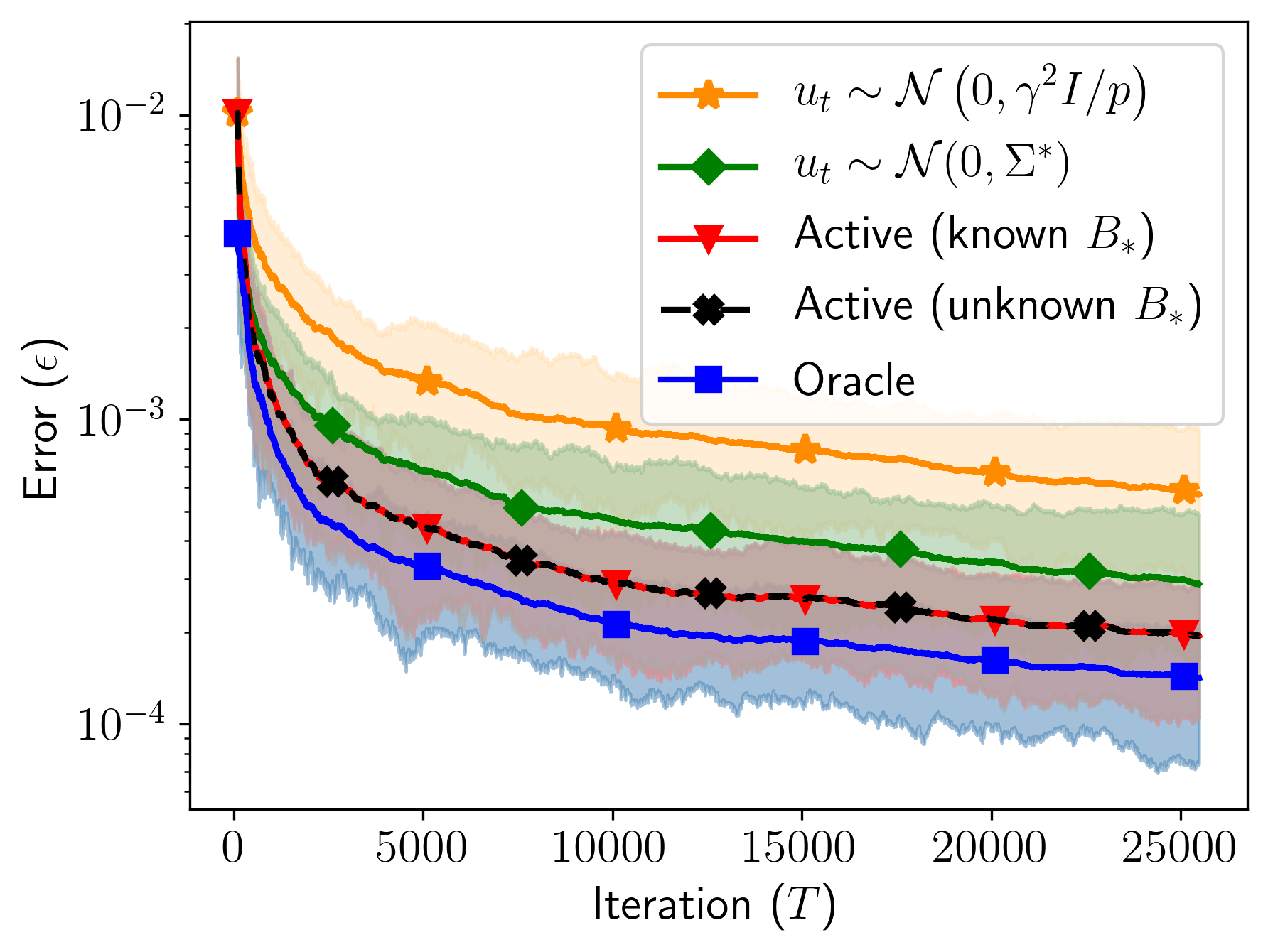}
  \captionof{figure}{$A_*$ Jordan block with \\ $d = 4$, $\rho(A_*) = 0.9$, $B_* = I$}
  \label{fig:unknownB2_jordanA}
\end{minipage}%
\begin{minipage}{.5\textwidth}
  \centering
  \captionsetup{justification=centering}
  \includegraphics[width=\linewidth]{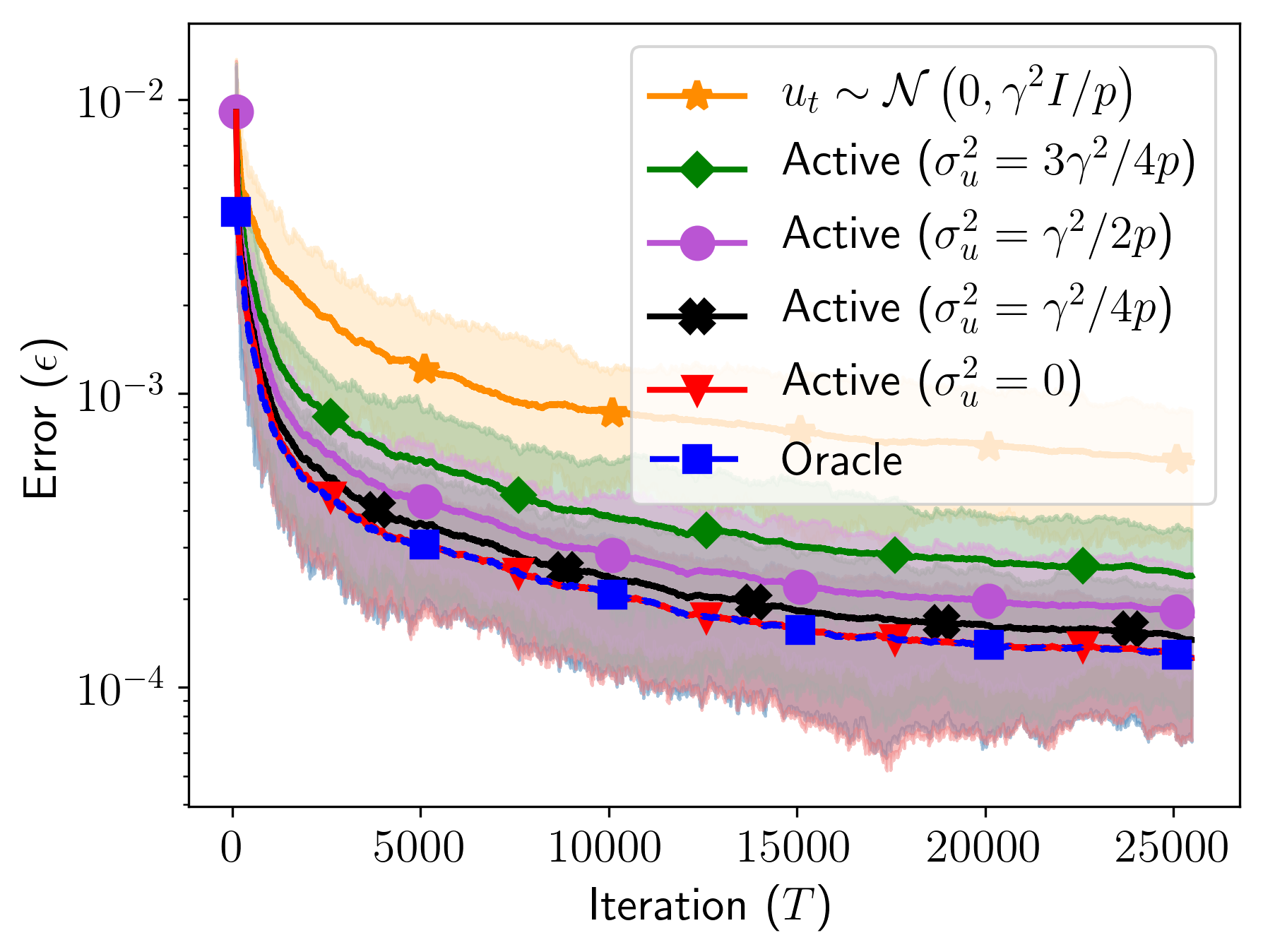}
  \captionof{figure}{$A_*$ Jordan block with  $d = 4$, $\rho(A_*) = 0.9$,  $B_* = I$, varying $\sigma_u^2$}
  \label{fig:greedyA2}
\end{minipage}
\end{figure}

We set $T_0 = 100, k_0 = 20$. Rather than running the \texttt{UpdateInputs} function as stated, we plan greedily with respect to $\hat{A}$---we do not restrict the set of allowable frequencies and set $U \leftarrow \texttt{OptInput}_{k_i}(\hat{A}_i,B_*,\gamma^2/2,[k_{i+1}],\{ x_t \}_{t=1}^T)$. In every experiment we solve $\texttt{OptInput}$ from a single random initialization and do not restart multiple times to obtain a globally optimal solution. We plot the error $\| \hat{A} - A_* \|_2$ against the iteration number. The solid lines show the averages over 50 trials (100 for Figure \ref{fig:randomA}) and the shaded regions indicate the 10\% and 90\% percentiles.

Figures \ref{fig:diagA} and \ref{fig:randomA} illustrate the effectiveness of our approach as compared to exciting the system with noise---Algorithm \ref{alg:active_lds_noise} dramatically outperforms noise-based approaches and performs nearly as well as the optimal. Figure \ref{fig:unknownB2_jordanA} investigates the performance of our algorithm when $B_*$ is unknown. Here we simultaneously solve for $A_*$ and $B_*$ and use our estimate of $B_*$ when optimizing our inputs. As can be seen, this barely affects the algorithm's performance.

At each epoch, Algorithm \ref{alg:active_lds_noise} devotes some amount of input energy to playing random noise. Let $\sigma_u^2$ denote the variance of this noise. By default in Algorithm \ref{alg:active_lds_noise} we set $\sigma_u^2 = \frac{\gamma^2}{2p}$. Figure \ref{fig:greedyA2} illustrates the performance of Algorithm \ref{alg:active_lds_noise} when $\sigma_u^2$ is varied. For a given $\sigma_u^2$, all additional energy is devoted to the sinusoidal component of the input. As this plot illustrates, noise is not needed in practice to effectively learn and, when all  energy is devoted to the sinusoidal inputs, the performance of Algorithm \ref{alg:active_lds_noise} almost immediately matches that of the optimal.

\section{Discussion} 
In this work we have presented an algorithm for active identification of linear dynamical systems. We show that our algorithm achieves optimal asymptotic rates and present finite time performance bounds quantifying how the interactions between the input and the system affect the estimation. This work opens up several possible directions for future work.
\begin{itemize}[noitemsep,leftmargin=*]
\item  \texttt{OptInput} is nonconvex so a globally optimal solution cannot be efficiently found. In practice, an alternating minimization approach can be used to compute a local optimum. While solving \texttt{OptInput} may be difficult, as our bounds show, the quantity being optimized is intrinsic to the problem. Developing algorithms to efficiently solve \texttt{OptInput} is an interesting future direction.
\item Recent works in system identification \cite{simchowitz2018learning,sarkar2018how} have emphasized obtaining bounds that do not scale with the mixing time of the system. Our error bounds do not scale with this quantity yet they require the transient effects of the inputs to have decayed. This condition seems necessary to cleanly quantify the performance and design inputs, yet may be possible to remove with a careful analysis of the transient behavior.
\item This work only considers exciting the system with sinusoidal inputs. While we show this is sufficient to achieve optimal rates, one could also imagine choosing inputs that were a function of the current state. \cite{dean2018regret} provides rates when a linear state feedback controller is used, but does not discuss how the choice of feedback could improve estimation. It is unclear a priori how effective it could be. At minimum, a carefully designed state feedback controller could be used to mitigate transient effects. We leave this direction for future work. 
\item A recent work \cite{gonzalez2019finite} develops finite time bounds for estimating SISO AR($n$) systems with $n > 1$. Extending this to MIMO AR($n$) systems and allowing for active input design is an open problem and exciting future direction. 
\end{itemize}

\subsection*{Acknowledgements}
 The authors would like to thank Yue Sun and Max Simchowitz for helpful comments. The work of AW was supported by an NSF GFRP Fellowship DGE-1762114. The work of KJ was supported in part by grant NSF RI 1907907.

\bibliography{bibliography.bib}


\appendix

\tableofcontents

\section{Notation}\label{sec:notation}

\begin{table}[H]
\begin{center}
\begin{tabular}{ |l| } 
 \hline
 LDS Notation  \\ 
 \hline
 $d$ is the state dimension \\
 $p$ is the input dimension \\
 $\sigma^2$ is the variance of the process noise \\
   $\sigma_u^2$ is the variance of the exploration noise (set by default to $\gamma^2/(2p)$) \\
 $x_t^u$ system state due only to deterministic input, $x_{t+1}^u = A_* x_t^u + B_* u_t$ \\ 
 $x_t^\eta$ system state due only to noise, $x_{t+1}^\eta = A_* x_t^\eta + \eta_t$ \\
 $A = PJP^{-1}$ denotes Jordan decomposition of $A$ \\
 $J_\ell$, $\ell = 1,...,r$ is $\ell$th Jordan block of $J$ \\
 $U(e^{j \frac{2 \pi \ell}{k}})$ denotes the Discrete Fourier Transform of $\{ u_t \}_{t=1}^k$ \\
 $U_\ell = U(e^{j \frac{2 \pi \ell}{k}})$ \\
 $G(e^{j \theta}) = (e^{j \theta} I - A_*)^{-1} B_*$ \\
 $\rho(A)$ spectral radius of $A$ \\
 $\bar{\rho}(A) = 1/2 + \rho(A)/2$ \\
 $\beta(A,\rho) = \sup \left \{ \| A^k \|_2 \rho^{-k} \ : \ k \geq 0 \right \}$ \\
 $\beta(A) = \beta(A,1/2 + \rho(A)/2)$ \\
 $\kappa(A) = \| P \|_2 \| P^{-1} \|_2$ \\
 $\Gamma_t(A) = \sum_{s=0}^{t-1} (A^{s}) (A^{s})^\top$ \\
 $\Gamma_t^{B} (A) = \sum_{s=0}^{t-1} (A^{s} B) (A^{s} B)^\top$ \\
 $ \Gamma_t^\eta(A) = \sigma^2 \Gamma_t(A) + \sigma_u^2 \Gamma_t^{B_*}(A)$\\
 $\Gamma_t = \Gamma_t(A_*)$ \\
 $\Gamma_t^{B_*} = \Gamma_t^{B_*} (A_*)$ \\
 $\Gamma_{t}^\eta = \Gamma_t^\eta(A_*) $ \\
 $\Gamma_{k,t_0}^u = \frac{1}{\gamma^2} \frac{1}{k} \sum_{s = t_0 + 1}^{t_0+k} x_s^u {x_s^u}^\top$ \\
 $\Gamma_k^u(A,B) = \frac{1}{\gamma^2} \frac{1}{k^2} \sum_{\ell = 0}^{k-1} (e^{j \frac{2 \pi \ell}{k}} I - A)^{-1} B U(e^{j \frac{2 \pi \ell}{k}}) U(e^{j \frac{2 \pi \ell}{k}})^H B^H (e^{j \frac{2 \pi \ell}{k}} I - A)^{-H}$ \\
 $ \Gamma_k^u = \Gamma_k^u(A_*,B_*)$ \\
 $\tilde{\Gamma}_{k,t_0}^u = \gamma^2 \Gamma_{k,t_0}^u$ \\
 $\tilde{\Gamma}_k^u = \gamma^2 \Gamma_k^u$ \\
 $H_k(A,B,U,\mathcal{I}) = \sum_{\ell \in \mathcal{I}} (e^{j \frac{2\pi\ell}{k}} I - A)^{-1} B U(e^{j \frac{2\pi\ell}{k}}) U(e^{j \frac{2\pi\ell}{k}})^H B^H (e^{j \frac{2\pi\ell}{k}} I - A)^{-H} $ \\
 $\bar{\Gamma}_T$ high probability upper bound, in PSD sense, on $\sum_{t=1}^T x_t x_t^\top$ \\
 $T_{ss}(\zeta,k,x_0^u) = \mathcal{O} \left (  \frac{1}{\log \rho(A_*)} \log \left ( \frac{k \zeta(1-\rho(A_*)^2)}{2 \| x_0^u - x_0^{ss} \|_2^2 \beta(A_*)^2}  \right ) \right )$ \\
 $T_{ss}(\zeta,k_{i+1}) = \mathcal{O} \left ( \frac{1}{\log \frac{1}{\rho(A_*)}} \left ( \log \left ( \frac{4 \beta(A_*) \| B_* \|_2 k_{i+1} \gamma}{1-\rho(A_*)^{k_{i+1}}} +  \sqrt{ 2 Tr \left (\sigma^2 \Gamma_{T} + \frac{\gamma^2}{2p} \Gamma_{T}^{B_*} \right ) \left ( 1 + \frac{1}{c} \log \frac{4}{\delta} \right )} \right )  \right . \right . $ \\
\hspace{5cm} $ \left . \left . +  \log \left ( \frac{4  \beta(A_*) \gamma \max_{\ell = 1,...,k_{i+1}} \| (e^{j \frac{2 \pi \ell}{k_{i+1}}} I - A_*)^{-1} B_* \|_2}{\zeta \sqrt{k_{i+1}}  \sqrt{1-\rho(A_*)^2}}  \right ) \right )  \right )$ \\
 \hline
 \end{tabular}
\end{center}
\end{table}

\begin{table}[H]
\begin{center}
\begin{tabular}{ |l| } 
\hline
 Active System Identification Notation \\
\hline
 $\mathcal{U}_{\gamma^2} = \left \{ u_1,...,u_k \in \mathbb{R}^p \ : \ \sum_{\ell=1}^k U(e^{j 2 \pi \ell / k})^H U(e^{j 2 \pi \ell / k}) \leq k^2 \gamma^2 \right \}$ \\
 $\bar{\mathcal{U}}_{\gamma^2} = \left \{ u_1,...,u_k \in \mathbb{R}^p \ : \ \sum_{\ell=1}^k U(e^{j 2 \pi \ell / k})^H U(e^{j 2 \pi \ell / k}) \leq k^2 \gamma^2, \ \sum_{t=1}^k u_t = 0 \right \}$ \\
$\texttt{OptInput}_k(A,B,\gamma^2,\mathcal{I},\{ x_t \}_{t=1}^T) = \begin{matrix*}[l]  \max_{u_1,...,u_k \in \mathbb{R}^p} \ \lambda_{\min} \left ( \gamma^2 \bar{T} \Gamma_k^u(A,B) +  \sum_{t=1}^T x_t x_t^\top \right ) \\
 \text{s.t.} \ \ u_1,...,u_k \in \bar{\mathcal{U}}_{\gamma^2}, U(e^{j 2 \pi \ell / k}) = 0, \forall \ell \not\in \mathcal{I}
 \end{matrix*}$ \\
 $\epsilon_S(A,B,\gamma^2,k_{i},\{x_t\}_{t=1}^T,\delta) = \min \Bigg \{ \frac{\frac{27}{256 T_i \gamma^2} \texttt{OptInput}_{k_{i}}(A,B,\gamma^2/2,[k_{i}],\{x_t\}_{t=1}^T) }{\max\limits_{w \in \mathcal{M}(\hat{A}_i,\{x_t\}_{t=1}^T), \ell \in [k_{i}]  }  \| w^\top (e^{j \frac{2 \pi \ell}{k_{i}}} I - A)^{-1} \|_2^2 \|  (e^{j \frac{2 \pi \ell}{k_{i}}} I - A)^{-1}\|_2 \| B \|_2^2}, $ \\
 \hspace{8cm} $\frac{1}{ \max_{\ell \in [k_{i}]} 5 \| (e^{j \frac{2 \pi \ell}{k_{i}}} I - A )^{-1} \|_2} \Bigg \}$ \\
 $\bar{\epsilon}_S(A,B,\gamma^2,T,\delta) = \min \Bigg \{ \frac{\frac{27}{256 (2T + T_0) \gamma^2 \| B \|_2^2} \texttt{OptInput}_{2k(T)} \left (A,B,\gamma^2,[2k(T)], c T \Gamma_{k(T)}^\eta   \right )}{\max\limits_{w \in \bar{\mathcal{M}}_{2k(T)}(A,B,\delta,\gamma^2/2), \ell \in [2k(T)]  }   \| w^\top (e^{j \frac{2 \pi \ell}{2k(T)}} I - A)^{-1} \|_2^2 \|  (e^{j \frac{2 \pi \ell}{2k(T)}} I - A)^{-1}\|_2 } ,$ \\
 \hspace{8cm} $\frac{1}{\max_{\ell \in [2k(T)]} 5 \| (e^{j \frac{2 \pi \ell}{2k(T)}} I - A )^{-1} \|_2} \Bigg \} $ \\
 $L(A,B,U,\epsilon,\mathcal{I},w) = \max\limits_{\substack{\Delta \in \mathbb{R}^{d \times d}, \| \Delta \|_2 = 1 \\ \delta \in [0,\epsilon]}} 2 \left | \sum_{\ell \in \mathcal{I}} w^\top (e^{j \frac{2\pi\ell}{k}} I - A -  \delta \Delta)^{-1} \Delta (e^{j \frac{2\pi\ell}{k}} I - A - \delta \Delta)^{-1}  \right . $ \\
 \hspace{6cm} $\left .  \cdot \ B U(e^{j \frac{2\pi\ell}{k}}) U(e^{j \frac{2\pi\ell}{k}})^H B^H (e^{j \frac{2\pi\ell}{k}} I - A - \delta \Delta)^{-H} w \right |$  \\
$ \mathcal{M}(A, \{ x_t \}_{t=1}^T) \leftarrow \bigg \{ w \in \mathcal{S}^{d-1} \ : \ \frac{k^2}{2T+T_0} \sum_{t=1}^T (w^\top x_t)^2 \leq  \frac{k^2}{2T+T_0} \sum_{t=1}^T ({w'}^\top x_t)^2 $ \\
\hspace{2cm} $ + \  \min_{w' \in \mathcal{S}^{d-1}} \frac{4}{3} \gamma^2 \max_{\ell \in [k] \ : \  \epsilon \leq \left ( 4 \| (e^{j \frac{2 \pi \ell}{k}} I - A)^{-1} \|_2 \right )^{-1} } \| {w'}^\top (e^{j \frac{2 \pi \ell}{k}} I - A)^{-1} B \|_2^2   \bigg \}$ \\
 $\mathcal{M}(A,\hat{A},\{ x_t \}_{t=1}^T,\mathcal{I}) = \bigg \{ w \in \mathcal{S}^{d-1} \ : \ \frac{k^2}{2T + T_0} \sum_{t=1}^T (w^\top x_t)^2 \leq \min\limits_{w' \in \mathcal{S}^{d-1}}  \frac{k^2}{2T + T_0} \sum_{t=1}^T ({w'}^\top x_t)^2$ \\
 \hspace{3cm} $ + \ \gamma^2  \max\limits_{i \in \mathcal{I}} \max \{ \| {w'}^\top (e^{j \frac{2\pi i}{k}} I - A)^{-1} B \|_2^2, \| {w'}^\top (e^{j \frac{2\pi i}{k}} I - \hat{A})^{-1} B \|_2^2 \}   \bigg \}$ \\
 $\bar{\mathcal{M}}_k(A,B,\delta,\gamma^2) = \bigg \{ w \in \mathcal{S}^{d-1} \ : \  \frac{T}{2T + T_0} w^\top \Gamma_k^\eta w \leq \min\limits_{\ell \in [r] } \max\limits_{\theta \in [0,2\pi]} 6 \gamma^2 \kappa(A)^2 \| (e^{j \theta} I - J_\ell)^{-1} \|_2^2 \| B \|_2^2$ \\
\hspace{3cm} $+ \ 2  \left ( 1 +  \log \frac{2}{\delta} \right )  \left (\frac{\kappa(A)^2 \beta(J_\ell)^2 (\sigma^2  + \gamma^2 / p \| B \|_2^2)}{1 - \bar{\rho}(J_\ell)^2}   \right )  + 16 \frac{\kappa(A)^2 \| B \|_2^2 \beta(J_\ell)^2 \gamma^2 }{(1 - \bar{\rho}(J_\ell)^{2k_0})(1 - \bar{\rho}(J_\ell)^{2})}  \bigg \}$ \\
 $k(T)$ denotes value of $k_i$ for given $T$, if $T$ at epoch boundary, denotes value for previous epoch  \\
$\theta_i = \frac{2 \pi i}{k}$ \\
 \hline
\end{tabular}
\end{center}
\end{table}

\begin{table}[H]
\begin{center}
\begin{tabular}{ |l| } 
 \hline
 Standard Mathematical Notation  \\ 
 \hline
 $\mathcal{S}^{d-1} = \{ v \in \mathbb{R}^d \ : \ \| v \|_2 = 1 \}$ \\ 
 $\| \ . \ \|_2$ denotes matrix operator norm and vector 2-norm \\
 $\| \ . \ \|_F$ denotes matrix Frobenius norm \\
 $[k] = \{ 1,2,3,...,k \}$ \\ 
 \hline
\end{tabular}
\end{center}
\end{table}

\begin{algorithm}[H]
\begin{algorithmic}[1]
\Function{\texttt{UpdateInputs}}{$A$,$B$,$\{ x_t \}_{t=1}^T$,$\gamma^2$,$k$,$\epsilon$,$FT$}
\State    {\color{blue} // Form set of directions that may correspond to minimum eigenvalue}
\State \Longunderstack[l]{$ \mathcal{M}(A, \{ x_t \}_{t=1}^T) \leftarrow \bigg \{ w \in \mathcal{S}^{d-1} \ : \ \frac{k^2}{2T+T_0} \sum_{t=1}^T (w^\top x_t)^2 \leq  \frac{k^2}{2T+T_0} \sum_{t=1}^T ({w'}^\top x_t)^2 + $ \\ \\  $\min_{w' \in \mathcal{S}^{d-1}} \frac{4}{3} \gamma^2 \max_{\ell \in [k] \ : \  \epsilon \leq \left ( 4 \| (e^{j \frac{2 \pi \ell}{k}} I - A)^{-1} \|_2 \right )^{-1} } \| {w'}^\top (e^{j \frac{2 \pi \ell}{k}} I - A)^{-1} B \|_2^2   \bigg \}$ } \label{alg_line:mathcalM} 
\State \textit{\color{blue} // Check if $\epsilon$ small enough to plan with all frequencies}
\If{$\epsilon \leq \min_{\ell \in [k]} \left ( 4 \| (e^{j \frac{2 \pi \ell}{k}} I - A)^{-1} \|_2 \right )^{-1}$ \\ \hspace*{10mm} and $\max_{w \in \mathcal{M}(A, \{ x_t \}_{t=1}^T), \ell \in [k]} \frac{32}{3} \epsilon (2T+ T_0) \gamma^2 \| w^\top (e^{j \frac{2 \pi \ell}{k}} I - A)^{-1} \|_2^2 \frac{\|  (e^{j \frac{2 \pi \ell}{k}} I - A)^{-1} B \|_2^2}{\|  (e^{j \frac{2 \pi \ell}{k}} I - A)^{-1} \|_2}$ \\ \hspace*{10mm} $\leq \texttt{OptInput}_{k}(A,B,\gamma^2/2,[k],\{x_t\}_{t=1}^T)$ \label{alg_line:check_opt}} 
  	\State $\mathcal{I} = [k]$ 
\Else
  	\State \textit{\color{blue} // Otherwise, set $\mathcal{I}$ to include frequencies we can plan effectively with}
  	\State $\mathcal{I} \leftarrow \{ \}$ 
  	\For{$\ell = 1,2,3,...,k$ \label{alg_line:construct_I}}
		\State \textit{\color{blue} // Check if we can plan optimally with frequency $\ell$}
  		\If{$\epsilon \leq \left ( 4 \| (e^{j \frac{2 \pi \ell}{k}} I - A)^{-1} \|_2 \right )^{-1}$ and  $\max_{w \in \mathcal{M}(A, \{ x_t \}_{t=1}^T)} \frac{32}{3} \epsilon (2T+$ \hspace*{20mm} $T_0) 		\gamma^2 \| w^\top (e^{j \frac{2 \pi \ell}{k}} I - A)^{-1} \|_2^2 \frac{\|  (e^{j \frac{2 \pi \ell}{k}} I - A)^{-1} B \|_2^2}{\|  (e^{j \frac{2 \pi \ell}{k}} I - A)^{-1} \|_2} \leq 		\lambda_{\min} \left ( \sum_{t=1}^T x_t x_t^\top \right )$}
  			\State $\mathcal{I} \leftarrow \mathcal{I} \cup \ell$
 	 	\EndIf
  	\EndFor
\EndIf
\State \textit{\color{blue} // Update inputs}
 \If{$FT == $ \texttt{True}}
    	\State $\sigma_u^2 \leftarrow \frac{\gamma^2}{2 p}$, $u \leftarrow \texttt{OptInput}_k(A,B, \gamma^2 - p \sigma_u^2, \mathcal{I},\{ x_t \}_{t=1}^T)$  \label{alg_line:update_inputs2}
	\Else
	\State $\sigma_u^2 \leftarrow \frac{\gamma^2}{2 p}$, $u \leftarrow \texttt{OptInput}_k \left (A,B, \gamma^2 - p \sigma_u^2, \mathcal{I},  (2T + T_0) \sigma^2 \Gamma_k(A)  \right )$  
   \EndIf
\State \Return $ U $ 
\EndFunction
\end{algorithmic}
\caption*{Full Definition of \texttt{UpdateInputs}}
\end{algorithm}

Several comments on notation are in order. First, note that $\beta(A,\rho)$ is the smallest value such that $\| A^k \|_2 \leq \beta(A,r) \rho^k$ for all $k \geq 0$. $\beta(A,\rho)$ is finite as long as $\rho > \rho(A)$. More generally, we can upper bound $\beta(A,\rho)$ as \cite{tu2017non}:
$$ \beta(A,\rho) \leq \max_{|z| \geq 1} \| (z \rho I - A)^{-1} \|_2 = \max_{\theta \in [0,2\pi]} \| (\rho e^{j \theta} I - A)^{-1} \|_2 $$
As $r$ is increased, $\beta(A,\rho)$ will decrease, but the decay rate will be slower. Note that if we set $\rho = \bar{\rho}(A) = \frac{1}{2} + \frac{1}{2} \rho(A)$ and $\beta(A) = \beta(A,\frac{1}{2} + \frac{1}{2} \rho(A))$, we have:
$$ \frac{1}{1 - \rho} = \frac{2}{1 - \rho(A)}, \ \ \ \ \max_{\theta \in [0,2\pi]} \| (\rho e^{j \theta} I - A)^{-1} \|_2 = 2 \max_{\theta \in [0,2\pi]} \| (e^{j \theta} I - A)^{-1} \|_2$$
so the cumulative behavior of the transient, which corresponds to $\frac{1}{1-\rho}$, and the upper bound on $\beta(A,\rho)$, will each be within a factor of 2 of their optimal possible values. Throughout the appendix, we will upper bound $\| A^\ell \|_2 \leq \beta(A) \bar{\rho}(A)^\ell$. In nearly all cases, however, the expressions obtained that contain $\bar{\rho}(A)$ can be replaced with a $\rho(A)$ by adding a factor of 2.

To simplify notation throughout the proofs, we will let $\Gamma_k^\eta = \sigma^2 \Gamma_k + \sigma_u^2 \Gamma_k^{B_*}$ and $\tilde{\Gamma}_k^u = \gamma^2 \Gamma_k^u$. Throughout the appendix, we will let $\sigma_u^2$ refer to the variance of the exploration noise, which is set by default to $\gamma^2/(2p)$.

\section{Algorithm \ref{alg:active_lds_noise} Performance Results}\label{sec:alg_perf}

We first present the full version of Theorem \ref{cor:largek_alg_opt_informal}.

\begin{theorem}\label{cor:largek_alg_opt} 
(\textbf{Full version of Theorem \ref{cor:largek_alg_opt_informal}}) Assume that  $\gamma^2 \geq \frac{(1- \rho(A_*))^2}{2 \beta(A_*)^2}$, and:
\begin{equation}\label{eq:alg_perf_interp_ktinit}
T_0 \geq c k_0 \left ( \log \frac{1}{\delta} + d + d \log \left (   \frac{2 \beta(A_*)^2 \gamma^2 }{ (1 - \rho(A_*))^2} (1 + T_0)  + \frac{ 4 \beta(A_*)^2 d (\sigma^2 + \gamma^2 \| B_* \|_2)}{1 - \rho(A_*)^2}  ( 1 + \log \frac{2}{\delta} )  \right )  \right )
\end{equation}
Then for any:
\begin{align}\label{eq:thm1_burnin}
\begin{split}
T & \geq \max \left \{ \frac{9}{2} T_{ss} \left ( c_1 \lambda_{\min}\left (\sigma^2 \Gamma_{k(T)} + \frac{\gamma^2}{p} \Gamma_{k(T)}^{B_*} \right ), \frac{k(T)}{2} \right ), \right . \\
& \ \ \ \ \ \ \ \ \ \ \ \ \ \ \ \  \left . c_2 \sigma^2 \frac{\log \frac{1}{\delta} + d + \log \det \left ( \bar{\Gamma}_{T} \left ( \sigma^2 \Gamma_{k(T)} + \frac{\gamma^2}{p} \Gamma_{k(T)}^{B_*} \right )^{-1} + I \right )}{ \bar{\epsilon}_S(A_*,B_*,\gamma^2,T,\delta)^2  \lambda_{\min} \left (\sigma^2 \Gamma_{k(T)} + \frac{\gamma^2}{p} \Gamma_{k(T)}^{B_*} \right )} \right \}
\end{split}
\end{align}
Algorithm \ref{alg:active_lds_noise} with $FT = $ \texttt{True} will achieve the following rate:
$$ \mathbb{P} \left [ \| \hat{A} - A_* \|_2 \leq C \sigma \sqrt{\frac{\log \frac{1}{\delta} + d + \log \det \left ( \bar{\Gamma}_{T } \left ( \sigma^2 \Gamma_{k(T)} + \frac{\gamma^2}{p} \Gamma_{k(T)}^{B_*} \right )^{-1} + I \right )}{T \lambda_{\min} \left (  \sigma^2 \Gamma_{k(T)}  + \gamma^2 \Gamma_{k(T)}^{u^*} \right )}} \right ] \geq 1 - 9 \delta$$
and will produce inputs satisfying $\mathbb{E} \left [ 1/T \sum_{t=1}^T u_t^\top u_t \right ] \leq \gamma^2$. Here $c_1, c_2, C$ are universal constants, $u^*$ is the solution to {\normalfont \texttt{OptInput}}$_{k(T)}(A_*,B_*,\gamma^2,k(T),0)$, and $\bar{\Gamma}_T = 16  \frac{\beta(A_*)^2 \gamma^2 }{ (1 - \rho(A_*))^2} (1 + T) I + 4  \left (  tr \left (\sigma^2 \Gamma_{T} + \frac{\gamma^2}{p} \Gamma_T^{B_*} \right ) \left ( 1 +  \log \frac{2}{\delta} \right ) I  \right )$.
\end{theorem}

Several additional remarks are in order.

\begin{remark}
For Theorem \ref{cor:largek_alg_opt} to hold, $T_0$ and $k_0$ must be set to satisfy (\ref{eq:alg_perf_interp_ktinit}). This condition is necessary to guarantee that the burn-in time required by Theorem \ref{thm:concentration2} is met at each epoch. Satisfying this condition requires knowledge of the unknown system so, in practice, we cannot guarantee that it will be met for some $T_0,k_0$. However, since Algorithm \ref{alg:active_lds_noise} increases $T_i$ faster than $k_i$, regardless of how $T_0,k_0$ are set, it will eventually satisfy the burn-in condition of Theorem \ref{thm:concentration2}, and  so the conclusion of Theorem \ref{cor:largek_alg_opt} will eventually hold. 
\end{remark}

\begin{remark}
Every line in \texttt{UpdateInputs}, with the exceptions of solving $\texttt{OptInput}$, is at worst a convex program and can be solved efficiently. Computing $\mathcal{M}(A,\{ x_t \}_{t=1}^T)$ in line \ref{alg_line:mathcalM} involves a linear search over $\ell \in [k]$ and the computation of a minimum eigenvalue for each $\ell$. $\mathcal{M}(A,\{ x_t \}_{t=1}^T)$ will be an ellipsoid. Line \ref{alg_line:check_opt} and line \ref{alg_line:construct_I} also involve iterating over all $\ell \in [k]$ and for each $\ell$, maximizing a quadratic over an ellipsoid. Since the maximization of a quadratic over an ellipsoid can be solved via a single SVD, this step can be efficiently completed.  While $k$ is growing exponentially with the epoch, we only call \texttt{UpdateInputs} once per epoch. Since the epoch length is also increasing exponentially, the number of epochs is only logarithmic in $T$. Thus, the total number of flops is only linear in $T$. In practice, one should simply stop increasing k when a sufficiently fine discretization of the space is reached to obtain close to optimal performance. Experimentally, we found this worked quite well.
\end{remark}

\begin{remark}
The only constraint we place on the inputs is that their average power is bounded by some value. This constraint allows for signals with large amplitudes, a situation which is often highly undesirable in practice. To avoid this possibility, further constraints could be added $\texttt{OptInput}$ to guarantee that the input computed has bounded amplitude as well as power. Unfortunately, amplitude constraints are non-trivial to enforce when optimizing in the frequency domain. Further, adding this constraint would cause us to lose the guarantee of global optimality of inputs. In practice, we have observed that the optimal inputs typically do not exhibit large spikes are other such undesirable behavior. 
\end{remark}

\begin{remark}
The restriction that $\rho(A_*) < 1$ is necessary to guarantee that the system will reach steady-state when a new input is played. As such, all our finite time results fundamentally depend on this assumption. A first step towards relaxing it would be proving a version of Proposition \ref{prop:covbound} that does not require the system has reached steady state. We leave this for future work.

We also note that, in some sense, the interesting regime for active system identification is when $\rho(A_*) < 1$. As was shown in \cite{sarkar2018how}, when all modes in $A_*$ are unstable, the system can be estimated at an exponential rate. Thus, in this case, active identification is likely unnecessary. A more interesting regime may be when some eigenvalues of $A_*$ have magnitude greater than 1, and some have magnitude less than 1. In this case active identification could be used to excite the modes corresponding to the smaller eigenvalues. We leave this direction for future work.
\end{remark}

We next present our master theorem quantifying the performance of Algorithm \ref{alg:active_lds_noise}. Algorithm \ref{alg:active_lds_noise} operates in three regimes. In the first regime, when $T_i$ is not large enough for the system to reach steady state, we are only able to guarantee learning due to the contribution of the noise. In the second regime, $T_i$ is large enough for the system to reach steady state but $\epsilon_i$ is not small enough for all frequencies to be playable. Finally, in the third regime, $T_i$ is large enough to reach steady state and all frequencies are playable, allowing us to attain the optimal performance. All three regimes are quantified in Theorem \ref{thm:ft_alg_perf_master}.

\begin{theorem}\label{thm:ft_alg_perf_master}
Assume that $\gamma^2 \geq \frac{(1- \rho(A_*))^2}{2 \beta(A_*)^2}$, and:
\begin{equation}\label{eq:master_thm_tkinit}
T_0 \geq c k_0 \left ( \log \frac{1}{\delta} + d + d \log \left (   \frac{2 \beta(A_*)^2 \gamma^2 }{ (1 - \rho(A_*))^2} (1 + T_0)  + \frac{ 4 \beta(A_*)^2 d (\sigma^2 + \sigma_u^2 \| B_* \|_2)}{1 - \rho(A_*)^2}  ( 1 + \log \frac{2}{\delta} )  \right )  \right )
\end{equation}
Let $  \bar{\Gamma}_T = 16  \frac{\beta(A_*)^2 \gamma^2 }{ (1 - \rho(A_*))^2} (1 + T) I + 4  \left (  tr \left (\sigma^2 \Gamma_{T} + \frac{\gamma^2}{p} \Gamma_T^{B_*} \right ) \left ( 1 +  \log \frac{2}{\delta} \right ) I  \right )$ . Then Algorithm \ref{alg:active_lds_noise} with $FT = $ \texttt{True} will have: 
\begin{enumerate}[(a)]
\item \label{thm:ft_alg_perf_master:a} For any $i$: 
$$\mathbb{P} \left [ \| \hat{A}_i - A_* \|_2 \leq C_1 \sigma \sqrt{\frac{\log \frac{1}{\delta} + d + \log \det ( \bar{\Gamma}_T (\Gamma_{k_i}^\eta )^{-1} + I)}{T \lambda_{\min}(\Gamma_{k_i}^\eta )}} \right ] \geq 1- 3 \delta$$
\item \label{thm:ft_alg_perf_master:b} If:
$$ T_i \geq 3 T_{ss} \left ( \frac{1}{10} \lambda_{\min}(\tilde{\Gamma}_{k_i}^{u_i}), k_i \right )$$
then:
$$\mathbb{P} \left [ \| \hat{A}_i - A_* \|_2 \leq C_2 \sigma \sqrt{\frac{\log \frac{1}{\delta} + d + \log \det ( \bar{\Gamma}_T (\Gamma_{k_i}^\eta + \tilde{\Gamma}_{k_i}^{u_i^*})^{-1} + I)}{T \lambda_{\min}(\Gamma_{k_i}^\eta + \tilde{\Gamma}_{k_i}^{u_i^*})}} \right ] \geq 1- 5 \delta$$
where $u_i^*$ is the solution to ${\normalfont \texttt{OptInput}}_{k_{i}}(A_*,B_*,\gamma^2,\mathcal{I}_i,\{ x_t \}_{t=1}^{T-T_i})$.
\item \label{thm:ft_alg_perf_master:c} If:
$$ T_i \geq \max \left \{ 3 T_{ss}  \left ( \frac{1}{10} \lambda_{\min}(\tilde{\Gamma}_{k_i}^{u_i}), k_i \right ), C_3^2 \sigma^2 \frac{\log \frac{1}{\delta} + d + \log \det ( \bar{\Gamma}_{T} (\Gamma_{k_{i-1}}^\eta )^{-1} + I)}{\bar{\epsilon}_S(A_*,B_*,\gamma^2,T - T_i,\delta)^2  \lambda_{\min}(\Gamma_{k_{i-1}}^\eta)} \right \}$$ 
then:
$$\mathbb{P} \left [ \| \hat{A}_i - A_* \|_2 \leq C_4 \sigma \sqrt{\frac{\log \frac{1}{\delta} + d + \log \det ( \bar{\Gamma}_T (\Gamma_{k_i}^\eta + \tilde{\Gamma}_{k_i}^{u_i^*})^{-1} + I)}{T \lambda_{\min}(\Gamma_{k_i}^\eta + \tilde{\Gamma}_{k_i}^{u_i^*})}} \right ] \geq 1- 9 \delta$$
where $u_i^*$ is the solution to ${\normalfont \texttt{OptInput}}_{k_{i}}(A_*,B_*,\gamma^2,[k_{i}],\{ x_t \}_{t=1}^{T-T_i})$.
\end{enumerate}
In all cases, the inputs produced will satisfy:
\begin{equation*}
\mathbb{E} \left [ \frac{1}{T} \sum_{t=1}^T u_t^\top u_t \right ] \leq \gamma^2
\end{equation*}
Here $C_1,C_2,C_3,C_4$ are universal constants.
\end{theorem}

\subsection{Proof of Theorem \ref{cor:largek_alg_opt_informal} and Theorem \ref{cor:largek_alg_opt}}\label{sec:alg_ft_full_proof}
The proof of Theorem  \ref{cor:largek_alg_opt} follows an event-based analysis. We define several events, show that they all hold with high probability, and that together they imply the rate given in Theorem  \ref{cor:largek_alg_opt} holds. We outline the steps at a high level here. 

We first must show that the estimate attained at the $i-1$th epoch is sufficiently accurate to guarantee that we are playing inputs that achieve a response close to optimal. Defining the event $\mathcal{E}_7$ to be the event that $\epsilon_{i-1}$ is this small, Theorem \ref{thm:ft_alg_perf_master} shows that this holds with high probability. To show that the value of $\epsilon_{i-1}$ is sufficiently small to guarantee that our inputs are nearly optimal, we must show that $\epsilon_S \geq \bar{\epsilon}_S$. This requires controlling the covariates in a specific direction, $w_{\min}$ which we define below. Event $\mathcal{E}_6$ is the event on which this is controlled and Lemma \ref{lem:upper_bound} shows that it holds with high probability. On this event, Theorem \ref{thm:opt_sln_perturbation} and Lemmas \ref{lem:epss_lb} and \ref{lem:ft_sufficient} guarantee that, given $\epsilon_{i-1}$ this small, we will have that our inputs achieve a nearly optimal response. 

The remaining events are needed to guarantee our estimation rate at epoch $i$ holds. Event $\mathcal{E}_1$ guarantees an upper bound on the covariates. $\mathcal{E}_2$ and $\mathcal{E}_3$ are both lower bounds on the covariates. $\mathcal{E}_2$ lower bounds the covariates from all epochs prior to epoch $i$ in terms of the noise and $\mathcal{E}_3$ lower bounds the covariates from the $i$th epochs in terms of the input. The former is necessary for more technical reasons while the latter allows us to lower bound the covariates in terms of the inputs, which ultimately yields the rate that depends on the input response. $\mathcal{E}_1$ is shown to hold with high probability by Lemma \ref{lem:upper_bound} and, conditioned on the covariates being upper bounded Lemma \ref{lem:cov_lb_noise} shows that $\mathcal{E}_2$ holds with high probability.

A slightly more subtle issue arises in showing that $\mathcal{E}_3$ holds with high probability. For $\mathcal{E}_3$ to hold, we must have that $T$ is large enough to guarantee that the system has reached steady state in the $i$th epoch. Guaranteeing the steady state condition is reached requires the initial state at the start of the epoch, $x_{T - T_i}$, to be bounded. Given that such a bound holds, we can guarantee, in terms of this bound, that $T$ will be sufficiently large for the system to reach steady state.  Event $\mathcal{E}_5$ gives this upper bound on $x_{T - T_i}$ and Lemma \ref{lem:alg_xtnorm_bound} shows that it holds with high probability. Given this and the burn-in condition required by Theorem \ref{cor:largek_alg_opt}, it follows that the system will have reached steady state at epoch $i$. This, combined with $\mathcal{E}_1$ holding, allows us to apply Corollary \ref{cor:cov_lb_inputs_burnin} to show that $\mathcal{E}_3$ holds with high probability.  

Event $\mathcal{E}_4$ next shows that the self-normalized term in the error is bounded. On the event that the covariates are upper and lower bounded as in $\mathcal{E}_1,\mathcal{E}_2,\mathcal{E}_3$, $\mathcal{E}_4$ holds with high probability by Lemma \ref{lem:self_normalized_past_data2}. 

Finally, we show that if all of these events hold simultaneously, the ``good'' event, $\mathcal{A}$, which guarantees that the rate in Theorem \ref{cor:largek_alg_opt} holds, is always true. Since all of these events hold with high probability, it then follows that $\mathcal{A}$ holds with high probability.

\begin{proof}
Throughout we will let $T = \sum_{j=0}^i T_i$, the total time that has elapsed after $i$ epochs. 

Let $\bar{u}^*$ be the solution to $\texttt{OptInput}_{k_i}(A_*,B_*,\gamma^2,k_i,0)$ and define the following events:
\begin{align*}
& \mathcal{A} := \left \{ \| \hat{A}_i - A_* \|_2 \leq C \sigma \sqrt{\frac{\log \frac{1}{\delta} + d + \log \det ( \bar{\Gamma}_T (\Gamma_{k_i}^\eta + \tilde{\Gamma}_{k_i}^{u_i})^{-1} + I)}{T \lambda_{\min}\left (\Gamma_{k_i}^\eta + \tilde{\Gamma}_{k_i}^{\bar{u}^*} \right )}} \right \} \\
& \mathcal{E}_1 := \left \{  \sum_{t=1}^T x_t x_t^\top \preceq T \bar{\Gamma}_T \right \} \\
& \mathcal{E}_2 := \left \{   \sum_{t=1}^{T - T_i} x_t x_t^\top \succeq c_1 T_{i-1} \Gamma_{k_{i}}^\eta \right \} \\
& \mathcal{E}_3 :=  \left \{ \sum_{t=T - T_i}^T x_t x_t^\top + \sum_{t=1}^{T-T_i} x_t x_t^\top \succeq c_2 T_i \tilde{\Gamma}_{k_i}^{u_i} + \frac{1}{2} \sum_{t=1}^{T-T_i} x_t x_t^\top \right \} \\
& \mathcal{E}_4 := \left \{ \left \| \left ( \sum_{t=1}^{T} x_t x_t^\top \right )^{-1/2} \sum_{t=1}^{T} x_t \eta_t^\top \right \|_2 \leq c_3 \sigma \sqrt{\log \frac{1}{\delta} + d + \log \det ( \bar{\Gamma}_{T } (\Gamma_{k_{i}}^\eta +  \tilde{\Gamma}_{k_i}^{u_i})^{-1} + I)} \right \} \\
& \mathcal{E}_5 := \left \{ \| x_{T - T_i} \|_2 \leq \frac{2\beta(A_*) \| B_* \|_2 k_{i-1} \gamma}{1 - \bar{\rho}(A_*)^{k_{i-1}}} + \sqrt{2 tr \left ( \Gamma_{T-T_i}^\eta \right ) \left ( 1 + \frac{1}{c_4} \log \frac{4}{\delta} \right )}  \right \} \\
& \mathcal{E}_6 := \left \{ \sum_{t=1}^T (w_{\min}^\top x_t)^2 \leq 4 \sum_{t=1}^T (w_{\min}^\top x_t^u)^2 + 4 T \left ( 1 + \log \frac{2}{\delta} \right ) w_{\min}^\top ( \sigma^2 \Gamma_T + \sigma_u^2 \Gamma_T^{B_*})w_{\min} \right \} \\
& \mathcal{E}_7 := \left \{ \epsilon_{i-1} \leq \bar{\epsilon}_S(A_*,B_*,\gamma^2,T,\delta) \right \}
\end{align*}
Let $A_* = PJP^{-1}$ and $p_i$ denote the columns of $P$. Here $w_{\min}$ is any unit norm vector such that $w_{\min}^\top p_i = 0$ for all $p_i$ that do not correspond to the minimum eigenvalue of $A_*$. We wish to bound $\mathbb{P}[\mathcal{A}^c]$. The following set of inequalities obviously holds:
\begin{align*}
\mathbb{P}[\mathcal{A}^c] & \leq \mathbb{P}[\mathcal{A}^c \cap \mathcal{E}_1 ] + \mathbb{P} [ \mathcal{E}_1^c] \\
& \leq \mathbb{P}[\mathcal{A}^c \cap \mathcal{E}_1 \cap \mathcal{E}_5 ] + \mathbb{P}[\mathcal{E}_1^c] + \mathbb{P} [ \mathcal{E}_5^c] \\
& \leq \mathbb{P}[\mathcal{A}^c  \cap \mathcal{E}_1 \cap \mathcal{E}_2 \cap \mathcal{E}_5 ] + \mathbb{P}[\mathcal{E}_1^c] + \mathbb{P}[\mathcal{E}_1 \cap \mathcal{E}_2^c] + \mathbb{P} [ \mathcal{E}_5^c] \\
& \leq \mathbb{P}[\mathcal{A}^c \cap \mathcal{E}_1 \cap \mathcal{E}_2 \cap \mathcal{E}_3 \cap \mathcal{E}_5  ] + \mathbb{P}[\mathcal{E}_1^c] + \mathbb{P}[\mathcal{E}_1 \cap \mathcal{E}_2^c] + \mathbb{P}[\mathcal{E}_1 \cap \mathcal{E}_2 \cap \mathcal{E}_5 \cap \mathcal{E}_3^c ]   + \mathbb{P} [ \mathcal{E}_5^c] \\
& \leq \mathbb{P}[\mathcal{A}^c \cap \mathcal{E}_1 \cap \mathcal{E}_2 \cap \mathcal{E}_3 \cap \mathcal{E}_4 \cap \mathcal{E}_5] + \mathbb{P}[\mathcal{E}_1^c] + \mathbb{P}[\mathcal{E}_1 \cap \mathcal{E}_2^c] + \mathbb{P}[\mathcal{E}_1 \cap \mathcal{E}_2 \cap \mathcal{E}_5 \cap \mathcal{E}_3^c] \\
& \ \ \ \ \ \ \ \ \ \ \ \ + \mathbb{P}[ \mathcal{E}_1 \cap \mathcal{E}_2 \cap \mathcal{E}_3 \cap \mathcal{E}_4^c]  
  + \mathbb{P} [ \mathcal{E}_5^c] \\
& \leq \mathbb{P}[\mathcal{A}^c  \cap \mathcal{E}_1 \cap \mathcal{E}_2 \cap \mathcal{E}_3 \cap \mathcal{E}_4 \cap \mathcal{E}_5 \cap \mathcal{E}_6] + \mathbb{P}[\mathcal{E}_1^c] + \mathbb{P}[\mathcal{E}_1 \cap \mathcal{E}_2^c] + \mathbb{P}[\mathcal{E}_1 \cap \mathcal{E}_2 \cap \mathcal{E}_5 \cap \mathcal{E}_3^c] \\
& \ \ \ \ \ \ \ \ \ \ \ \ + \mathbb{P}[ \mathcal{E}_1 \cap \mathcal{E}_2 \cap \mathcal{E}_3 \cap \mathcal{E}_4^c]  
  + \mathbb{P} [ \mathcal{E}_5^c] + \mathbb{P}[\mathcal{E}_6^c] \\
& \leq \mathbb{P}[\mathcal{A}^c  \cap \mathcal{E}_1 \cap \mathcal{E}_2 \cap \mathcal{E}_3 \cap \mathcal{E}_4 \cap \mathcal{E}_5 \cap \mathcal{E}_6 \cap \mathcal{E}_7] + \mathbb{P}[\mathcal{E}_1^c] + \mathbb{P}[\mathcal{E}_1 \cap \mathcal{E}_2^c] + \mathbb{P}[\mathcal{E}_1 \cap \mathcal{E}_2 \cap \mathcal{E}_5 \cap \mathcal{E}_3^c] \\
& \ \ \ \ \ \ \ \ \ \ \ \ + \mathbb{P}[ \mathcal{E}_1 \cap \mathcal{E}_2 \cap \mathcal{E}_3 \cap \mathcal{E}_4^c]  
  + \mathbb{P} [ \mathcal{E}_5^c] + \mathbb{P}[\mathcal{E}_6^c] + \mathbb{P}[\mathcal{E}_7^c]
\end{align*}
By part \ref{thm:ft_alg_perf_master:a} of Theorem \ref{thm:ft_alg_perf_master} it follows that if:
$$ T_{i-1} \geq C^2 \sigma^2 \frac{\log \frac{1}{\delta} + d + \log \det ( \bar{\Gamma}_{T} (\Gamma_{k_{i-1}}^\eta )^{-1} + I)}{\bigg ( \bar{\epsilon}_S(A_*,B_*,\gamma^2,T,\delta) \bigg )^2  \lambda_{\min}(\Gamma_{k_{i-1}}^\eta)}$$
then $\mathbb{P}[\mathcal{E}_7^c] \leq 3 \delta$.  By Lemma \ref{lem:alg_xtnorm_bound} we have that $\mathbb{P}[\mathcal{E}_5^c] \leq \delta$. In Lemma \ref{lem:logdet_upperbound} we proved that:
\begin{align*}
\tilde{\Gamma}_{T,0}^u = \frac{1}{T} \sum_{t=1}^T x_t^u {x_t^u}^\top & \preceq   \frac{\beta(A_*)^2 \gamma^2 }{2 (1 - \bar{\rho}(A_*))^2} (1 + T) I
\end{align*}
which allows us to deterministically upper bound:
\begin{align*}
& 4 \sum_{j=0}^i \left ( \tilde{\Gamma}_{T_j}^{u_j} +  tr \left (\Gamma_{T_j}^\eta \right ) \left ( 1 +  \log \frac{2}{\delta} \right ) I  \right ) \preceq 4  \frac{\beta(A_*)^2 \gamma^2 }{ (1 - \bar{\rho}(A_*))^2} (1 + T) I + 4  \left (  tr \left (\Gamma_{T}^\eta \right ) \left ( 1 +  \log \frac{2}{\delta} \right ) I  \right )
\end{align*}
By Lemma \ref{lem:upper_bound} we then have that $\mathbb{P}[\mathcal{E}_1^c] \leq \delta$. By Lemma \ref{lem:upper_bound}, $\mathbb{P}[\mathcal{E}_6^c] \leq \delta$. Note that on the event $\mathcal{E}_1$ by Lemmas \ref{lem:burn_in3} and \ref{lem:logdet_upperbound} the burn-in time required by Lemma \ref{lem:cov_lb_noise}  will be met at the end of epoch $i$ assuming that $k_0,T_0$ are chosen to satisfy (\ref{eq:alg_perf_interp_ktinit}). Since $u_{i-1}$ is random we cannot apply Lemma \ref{lem:cov_lb_noise} to bound this directly, however:
\begin{align*}
\mathbb{P} [\mathcal{E}_1 \cap  \mathcal{E}_2^c] & = \mathbb{E} [ \mathbb{I} \{\mathcal{E}_1 \cap \mathcal{E}_2^c  \} ] \\
& = \mathbb{E} [ \mathbb{E} [ \mathbb{I} \{ \mathcal{E}_1 \cap \mathcal{E}_2^c \} | \mathcal{F}_{T - T_i - T_{i-1}} ] ] \\
& \overset{(a)}{\leq} \mathbb{E} \left [ \mathbb{E} \left [ \mathbb{I} \left \{ \sum_{t=T - T_i - T_{i-1}}^{T - T_{i}} x_t x_t^\top \not\succeq c_1 T_{i-1} \Gamma_{k_i}^\eta,  \sum_{t=1}^T x_t x_t^\top \preceq T \bar{\Gamma}_T \right \} | \mathcal{F}_{T - T_i - T_{i-1}} \right ] \right ] \\
& \overset{(b)}{\leq} \delta
\end{align*}
Here $(a)$ follows since $\sum_{t=T - T_i - T_{i-1}}^{T - T_{i}} x_t x_t^\top \preceq \sum_{t=1}^{T - T_{i}} x_t x_t^\top$ so $\mathbb{P}[\sum_{t=T - T_i - T_{i-1}}^{T - T_{i}} x_t x_t^\top \not\succeq c_1 T_{i-1} \Gamma_{k_i}^\eta] \geq \mathbb{P}[\sum_{t=1}^{T - T_{i}} x_t x_t^\top \not\succeq c_1 T_{i-1} \Gamma_{k_i}^\eta]$, and $(b)$ follows by applying Lemma \ref{lem:cov_lb_noise} since $u_{i-1}$ is deterministic on $\mathcal{F}_{T - T_i - T_{i-1}}$. 

A similar argument can be applied to bound $\mathbb{P}[\mathcal{E}_1 \cap \mathcal{E}_2 \cap \mathcal{E}_5 \cap \mathcal{E}_3^c]$. By Corollary \ref{cor:cov_lb_inputs_burnin}, we will have that $\mathbb{P}[ \mathcal{E}_1 \cap \mathcal{E}_2 \cap \mathcal{E}_5 \cap \mathcal{E}_3^c] \leq \delta$ so long as the steady state condition required by Proposition \ref{prop:covbound} is met for every $w \in \mathcal{S}^{d-1}$ where $c T_i w^\top \tilde{\Gamma}_{k_i}^{u_i} w \geq \sum_{t=1}^{T-T_i} (w^\top x_t)^2$. That is, we need:
\begin{equation*}
\left | \sum_{t=T'+1}^{T' + k } \left (w^\top x_t^{u_i} - \frac{1}{k} \sum_{t'=T'+1}^{T' + k } w^\top x_{t'}^{u_i} \right )^2 - k_i w^\top \tilde{\Gamma}_{k_i}^{u_i} w \right | \leq \frac{1}{10} k_i w^\top \tilde{\Gamma}_{k+i}^{u_i} w 
\end{equation*}
for all $w$ meeting this condition. On the event $\mathcal{E}_5$, by Corollary \ref{cor:tss_bound}, this burn in time will be reached as long as:
$$ T' \geq \max_{w \in \mathcal{S}^{d-1} \ : \ c T_i w^\top \tilde{\Gamma}_{k_i}^{u_i} w \geq \sum_{t=1}^{T-T_i} (w^\top x_t)^2 } T_{ss} \left ( \frac{1}{10} w^\top \tilde{\Gamma}_{k_i}^{u_i} w  , k_i \right ) $$
Note that on the event $\mathcal{E}_2$ and since $T_{ss}$ increases as its first argument decreases, we have:
\begin{align*} 
& \max_{w \in \mathcal{S}^{d-1} \ : \ c T_i w^\top \tilde{\Gamma}_{k_i}^{u_i} w \geq \sum_{t=1}^{T-T_i} (w^\top x_t)^2 } T_{ss} \left ( \frac{1}{10} w^\top \tilde{\Gamma}_{k_i}^{u_i} w  , k_i \right ) \\
\leq \ &  \max_{w \in \mathcal{S}^{d-1} \ : \ c T_i w^\top \tilde{\Gamma}_{k_i}^{u_i} w \geq \sum_{t=1}^{T-T_i} (w^\top x_t)^2 } T_{ss} \left ( \frac{1}{10c} \frac{1}{T_i} \sum_{t=1}^{T-T_i} (w^\top x_t)^2  , k_i \right ) \\
\leq \ &  \max_{w \in \mathcal{S}^{d-1} \ : \ c T_i w^\top \tilde{\Gamma}_{k_i}^{u_i} w \geq \sum_{t=1}^{T-T_i} (w^\top x_t)^2 } T_{ss} \left ( \frac{c_1}{10c} \frac{T_{i-1}}{T_i} \Gamma_{k_i}^\eta  , k_i \right ) \\
\leq \ & T_{ss} \left ( c_5 \lambda_{\min}(\Gamma_{k_i}^\eta), k_i \right )
\end{align*}
By Lemmas \ref{lem:burn_in3} and \ref{lem:logdet_upperbound} and on the event $\mathcal{E}_1$, assuming that $k_0,T_0$ are chosen to satisfy (\ref{eq:alg_perf_interp_ktinit}), we will have that:
$$ T_i \geq c k_i \left ( d + \log \frac{1}{\delta} + \log \det ( \bar{\Gamma}_T (\tilde{\Gamma}_{k_i}^{u_i})^{-1} ) \right )$$
so if $T_i \geq 3 T_{ss} \left ( c_5 \lambda_{\min}(\Gamma_{k_i}^\eta), k_i \right )$, then:
$$ T_i \geq 2T_{ss} \left ( c_5 \lambda_{\min}(\Gamma_{k_i}^\eta), k_i \right ) + \frac{1}{3} c k_i \left ( d + \log \frac{1}{\delta} + \log \det ( \bar{\Gamma}_T (\tilde{\Gamma}_{k_i}^{u_i})^{-1} ) \right )$$
Noting that on $\mathcal{E}_1$, we will have that $\sum_{t=1}^{T-T_i} x_t x_t^\top \preceq T \bar{\Gamma}_T$, we see then that the burn-in time required by Corollary \ref{cor:cov_lb_inputs_burnin} will be met if $T_i \geq 3 T_{ss} \left ( c_5 \lambda_{\min}(\Gamma_{k_i}^\eta), k_i \right )$. Repeating the same calculation we used to bound $\mathbb{P}[ \mathcal{E}_1 \cap \mathcal{E}_2^c]$ to handle the fact that $u_i$ is random, we conclude, by Corollary \ref{cor:cov_lb_inputs_burnin}, that $\mathbb{P}[\mathcal{E}_1 \cap \mathcal{E}_2 \cap \mathcal{E}_5 \cap \mathcal{E}_3^c] \leq  \delta$.

By Lemma \ref{lem:self_normalized_past_data2}, we have directly that $\mathbb{P}[ \mathcal{E}_1 \cap \mathcal{E}_2 \cap \mathcal{E}_3 \cap \mathcal{E}_4^c]  \leq \delta$.

Finally, we must bound $\mathbb{P}[\mathcal{A}^c \cap \mathcal{E}_1 \cap \mathcal{E}_2 \cap \mathcal{E}_3 \cap \mathcal{E}_4 \cap \mathcal{E}_5 \cap \mathcal{E}_6 \cap \mathcal{E}_7]$. We can decompose the error as:
\begin{align*}
\| \hat{A}_i - A_* \|_2 & = \| (X X^\top)^{-1} X^\top E \|_2 \\
& \leq \| (X X^\top)^{-1/2} \|_2 \| (X X^\top)^{-1/2} X^\top E \|_2 \\
& = \lambda_{\min}(X X^\top)^{-1/2} \| (X X^\top)^{-1/2} X^\top E \|_2
\end{align*}
On the event $\mathcal{E}_1 \cap \mathcal{E}_2 \cap \mathcal{E}_3 \cap \mathcal{E}_4 \cap \mathcal{E}_5 \cap \mathcal{E}_6 \cap \mathcal{E}_7$, we will have that:
$$  \| (X X^\top)^{-1/2} X^\top E \|_2 \leq c_3 \sigma \sqrt{\log \frac{1}{\delta} + d + \log \det ( \bar{\Gamma}_{T } (\Gamma_{k_{i}}^\eta + \tilde{\Gamma}_{k_i}^{u_i})^{-1} + I)}$$
Furthermore, on this event, we will have that $\epsilon_{i-1} \leq \bar{\epsilon}_S(A_*,B_*,\gamma^2,T,\delta)$ and all the conditions of Lemma \ref{lem:epss_lb} will be met so $\epsilon_S(A_*,B_*,\gamma^2,k_i,\{ x_t \}_{t=1}^T, \delta) \geq \bar{\epsilon}_S(A_*,B_*,\gamma^2,T,\delta)$. This implies that $\epsilon_{i-1} \leq \epsilon_S(A_*,B_*,\gamma^2,k_{i-1},\{x_t\}_{t=1}^{T-T_i},\delta)$ so by Lemma \ref{lem:ft_sufficient}, $\mathcal{I}_{i} = [k_{i}]$ and:
\begin{align}\label{eq:pert_bound_main_thm}
\begin{split}
& \left | \lambda_{\min} \left ( \frac{T_i}{k_i^2} H_{k_{i}}(A_*,B_*,U^*,[k_{i}]) + \sum_{t=1}^{T-T_i} x_t x_t^\top \right ) -  \lambda_{\min} \left ( \frac{T_i}{k_i^2} H_{k_{i}}(A_*,B_*,\hat{U},[k_{i}]) +  \sum_{t=1}^{T-T_i} x_t x_t^\top \right ) \right | \\
& \leq \frac{1}{2} \lambda_{\min} \left ( \frac{T_i}{k_i^2} H_{k_{i}}(A_*,B_*,U^*,[k_{i}]) +  \sum_{t=1}^{T-T_i} x_t x_t^\top \right ) 
\end{split}
\end{align} 
where $U^*$ is the solution to \texttt{OptInput}$_{k_{i}}(A_*,B_*,\gamma^2/2,[k_{i}],\{x_t\}_{t=1}^T)$ and $\hat{U}$ the solution to 

\noindent \texttt{OptInput}$_{k_{i}}(\hat{A}_{i-1},B_*,\gamma^2/2,[k_{i}],\{x_t\}_{t=1}^T)$. Furthermore, on this event we will have that:
\begin{align*}
\lambda_{\min}\left ( \sum_{t=1}^T x_t x_t^\top \right) & \overset{(a)}{\geq}  \lambda_{\min} \left ( \frac{1}{2} \sum_{t=1}^{T - T_i } x_t x_t^\top + c_2 T_i \tilde{\Gamma}_{k_i}^{u_i} \right ) \\
& \overset{(b)}{\geq} \frac{1}{2} \lambda_{\min} \left ( \sum_{t=1}^{T - T_i} x_t x_t^\top + c_2 T_i \tilde{\Gamma}_{k_i}^{u_i^*} \right ) \\
& \overset{(c)}{\geq} \frac{c_2}{2} \lambda_{\min} \left ( \sum_{t=1}^{T - T_i} x_t x_t^\top +  T_i \tilde{\Gamma}_{k_i}^{\bar{u}^*} \right ) \\
& \overset{(d)}{\geq}  \frac{c_3}{2} \lambda_{\min} \left (  T_{i-1} \Gamma_{k_i}^\eta +  T_i \tilde{\Gamma}_{k_i}^{\bar{u}^*} \right ) 
\end{align*}
where $(a)$ holds on $\mathcal{E}_3$, $(b)$ holds given (\ref{eq:pert_bound_main_thm}), $(c)$ holds since the inputs $u_i^*$ maximize the quantity $\lambda_{\min} \left ( \sum_{t=1}^{T - T_i} x_t x_t^\top + T_i \tilde{\Gamma}_{k_i}^{u_i^*} \right )$ under the power constraint, and $(d)$ holds on $\mathcal{E}_2$.  $T_i = \frac{2}{3} T + \frac{1}{3} T_0$ which implies that both $T_{i}$ and $T_{i-1}$ are greater than $\frac{1}{5} T$ so:
$$  \frac{c_3}{2} \lambda_{\min} \left (  T_{i-1} \Gamma_{k_i}^\eta +  T_i \tilde{\Gamma}_{k_i}^{\bar{u}^*} \right )  \geq  c_4 T \lambda_{\min} \left (  \Gamma_{k_i}^\eta +  \tilde{\Gamma}_{k_i}^{\bar{u}^*} \right ) $$
By the error decomposition above, on the event $\mathcal{E}_1 \cap \mathcal{E}_2 \cap \mathcal{E}_3 \cap \mathcal{E}_4 \cap \mathcal{E}_5 \cap \mathcal{E}_6 \cap \mathcal{E}_7$ it then follows that:
$$ \| \hat{A}_i - A_* \|_2 \leq C \sigma \sqrt{\frac{\log \frac{1}{\delta} + d + \log \det ( \bar{\Gamma}_{T } (\Gamma_{k_{i}}^\eta + \tilde{\Gamma}_{k_i}^{u_i})^{-1} + I)}{T \lambda_{\min} \left (  \Gamma_{k_i}^\eta +  \tilde{\Gamma}_{k_i}^{\bar{u}^*} \right )}}$$
so $\mathbb{P}[\mathcal{A}^c \cap \mathcal{E}_1 \cap \mathcal{E}_2 \cap \mathcal{E}_3 \cap \mathcal{E}_4 \cap \mathcal{E}_5  \cap \mathcal{E}_6 \cap \mathcal{E}_7] = 0$.

Combining all of this, we have that if:
$$ T_{i-1} \geq C^2 \sigma^2 \frac{\log \frac{1}{\delta} + d + \log \det ( \bar{\Gamma}_{T} (\Gamma_{k_{i-1}}^\eta )^{-1} + I)}{\bigg ( \bar{\epsilon}_S(A_*,B_*,\gamma^2,T,\delta) \bigg )^2  \lambda_{\min}(\Gamma_{k_{i-1}}^\eta)}$$
and $T_i \geq 3 T_{ss} \left ( c_5 \lambda_{\min}(\Gamma_{k_i}^{\eta}), k_i \right ) $ then:
$$ \mathbb{P} \left [ \| \hat{A}_i - A_* \|_2 \leq C \sigma \sqrt{\frac{\log \frac{1}{\delta} + d + \log \det ( \bar{\Gamma}_{T } (\Gamma_{k_{i}}^\eta + \tilde{\Gamma}_{k_i}^{u_i})^{-1} + I)}{T \lambda_{\min} \left (  \Gamma_{k_i}^\eta +  \tilde{\Gamma}_{k_i}^{\bar{u}^*} \right )}} \right ] \geq 1 - 9 \delta$$
To eliminate dependance on $i$, note that $T = \sum_{j=0}^i 3^j T_0 = \frac{T_0}{2} (3^{i+1} - 1)$ which implies that $i = \log ( 2T / T_0 + 1)/\log 3 - 1$, and that $T_i = \frac{2}{3} T + \frac{1}{3} T_0$ and $T_{i-1} = \frac{2}{9} T + \frac{1}{9}T_0$. We then have that if:
$$ \frac{2}{9} T + \frac{1}{9}T_0 \geq C^2 \sigma^2 \frac{\log \frac{1}{\delta} + d + \log \det ( \bar{\Gamma}_{T} (\Gamma_{k(T)/2}^\eta )^{-1} + I)}{\bigg ( \bar{\epsilon}_S(A_*,B_*,\gamma^2,T,\delta) \bigg )^2  \lambda_{\min}(\Gamma_{k(T)/2}^\eta)}$$
and $\frac{2}{3} T + \frac{1}{3} T_0 \geq 3 T_{ss} \left ( c_5 \lambda_{\min}(\Gamma_{k(T)/2}^{\eta}), k(T)/2 \right ) $ that:
$$ \mathbb{P} \left [ \| \hat{A}_i - A_* \|_2 \leq C \sigma \sqrt{\frac{\log \frac{1}{\delta} + d + \log \det ( \bar{\Gamma}_{T } (\Gamma_{k(T)}^\eta )^{-1} + I)}{T \lambda_{\min} \left (  \Gamma_{k(T)}^\eta +  \tilde{\Gamma}_{k(T)}^{\bar{u}^*} \right )}} \right ] \geq 1 - 9 \delta$$

\end{proof}

\subsection{Proof of Theorem \ref{thm:ft_alg_perf_master}}
Throughout we will let $T = \sum_{j=0}^i T_i$, the total time that has elapsed after $i$ epochs. 

We first note that the bound on expected power of the inputs follows directly from Lemma \ref{lem:alg_power_constraint} and by the power constraint imposed in $\texttt{OptInput}$.

\subsubsection{Proof of Theorem \ref{thm:ft_alg_perf_master} part \ref{thm:ft_alg_perf_master:a}} 
Let:
$$ \mathcal{A} := \left \{ \| \hat{A}_i - A_* \|_2 \leq C \sigma \sqrt{\frac{\log \frac{1}{\delta} + d + \log \det ( \bar{\Gamma}_T (\Gamma_{k_i}^\eta)^{-1} + I)}{T \lambda_{\min}(\Gamma_{k_i}^\eta)}} \right \}$$
be the event that our desired error bound holds, and define the following events:
\begin{align*}
& \mathcal{E}_1 := \left \{  \sum_{t=1}^T x_t x_t^\top \preceq T \bar{\Gamma}_T \right \} \\
& \mathcal{E}_2 := \left \{   \sum_{t=1}^{T} x_t x_t^\top \succeq c_1 T \Gamma_{k_i}^\eta \right \} \\
& \mathcal{E}_3 := \left \{ \left \| \left ( \sum_{t=1}^{T} x_t x_t^\top \right )^{-1/2} \sum_{t=1}^{T} x_t \eta_t^\top \right \|_2 \leq c_3 \sigma \sqrt{\log \frac{1}{\delta} + d + \log \det ( \bar{\Gamma}_{T } (\Gamma_{k_{i}}^\eta)^{-1} + I)} \right \} 
\end{align*}
We wish to bound $\mathbb{P}[\mathcal{A}^c]$. The following inequalities obviously hold:
\begin{align*}
\mathbb{P}[\mathcal{A}^c] & \leq \mathbb{P}[\mathcal{A}^c \cap \mathcal{E}_1] + \mathbb{P}[\mathcal{E}_1^c] \\
& \leq \mathbb{P}[\mathcal{A}^c \cap \mathcal{E}_1 \cap \mathcal{E}_2] + \mathbb{P}[\mathcal{E}_1 \cap \mathcal{E}_2^c] + \mathbb{P}[\mathcal{E}_1^c] \\
& \leq \mathbb{P}[\mathcal{A}^c \cap \mathcal{E}_1 \cap \mathcal{E}_2 \cap \mathcal{E}_3] + \mathbb{P}[\mathcal{E}_1 \cap \mathcal{E}_2 \cap \mathcal{E}_3^c] + \mathbb{P}[\mathcal{E}_1 \cap \mathcal{E}_2^c] + \mathbb{P}[\mathcal{E}_1^c]
\end{align*}
By Lemma \ref{lem:upper_bound}, and since, following the proof of Lemma \ref{lem:logdet_upperbound}:
$$4 \sum_{j=0}^i   \Gamma_{T_j}^{u_j} + 4 tr \left (\Gamma_{T}^\eta \right ) \left ( 1 +  \log \frac{2}{\delta} \right ) I  \preceq 4  \frac{\beta(A_*)^2 \gamma^2 }{ (1 - \bar{\rho}(A_*))^2} (1 + T) I + 4  \left (  tr \left (\Gamma_{T}^\eta \right ) \left ( 1 +  \log \frac{2}{\delta} \right ) I  \right )$$
we have that $\mathbb{P}[\mathcal{E}_1^c] \leq \delta$. Note that on the event $\mathcal{E}_1$, by Lemmas \ref{lem:burn_in3} and \ref{lem:logdet_upperbound} the burn-in time required by Lemma \ref{lem:cov_lb_noise}  will be met at the end of epoch $i$ assuming that $k_0,T_0$ are chosen to satisfy (\ref{eq:master_thm_tkinit}). Therefore, by Lemma \ref{lem:cov_lb_noise}, $\mathbb{P}[\mathcal{E}_1 \cap \mathcal{E}_2^c ] \leq \delta$. 

By Lemma \ref{lem:self_normalized_past_data2}, we have directly that $\mathbb{P}[ \mathcal{E}_1 \cap \mathcal{E}_2 \cap \mathcal{E}_3^c]  \leq \delta$. 

Finally, to bound $\mathbb{P}[\mathcal{A}^c \cap \mathcal{E}_1 \cap \mathcal{E}_2 \cap \mathcal{E}_3]$, note that:
\begin{align*}
\| \hat{A}_i - A_* \|_2 & = \| (X X^\top)^{-1} X^\top E \|_2 \\
& \leq \| (X X^\top)^{-1/2} \|_2 \| (X X^\top)^{-1/2} X^\top E \|_2 \\
& = \lambda_{\min}(X X^\top)^{-1/2} \| (X X^\top)^{-1/2} X^\top E \|_2
\end{align*}
On the event $\mathcal{E}_1 \cap \mathcal{E}_2 \cap \mathcal{E}_3$, we have that:
$$  \lambda_{\min}(X X^\top)^{-1/2} \leq C \sqrt{\frac{1}{T \lambda_{\min}(\Gamma_{k_i}^\eta)}}$$
and:
$$ \| (X X^\top)^{-1/2} X^\top E \|_2 \leq c_3 \sigma \sqrt{\log \frac{1}{\delta} + d + \log \det ( \bar{\Gamma}_{T } (\Gamma_{k_{i}}^\eta )^{-1} + I)}$$
Thus:
$$\| \hat{A}_i - A_* \|_2 \leq C \sigma \sqrt{\frac{\log \frac{1}{\delta} + d + \log \det ( \bar{\Gamma}_T (\Gamma_{k_i}^\eta)^{-1} + I)}{T \lambda_{\min}(\Gamma_{k_i}^\eta )}}$$ 
so $\mathbb{P}[\mathcal{A}^c \cap \mathcal{E}_1 \cap \mathcal{E}_2 \cap \mathcal{E}_3 ] = 0$. 
Combining everything, it follows that:
$$\mathbb{P} \left [ \| \hat{A}_i - A_* \|_2 \leq C \sigma \sqrt{\frac{\log \frac{1}{\delta} + d + \log \det ( \bar{\Gamma}_T (\Gamma_{k_i}^\eta )^{-1} + I)}{T \lambda_{\min}(\Gamma_{k_i}^\eta )}} \right ] \geq 1- 3 \delta$$

\subsubsection{Bounding the error with respect to inputs}\label{sec:thmmaster_pf_errorinput}
The proof of this mirrors closely the proof above but now with inputs included. Let:
$$ \mathcal{A} := \left \{ \| \hat{A}_i - A_* \|_2 \leq C \sigma \sqrt{\frac{\log \frac{1}{\delta} + d + \log \det ( \bar{\Gamma}_T (\Gamma_{k_i}^\eta + \tilde{\Gamma}_{k_i}^{u_i})^{-1} + I)}{T \lambda_{\min}(\Gamma_{k_i}^\eta + \tilde{\Gamma}_{k_i}^{u_i})}} \right \}$$
be the event that our desired error bound holds, and define the following events:
\begin{align*}
& \mathcal{E}_1 := \left \{  \sum_{t=1}^T x_t x_t^\top \preceq T \bar{\Gamma}_T \right \} \\
& \mathcal{E}_2 := \left \{   \sum_{t=1}^{T} x_t x_t^\top \succeq c_1 T \Gamma_{k_i}^\eta \right \} \\
& \mathcal{E}_3 :=  \left \{ \sum_{t=1}^T x_t x_t^\top \succeq c_2 T \tilde{\Gamma}_{k_i}^{u_i} \right \} \\
& \mathcal{E}_4 := \left \{ \left \| \left ( \sum_{t=1}^{T} x_t x_t^\top \right )^{-1/2} \sum_{t=1}^{T} x_t \eta_t^\top \right \|_2 \leq c_3 \sigma \sqrt{\log \frac{1}{\delta} + d + \log \det ( \bar{\Gamma}_{T } (\Gamma_{k_{i}}^\eta + \tilde{\Gamma}_{k_i}^{u_i})^{-1} + I)} \right \} \\
& \mathcal{E}_5 := \left \{ \| x_{T - T_i} \|_2 \leq \frac{2\beta(A_*) \| B_* \|_2 k_{i-1} \gamma}{1 - \bar{\rho}(A_*)^{k_{i-1}}} + \sqrt{2 tr \left ( \Gamma_{T-T_i}^\eta \right ) \left ( 1 + \frac{1}{c_4} \log \frac{4}{\delta} \right )}  \right \}
\end{align*}
We wish to bound $\mathbb{P}[\mathcal{A}^c]$. The following set of inequalities hold:
\begin{align*}
\mathbb{P}[\mathcal{A}^c] & \leq \mathbb{P}[\mathcal{A}^c \cap \mathcal{E}_1 ] + \mathbb{P} [ \mathcal{E}_1^c] \\
& \leq \mathbb{P}[\mathcal{A}^c \cap \mathcal{E}_1 \cap \mathcal{E}_5 ] + \mathbb{P}[\mathcal{E}_1^c] + \mathbb{P} [ \mathcal{E}_5^c] \\
& \leq \mathbb{P}[\mathcal{A}^c \cap \mathcal{E}_1 \cap \mathcal{E}_2 \cap \mathcal{E}_5] + \mathbb{P}[\mathcal{E}_1^c] + \mathbb{P}[\mathcal{E}_1 \cap \mathcal{E}_2^c]  + \mathbb{P} [ \mathcal{E}_5^c] \\
& \leq \mathbb{P}[\mathcal{A}^c  \cap \mathcal{E}_1 \cap \mathcal{E}_2 \cap \mathcal{E}_3 \cap \mathcal{E}_5]  + \mathbb{P}[\mathcal{E}_1^c] + \mathbb{P}[\mathcal{E}_1 \cap \mathcal{E}_2^c] + \mathbb{P}[\mathcal{E}_1 \cap \mathcal{E}_5 \cap \mathcal{E}_3^c]  + \mathbb{P} [ \mathcal{E}_5^c] \\
& \leq \mathbb{P}[\mathcal{A}^c \cap \mathcal{E}_1 \cap \mathcal{E}_2 \cap \mathcal{E}_3 \cap \mathcal{E}_4 \cap \mathcal{E}_5 ] + \mathbb{P}[\mathcal{E}_1^c] + \mathbb{P}[\mathcal{E}_1 \cap \mathcal{E}_2^c] + \mathbb{P}[\mathcal{E}_5 \cap \mathcal{E}_1 \cap \mathcal{E}_3^c] \\
& \ \ \ \ \ \ \ \ \ \ + \mathbb{P}[ \mathcal{E}_1 \cap \mathcal{E}_2 \cap \mathcal{E}_3 \cap \mathcal{E}_4^c] + \mathbb{P} [ \mathcal{E}_5^c]
\end{align*}
By Lemma \ref{lem:alg_xtnorm_bound} we have that $\mathbb{P}[\mathcal{E}_5^c] \leq \delta$. By Lemma \ref{lem:upper_bound} and since:
$$4 \sum_{j=0}^i   \Gamma_{T_j}^{u_j} + 4 tr \left (\Gamma_{T}^\eta \right ) \left ( 1 +  \log \frac{2}{\delta} \right ) I  \preceq 4  \frac{\beta(A_*)^2 \gamma^2 }{ (1 - \bar{\rho}(A_*))^2} (1 + T) I + 4  \left (  tr \left (\Gamma_{T}^\eta \right ) \left ( 1 +  \log \frac{2}{\delta} \right ) I  \right )$$
we have that $\mathbb{P}[\mathcal{E}_1^c] \leq \delta$. Note that on the event $\mathcal{E}_1$, by Lemmas \ref{lem:burn_in3} and \ref{lem:logdet_upperbound} the burn-in time required by Lemma \ref{lem:cov_lb_noise}  will be met at the end of epoch $i$ assuming that $k_0,T_0$ are chosen to satisfy (\ref{eq:master_thm_tkinit}). Since $u_i$ is random we cannot apply Lemma \ref{lem:cov_lb_noise} to bound this directly, however:
\begin{align*}
\mathbb{P} [ \mathcal{E}_1 \cap \mathcal{E}_2^c] & = \mathbb{E} [ \mathbb{I} \{\mathcal{E}_1 \cap \mathcal{E}_2^c\} ] \\
& = \mathbb{E} [ \mathbb{E} [ \mathbb{I} \{ \mathcal{E}_1 \cap \mathcal{E}_2^c \} | \mathcal{F}_{T - T_i} ] ] \\
& \leq \mathbb{E} \left [ \mathbb{E} \left [ \mathbb{I} \left \{ \sum_{t=T - T_i}^T x_t x_t^\top \not\succeq c_1 T \Gamma_{k_i}^\eta, \sum_{t=1}^T x_t x_t^\top \preceq T \bar{\Gamma}_T \right \} | \mathcal{F}_{T - T_i} \right ] \right ] \\
& \leq \delta
\end{align*}
where the last inequality follows by applying Lemma \ref{lem:cov_lb_noise} since $u_i$ is deterministic on $\mathcal{F}_{T - T_i}$ and noting that $T_i = \frac{2}{3} T + \frac{1}{3} T_0$. 

A similar argument can be applied to bound $\mathbb{P}[\mathcal{E}_1 \cap \mathcal{E}_5 \cap \mathcal{E}_3^c]$. Note that on the event $ \mathcal{E}_1$, by Lemmas \ref{lem:burn_in3} and \ref{lem:logdet_upperbound} and assuming that $k_0,T_0$ are chosen to satisfy (\ref{eq:master_thm_tkinit}), we will have that:
$$ T_i \geq c k_i \left ( d + \log \frac{1}{\delta} + \log \det ( \bar{\Gamma}_T (\tilde{\Gamma}_{k_i}^{u_i})^{-1} ) \right )$$
so if $T_i \geq 3 T_{ss} \left ( \frac{1}{10} \lambda_{\min}(\tilde{\Gamma}_{k_i}^{u_i}), k_i \right ) $, then:
$$ T_i \geq 2 T_{ss} \left ( \frac{1}{10} \lambda_{\min}(\tilde{\Gamma}_{k_i}^{u_i}), k_i \right )  + \frac{1}{3} c k_i \left ( d + \log \frac{1}{\delta} + \log \det ( \bar{\Gamma}_T (\tilde{\Gamma}_{k_i}^{u_i})^{-1} ) \right )$$
On the event $\mathcal{E}_5$, by Corollary \ref{cor:tss_bound}, $T_{ss} \left ( \frac{1}{10} \lambda_{\min}(\tilde{\Gamma}_{k_i}^{u_i}), k_i \right )$ will then be sufficiently large for the system to reach steady state so the burn-in time required by Lemma \ref{lem:cov_lb_inputs} will be met. Then repeating the same calculation as above to handle the fact that $u_i$ are random, we get that $\mathbb{P}[\mathcal{E}_1 \cap \mathcal{E}_5 \cap \mathcal{E}_3^c] \leq \delta$.

By Lemma \ref{lem:self_normalized_past_data2}, we have directly that $\mathbb{P}[ \mathcal{E}_1 \cap \mathcal{E}_2 \cap \mathcal{E}_3 \cap \mathcal{E}_4^c] \leq \delta$. 

Finally, to bound $\mathbb{P}[\mathcal{A}^c  \cap \mathcal{E}_1 \cap \mathcal{E}_2 \cap \mathcal{E}_3 \cap \mathcal{E}_4 \cap \mathcal{E}_5]$, note that:
\begin{align*}
\| \hat{A}_i - A_* \|_2 & = \| (X X^\top)^{-1} X^\top E \|_2 \\
& \leq \| (X X^\top)^{-1/2} \|_2 \| (X X^\top)^{-1/2} X^\top E \|_2 \\
& = \lambda_{\min}(X X^\top)^{-1/2} \| (X X^\top)^{-1/2} X^\top E \|_2
\end{align*}
On the event $\mathcal{E}_1 \cap \mathcal{E}_2 \cap \mathcal{E}_3 \cap \mathcal{E}_4 \cap \mathcal{E}_5$, we have that:
$$  \lambda_{\min}(X X^\top)^{-1/2} \leq C \sqrt{\frac{1}{T \lambda_{\min}(\Gamma_{k_i}^\eta + \tilde{\Gamma}_{k_i}^{u_i})}}$$
and:
$$ \| (X X^\top)^{-1/2} X^\top E \|_2 \leq c_3 \sigma \sqrt{\log \frac{1}{\delta} + d + \log \det ( \bar{\Gamma}_{T } (\Gamma_{k_{i}}^\eta + \tilde{\Gamma}_{k_i}^{u_i})^{-1} + I)}$$
Thus:
$$\| \hat{A}_i - A_* \|_2 \leq C \sigma \sqrt{\frac{\log \frac{1}{\delta} + d + \log \det ( \bar{\Gamma}_T (\Gamma_{k_i}^\eta + \tilde{\Gamma}_{k_i}^{u_i})^{-1} + I)}{T \lambda_{\min}(\Gamma_{k_i}^\eta + \tilde{\Gamma}_{k_i}^{u_i})}}$$ 
so $\mathbb{P}[\mathcal{A}^c \cap \mathcal{E}_1 \cap \mathcal{E}_2 \cap \mathcal{E}_3 \cap \mathcal{E}_4 \cap \mathcal{E}_5] = 0$. 
Combining everything, it follows that:
$$\mathbb{P} \left [ \| \hat{A}_i - A_* \|_2 \leq C \sigma \sqrt{\frac{\log \frac{1}{\delta} + d + \log \det ( \bar{\Gamma}_T (\Gamma_{k_i}^\eta + \tilde{\Gamma}_{k_i}^{u_i})^{-1} + I)}{T \lambda_{\min}(\Gamma_{k_i}^\eta + \tilde{\Gamma}_{k_i}^{u_i})}} \right ] \geq 1- 5 \delta$$

\subsubsection{Proof of Theorem \ref{thm:ft_alg_perf_master} part \ref{thm:ft_alg_perf_master:b}}
To complete the result, we must show that the inputs $u_i$, which are computed based on our estimate of the system $\hat{A}_{i-1}$, are close to the optimal inputs computed on the true system, for a specific set of frequencies $\mathcal{I}_i$. That is:
\begin{align*}
& \left | \lambda_{\min} \left ( \frac{T_i}{k_i^2} H_{k_{i}}(A_*,B_*,U^*,\mathcal{I}_i) +  \sum_{t=1}^{T-T_i} x_t x_t^\top \right ) -  \lambda_{\min} \left ( \frac{T_i}{k_i^2} H_{k_{i}}(A_*,B_*,\hat{U},\mathcal{I}_i) +  \sum_{t=1}^{T-T_i} x_t x_t^\top \right ) \right | \\
& \leq \frac{1}{2} \lambda_{\min} \left ( \frac{T_i}{k_i^2} H_{k_{i}}(A_*,B_*,U^*,\mathcal{I}_i) +  \sum_{t=1}^{T-T_i} x_t x_t^\top \right ) 
\end{align*} 
where $U^*$ is the solution to \texttt{OptInput}$_{k_{i}}(A_*,B_*,\gamma^2/2,\mathcal{I}_i,\{x_t\}_{t=1}^{T-T_i})$ and $\hat{U}$ the solution to 

\noindent \texttt{OptInput}$_{k_{i}}(\hat{A}_{i-1},B_*,\gamma^2/2,\mathcal{I}_i,\{x_t\}_{t=1}^{T-T_i})$. By Lemma \ref{lem:tf_approx_bound}, if $\epsilon_{i-1} \leq ( 4 \| (e^{j \frac{2 \pi \ell}{k_{i}}} I - \hat{A}_{i-1} )^{-1} \|_2 )^{-1}$, then:
\begin{equation*}
\| (e^{j \frac{2 \pi \ell}{k_{i}}} I - A_*)^{-1} \|_2 \leq \frac{4}{3} \| (e^{j \frac{2 \pi \ell}{k_{i}}} I - \hat{A}_{i-1} )^{-1} \|_2
\end{equation*}
and:
\begin{equation*}
\| w^\top (e^{j \frac{2 \pi \ell}{k_{i}}} I - A_*)^{-1} \|_2 \leq \frac{4}{3} \| w^\top (e^{j \frac{2 \pi \ell}{k_{i}}} I - \hat{A}_{i-1})^{-1} \|_2
\end{equation*}
this then implies that $\epsilon_{i-1} \leq ( 3 \| (e^{j \frac{2 \pi \ell}{k_{i}}} I - A_*)^{-1} \|_2 )^{-1}$ so, again by Lemma \ref{lem:tf_approx_bound}:
\begin{equation*}
\| (e^{j \frac{2 \pi \ell}{k_{i}}} I - \hat{A}_{i-1})^{-1} \|_2 \leq \frac{3}{2} \| (e^{j \frac{2 \pi \ell}{k_{i}}} I - A_*)^{-1} \|_2
\end{equation*}
Assuming this condition is satisfied for a particular $\ell$, then:
\begin{align*}
& \max_{w \in \mathcal{M}(\hat{A}_{i-1},\{ x_t \}_{t=1}^{T - T_i})} \frac{27}{4} \epsilon_{i-1} T_{i} \gamma^2 \| w^\top (e^{j \frac{2 \pi \ell}{k_{i}}} I - A_*)^{-1} \|_2^2 \frac{\|  (e^{j \frac{2 \pi \ell}{k_{i}}} I - A_*)^{-1} B_* \|_2^2}{\|  (e^{j \frac{2 \pi \ell}{k_{i}}} I - A_*)^{-1} \|_2} \\
& \ \ \ \ \ \ \ \ \ \ \ \ \ \leq \max_{w \in \mathcal{M}(\hat{A}_{i-1},\{ x_t \}_{t=1}^{T - T_i})} \frac{32}{3} \epsilon_{i-1} T_{i} \gamma^2 \| w^\top (e^{j \frac{2 \pi \ell}{k_{i}}} I - \hat{A}_{i-1})^{-1} \|_2^2 \frac{\|  (e^{j \frac{2 \pi \ell}{k_{i}}} I - \hat{A}_{i-1})^{-1} B_* \|_2^2}{\|  (e^{j \frac{2 \pi \ell}{k_{i}}} I - \hat{A}_{i-1})^{-1} \|_2}
\end{align*}
Note that $\texttt{OptInput}_{k_{i}}(A_*,B_*,\gamma^2/2,\mathcal{I}_{i},\{ x_t \}_{t=1}^{T-T_i}) \geq \lambda_{\min}\left ( \sum_{t=1}^{T-T_i} x_t x_t^\top \right )$. So linking these together, if $\ell \in \mathcal{I}_i$, then:
\begin{align*}
& \max_{w \in \mathcal{M}(\hat{A}_{i-1},\{ x_t \}_{t=1}^{T - T_i})} \frac{27}{4} \epsilon_{i-1} T_{i} \gamma^2 \| w^\top (e^{j \frac{2 \pi \ell}{k_{i}}} I - A_*)^{-1} \|_2^2 \frac{\|  (e^{j \frac{2 \pi \ell}{k_{i}}} I - A_*)^{-1} B_* \|_2^2}{\|  (e^{j \frac{2 \pi \ell}{k_{i}}} I - A_*)^{-1} \|_2} \\
 \leq & \  \frac{1}{2} \texttt{OptInput}_{k_{i}}(A_*,B_*,\gamma^2/2,\mathcal{I}_{i},\{ x_t \}_{t=1}^{T-T_i})
\end{align*}
Since $\epsilon_{i-1} \leq ( 3 \| (e^{j \frac{2 \pi \ell}{k_{i}}} I - A_*)^{-1} \|_2 )^{-1}$ for all $\ell \in \mathcal{I}_i$, we can invoke Lemma \ref{lem:perturbation_sufficient} to get that:
\begin{align*}
&  \max_{w \in \mathcal{M}(\hat{A}_{i-1},\{ x_t \}_{t=1}^{T - T_i}), U \in \mathcal{U}_{\gamma^2/2}} 2 \frac{T_{i}}{k_{i}^2} \epsilon_{i-1} L(A_*,B_*,U,\epsilon_{i-1},\mathcal{I}_i,w) \\
\leq \ & \max_{\substack{w \in \mathcal{M}(\hat{A}_{i-1},\{ x_t \}_{t=1}^{T - T_i}), \ell \in \mathcal{I}_i}} \frac{27}{4} \epsilon_{i-1} T_{i} \gamma^2 \| w^\top (e^{j \frac{2\pi \ell}{k_i}} I - A_*)^{-1} \|_2^2  \frac{\| (e^{j \frac{2\pi \ell}{k_i}} I - A_*)^{-1} B_* \|_2^2}{\|  (e^{j \frac{2\pi \ell}{k_i}} I - A_*)^{-1} \|_2}
\end{align*}
so applying Theorem \ref{thm:opt_sln_perturbation} and Lemma \ref{lem:m_alg_subset}:
\begin{align*}
& \left | \texttt{OptInput}_{k_{i}}(A_*,B_*,\gamma^2/2,\mathcal{I}_{i},\{ x_t \}_{t=1}^{T-T_i}) - \texttt{OptInput}_{k_{i}}(\hat{A}_{i-1},B_*,\gamma^2/2,\mathcal{I}_{i},\{ x_t \}_{t=1}^{T-T_i}) \right | \\
& \leq \max_{\substack{U \in \mathcal{U}_{\gamma^2/2} \\ w \in \mathcal{M}(A_*,\hat{A}_{i-1}, \{ x_t \}_{t=1}^{T-T_i},\mathcal{I}_{i})}} 2 \frac{T_{i}}{k_{i}^2} \epsilon_{i-1} L(A_*,B_*,U,\epsilon_{i-1}, \mathcal{I}_i, w) \\
& \leq \max_{\substack{U \in \mathcal{U}_{\gamma^2/2} \\ w \in \mathcal{M}(\hat{A}_{i-1},\{ x_t \}_{t=1}^{T - T_i})}} 2 \frac{T_{i}}{k_{i}^2} \epsilon_{i-1} L(A_*,B_*,U,\epsilon_{i-1}, \mathcal{I}_i, w) \\
& \leq \frac{1}{2} \texttt{OptInput}_{k_{i}}(A_*,B_*,\gamma^2/2,\mathcal{I}_{i},\{ x_t \}_{t=1}^{T-T_i})
\end{align*}
which is the desired conclusion.

\subsubsection{Proof of Theorem \ref{thm:ft_alg_perf_master} part \ref{thm:ft_alg_perf_master:c}}
Let:
$$ \mathcal{A} := \left \{ \| \hat{A}_i - A_* \|_2 \leq C \sigma \sqrt{\frac{\log \frac{1}{\delta} + d + \log \det ( \bar{\Gamma}_T (\Gamma_{k_i}^\eta + \tilde{\Gamma}_{k_i}^{u_i})^{-1} + I)}{T \lambda_{\min}(\Gamma_{k_i}^\eta + \tilde{\Gamma}_{k_i}^{u_i^*})}} \right \}$$
$$\mathcal{E}_6 := \left \{ \epsilon_{i-1} \leq C \sigma \sqrt{\frac{\log \frac{1}{\delta} + d + \log \det ( \bar{\Gamma}_{T} (\Gamma_{k_{i-1}}^\eta )^{-1} + I)}{T_{i-1} \lambda_{\min}(\Gamma_{k_{i-1}}^\eta )}} \right \} $$
$$ \mathcal{E}_7 := \left \{ \sum_{t=1}^T (w_{\min}^\top x_t)^2 \leq 4 \sum_{t=1}^T (w_{\min}^\top x_t^u)^2 + 4 T \left ( 1 + \log \frac{2}{\delta} \right ) w_{\min}^\top ( \sigma^2 \Gamma_T + \sigma_u^2 \Gamma_T^{B_*})w_{\min} \right \} $$
Let $A_* = PJP^{-1}$ and $p_i$ denote the columns of $P$. Here $w_{\min}$ is any unit norm vector such that $w_{\min}^\top p_i = 0$ for all $p_i$ that do not correspond to the minimum eigenvalue of $A_*$. We can follow the proof outlined in Section \ref{sec:thmmaster_pf_errorinput} up to the final step, adding in the events $\mathcal{E}_6,\mathcal{E}_7$:
$$\mathbb{P}[\mathcal{A}^c  \cap \mathcal{E}_1 \cap \mathcal{E}_2 \cap \mathcal{E}_3 \cap \mathcal{E}_5] \leq \mathbb{P}[\mathcal{A}^c  \cap \mathcal{E}_1 \cap \mathcal{E}_2 \cap \mathcal{E}_3 \cap \mathcal{E}_5 \cap \mathcal{E}_6 \cap \mathcal{E}_7] + \mathbb{P}[\mathcal{E}_6^c] + \mathbb{P}[\mathcal{E}_7^c] $$
By Lemma \ref{lem:upper_bound}, $\mathbb{P}[\mathcal{E}_7^c] \leq \delta$. By part \ref{thm:ft_alg_perf_master:a}, we will have that $\mathbb{P}[\mathcal{E}_6] \geq 1 - 3 \delta$. We would like to guarantee that $\epsilon_{i-1} \leq \bar{\epsilon}_S(A_*,B_*,\gamma^2,T - T_i,\delta)$. On the event $\mathcal{E}_6$, a sufficient condition to achieve this is:
$$ T_{i-1} \geq C^2 \sigma^2 \frac{\log \frac{1}{\delta} + d + \log \det ( \bar{\Gamma}_{T} (\Gamma_{k_{i-1}}^\eta )^{-1} + I)}{\bigg ( \bar{\epsilon}_S(A_*,B_*,\gamma^2,T - T_i,\delta) \bigg )^2  \lambda_{\min}(\Gamma_{k_{i-1}}^\eta)}$$
On the event $\mathcal{E}_1 \cap \mathcal{E}_2 \cap \mathcal{E}_3 \cap \mathcal{E}_5 \cap \mathcal{E}_6 \cap \mathcal{E}_7$, by Lemma \ref{lem:epss_lb}, we will have that $\epsilon_S(A,B,\gamma^2,k_{i-1},\{x_t\}_{t=1}^{T-T_i},\delta) \geq \bar{\epsilon}_S(A_*,B_*,\gamma^2,T - T_i,\delta)$, so by Lemma \ref{lem:ft_sufficient}, then we will have that $\mathcal{I}_{i} = [k_{i}]$ and that:
\begin{align*}
& \left | \lambda_{\min} \left ( \frac{T_i}{k_i^2} H_{k_{i}}(A_*,B_*,U^*,[k_{i}]) +  \sum_{t=1}^{T-T_i} x_t x_t^\top \right ) -  \lambda_{\min} \left ( \frac{T_i}{k_i^2} H_{k_{i}}(A_*,B_*,\hat{U},[k_{i}]) +  \sum_{t=1}^{T-T_i} x_t x_t^\top \right ) \right | \\
& \leq \frac{1}{2} \lambda_{\min} \left ( \frac{T_i}{k_i^2} H_{k_{i}}(A_*,B_*,U^*,[k_{i}]) +  \sum_{t=1}^{T-T_i} x_t x_t^\top \right ) 
\end{align*} 
where $U^*$ is the solution to \texttt{OptInput}$_{k_{i}}(A_*,B_*,\gamma^2/2,[k_{i}],\{x_t\}_{t=1}^{T-T_i})$ and $\hat{U}$ the solution to 

\noindent \texttt{OptInput}$_{k_{i}}(\hat{A}_{i-1},B_*,\gamma^2/2,[k_{i}],\{x_t\}_{t=1}^{T-T_i})$. So it follows that on the event $\mathcal{E}_1 \cap \mathcal{E}_2 \cap \mathcal{E}_3 \cap \mathcal{E}_5 \cap \mathcal{E}_6 \cap \mathcal{E}_7$, $\mathcal{A}$ will also hold, so $\mathbb{P}[\mathcal{A}^c  \cap \mathcal{E}_1 \cap \mathcal{E}_2 \cap \mathcal{E}_3 \cap \mathcal{E}_5 \cap \mathcal{E}_6 \cap \mathcal{E}_7] = 0$. We can then apply part \ref{thm:ft_alg_perf_master:a} to get that, so long as $T_{i-1}$ meets the condition above and $T_i \geq 3 T_{ss} \left ( \frac{1}{10} \lambda_{\min}(\tilde{\Gamma}_{k_i}^{u_i^*}), k_i \right )$:
$$\mathbb{P} \left [ \| \hat{A}_i - A_* \|_2 \leq C \sigma \sqrt{\frac{\log \frac{1}{\delta} + d + \log \det ( \bar{\Gamma}_T (\Gamma_{k_i}^\eta + \tilde{\Gamma}_{k_i}^{u_i^*})^{-1} + I)}{T \lambda_{\min}(\Gamma_{k_i}^\eta + \tilde{\Gamma}_{k_i}^{u_i^*})}} \right ] \geq 1- 9 \delta$$

\subsection{Proof of Theorem \ref{thm:asymp}}\label{sec:proof_asymp}
\begin{proof}
Throughout we will let $T = \sum_{j=0}^i T_i$, the total time that has elapsed after $i$ epochs. 

By Lemma \ref{lem:cov_limit_existence}, we know that:
\begin{equation*}
c^* := \lim_{i \rightarrow \infty} \lambda_{\min} ( \sigma^2\Gamma_{k_0 2^i} + \tilde{\Gamma}_{k_0 2^i}^{u^*} ) 
\end{equation*}
exists and is finite, where here $u^*$ is the set of inputs in $\mathcal{U}_{\gamma^2}$ that maximizes $\lambda_{\min} ( \sigma^2\Gamma_{k_0 2^i} + \tilde{\Gamma}_{k_0 2^i}^{u} ) $. It follows then that there exists $i_0$ such that, for all $i \geq i_0$, we will have:
\begin{equation*}
\left | \lambda_{\min} (\sigma^2 \Gamma_{k_0 2^i} + \tilde{\Gamma}_{k_0 2^i}^{u^*} ) - c^* \right | \leq \frac{1}{4} c^*
\end{equation*}

By Corollary \ref{cor:noise_planning_per} and Lemma \ref{lem:gamma_perturbation}, for small enough $\epsilon$ and some $i_1$, we will have that:
\begin{equation*}
\left | \lambda_{\min} (\sigma^2 \Gamma_{k_0 2^i} + \tilde{\Gamma}_{k_0 2^i}^{u^*} ) - \lambda_{\min} (\sigma^2 \Gamma_{k_0 2^i} + \tilde{\Gamma}_{k_0 2^i}^{\hat{u}} ) \right | \leq \frac{1}{4} c^*
\end{equation*}
for all $i \geq i_1$, where $\hat{u}$ is the set of inputs in $\bar{\mathcal{U}}_{\gamma^2}$ that maximizes $\lambda_{\min} ( \sigma^2\Gamma_{k_0 2^i}(\hat{A}_{i-1}) + \tilde{\Gamma}_{k_0 2^i}^{u}(\hat{A}_{i-1},B_*) ) $, the set of inputs computed when $FT = $ \texttt{False}. Denote this small enough $\epsilon$ as $\epsilon_\infty$, and set $\epsilon_\infty$ small enough so that $\epsilon_\infty \leq \bar{\epsilon}(A_*,B_*,\gamma^2,T,\delta)$ for all $T, \delta$, which will guarantee that we are playing all frequencies, and small enough that $\hat{A}_{i-1}$ has spectral radius less than 1. Note that $\bar{\epsilon}(A_*,B_*,\gamma^2,T,\delta)$ is finite and greater than 0 as $\delta \rightarrow 0$ and $T \rightarrow \infty$. Note also that the fact that we allow $u^*$ to have a DC component and do not allow $\hat{u}$ to have a DC component does not affect the above result since, by Lemma \ref{lem:discrete1}, transfer functions are continuous in frequency. For large enough $i$, we can then make the response of the system without DC input arbitrarily close to the response of the system with DC input, by inputing energy at increasing lower frequencies. Combining these give that, for all $i \geq \max \{ i_0, i_1 \}$, we will have:
\begin{equation*}
\left |  \lambda_{\min} ( \sigma^2\Gamma_{k_0 2^i} + \tilde{\Gamma}_{k_0 2^i}^{\hat{u}} ) - c^* \right | \leq \frac{1}{2} c^*
\end{equation*}
which implies:
\begin{equation}\label{eq:thm_asympt1}
\lambda_{\min} (\sigma^2 \Gamma_{k_0 2^i} + \tilde{\Gamma}_{k_0 2^i}^{\hat{u}} ) \geq \frac{1}{2} c^*
\end{equation}

Modifying the burn-in time of Theorem \ref{cor:largek_alg_opt}  to:
\begin{align*}
\begin{split}
T_i& \geq \max \left \{ \frac{9}{2} T_{ss} \left ( c_1 \lambda_{\min}(\Gamma_{k_i}^{\eta}), k_i \right ), c_2 \sigma^2 \frac{\log \frac{1}{\delta} + d + \log \det ( \bar{\Gamma}_{T} (\Gamma_{k_{i-1}}^\eta )^{-1} + I)}{\epsilon_\infty^2  \lambda_{\min}(\Gamma_{k_{i-1}}^\eta)} \right \}
\end{split}
\end{align*}
Assuming this burn in time is met and:
\begin{equation*}
T_i \geq c k_i \left ( \log \frac{1}{\delta} + d + d \log \left (   \frac{2 \beta(A_*)^2 \gamma^2 }{ (1 - \bar{\rho}(A_*))^2} (1 + T)  + \frac{ 4 \beta(A_*)^2 d (\sigma^2 + \sigma_u^2 \| B_* \|_2)}{1 - \bar{\rho}(A_*)^2}  ( 1 + \log \frac{2}{\delta} )  \right )  \right )
\end{equation*}
then by Theorem \ref{cor:largek_alg_opt}, we will have that:
\begin{equation*}
\mathbb{P} \left [ \| \hat{A} - A_* \|_2 > \epsilon \right ] \leq \delta
\end{equation*}
so long as (where here we use the fact that $k_i = k(T)$):
\begin{equation*}
\epsilon \geq C \sigma \sqrt{ \frac{d  + \log \det  \left ( \bar{\Gamma}_T \left (\Gamma_{k(T)}^\eta + \tilde{\Gamma}_{k(T)}^{u^*} \right )^{-1} + I \right )  + \log \frac{1}{\delta}}{T \lambda_{\min}\left ( \Gamma_{k(T)}^\eta + \tilde{\Gamma}_{k(T)}^{u^*} \right )}}
\end{equation*}
or equivalently:
\begin{equation}\label{eq:thm_asymp_ft}
T \geq C \sigma^2  \frac{d  + \log \det  \left ( \bar{\Gamma}_T \left (\Gamma_{k(T)}^\eta + \tilde{\Gamma}_{k(T)}^{u^*} \right )^{-1} + I \right )  + \log \frac{1}{\delta}}{\epsilon^2 \lambda_{\min}\left ( \Gamma_{k(T)}^\eta + \tilde{\Gamma}_{k(T)}^{u^*} \right )}
\end{equation}
where $u^*$ is defined as above and here we use (\ref{eq:thm_asympt1}). Note that by modifying the burn-in time of Theorem \ref{cor:largek_alg_opt}, replacing $\bar{\epsilon}(A_*,B_*,\gamma^2,T,\delta)$ with $\epsilon_\infty$, by the definition of $\epsilon_\infty$, we will have the inputs being played are optimal with the flag $FT = $ \texttt{False}, since $\epsilon_\infty \leq \bar{\epsilon}(A_*,B_*,\gamma^2,T,\delta)$. As noted above, $\lambda_{\min}\left ( \Gamma_{k(T)}^\eta + \tilde{\Gamma}_{k(T)}^{u^*} \right )$ is upper bounded by a constant independent of $T$ and $\delta$. Thus, as $\delta \rightarrow 0$, the condition (\ref{eq:thm_asymp_ft}) will force $T \rightarrow \infty$. This implies that for small enough $\delta$, we will have $k(T) \geq k_0 2^{\max \{ i_0, i_1 \}}$. In this case, then, we will have:
\begin{equation*}
C' \sigma^2  \frac{d  + \log \det  \left ( \frac{1}{c^*} \bar{\Gamma}_T  \right )  + \log \frac{1}{\delta}}{\epsilon^2 c^*} \geq C \sigma^2  \frac{d  + \log \det  \left ( \bar{\Gamma}_T \left (\Gamma_{k(T)}^\eta + \tilde{\Gamma}_{k(T)}^{u^*} \right )^{-1} + I \right )  + \log \frac{1}{\delta}}{\epsilon^2 \lambda_{\min}\left ( \Gamma_{k(T)}^\eta + \tilde{\Gamma}_{k(T)}^{u^*} \right )}
\end{equation*}
Defining $\bar{\tau}_{\epsilon \delta}$ to be a solution to:
\begin{equation*}
\bar{\tau}_{\epsilon \delta} \geq C' \sigma^2  \frac{d  + \log \det  \left ( \frac{1}{c^*} \bar{\Gamma}_{\bar{\tau}_{\epsilon \delta}}  \right )  + \log \frac{1}{\delta}}{\epsilon^2 c^*}
\end{equation*}
for small enough $\epsilon, \delta$, it then follows by Theorem \ref{cor:largek_alg_opt} that for any $T$ at an epoch boundary, so long as $T \geq \bar{\tau}_{\epsilon \delta}$ and the burn-in condition is met, we will have that:
\begin{equation*}
\mathbb{P} \left [ \| \hat{A} - A_* \|_2 > \epsilon \right ] \leq \delta
\end{equation*}
The above definition of $\bar{\tau}_{\epsilon \delta}$ implies that necessarily $\bar{\tau}_{\epsilon \delta} \geq \frac{C' \sigma^2 \log \frac{1}{\delta}}{\epsilon^2 c^*}$ so as $\delta \rightarrow 0$, we will have that $\bar{\tau}_{\epsilon \delta} \rightarrow \infty$. By definition:
\begin{align*}
\bar{\Gamma}_T & = 2  \frac{\beta(A_*)^2 \gamma^2 }{ (1 - \bar{\rho}(A_*))^2} (1 + T) I + 4  \left (  tr \left (\Gamma_{T}^\eta \right ) \left ( 1 +  \log \frac{2}{\delta} \right ) I  \right ) 
\end{align*}
so:
\begin{align*}
\log \det \left ( \frac{1}{c^*} \bar{\Gamma}_{\bar{\tau}_{\epsilon \delta}} \right ) & = d \log \left ( 2  \frac{\beta(A_*)^2 \gamma^2 }{c^* (1 - \bar{\rho}(A_*))^2} (1 + \bar{\tau}_{\epsilon \delta})  + 4 \frac{tr \left (\Gamma_{\bar{\tau}_{\epsilon \delta}}^\eta \right )}{c^*} \left ( 1 +  \log \frac{2}{\delta} \right )    \right ) \\
& \leq d \log \left ( 2  \frac{\beta(A_*)^2 \gamma^2 }{c^* (1 - \bar{\rho}(A_*))^2} (1 + \bar{\tau}_{\epsilon \delta})  \right ) + d \log \left ( 4 \frac{tr \left (\Gamma_{\infty}^\eta \right )}{c^*} \left ( 1 +  \log \frac{2}{\delta} \right )    \right )
\end{align*}
where the inequality will hold for small enough $\delta$. Since $\bar{\tau}_{\epsilon \delta} \rightarrow \infty$ as $\delta \rightarrow 0$, it follows that for small enough $\delta$, we will have that:
\begin{equation*}
d \log \left ( 2  \frac{\beta(A_*)^2 \gamma^2 }{c^* (1 - \bar{\rho}(A_*))^2} (1 + \bar{\tau}_{\epsilon \delta})  \right ) \leq \frac{\epsilon^2 c^*}{2 C' \sigma^2} \bar{\tau}_{\epsilon \delta}
\end{equation*}
Thus, for small enough $\delta$, we will have that:
\begin{equation}\label{eq:asympt_logtau}
C' \sigma^2  \frac{d  + \log \det  \left ( \frac{1}{c^*} \bar{\Gamma}_{\bar{\tau}_{\epsilon \delta}}  \right )  + \log \frac{1}{\delta}}{\epsilon^2 c^*} \leq C' \sigma^2  \frac{d  +d \log \left ( 4 \frac{tr \left (\Gamma_{\infty}^\eta \right )}{c^*} \left ( 1 +  \log \frac{2}{\delta} \right )    \right )  + \log \frac{1}{\delta}}{\epsilon^2 c^*} + \frac{\bar{\tau}_{\epsilon \delta}}{2}
\end{equation}
So if:
\begin{equation*}
\bar{\tau}_{\epsilon \delta} \geq 2 C' \sigma^2  \frac{d  +d \log \left ( 4 \frac{tr \left (\Gamma_{\infty}^\eta \right )}{c^*} \left ( 1 +  \log \frac{2}{\delta} \right )    \right )  + \log \frac{1}{\delta}}{\epsilon^2 c^*}
\end{equation*}
we will have that for any $T \geq \bar{\tau}_{\epsilon \delta}$, so long as the burn-in condition is met:
\begin{equation*}
\mathbb{P} \left [ \| \hat{A} - A_* \|_2 > \epsilon \right ] \leq \delta
\end{equation*}
We can set: 
\begin{equation*}
\tau_{\epsilon \delta} := 2 C' \sigma^2  \frac{d  +d \log \left ( 4 \frac{tr \left (\Gamma_{\infty}^\eta \right )}{c^*} \left ( 1 +  \log \frac{2}{\delta} \right )    \right )  + \log \frac{1}{\delta}}{\epsilon^2 c^*}
\end{equation*}
and then:
\begin{equation*}
\lim_{\delta \rightarrow 0} \frac{\tau_{\epsilon \delta}}{\log \frac{1}{\delta}} = \frac{C \sigma^2}{\epsilon^2 c^*}
\end{equation*}
It remains to show that the modified burn-in time required by Theorem \ref{cor:largek_alg_opt} is met as $\delta \rightarrow 0$. That is, we need to ensure that as $\delta \rightarrow 0$:
\begin{align*}
\begin{split}
\tau_{\epsilon \delta} & \geq \max \left \{ 9 T_{ss} \left ( c_1 \lambda_{\min}(\Gamma_{k_{i}}^{\eta}), k_i \right ), c_2 \sigma^2 \frac{\log \frac{1}{\delta} + d + \log \det ( \bar{\Gamma}_{\tau_{\epsilon \delta}} (\Gamma_{k_{i-1}}^\eta )^{-1} + I)}{\epsilon_\infty^2  \lambda_{\min}(\Gamma_{k_{i-1}}^\eta)} \right \}
\end{split}
\end{align*}
where here we have replaced $T_i$ by $\tau_{\epsilon \delta}$ by noting that $T_i \geq \frac{\tau_{\epsilon \delta}}{2}$ if $\tau_{\epsilon \delta}$ is at an epoch boundary, since $T_i = \frac{2}{3} T + \frac{1}{3} T_0$.  By what we have shown and by definition of $\tau_{\epsilon \delta}$, so long as $\epsilon < \epsilon_\infty$ and for small enough $\delta$, we automatically have that:
\begin{equation*}
\tau_{\epsilon \delta} \geq c_2 \sigma^2  \frac{d  + \log \det  \left ( \bar{\Gamma}_{\tau_{\epsilon \delta}} (\Gamma_{k_{i-1}}^\eta)^{-1} + I \right )  + \log \frac{1}{\delta}}{\epsilon_\infty^2  \lambda_{\min}\left (\Gamma_{k_{i-1}}^\eta \right )}
\end{equation*}
To see that eventually:
\begin{equation*}
\tau_{\epsilon \delta} \geq 9 T_{ss} \left ( c_1 \lambda_{\min}(\Gamma_{k_i}^{\eta}), k_i\right )
\end{equation*}
Note that $\lambda_{\min}(\Gamma_{k_i}^{\eta}) > 0$, and that the dependance in $T_{ss}$ is logarithmic in $\tau_{\epsilon \delta}$, and scales as $\log \log \frac{1}{\delta}$. Thus, using the same argument as what we used above in (\ref{eq:asympt_logtau}), since $\tau_{\epsilon \delta}$ increases as $\log \frac{1}{\delta}$, a term linear in $\tau_{\epsilon \delta}$ will eventually exceed a term logarithmic in $\tau_{\epsilon \delta}$ for small enough $\delta$, so we will eventually have that the burn-in condition is met. Finally, we see that the condition:
\begin{equation*}
T_i \geq c k_i \left ( \log \frac{1}{\delta} + d + d \log \left (   \frac{2 \beta(A_*)^2 \gamma^2 }{ (1 - \bar{\rho}(A_*))^2} (1 + T)  + \frac{ 4 \beta(A_*)^2 d (\sigma^2 + \sigma_u^2 \| B_* \|_2)}{1 - \bar{\rho}(A_*)^2}  ( 1 + \log \frac{2}{\delta} )  \right )  \right )
\end{equation*}
will be met eventually regardless of how $k_0,T_0$ are set since, as noted, $\tau_{\epsilon \delta} \rightarrow \infty$ as $\delta \rightarrow 0$, implying that the number of epochs will go to infinity as $\delta \rightarrow 0$. Since $T_i$ increases faster than $k_i$, eventually the left hand side of the above inequality will be greater than the right hand side.
\end{proof}


\section{Special Cases of Theorem \ref{cor:largek_alg_opt}}\label{sec:main_thm_cors}
\begin{corollary}\label{cor:symmetric_a}
\textbf{(Full version of Corollary \ref{cor:symmetric_a_informal})} Assume the assumptions outlined in Section \ref{sec:interpreting} for the case where $A_*$ is diagonalizable by a unitary matrix are met. Then after:
\begin{align*}
T & \geq c \max \Bigg \{ \frac{T_0^2}{k_0^2} \max_{i=1,...,d} \frac{i^2}{(1 - \lambda_i)^2}, \frac{ \log \left ( \frac{  \| \mathbf{1} - \lambda \|_2^2 k}{  1 - \lambda_1  } +  \frac{  \| \mathbf{1} - \lambda \|_2^2}{ \gamma ( 1 - \lambda_1) }  \sqrt{   \left (\sum_{i=1}^d \frac{\sigma^2 + \gamma^2 / d}{1 - \lambda_i}  \right )  \log \frac{1}{\delta}} \right ) }{1 - \lambda_1}, \\
& \ \ \ \ \ \ \ \ \ \ \ \ \ \ \ \frac{\sigma^2 \| \mathbf{1} - \lambda \|_2^4 }{\sigma^2 + \gamma^2 / d} \frac{d \log \left (  \frac{\| \mathbf{1} - \lambda \|_2^2}{(1 - \lambda_1)^2} T +  \log \frac{1}{\delta}  \sum_{i=1}^d \frac{d}{1 - \lambda_i}  \right )   + \log \frac{1}{\delta}}{(1 - \lambda_1)^4 \left ( \log \frac{1}{\delta} \right )^2} \Bigg \}
\end{align*}
steps, Algorithm \ref{alg:active_lds_noise} will attain the following rate:
\begin{equation*}
\mathbb{P} \left [ \| \hat{A} - A_* \|_2  >  C \sqrt{\frac{\sigma^2 \| \mathbf{1} - \lambda \|_2^2}{\gamma^2 + \sigma^2 \| \mathbf{1} - \lambda \|_2^2}}  \sqrt{ \frac{ d \log \left (  \frac{\| \mathbf{1} - \lambda \|_2^2}{(1 - \lambda_1)^2} T +  \log \frac{1}{\delta}  \sum_{i=1}^d \frac{d}{1 - \lambda_i}  \right )   + \log \frac{1}{\delta}}{T }} \right ] \leq 9 \delta
\end{equation*}
while simply playing $u_t \sim \mathcal{N}(0, \frac{\gamma^2}{d} I)$ for all time will yield the following rate:
\begin{equation*}
\mathbb{P} \left [ \| \hat{A} - A_* \|_2  >  C \sqrt{\frac{\sigma^2 d}{\gamma^2 +  d \sigma^2 }}  \sqrt{ \frac{ d \log \left (  \frac{T}{(1 - \lambda_1)^2}  +  \log \frac{1}{\delta}  \sum_{i=1}^d \frac{d}{1 - \lambda_i}  \right )   + \log \frac{1}{\delta}}{T }} \right ] \leq 3 \delta
\end{equation*}
\end{corollary}

\subsection{Proof of Corollary \ref{cor:symmetric_a_informal} and Corollary \ref{cor:symmetric_a}}
\begin{proof}
The above rate can be attained by the input:
\begin{equation*}
u_t = \sum_{i=1}^d a_i v_i \cos \left ( \frac{2 \pi i}{k} t \right )
\end{equation*}
for $k \geq \mathcal{O} \left ( \max_{i=1,...,d} \frac{i}{1 - \lambda_i} \right )$ and some $a_i$ to be specified satisfying:
\begin{equation*}
\sum_{i=1}^d a_i^2 = \gamma^2
\end{equation*}
To see this, note that with this input we will have that:
\begin{align*}
H_k(A_*,B_*,U,[k]) & =  \sum_{i = 1}^d a_i^2 (e^{j 2 \pi i / k} I - A_*)^{-1} v_i v_i^\top (e^{j 2 \pi i / k} I - A_* )^{-H} \\
& =  V \left [ \sum_{i = 1}^d a_i^2 (e^{j 2 \pi i / k} I - \Lambda)^{-1} V^\top v_i v_i^\top V (e^{j 2 \pi i / k} I - \Lambda )^{-H} \right ] V^\top \\
& =  V \left [ \sum_{i = 1}^d \frac{a_i^2}{ (e^{j 2 \pi i / k} - \lambda_i) (e^{-j 2 \pi i / k} - \lambda_i)} e_i e_i^\top \right ] V^\top
\end{align*}
Note that:
\begin{align*}
(e^{j 2 \pi i / k} - \lambda_i) & (e^{-j 2 \pi i / k} - \lambda_i)  = 1 + \lambda_i^2 - \lambda_i \left ( e^{-j \frac{2 \pi i}{k}} + e^{j \frac{2 \pi i}{k}} \right ) \\ 
& = 1 + \lambda_i^2 - 2 \lambda_i \cos \frac{2 \pi i}{k}  \approx 1 + \lambda_i^2 - 2 \lambda_i \left ( 1 - \frac{2 \pi^2 i^2}{k^2} \right ) \\
& = (1 - \lambda_i)^2 + \frac{4 \lambda_i \pi^2 i^2}{k^2}  = \mathcal{O} \left ( (1 - \lambda_i)^2 \right )
\end{align*}
where the last equality will hold as long as:
\begin{equation*}
\frac{4 \lambda_i \pi^2 i^2}{k^2} \leq  (1 - \lambda_i)^2  \implies \frac{2 \sqrt{\lambda_i} \pi i}{1 - \lambda_i} \leq k
\end{equation*}

Assume that $k$ satisfies this, then:
\begin{align*}
H_k(A_*,B_*,U,[k]) & = \mathcal{O} \left ( V \left [ \sum_{i = 1}^d \frac{a_i^2}{ (1 - \lambda_i)^2} e_i e_i^\top \right ] V^\top \right )
\end{align*}
Choosing $a_i^2 = \frac{\gamma^2}{\| \mathbf{1} - \lambda \|_2^2} (1 - \lambda_i)^2$, the energy constraint will be satisfied since:
\begin{equation*}
\sum_{i=1}^d a_i^2 = \frac{\gamma^2}{\| \mathbf{1} - \lambda \|_2^2} \sum_{i=1}^d (1 - \lambda_i)^2 = \gamma^2
\end{equation*} 
and:
\begin{align*}
H_k(A_*,B_*,U,[k]) & = \mathcal{O} \left ( \frac{\gamma^2}{\| \mathbf{1} - \lambda \|_2^2} V \left [ \sum_{i = 1}^d  e_i e_i^\top \right ] V^\top \right ) = \mathcal{O} \left ( \frac{\gamma^2}{\| \mathbf{1} - \lambda \|_2^2} I \right )
\end{align*}
Thus, we will have that $\lambda_{\min} \left (  \tilde{\Gamma}_k^u \right ) = \mathcal{O} \left ( \frac{\gamma^2}{\| \mathbf{1} - \lambda \|_2^2}  \right )$ so $\lambda_{\min} \left (\Gamma^\eta_k + \tilde{\Gamma}_k^u \right ) \geq \mathcal{O} \left ( \frac{\sigma^2}{1 - \lambda_d}  +  \frac{\gamma^2}{\| \mathbf{1} - \lambda \|_2^2} \right )$. Since we have constructed a feasible input and Algorithm \ref{alg:active_lds_noise} constructs the optimal input on the true system (assuming $T$ is large enough), it follows that Algorithm \ref{alg:active_lds_noise} will perform at least this well.

Theorem \ref{cor:largek_alg_opt} then immediately gives that for sufficiently large $T$:
\begin{equation*}
\mathbb{P} \left [ \| \hat{A} - A_* \|_2  >  C \sigma \sqrt{ \frac{ d \log \left (  \frac{\| \mathbf{1} - \lambda \|_2^2}{(1 - \lambda_1)^2} T +  \log \frac{1}{\delta}  \sum_{i=1}^d \frac{d}{1 - \lambda_i}  \right )   + \log \frac{1}{\delta}}{T \left ( \frac{\sigma^2}{1 - \lambda_d}  +  \frac{\gamma^2}{\| \mathbf{1} - \lambda \|_2^2} \right ) }} \right ] \leq 9 \delta
\end{equation*}
Since $\bar{\Gamma}_T = 2  \frac{\beta(A_*)^2 \gamma^2 }{ (1 - \bar{\rho}(A_*))^2} (1 + T) I + 4  \left (  tr \left (\Gamma_{T}^\eta \right ) \left ( 1 +  \log \frac{2}{\delta} \right ) I  \right )$ and:
\begin{align*}
\log \det  \left ( \bar{\Gamma}_T \left (\Gamma_k^\eta + \tilde{\Gamma}_k^{u^*} \right )^{-1} + I \right ) & \leq c d \log  \left ( \frac{  \frac{ \gamma^2 }{ (1 - \lambda_1)^2}  T  +    tr \left (\Gamma_{T}^\eta \right )   \log \frac{1}{\delta}  }{ \frac{\sigma^2}{1 - \lambda_d}  +  \frac{\gamma^2}{\| \mathbf{1} - \lambda \|_2^2} } \right ) \\
& \leq c d \log \left ( \frac{\| \mathbf{1} - \lambda \|_2^2}{(1 - \lambda_1)^2} T +  \left ( \sum_{i=1}^d \frac{\gamma^2/d + \sigma^2}{1 - \lambda_i} \right ) \frac{d}{\sigma^2 + \gamma^2} \log \frac{1}{\delta} \right ) 
\end{align*}
It remains then to quantify how large $T$ must be to achieve this rate. From Theorem \ref{cor:largek_alg_opt}, we know that we must have:
\begin{align}
\begin{split}
T & \geq \max \Bigg \{ 2 T_{ss} \left (\frac{1}{10} \lambda_{\min} \left ( \tilde{\Gamma}_{k}^{u} \right ), k \right ),   c_2 \sigma^2  \frac{d  + \log \det  \left ( \bar{\Gamma}_T {\Gamma_k^\eta}^{-1} \right )  + \log \frac{1}{\delta}}{\bigg ( \bar{\epsilon}_S(A_*,B_*,\gamma^2,T,\delta) \bigg )^2  \lambda_{\min}\left (\Gamma_k^\eta \right )} \Bigg \}
\end{split}
\end{align}
and from above we need $k = \mathcal{O} \left ( \max_{i=1,...,d} \frac{i}{1 - \lambda_i} \right )$. To achieve this condition on $k$, Lemma \ref{lem:k_Tbound} lets us lower bound $k$ as $k \geq \frac{\sqrt{2}}{2} \frac{k_0}{T_0} \sqrt{T}$ so if $T \geq \mathcal{O} \left ( \frac{T_0^2}{k_0^2} \max_{i=1,...,d} \frac{i^2}{(1 - \lambda_i)^2} \right )$, then $k$ will be sufficiently large.

We already know that $\lambda_{\min} \left ( \tilde{\Gamma}_{k}^{u^*} \right ) = \mathcal{O} \left ( \frac{\gamma^2}{\| \mathbf{1} - \lambda \|_2^2} \right )$. In this case then, by Corollary \ref{cor:tss_bound}:
\begin{align*}
T_{ss} \left (\frac{1}{10} \lambda_{\min} \left ( \tilde{\Gamma}_{k}^{u} \right ), k \right ) = & \mathcal{O} \Bigg ( \max \Bigg \{ \frac{1}{\log \frac{1}{\bar{\rho}(A_*)}} \left (  \log \left ( \frac{ k \gamma}{1-\bar{\rho}(A_*)^{k}} + \sqrt{   \left (\sum_{i=1}^d \frac{\sigma^2 + \gamma^2 / d}{1 - \lambda_i}  \right ) \log \frac{1}{\delta}} \right ) \right . \\
& \ \ \ \ \ \ \ \ \ \ \ \ \ \ \ \ \ \  + \left .  \log \left ( \frac{\| \mathbf{1} - \lambda \|_2^2}{ k \gamma^2 (1-\bar{\rho}(A_*)^2)}  \right ) \right ), \\
& \ \ \ \ \ \ \ \ \ \ \ \ \ \ \frac{1}{\log \frac{1}{\bar{\rho}(A_*)}} \left ( \log \left ( \frac{k \gamma}{1-\bar{\rho}(A_*)^{k}} +  \sqrt{  \left (\sum_{i=1}^d \frac{\sigma^2 + \gamma^2 / d}{1 - \lambda_i}  \right )   \log \frac{1}{\delta}} \right ) \right . \\
& \ \ \ \ \ \ \ \ \ \ \ \ \ \ \ \ \ \  + \left . \log \left ( \frac{  \| \mathbf{1} - \lambda \|_2^2}{ \gamma \sqrt{k} ( 1 - \bar{\rho}(A_*))^{3/2} }   \right ) \right ) \Bigg \} \Bigg ) \\
& \leq \mathcal{O} \left ( \frac{ \log \left ( \frac{ k \gamma}{1-\bar{\rho}(A_*)^{k}} + \sqrt{   \left (\sum_{i=1}^d \frac{\sigma^2 + \gamma^2 / d}{1 - \lambda_i}  \right )  \log \frac{1}{\delta} } \right ) + \log \left ( \frac{  \| \mathbf{1} - \lambda \|_2^2}{ \gamma \sqrt{k} ( 1 - \bar{\rho}(A_*))^{3/2} }    \right )}{1 - \bar{\rho}(A_*)} \right ) \\
& \leq \mathcal{O} \left ( \frac{ \log \left ( \frac{  \| \mathbf{1} - \lambda \|_2^2 k}{  1 - \bar{\rho}(A_*)  } +  \frac{  \| \mathbf{1} - \lambda \|_2^2}{ \gamma ( 1 - \bar{\rho}(A_*)) }  \sqrt{   \left (\sum_{i=1}^d \frac{\sigma^2 + \gamma^2 / d}{1 - \lambda_i}  \right )  \log \frac{1}{\delta}} \right ) }{1 - \bar{\rho}(A_*)} \right )
\end{align*}
where the first inequality holds since $\log \frac{1}{\bar{\rho}(A_*)} \approx 1 - \bar{\rho}(A_*)$ for $\bar{\rho}(A_*)$ close to 1 and the second holds by our lower bound on $k$.

To bound $\bar{\epsilon}_S(A_*,B_*,\gamma^2,T,\delta)$, we must first bound $\bar{\mathcal{M}}_k (A_*,B_*,\delta,\gamma^2/2)$. We see in our case that:
\begin{align*}
\bar{\mathcal{M}}_k (A_*,B_*,\delta,\gamma^2/2) & \subseteq \left \{ V w \ : \ w \in \mathcal{S}^{d-1},   \sum_{i=1}^d \frac{w_i^2}{1 - \lambda_i^2} \leq  c_1 \frac{\gamma^2}{(1 - \lambda_d)^2} + c_2  \log \frac{1}{\delta}   \left (\frac{\sigma^2  + \gamma^2 / d}{1 - \lambda_d^2}   \right )   \right \}
\end{align*}
Note that this implies that, for any $u \in \bar{\mathcal{M}}_k (A_*,B_*,\delta,\gamma^2/2)$, denoting $w_i = [V^\top u]_i$, we will have:
\begin{equation*}
w_1 \leq c\sqrt{1 - \lambda_1} \left ( \frac{ \gamma}{1 - \lambda_d} +  \sqrt{ \log \frac{1}{\delta}}   \left (\frac{\sigma  + \gamma / \sqrt{d}}{\sqrt{1 - \lambda_d}} \right ) \right ) \leq c (\sigma + \gamma) \sqrt{ (1 - \lambda_1) \log \frac{1}{\delta}} 
\end{equation*}
Then we will have that:
\begin{align*}
\max_{w \in \bar{\mathcal{M}}_k (A_*,B_*,\delta,\gamma^2/2), \theta \in [0,2\pi]} \| w^\top (e^{j \theta} I - A_*)^{-1} \|_2^2 \|  (e^{j \theta} I - A_*)^{-1}\|_2 & \approx \max_{w \in \bar{\mathcal{M}}_k (A_*,B_*,\delta,\gamma^2/2)} \frac{w_1^2}{(1 - \lambda_1)^3} \\
& \leq c \frac{\sigma^2 + \gamma^2}{(1 - \lambda_1)^2} \log \frac{1}{\delta} 
\end{align*}
Based on our choice of inputs:
\begin{align*}
\frac{1}{2T + T_0} \texttt{OptInput}_{k} \left (A_*,B_*,\gamma^2,\left [k \right ], c_2 T \Gamma_{ k}^\eta \right ) \geq \mathcal{O} \left ( \sigma^2 + \frac{\gamma^2}{\| \mathbf{1} - \lambda \|_2^2} \right )
\end{align*}
So combining these, we can lower bound $\bar{\epsilon}_S(A_*,B_*,\gamma^2,T,\delta)$ as:
\begin{equation*}
\bar{\epsilon}_S(A_*,B_*,\gamma^2,T,\delta) \geq \mathcal{O} \left ( \frac{(1 - \lambda_1)^2}{\| \mathbf{1} - \lambda \|_2^2 \log \frac{1}{\delta}} \right )
\end{equation*}
We can then write the burn in time from Theorem \ref{cor:largek_alg_opt} as:
\begin{align*}
T & \geq c \max \Bigg \{ \frac{T_0^2}{k_0^2} \max_{i=1,...,d} \frac{i^2}{(1 - \lambda_i)^2}, \frac{ \log \left ( \frac{  \| \mathbf{1} - \lambda \|_2^2 k}{  1 - \lambda_1  } +  \frac{  \| \mathbf{1} - \lambda \|_2^2}{ \gamma ( 1 - \lambda_1) }  \sqrt{   \left (\sum_{i=1}^d \frac{\sigma^2 + \gamma^2 / d}{1 - \lambda_i}  \right )  \log \frac{1}{\delta}} \right ) }{1 - \lambda_1}, \\
& \ \ \ \ \ \ \ \ \ \ \ \ \ \ \ \frac{\sigma^2 \| \mathbf{1} - \lambda \|_2^4 }{\sigma^2 + \gamma^2 / d} \frac{d \log \left (  \frac{\| \mathbf{1} - \lambda \|_2^2}{(1 - \lambda_1)^2} T +  \log \frac{1}{\delta}  \sum_{i=1}^d \frac{d}{1 - \lambda_i}  \right )   + \log \frac{1}{\delta}}{(1 - \lambda_1)^4 \left ( \log \frac{1}{\delta} \right )^2} \Bigg \}
\end{align*}

The rate in the case where we simply play $u_t \sim \mathcal{N}(0, \frac{\gamma^2}{d} I)$ for all time follows from Theorem \ref{thm:concentration2}.
\end{proof}

\subsection{Proof of Corollary \ref{cor:block_diagonal_a}}
\begin{proof}
Since $\| A_* - \hat{A} \|_2 = \max_{j=1,...,m} \| A_j - \hat{A}_j \|_2$ (assuming $\hat{A}$ has the same block diagonal structure), to minimize the error in the estimate we want to minimize the maximum error in the estimate of each subsystem. By Theorem \ref{cor:largek_alg_opt}, once the burn-in time is reached, the estimation error for each subsystem will behave as:
\begin{equation*}
\mathbb{P} \left [ \| \hat{A}_j - A_j \|_2  >  C \sigma \sqrt{ \frac{d_j  + \log \det  \left ( \bar{\Gamma}_T^j \left (\Gamma_{k(T)}^{\eta,j} + \tilde{\Gamma}_{k(T)}^{u^*,j} \right )^{-1} + I \right )  + \log \frac{1}{\delta}}{T \lambda_{\min}\left ( \Gamma_{k(T)}^{\eta,j} + \tilde{\Gamma}_{k(T)}^{u^*,j} \right )}} \right ] \leq 9 \delta
\end{equation*}
where we let $\Gamma^j$ denote the covariates for the $j$th subsystem. For simplicity assume that:
$$\lambda_{\min}\left ( \Gamma_{k(T)}^{\eta,j} + \tilde{\Gamma}_{k(T)}^{u^*,j} \right ) \approx \lambda_{\min}\left (  \tilde{\Gamma}_{k(T)}^{u^*,j} \right ) =: \gamma_j^2 \lambda_{\min}^{*,j} $$
where here we let $\lambda_{\min}^{*,j}$ denote the optimal response of the system to inputs with power 1, and $\gamma_j^2$ the true amount of power inputed to the $j$th block.

Ignoring log factors, the optimal thing to do is to then set:
\begin{equation}\label{eq:block_diag_opt}
\frac{d_\ell}{\gamma_\ell^2 \lambda_{\min}^{*,\ell}} = \frac{d_j}{\gamma_j^2 \lambda_{\min}^{*,j}}
\end{equation}
for all $\ell,j \in [m]$, as this will make the estimation error equal for each subsystem, minimizing the overall error. Meeting this constraint and the power constraint, the following condition will then be met for any $j$:
\begin{equation*}
\gamma_j^2 \frac{\lambda_{\min}^{*,j}}{d_j} \sum_{\ell=1}^m \frac{d_\ell}{\lambda_{\min}^{*,\ell}} = \gamma^2 \implies \gamma_j^2 = \frac{d_j \gamma^2}{\lambda_{\min}^{*,j} \sum_{\ell=1}^m \frac{d_\ell}{\lambda_{\min}^{*,\ell}}}
\end{equation*}

Given this, we then have that:
\begin{equation*}
\mathbb{P} \left [ \| \hat{A}_{j} - A_j \|_2  >  C \sigma \sqrt{\frac{\sum_{\ell=1}^m \frac{d_\ell}{\lambda_{\min}^{*,\ell}}}{d_j \gamma^2}} \sqrt{ \frac{d_j  + \log \det  \left ( \bar{\Gamma}_T^j \left (\Gamma_{k(T)}^{\eta,j} + \tilde{\Gamma}_{k(T)}^{u^*,j} \right )^{-1} + I \right )  + \log \frac{1}{\delta}}{T}}  \right ] \leq 9 \delta
\end{equation*}
Thus, with high probability, we will have that:
\begin{equation*}
\epsilon = \tilde{\mathcal{O}} \left ( \sqrt{\frac{\sum_{\ell=1}^m \frac{d_\ell}{\lambda_{\min}^{*,\ell}}}{ \gamma^2 T}} \right )
\end{equation*}
In contrast, if we simply input random noise into the system---that is, set $u_t \sim \mathcal{N}(0, \frac{\gamma^2}{p} I)$---then in the $i$th block we will achieve the rate:
\begin{equation*}
\mathbb{P} \left [ \| \hat{A}_{j} - A_j \|_2  >  C \sigma \sqrt{ \frac{d_j  + \log \det  \left ( \bar{\Gamma}_T^j \left (\frac{\gamma^2}{p} \Gamma_{k(T)}^{B,j} + \sigma^2 \Gamma_{k(T)}^{j}  \right )^{-1} + I\right )  + \log \frac{1}{\delta}}{T \lambda_{\min}\left ( \frac{\gamma^2}{p} \Gamma_{k(T)}^{B,j} + \sigma^2 \Gamma_{k(T)}^{j}  \right )}} \right ] \leq 3 \delta
\end{equation*}
so, with high probability, noting that by construction $p \geq m$:
\begin{equation*}
\epsilon = \tilde{\mathcal{O}} \left ( \max_{j=1,...,m} \sqrt{\frac{d_j m}{\gamma^2 T  \lambda_{\min} \left ( \Gamma_{k(T)}^{B,j} \right ) }} \right )
\end{equation*}
To achieve the adaptive rate, Algorithm \ref{alg:active_lds_noise} can be run separately for each subsystem. After the optimal solution for each subsystem is found, the power $\gamma_j^2$ input to each subsystem can then be adjusted so that the empirical version of (\ref{eq:block_diag_opt}) is satisfied. Once the burn-in time from Theorem \ref{cor:largek_alg_opt} is met for each subsystem, our estimates of $\lambda_{\min}^{*,j}$ will be sufficiently accurate to guarantee that (\ref{eq:block_diag_opt}) will be met on the true system, and we will then achieve the optimal adaptive rate.
\end{proof}


\section{Algorithm \ref{alg:active_lds_noise} Performance Lemmas}

\subsection{Quantifying When $\epsilon_{i-1}$ Small Enough for $u_i \approx u_i^*$}
\begin{lemma}\label{lem:ft_sufficient}\com{[Done]}
If:
\begin{align*}
\epsilon_i & \leq \min \Bigg \{ \frac{{\normalfont \texttt{OptInput}}_{k_{i+1}}(A_*,B_*,\gamma^2/2,[k_{i+1}],\{x_t\}_{t=1}^T) }{\max_{w \in \mathcal{M}(\hat{A}_i,\{x_t\}_{t=1}^T), \ell \in [k_{i+1}]  } \frac{256}{27}  T_{i+1} \gamma^2  \| w^\top (e^{j \frac{2 \pi \ell}{k_{i+1}}} I - A_*)^{-1} \|_2^2 \|  (e^{j \frac{2 \pi \ell}{k_{i+1}}} I - A_*)^{-1}\|_2 \| B_* \|_2^2} , \\
& \ \ \ \ \ \ \ \ \ \ \ \  \frac{1}{\max_{\ell \in [k_{i+1}]} 5 \| (e^{j \frac{2 \pi \ell}{k_{i+1}}} I - A_* )^{-1} \|_2} \Bigg \} =: \epsilon_S(A_*,B_*,\gamma^2,k_{i+1},\{x_t\}_{t=1}^T,\delta)
\end{align*}
then $\mathcal{I}_{i+1} = [k_{i+1}]$ and:
\begin{align*}
& \left | \lambda_{\min} \left ( \frac{T_{i+1}}{k_{i+1}^2} H_{k_{i+1}}(A_*,B_*,U^*,[k_{i+1}]) +  \sum_{t=1}^T x_t x_t^\top \right ) -  \lambda_{\min} \left ( \frac{T_{i+1}}{k_{i+1}^2} H_{k_{i+1}}(A_*,B_*,\hat{U},[k_{i+1}]) +  \sum_{t=1}^T x_t x_t^\top \right ) \right | \\
& \leq \frac{1}{3} \lambda_{\min} \left ( \frac{T_{i+1}}{k_{i+1}^2} H_{k_{i+1}}(A_*,B_*,U^*,[k_{i+1}]) +  \sum_{t=1}^T x_t x_t^\top \right ) 
\end{align*}
where $U^*$ is the solution to {\normalfont \texttt{OptInput}}$_{k_{i+1}}(A_*,B_*,\gamma^2/2,[k_{i+1}],\{x_t\}_{t=1}^T)$ and $\hat{U}$ the solution to {\normalfont \texttt{OptInput}}$_{k_{i+1}}(\hat{A}_i,B_*,\gamma^2/2,[k_{i+1}],\{x_t\}_{t=1}^T)$.
\end{lemma}
\begin{proof} 
By Lemma \ref{lem:tf_approx_bound}, if $\epsilon_i \leq ( 4 \| (e^{j \theta} I - \hat{A}_i )^{-1} \|_2 )^{-1}$, then:
\begin{equation*}
\| (e^{j \theta} I - A_*)^{-1} \|_2 \leq \frac{4}{3} \| (e^{j \theta} I - \hat{A}_i )^{-1} \|_2
\end{equation*}
this then implies that $\epsilon_i \leq ( 3 \| (e^{j \theta} I - A_*)^{-1} \|_2 )^{-1}$ so, again by Lemma \ref{lem:tf_approx_bound}:
\begin{equation*}
\| (e^{j \theta} I - \hat{A}_i)^{-1} \|_2 \leq \frac{3}{2} \| (e^{j \theta} I - A_*)^{-1} \|_2
\end{equation*}
Thus, if $\epsilon_i \leq ( 4 \| (e^{j \theta} I - \hat{A}_i )^{-1} \|_2 )^{-1}$, we can upper bound:
\begin{align}\label{eq:perflem1}
\begin{split}
\| w^\top (e^{j \frac{2 \pi \ell}{k_{i+1}}} I - \hat{A}_i)^{-1} \|_2^2 \frac{\|  (e^{j \frac{2 \pi \ell}{k_{i+1}}} I - \hat{A}_i)^{-1} B_* \|_2^2}{\|  (e^{j \frac{2 \pi \ell}{k_{i+1}}} I - \hat{A}_i)^{-1} \|_2} & \leq \| w^\top (e^{j \frac{2 \pi \ell}{k_{i+1}}} I - \hat{A}_i)^{-1} \|_2^2 \|  (e^{j \frac{2 \pi \ell}{k_{i+1}}} I - \hat{A}_i)^{-1}\|_2 \| B_* \|_2^2 \\
& \leq \frac{27}{8} \| w^\top (e^{j \frac{2 \pi \ell}{k_{i+1}}} I - A_*)^{-1} \|_2^2 \|  (e^{j \frac{2 \pi \ell}{k_{i+1}}} I - A_*)^{-1}\|_2 \| B_* \|_2^2
\end{split}
\end{align}
Applying Lemma \ref{lem:tf_approx_bound} again, a sufficient condition for $\epsilon_i \leq ( 4 \| (e^{j \theta} I - \hat{A}_i )^{-1} \|_2 )^{-1}$ is $\epsilon_i \leq ( 5 \| (e^{j \theta} I - A_* )^{-1} \|_2 )^{-1}$.

Assume now that $\epsilon_i \leq ( \max_{\ell \in [k_{i+1}]} 5 \| (e^{j \frac{2 \pi \ell}{k_{i+1}}} I - A_* )^{-1} \|_2 )^{-1}$. From the analysis in the proof of Theorem \ref{thm:opt_sln_perturbation}, it follows that:
\begin{align*}
& \left | \texttt{OptInput}_{k_{i+1}}(A_*,B_*,\gamma^2/2,[k_{i+1}],\{x_t\}_{t=1}^T) - \texttt{OptInput}_{k_{i+1}}(\hat{A}_i,B_*,\gamma^2/2,[k_{i+1}],\{x_t\}_{t=1}^T) \right | \\
\leq \ & \max_{\substack{w \in \mathcal{M}(A_*,\hat{A}_i,\{x_t\}_{t=1}^T,\mathcal{I}_{i+1}) \\ U \in \mathcal{U}_{\gamma^2/2}}} \frac{T_{i+1}}{k_{i+1}^2} \epsilon_i L(A_*,B_*,U,\epsilon_i,[k_{i+1}],w) \\
\overset{(a)}{\leq} \ & \max_{\substack{w \in \mathcal{M}(A_*,\hat{A}_i,\{x_t\}_{t=1}^T,\mathcal{I}_{i+1}) \\ \ell \in [k_{i+1}]}} \frac{125}{64} \epsilon_i T_{i+1} \gamma^2 \| w^\top (e^{j \frac{2 \pi \ell}{k_{i+1}}} I - A_*)^{-1} \|_2^2 \frac{\| (e^{j \frac{2 \pi \ell}{k_{i+1}}} I - A_*)^{-1} B_* \|_2^2}{\| (e^{j \frac{2 \pi \ell}{k_{i+1}}} I - A_*)^{-1} \|_2} \\
\overset{(b)}{\leq} \ & \max_{\substack{w \in \mathcal{M}(\hat{A}_i,\{x_t\}_{t=1}^T) \\ \ell \in [k_{i+1}]}} \frac{125}{64} \epsilon_i T_{i+1} \gamma^2 \| w^\top (e^{j \frac{2 \pi \ell}{k_{i+1}}} I - A_*)^{-1} \|_2^2 \frac{\| (e^{j \frac{2 \pi \ell}{k_{i+1}}} I - A_*)^{-1} B_* \|_2^2}{\| (e^{j \frac{2 \pi \ell}{k_{i+1}}} I - A_*)^{-1} \|_2} 
\end{align*}
where the inequality $(a)$ follows from Lemma \ref{lem:perturbation_sufficient} (with a slight readjustment of constants) and $(b)$ follows from Lemma \ref{lem:m_alg_subset}. Thus, if we can guarantee that:
\begin{align}\label{eq:perflem2}
\begin{split}
& \max_{\substack{w \in \mathcal{M}(\hat{A}_i,\{x_t\}_{t=1}^T) \\ \ell \in [k_{i+1}]}} \frac{125}{64} \epsilon_i T_{i+1} \gamma^2 \| w^\top (e^{j \frac{2 \pi \ell}{k_{i+1}}} I - A_*)^{-1} \|_2^2 \frac{\| (e^{j \frac{2 \pi \ell}{k_{i+1}}} I - A_*)^{-1} B_* \|_2^2}{\| (e^{j \frac{2 \pi \ell}{k_{i+1}}} I - A_*)^{-1} \|_2}  \\
& \ \ \ \ \ \ \ \ \ \ \ \ \ \ \ \ \leq \frac{1}{2} \texttt{OptInput}_{k_{i+1}}(A_*,B_*,\gamma^2/2,[k_{i+1}],\{x_t\}_{t=1}^T)
\end{split}
\end{align}
then it will follow that:
\begin{equation*}
\texttt{OptInput}_{k_{i+1}}(A_*,B_*,\gamma^2/2,[k_{i+1}],\{x_t\}_{t=1}^T) \leq 2 \texttt{OptInput}_{k_{i+1}}(\hat{A}_i,B_*,\gamma^2/2,[k_{i+1}],\{x_t\}_{t=1}^T) 
\end{equation*}
Assume $\epsilon_i$ is small enough to satisfy this. Then, with (\ref{eq:perflem1}), it follows that if:
\begin{align*}
& \max_{w \in \mathcal{M}(\hat{A}_i,\{x_t\}_{t=1}^T) } \frac{256}{81} \epsilon_i T_{i+1} \gamma^2  \| w^\top (e^{j \frac{2 \pi \ell}{k_{i+1}}} I - A_*)^{-1} \|_2^2 \|  (e^{j \frac{2 \pi \ell}{k_{i+1}}} I - A_*)^{-1}\|_2 \| B_* \|_2^2  \\
& \ \ \ \ \ \ \ \ \ \ \ \ \ \ \ \ \leq \frac{1}{3} \texttt{OptInput}_{k_{i+1}}(A_*,B_*,\gamma^2/2,[k_{i+1}],\{x_t\}_{t=1}^T) 
\end{align*}
then:
\begin{align*}
& \max_{w \in \mathcal{M}(\hat{A}_i,\{x_t\}_{t=1}^T)} \frac{32}{3} \epsilon_i T_{i+1} \gamma^2 \| w^\top (e^{j \frac{2 \pi \ell}{k_{i+1}}} I - \hat{A}_i)^{-1} \|_2^2 \frac{\|  (e^{j \frac{2 \pi \ell}{k_{i+1}}} I - \hat{A}_i)^{-1} B_* \|_2^2}{\|  (e^{j \frac{2 \pi \ell}{k_{i+1}}} I - \hat{A}_i)^{-1} \|_2} \\
& \ \ \ \ \ \ \ \ \ \ \ \ \ \ \ \ \leq \frac{2}{3} \texttt{OptInput}_{k_{i+1}}(\hat{A}_i,B_*,\gamma^2/2,[k_{i+1}],\{x_t\}_{t=1}^T)
\end{align*}
so $\ell \in \mathcal{I}_{i+1}$. Note that this condition will also imply that (\ref{eq:perflem2}) holds. Combining all of this, it follows that if:
\begin{align*}
\epsilon_i & \leq \min \Bigg \{ \frac{\texttt{OptInput}_{k_{i+1}}(A_*,B_*,\gamma^2/2,[k_{i+1}],\{x_t\}_{t=1}^T) }{\max_{w \in \mathcal{M}(\hat{A}_i,\{x_t\}_{t=1}^T), \ell \in [k_{i+1}] } \frac{256}{27}  T_{i+1} \gamma^2  \| w^\top (e^{j \frac{2 \pi \ell}{k_{i+1}}} I - A_*)^{-1} \|_2^2 \|  (e^{j \frac{2 \pi \ell}{k_{i+1}}} I - A_*)^{-1}\|_2 \| B_* \|_2^2} , \\
& \ \ \ \ \ \ \ \ \ \ \ \  \frac{1}{\max_{\ell \in [k_{i+1}]} 5 \| (e^{j \frac{2 \pi \ell}{k_{i+1}}} I - A_* )^{-1} \|_2} \Bigg \}
\end{align*}
then $\mathcal{I}_{i+1} = [k_{i+1}]$. Finally, we see that the perturbation bound holds by applying Theorem \ref{thm:opt_sln_perturbation} and our condition on $\epsilon_i$, since:
\begin{align*}
& \max_{\substack{w \in \mathcal{M}(A_*,\hat{A}_i,\{x_t\}_{t=1}^T,\mathcal{I}_{i+1}) \\ U \in \mathcal{U}_{\gamma^2/2}}} 2 \frac{T_{i+1}}{k_{i+1}^2} \epsilon_i L(A_*,B_*,U,\epsilon_i,[k_{i+1}],w) \\
\leq \ & \max_{\substack{w \in \mathcal{M}(A_*,\hat{A}_i,\{x_t\}_{t=1}^T,\mathcal{I}_{i+1}) \\ \ell \in [k_{i+1}]}} \frac{125}{32} \epsilon_i T_{i+1} \gamma^2  \| w^\top (e^{j \frac{2 \pi \ell}{k_{i+1}}} I - A_*)^{-1} \|_2^2 \|  (e^{j \frac{2 \pi \ell}{k_{i+1}}} I - A_*)^{-1}\|_2 \| B_* \|_2^2 \\
\leq \ & \max_{\substack{w \in \mathcal{M}(\hat{A}_i,\{x_t\}_{t=1}^T) \\ \ell \in [k_{i+1}]}} \frac{125}{32} \epsilon_i T_{i+1} \gamma^2  \| w^\top (e^{j \frac{2 \pi \ell}{k_{i+1}}} I - A_*)^{-1} \|_2^2 \|  (e^{j \frac{2 \pi \ell}{k_{i+1}}} I - A_*)^{-1}\|_2 \| B_* \|_2^2 \\
\leq \ & \frac{1}{2} \texttt{OptInput}_{k_{i+1}}(A_*,B_*,\gamma^2/2,[k_{i+1}],\{x_t\}_{t=1}^T) 
\end{align*}
\end{proof}

\begin{lemma}\label{lem:epss_lb}\com{[Done.]}
On the events that:
\begin{equation*}
\sum_{t=1}^T x_t x_t^\top \succeq c T \Gamma_{k_i}^\eta 
\end{equation*}
\begin{align*}
\sum_{t=1}^T ({w'}^\top x_t)^2 & \leq 4 \sum_{t=1}^T ({w'}^\top x_t^u)^2 + 4 T \left ( 1 +  \log \frac{2}{\delta} \right ) {w'}^\top  \left (\sigma^2  \Gamma_T  + \sigma_u^2  \Gamma_T^{B_*}  \right )  w'  
\end{align*}
\begin{equation*}
\| A_* - \hat{A}_i \|_2 \leq \frac{1}{\max_{\ell \in \mathcal{I}_{i+1}} 2 \| (e^{j \theta_\ell} I - A_*)^{-1} \|_2}
\end{equation*}
for some $w'$ to be specified, we will have:
\begin{align*}
\epsilon_S(A_*,B_*,\gamma^2,k_{i+1},\{x_t\}_{t=1}^T,\delta) \geq \bar{\epsilon}_S(A_*,B_*,\gamma^2,T,\delta)
\end{align*}
where:
\begin{align*}
& \bar{\epsilon}_S(A_*,B_*,\gamma^2,T,\delta) \\
& := \min \Bigg \{ \frac{\frac{27}{256 (2T + T_0) \gamma^2} {\normalfont \texttt{OptInput}}_{2k(T)} \left (A_*,B_*,\gamma^2,[2k(T)], c T \Gamma_{k(T)}^\eta   \right )}{\max_{w \in \bar{\mathcal{M}}_{2k(T)}(A_*,B_*,\delta,\gamma^2/2), \ell \in [2k(T)]  }   \| w^\top (e^{j \frac{2 \pi \ell}{2k(T)}} I - A_*)^{-1} \|_2^2 \|  (e^{j \frac{2 \pi \ell}{2k(T)}} I - A_*)^{-1}\|_2 \| B_* \|_2^2} , \\
& \ \ \ \ \ \ \ \ \ \ \ \  \frac{1}{\max_{\ell \in [2k(T)]} 5 \| (e^{j \frac{2 \pi \ell}{2k(T)}} I - A_* )^{-1} \|_2} \Bigg \} 
\end{align*}
\end{lemma}
\begin{proof}
From the definition of $\texttt{OptInput}$, it is clear that:
\begin{align*}
& \texttt{OptInput}_{k_{i+1}}(A_*,B_*,\gamma^2,[k_{i+1}],\{ x_t \}_{t=1}^T)   \geq  \texttt{OptInput}_{k_{i+1}} \left (A_*,B_*,\gamma^2,[k_{i+1}], c_2 T \Gamma_{k_i}^\eta   \right ) 
\end{align*}
on the event $\sum_{t=1}^T x_t x_t^\top \succeq c_2 T \Gamma_{k_i}^\eta $. Further, conditioned on all three events assumed to hold, by Lemma \ref{lem:M_superset}, we have that:
$$\mathcal{M}(\hat{A}_i,\{ x_t \}_{t=1}^T ) \subseteq \bar{\mathcal{M}}_{k_{i+1}}(A_*,B_*,\delta,\gamma^2/2)$$
Finally, recall that $k_i = k(T)$. Combining all of this we have:
\begin{align*}
& \epsilon_S(A_*,B_*,\gamma^2,k_{i+1},\{x_t\}_{t=1}^T,\delta) \\
& = \min \Bigg \{ \frac{\texttt{OptInput}_{k_{i+1}}(A_*,B_*,\gamma^2/2,[k_{i+1}],\{x_t\}_{t=1}^T) }{\max_{w \in \mathcal{M}(\hat{A}_i,\{x_t\}_{t=1}^T), \ell \in [k_{i+1}]  } \frac{256}{27}  T_{i+1} \gamma^2  \| w^\top (e^{j \frac{2 \pi \ell}{k_{i+1}}} I - A_*)^{-1} \|_2^2 \|  (e^{j \frac{2 \pi \ell}{k_{i+1}}} I - A_*)^{-1}\|_2 \| B_* \|_2^2} , \\
& \ \ \ \ \ \ \ \ \ \ \ \  \frac{1}{\max_{\ell \in [k_{i+1}]} 5 \| (e^{j \frac{2 \pi \ell}{k_{i+1}}} I - A_* )^{-1} \|_2} \Bigg \} \\
& \geq \min \Bigg \{ \frac{\texttt{OptInput}_{k_{i+1}} \left (A_*,B_*,\gamma^2,[k_{i+1}], c T \Gamma_{k_i}^\eta   \right )}{\max_{w \in \bar{\mathcal{M}}_{k_{i+1}}(A_*,B_*,\delta,\gamma^2/2), \ell \in [k_{i+1}]  } \frac{256}{27}  T_{i+1} \gamma^2  \| w^\top (e^{j \frac{2 \pi \ell}{k_{i+1}}} I - A_*)^{-1} \|_2^2 \|  (e^{j \frac{2 \pi \ell}{k_{i+1}}} I - A_*)^{-1}\|_2 \| B_* \|_2^2} , \\
& \ \ \ \ \ \ \ \ \ \ \ \  \frac{1}{\max_{\ell \in [k_{i+1}]} 5 \| (e^{j \frac{2 \pi \ell}{k_{i+1}}} I - A_* )^{-1} \|_2} \Bigg \} 
\end{align*}

\end{proof}

\begin{lemma}\label{lem:m_alg_subset}
When calling {\normalfont \texttt{UpdateInputs}}, we will always have that:
$$ \mathcal{M}(A_*,\hat{A},\{ x_t \}_{t=1}^T,\mathcal{I}) \subseteq \mathcal{M}(\hat{A}, \{ x_t \}_{t=1}^T)$$
\end{lemma}
\begin{proof}
Recall that:
 \begin{align*}
  \mathcal{M}(\hat{A}, \{ x_t \}_{t=1}^T) = \bigg \{ & w \in \mathcal{S}^{d-1} \ : \ \frac{k^2}{2T+T_0} \sum_{t=1}^T (w^\top x_t)^2 \leq  \frac{k^2}{2T+T_0} \sum_{t=1}^T ({w'}^\top x_t)^2 \\
  & \ \ \ \ \ +  \min_{w' \in \mathcal{S}^{d-1}} \frac{4}{3} \gamma^2 \max_{\ell \in [k] \ : \  \epsilon \leq \left ( 4 \| (e^{j \frac{2 \pi \ell}{k}} I - \hat{A})^{-1} \|_2 \right )^{-1} } \| {w'}^\top (e^{j \frac{2 \pi \ell}{k}} I - \hat{A})^{-1} B_* \|_2^2   \bigg \}
 \end{align*}
 and:
\begin{align*} \mathcal{M}(A_*,\hat{A},\{ x_t \}_{t=1}^T, \mathcal{I}) = \bigg \{ & w \in \mathcal{S}^{d-1} \ : \ \frac{k^2}{2T + T_0} \sum_{t=1}^T (w^\top x_t)^2 \leq \min\limits_{w' \in \mathcal{S}^{d-1}}  \frac{k^2}{2T + T_0} \sum_{t=1}^T ({w'}^\top x_t)^2 \\
 & \ \ \ \ \  + \ \gamma^2  \max\limits_{i \in \mathcal{I}} \max \{ \| {w'}^\top (e^{j \frac{2\pi i}{k}} I - A_*)^{-1} B_* \|_2^2, \| {w'}^\top (e^{j \frac{2\pi i}{k}} I - \hat{A})^{-1}\|_2^2 \| B_* \|_2^2 \}   \bigg \}
 \end{align*}
 for any $\ell \in [k]$ satisfying $\epsilon \leq \left ( 4 \| (e^{j \frac{2 \pi \ell}{k}} I - \hat{A})^{-1} \|_2 \right )^{-1}$, by Lemma \ref{lem:tf_approx_bound}, we will have that:
 $$ \| (e^{j \frac{2 \pi \ell}{k}} I - A_*)^{-1} \|_2 \leq \frac{4}{3} \| (e^{j \frac{2 \pi \ell}{k}} I - \hat{A})^{-1} \|_2$$
 Since \texttt{UpdateInputs} only includes frequencies $\ell$ in $\mathcal{I}$ if  $\epsilon \leq \left ( 4 \| (e^{j \frac{2 \pi \ell}{k}} I - \hat{A})^{-1} \|_2 \right )^{-1}$, it follows that:
 $$ \max_{\ell \in \mathcal{I}} \| {w'}^\top (e^{j \frac{2\pi i}{k}} I - A_*)^{-1} B_* \|_2^2 \leq \max_{\ell \in [k] \ : \  \epsilon \leq \left ( 4 \| (e^{j \frac{2 \pi \ell}{k}} I - \hat{A})^{-1} \|_2 \right )^{-1} } \frac{4}{3} \| {w'}^\top (e^{j \frac{2 \pi \ell}{k}} I - \hat{A})^{-1} B_* \|_2^2$$
from which it follows that $ \mathcal{M}(A_*,\hat{A},\{ x_t \}_{t=1}^T,\mathcal{I}) \subseteq \mathcal{M}(\hat{A}, \{ x_t \}_{t=1}^T)$.
\end{proof}

\subsection{Meeting the Burn-In Time of Theorem \ref{thm:concentration2_2}}

\begin{lemma}\label{lem:logdet_upperbound}\com{[Done]}
$\log \det ( \bar{\Gamma}_T (\Gamma_{k}^\eta + \tilde{\Gamma}_k^u )^{-1}) \leq d \log \left (   \frac{2 \beta(A_*)^2 \gamma^2 }{ (1 - \bar{\rho}(A_*))^2} (1 + T)  + \frac{ 4 \beta(A_*)^2 d (\sigma^2 + \sigma_u^2 \| B_* \|_2)}{1 - \bar{\rho}(A_*)^2}  ( 1 + \log \frac{2}{\delta} )  \right )  $
\end{lemma}
\begin{proof}
 We have that $\bar{\Gamma}_T = 4 \left ( \tilde{\Gamma}_{T,0}^u + Tr(\Gamma_T^\eta ) ( 1 + \log \frac{2}{\delta} ) I \right )$ where $\tilde{\Gamma}_{T,0}^u = \frac{1}{T} \sum_{t=1}^T x_t^u {x_t^u}^\top$. Note that:
\begin{align*}
\| x_t^u \|_2 & = \left \| \sum_{s=0}^{t-1} A_*^{t-s-1} u_s \right \|_2  \leq  \sum_{s=0}^{t-1} \| A_*^{t-s-1} \|_2 \| u_s  \|_2   \leq  \beta(A_*) \gamma \sqrt{t} \sum_{s=0}^{t-1} \bar{\rho}(A_*)^{t-s-1}   \\
& =  \frac{\beta(A_*) \gamma \sqrt{t} ( 1 - \bar{\rho}(A_*)^t)}{1 - \bar{\rho}(A_*)}  \leq \frac{\beta(A_*) \gamma \sqrt{t} }{1 - \bar{\rho}(A_*)}
\end{align*}
which implies that:
\begin{equation*}
x_t^u {x_t^u}^\top \preceq  \frac{\beta(A_*)^2 \gamma^2 t }{(1 - \bar{\rho}(A_*))^2} I
\end{equation*}
so:
\begin{align*}
 \frac{1}{T} \sum_{t=1}^T x_t^u {x_t^u}^\top & \preceq \frac{1}{T} \frac{\beta(A_*)^2 \gamma^2 }{(1 - \bar{\rho}(A_*))^2} \sum_{t=1}^T t I  =  \frac{\beta(A_*)^2 \gamma^2 }{2(1 - \bar{\rho}(A_*))^2} (1 + T) I
\end{align*}
We also have that:
\begin{align*}
Tr(\Gamma_T^\eta ) & = \sigma^2 \sum_{t=0}^{T-1} Tr((A_*^t)^\top A_*^t) +  \sigma_u^2 \sum_{t=0}^{T-1} Tr( B_*^\top (A_*^t)^\top A_*^t B_*)  \\
& = \sigma^2 \sum_{t=0}^{T-1} \| A_*^t \|_F^2 +  \sigma_u^2 \sum_{t=0}^{T-1} \| A_*^t B_* \|_F^2 \\
& \leq \sigma^2 d \sum_{t=0}^{T-1} \| A_*^t \|_2^2 +  \sigma_u^2 d \sum_{t=0}^{T-1} \| A_*^t \|_2^2 \| B_* \|_2^2 \\
& \leq \sigma^2 \beta(A_*)^2 d \sum_{t=0}^{T-1} \bar{\rho}(A_*)^{2t} +  \sigma_u^2 \beta(A_*)^2 \| B_* \|_2^2 d \sum_{t=0}^{T-1} \bar{\rho}(A_*)^{2t}  \\
& \leq \frac{ \beta(A_*)^2 d (\sigma^2 + \sigma_u^2 \| B_* \|_2)}{1 - \bar{\rho}(A_*)} 
\end{align*}
This gives that:
\begin{align*}
\bar{\Gamma}_T \preceq 4 \left ( \frac{\beta(A_*)^2 \gamma^2 }{2 (1 - \bar{\rho}(A_*))^2} (1 + T)  + \frac{  \beta(A_*)^2 d (\sigma^2 + \sigma_u^2 \| B_* \|_2)}{1 - \bar{\rho}(A_*)^2}  ( 1 + \log \frac{2}{\delta} )  \right ) I
\end{align*}
Thus:
\begin{align*}
\log \det ( \bar{\Gamma}_T (\Gamma_{k}^\eta + \tilde{\Gamma}_k^u)^{-1}) & = \log \det ( \bar{\Gamma}_T ) - \log \det ( \Gamma_{T_k}^\eta + \tilde{\Gamma}_k^u ) \\
& \leq \log \det ( \bar{\Gamma}_T ) \\
& \leq \log \det \left ( 4 \left ( \frac{\beta(A_*)^2 \gamma^2 }{2 (1 - \bar{\rho}(A_*))^2} (1 + T)  + \frac{  \beta(A_*)^2 d (\sigma^2 + \sigma_u^2 \| B_* \|_2)}{1 - \bar{\rho}(A_*)^2}  ( 1 + \log \frac{2}{\delta} )  \right ) I \right ) \\
& = d \log \left (   \frac{2 \beta(A_*)^2 \gamma^2 }{ (1 - \bar{\rho}(A_*))^2} (1 + T)  + \frac{ 4 \beta(A_*)^2 d (\sigma^2 + \sigma_u^2 \| B_* \|_2)}{1 - \bar{\rho}(A_*)^2}  ( 1 + \log \frac{2}{\delta} )  \right )  
\end{align*}
\end{proof}

\begin{lemma}\label{lem:burn_in3} \com{[Done]}
Assume that $T \geq 16$, $\gamma^2 \geq \frac{(1- \bar{\rho}(A_*))^2}{2 \beta(A_*)^2}$, and:
\begin{equation*}
T_i \geq c k_i \left ( \log \frac{1}{\delta} + d + d \log \left (   \frac{2 \beta(A_*)^2 \gamma^2 }{ (1 - \bar{\rho}(A_*))^2} (1 + T)  + \frac{ 4 \beta(A_*)^2 d (\sigma^2 + \sigma_u^2 \| B_* \|_2)}{1 - \bar{\rho}(A_*)^2}  ( 1 + \log \frac{2}{\delta} )  \right )  \right )
\end{equation*}
then:
\begin{equation*}
3T_i \geq c 2k_i \left ( \log \frac{1}{\delta} + d + d \log \left (   \frac{2 \beta(A_*)^2 \gamma^2 }{ (1 - \bar{\rho}(A_*))^2} (1 + T + 3 T_i)  + \frac{ 4 \beta(A_*)^2 d (\sigma^2 + \sigma_u^2 \| B_* \|_2)}{1 - \bar{\rho}(A_*)^2}  ( 1 + \log \frac{2}{\delta} )  \right )  \right )
\end{equation*}
\end{lemma}
\begin{proof}
Since $T = \sum_{j=1}^i T_j$, we will have that:
\begin{align*}
& \log \left (   \frac{2 \beta(A_*)^2 \gamma^2 }{ (1 - \bar{\rho}(A_*))^2} (1 + T + 3 T_i)  + \frac{ 4 \beta(A_*)^2 d (\sigma^2 + \sigma_u^2 \| B_* \|_2)}{1 - \bar{\rho}(A_*)^2}  ( 1 + \log \frac{2}{\delta} )  \right ) \\
\leq \ & \log \left (   4 \frac{2 \beta(A_*)^2 \gamma^2 }{ (1 - \bar{\rho}(A_*))^2} (1 + T)  + 4 \frac{ 4 \beta(A_*)^2 d (\sigma^2 + \sigma_u^2 \| B_* \|_2)}{1 - \bar{\rho}(A_*)^2}  ( 1 + \log \frac{2}{\delta} )  \right ) \\
= \ & \log 4 + \log \left (    \frac{2 \beta(A_*)^2 \gamma^2 }{ (1 - \bar{\rho}(A_*))^2} (1 + T)  +  \frac{ 4 \beta(A_*)^2 d (\sigma^2 + \sigma_u^2 \| B_* \|_2)}{1 - \bar{\rho}(A_*)^2}  ( 1 + \log \frac{2}{\delta} )  \right )
\end{align*}
so:
\begin{align*}
& c 2k_i \left ( \log \frac{1}{\delta} + d + d \log \left (   \frac{2 \beta(A_*)^2 \gamma^2 }{ (1 - \bar{\rho}(A_*))^2} (1 + T + 3 T_i)  + \frac{ 4 \beta(A_*)^2 d (\sigma^2 + \sigma_u^2 \| B_* \|_2)}{1 - \bar{\rho}(A_*)^2}  ( 1 + \log \frac{2}{\delta} )  \right )  \right ) \\
\leq \ & c 2k_i \left ( \log \frac{1}{\delta} + d + d \log \left (   \frac{2 \beta(A_*)^2 \gamma^2 }{ (1 - \bar{\rho}(A_*))^2} (1 + T )  + \frac{ 4 \beta(A_*)^2 d (\sigma^2 + \sigma_u^2 \| B_* \|_2)}{1 - \bar{\rho}(A_*)^2}  ( 1 + \log \frac{2}{\delta} )  \right )  \right ) + c 2 k_i \log 4 \\
\leq \ & 2 T_i + c 2 k_i \log 4 \\
\leq \ & 3 T_i 
\end{align*}
where the second to last inequality follows assuming that $T \geq 16$ and $\gamma^2 \geq \frac{(1- \bar{\rho}(A_*))^2}{2 \beta(A_*)^2}$. 
\end{proof}

A direct corollary of Lemma \ref{lem:burn_in3} and Lemma \ref{lem:logdet_upperbound} is that, assuming $T_0 \geq 16$ and $\gamma^2 \geq \frac{(1- \bar{\rho}(A_*))^2}{2 \beta(A_*)^2}$, then, as long as:
\begin{equation*}\label{eq:lemburn_in2_init}
T_0 \geq c k_0 \left ( \log \frac{1}{\delta} + d + d \log \left (   \frac{2 \beta(A_*)^2 \gamma^2 }{ (1 - \bar{\rho}(A_*))^2} (1 + T_0)  + \frac{ 4 \beta(A_*)^2 d (\sigma^2 + \sigma_u^2 \| B_* \|_2)}{1 - \bar{\rho}(A_*)^2}  ( 1 + \log \frac{2}{\delta} )  \right )  \right )
\end{equation*}
the $k_i$ and $T_i$ used by Algorithm \ref{alg:active_lds_noise} will satisfy:
\begin{equation*}
T_i \geq  c k_i \left ( \log \frac{1}{\delta} + d + \log \det ( \bar{\Gamma}_T \Gamma^{-1}) \right )
\end{equation*}
for any $\Gamma \succeq 0$ and all $i$.

\subsection{Additional Lemmas}

\begin{lemma}\label{lem:alg_power_constraint} \com{[Done]}
For any $i$ and any $t \in [T - T_i, T_i - k_i]$, the inputs generated by Algorithm \ref{alg:active_lds_noise} will satisfy:
\begin{equation*}
\mathbb{E} \left [ \frac{1}{k_i} \sum_{s=t}^{t+k_i} u_s^\top u_s \right ]  \leq \gamma^2
\end{equation*}
\end{lemma}
\begin{proof}
Denote $u_t = \tilde{u}_t + \eta_t^u$ where $\tilde{u}_t$ is the solution to $\texttt{OptInput}_k(A_*,B_*,\gamma^2 - p \sigma_u^2,\mathcal{I},\{ x_t \}_{t=1}^T)$ and $\eta_t^u \sim \mathcal{N}(0,\sigma_u^2 I)$. Assume that $\sigma_u^2 \neq 0$. Then:
\begin{align*}
\mathbb{E} \left [ \frac{1}{k_i} \sum_{s=t}^{t+k_i} u_s^\top u_s \right ] & = \mathbb{E} \left [ \frac{1}{k_i} \sum_{s=t}^{t+k_i} \left ( { \tilde{u}_t}^\top  \tilde{u}_t + 2 { \tilde{u}_t}^\top \eta_t^u + {\eta_t^u}^\top \eta_t^u \right ) \right ] \\
& \overset{(a)}{=} \frac{1}{k_i} \sum_{s=t}^{t+k_i} { \tilde{u}_t}^\top  \tilde{u}_t + \frac{1}{k_i} \mathbb{E} \left [ \sum_{s=t}^{t+k_i} {\eta_t^u}^\top \eta_t^u  \right ] \\
& \overset{(b)}{=} \frac{1}{k_i} \sum_{s=t}^{t+k_i} { \tilde{u}_t}^\top  \tilde{u}_t + \frac{\gamma^2}{2} \\
& = \gamma^2
\end{align*}
where $(a)$ follows since $\tilde{u}_t$ and $\eta_t^u$ are independent, $(b)$ follows by our choice of $\sigma_u^2$ in Algorithm \ref{alg:active_lds_noise}. The final equality follows since, by construction, the inputs that are the solution to $\texttt{OptInput}_k(A_*,B_*,\gamma^2 - p \sigma_u^2,\mathcal{I},\{ x_t \}_{t=1}^T)$ will satisfy:
\begin{equation*}
\frac{1}{k} \sum_{s=t}^{t+k} \tilde{u}_s^\top \tilde{u}_s \leq \gamma^2 - p \sigma_u^2
\end{equation*}
for any $t \geq 0$. 
\end{proof}

\begin{lemma}\label{lem:alg_xtnorm_bound} \com{[Done]}
After $i$ epochs of running Algorithm \ref{alg:active_lds_noise}, we will have, with probability $1 - \delta$:
\begin{equation*}
\| x_t \|_2 \leq \frac{2 \beta(A_*) \| B_* \|_2 k_i \gamma}{1-\bar{\rho}(A_*)^{k_i}} +   \sqrt{ 2 Tr \left (\sigma^2 \Gamma_t + \frac{\gamma^2}{2p} \Gamma_t^{B_*} \right ) \left ( 1 + \frac{1}{c} \log \frac{4}{\delta} \right )} 
\end{equation*}
\end{lemma}
\begin{proof}
Let $x_t = x_t^u + x_t^{\eta,p} + x_t^{\eta,u}$ where $x_t^u$ is the response of the system due to the sinusoidal component of the input, $x_t^{\eta,p}$ is the response due to the process noise, and $x_t^{\eta,u}$ is the response due to the input noise. Note that this decomposition holds by linearity. Given this, we have $\| x_t \|_2  \leq \| x_t^{\eta,p} \|_2 + \| x_t^{\eta,u} \|_2 + \| x_t^u \|_2$. Then:
\begin{align*}
\| x_t^u \|_2 & \leq \left \| \sum_{s=0}^{t-1} A_*^{t-s-1} B_* u_s \right \|_2 \\
& \leq \beta(A_*) \| B_* \|_2 \sum_{s=0}^{t-1} \bar{\rho}(A_*)^{t-s-1} \| u_s \|_2 \\
& \leq \beta(A_*) \| B_* \|_2 \sum_{\ell=0}^{\lceil t / k_i \rceil - 2} \bar{\rho}(A_*)^{t - k_i(\ell + 1)-1} \sum_{s=k_i \ell}^{k_i (\ell + 1) - 1} \| u_s \|_2 + \beta(A_*) \| B_* \|_2 \sum_{s=(\lceil t / k_i \rceil - 1) k}^t \| u_s \|_2
\end{align*}
By construction, we will have that $\sum_{s=k_i \ell}^{k_i (\ell + 1) - 1} \| u_s \|_2^2 \leq k_i \gamma^2$ so long as $\ell$ is large enough that $k_i \ell$ is in epoch $i$. However, since $k_i$ is doubled at each epoch, this sum will contain an integer multiple of the period of the input regardless what the value of $\ell$ is, so we see that this inequality will hold for all values of $\ell$. This implies that for all $\ell$ (since $\| x \|_1 \leq \sqrt{n} \| x \|_2$ for any $x \in \mathbb{R}^n$), $\sum_{s=k_i \ell}^{k_i (\ell + 1) - 1} \| u_s \|_2 \leq \sqrt{k_i} \sqrt{\sum_{s=k_i \ell}^{k_i (\ell + 1) - 1} \| u_s \|_2^2} \leq k_i \gamma$. So:
\begin{align*}
& \beta(A_*) \| B_* \|_2 \sum_{\ell=0}^{\lceil t / k_i \rceil - 2} \bar{\rho}(A_*)^{t - k_i(\ell + 1)-1} \sum_{s=k_i \ell}^{k_i (\ell + 1) - 1} \| u_s \|_2 + \beta(A_*) \| B_* \|_2 \sum_{s=(\lceil t / k_i \rceil - 1) k}^t \| u_s \|_2 \\
\leq \ & \beta(A_*) \| B_* \|_2 k_i \gamma \sum_{\ell=0}^{\lceil t / k_i \rceil - 2} \bar{\rho}(A_*)^{t - k_i(\ell + 1)-1}  + \beta(A_*) \| B_* \|_2 k_i \gamma \\
= \ & \beta(A_*) \| B_* \|_2 k_i \gamma \frac{\bar{\rho}(A_*)^{t-1}}{\bar{\rho}(A_*)^{k_i}} \frac{\frac{1}{\bar{\rho}(A_*)^{k_i (\lceil t / k_i \rceil - 2)}} - \bar{\rho}(A_*)^{k_i}}{1 - \bar{\rho}(A_*)^{k_i}} + \beta(A_*) \| B_* \|_2 k_i \gamma \\
= \ &  \beta(A_*) \| B_* \|_2 k_i \gamma \frac{\frac{\bar{\rho}(A_*)^{k_i}}{\bar{\rho}(A_*)^{k_i \lceil t / k_i \rceil - t + 1} } - \bar{\rho}(A_*)^{t-1}}{1 - \bar{\rho}(A_*)^{k_i}}  + \beta(A_*) \| B_* \|_2 k_i \gamma \\
 \leq \ & \frac{\beta(A_*) \| B_* \|_2 k_i \gamma}{1 - \bar{\rho}(A_*)^{k_i}} + \beta(A_*) \| B_* \|_2 k_i \gamma \\
\leq \ & \frac{2\beta(A_*) \| B_* \|_2 k_i \gamma}{1 - \bar{\rho}(A_*)^{k_i}} 
\end{align*}
where the last inequality holds since if $t$ is divisible by $k_i$, $k_i \lceil t / k_i \rceil - t + 1 = 1$ so $\frac{\bar{\rho}(A_*)^{k_i}}{\bar{\rho}(A_*)^{k_i \lceil t / k_i \rceil - t + 1} } \leq 1$, and if $t$ is not divisible by $k_i$, $k_i \lceil t / k_i \rceil - t + 1 < k_i (t/k_i + 1) - t + 1 = k_i + 1$, and since $k_i \lceil t / k_i \rceil - t + 1$ is an integer, it follows that $\frac{\bar{\rho}(A_*)^{k_i}}{\bar{\rho}(A_*)^{k_i \lceil t / k_i \rceil - t + 1} } \leq 1$.

By definition:
\begin{equation*}
\| x_t^{\eta,p} \|_2 + \| x_t^{\eta,u} \|_2 = \left \| \sum_{s=0}^{t-1} A_*^{t - s - 1} \eta_s \right \|_2 + \left \| \sum_{s=0}^{t-1} A_*^{t-s-1} B_* \eta_s^u \right \|_2
\end{equation*}
where $\eta_s \sim \mathcal{N}(0,\sigma^2 I)$ and either $\eta_s^u = 0$ or $\eta_s^u \sim \mathcal{N}(0, \frac{\gamma^2}{2p} I)$. Note that:
\begin{equation*}
 \left \| \sum_{s=0}^{t-1} A_*^{t - s - 1} \eta_s \right \|_2^2 = \tilde{\eta}^\top \tilde{A}^\top \tilde{A} \tilde{\eta}
\end{equation*}
where:
\begin{equation*}
\tilde{A} = \begin{bmatrix} A_*^{t-1} & A_*^{t-2} & \ldots & A_* & I \end{bmatrix}, \ \ \ \ \tilde{\eta} = \begin{bmatrix} \eta_0 \\ \eta_1 \\ \vdots \\ \eta_{t-1} \end{bmatrix}
\end{equation*}
Noting that $\mathbb{E} \tilde{\eta}^\top \tilde{A}^\top \tilde{A} \tilde{\eta} = \sigma^2 Tr(\Gamma_t)$, we can then apply the Hanson-Wright inequality to get:
\begin{equation*}
\mathbb{P} \left [ | \tilde{\eta}^\top \tilde{A}^\top \tilde{A} \tilde{\eta} - \sigma^2 Tr(\Gamma_t) | \geq t \right ] \leq 2 \exp \left ( -c \min \left \{ \frac{t^2}{\sigma^4 \| \tilde{A}^\top \tilde{A} \|_F^2}, \frac{t}{\sigma^2 \| \tilde{A}^\top \tilde{A} \|_2} \right \} \right )
\end{equation*}
Setting $t =  \frac{\sigma^2 \| \tilde{A}^\top \tilde{A} \|_F^2}{c \| \tilde{A}^\top \tilde{A} \|_2} \log \frac{4}{\delta}$ the right hand side becomes:
\begin{equation*}
2 \exp \left ( - \frac{\| \tilde{A}^\top \tilde{A} \|_F^2}{ \| \tilde{A}^\top \tilde{A} \|_2^2}  \log \frac{4}{\delta} \right ) \leq \frac{\delta}{2}
\end{equation*}
where the inequality follows since $\frac{\| \tilde{A}^\top \tilde{A} \|_F^2}{ \| \tilde{A}^\top \tilde{A} \|_2^2} \geq 1$. So, with probability at least $1 - \delta / 2$, we will have:
\begin{align*}
\| x_t^{\eta,p} \|_2^2 & \leq \sigma^2 Tr(\Gamma_t) + \frac{\sigma^2 \| \tilde{A}^\top \tilde{A} \|_F^2}{c \| \tilde{A}^\top \tilde{A} \|_2} \log \frac{4}{\delta} \\
& \leq \sigma^2 Tr(\Gamma_t) \left ( 1 + \frac{1}{c} \log \frac{4}{\delta} \right )
\end{align*}
where the inequality holds since $\| \tilde{A}^\top \tilde{A} \|_F^2 \leq \| \tilde{A}^\top \tilde{A} \|_2 Tr(\tilde{A}^\top \tilde{A}) = \| \tilde{A}^\top \tilde{A} \|_2 Tr(\Gamma_t)$. Denoting $\tilde{A}^\top \tilde{A} = U \Lambda U^\top$, we see this is true since:
$$ \| \tilde{A}^\top \tilde{A} \|_F^2 = Tr(\tilde{A}^\top \tilde{A} \tilde{A}^\top \tilde{A}) = Tr(\Lambda^2) = \sum_{i=1}^n \lambda_i^2 \leq \left ( \max_i \lambda_i \right ) \sum_{i=1}^n \lambda_i = \| \tilde{A}^\top \tilde{A} \|_2 Tr(\tilde{A}^\top \tilde{A})$$
A similar calculation reveals that with probability at least $1 - \delta/2$:
\begin{equation*}
\| x_t^{\eta,u} \|_2^2 \leq \frac{\gamma^2}{2p} Tr(\Gamma_t^{B_*}) \left ( 1 + \frac{1}{c} \log \frac{4}{\delta} \right )
\end{equation*}
\end{proof}

\begin{corollary}\label{cor:tss_bound} \com{[Result is correct though it would be good to simplify it]}
After $i$ epochs of running Algorithm \ref{alg:active_lds_noise}, on the event that:
\begin{equation*}
\| x_t \|_2 \leq \frac{\beta(A_*) \| B_* \|_2 k_i \gamma}{1-\bar{\rho}(A_*)^{k_i}} +   \sqrt{ 2 Tr \left (\sigma^2 \Gamma_t + \frac{\gamma^2}{2p} \Gamma_t^{B_*} \right ) \left ( 1 + \frac{1}{c} \log \frac{4}{\delta} \right )} 
\end{equation*}
we will have:
\begin{align*}
T_{ss}(\zeta,k_{i+1},x_{\bar{T}_i}) & \leq \max \Bigg \{ \frac{1}{2 \log \frac{1}{\bar{\rho}(A_*)}} \left ( 2 \log \left ( \frac{4 \beta(A_*) \| B_* \|_2 k_{i+1} \gamma}{1-\bar{\rho}(A_*)^{k_{i+1}}} + \sqrt{ 2 Tr \left (\sigma^2 \Gamma_{T} + \frac{\gamma^2}{2p} \Gamma_{T}^{B_*} \right ) \left ( 1 + \frac{1}{c} \log \frac{4}{\delta} \right )} \right ) \right . \\
& \ \ \ \ \ \ \ \ \ \ \ \ \ \ \ \ \ \  + \left .  \log \left ( \frac{2  \beta(A_*)^2}{k_{i+1} \zeta(1-\bar{\rho}(A_*)^2)}  \right ) \right ), \\
& \ \ \ \ \ \ \ \ \ \ \ \ \ \ \frac{1}{\log \frac{1}{\bar{\rho}(A_*)}} \left ( \log \left ( \frac{4 \beta(A_*) \| B_* \|_2 k_{i+1} \gamma}{1-\bar{\rho}(A_*)^{k_{i+1}}} +  \sqrt{ 2 Tr \left (\sigma^2 \Gamma_{T} + \frac{\gamma^2}{2p} \Gamma_{T}^{B_*} \right ) \left ( 1 + \frac{1}{c} \log \frac{4}{\delta} \right )} \right ) \right . \\
& \ \ \ \ \ \ \ \ \ \ \ \ \ \ \ \ \ \  + \left . \log \left ( \frac{4  \beta(A_*) \gamma \max_{\ell = 1,...,k_{i+1}} \| (e^{j \frac{2 \pi \ell}{k_{i+1}}} I - A_*)^{-1} B_* \|_2}{\zeta \sqrt{k_{i+1}}  \sqrt{1-\bar{\rho}(A_*)^2}}  \right ) \right ) \Bigg \} =: T_{ss}(\zeta,k_{i+1})
\end{align*}
where  $T$ is the amount of time elapsed after $i$ epochs. 
\end{corollary}
\begin{proof}
From Lemma \ref{lem:burn_in}, we have:
\begin{align*}
T_{ss}(\zeta,k_{i+1},x_{\bar{T}_i}) & =  \max \left \{ \frac{1}{2 \log \bar{\rho}(A_*)} \log \left ( \frac{k_{i+1} \zeta(1-\bar{\rho}(A_*)^2)}{2 \| x_T - x_0^{ss,i+1} \|_2^2 \beta(A_*)^2}  \right ) , \right . \\
& \ \ \ \ \ \ \ \ \ \ \ \ \ \ \ \ \ \left .  \frac{1}{\log \bar{\rho}(A_*)} \log \left ( \frac{k_{i+1} \zeta \sqrt{1-\bar{\rho}(A_*)^2}}{4 \| x_T - x_0^{ss,i+1} \|_2 \beta(A_*) \sqrt{k_{i+1} w^\top \tilde{\Gamma}_{k_{i+1}}^{u_{i+1}} w}}  \right )  \right \} \\
& = \max \left \{ \frac{1}{2 \log \frac{1}{\bar{\rho}(A_*)}} \log \left ( \frac{2 \| x_T - x_0^{ss,i+1} \|_2^2 \beta(A_*)^2}{k_{i+1} \zeta(1-\bar{\rho}(A_*)^2)}  \right ) , \right . \\
& \ \ \ \ \ \ \ \ \ \ \ \ \ \ \ \ \ \left .  \frac{1}{\log \frac{1}{\bar{\rho}(A_*)}} \log \left ( \frac{4 \| x_T - x_0^{ss,i+1} \|_2 \beta(A_*) \sqrt{k_{i+1} w^\top \tilde{\Gamma}_{k_{i+1}}^{u_{i+1}} w}}{k_{i+1} \zeta \sqrt{1-\bar{\rho}(A_*)^2}}  \right )  \right \} 
\end{align*}
where $x_T$ is the state at the start of the $i+1$th epoch, and $x_0^{ss,i+1}$ is the initial state of the steady state response of the system to the inputs played at the $i+1$th epoch. From Lemma \ref{lem:alg_xtnorm_bound}, since the noise term will be 0, we can deterministically upper bound:
\begin{equation*}
\| x_0^{ss,i+1} \|_2 \leq \frac{2\beta(A_*) \| B_* \|_2 k_{i+1} \gamma}{1-\bar{\rho}(A_*)^{k_{i+1}}}
\end{equation*}
and also:
\begin{align*}
\| x_T \|_2 & \leq \frac{2\beta(A_*) \| B_* \|_2 k_i \gamma}{1-\bar{\rho}(A_*)^{k_i}} +  \sqrt{ 2 Tr \left (\sigma^2 \Gamma_T + \frac{\gamma^2}{2p} \Gamma_T^{B_*} \right ) \left ( 1 + \frac{1}{c} \log \frac{4}{\delta} \right )}  \\
& \leq \frac{2\beta(A_*) \| B_* \|_2 k_{i+1} \gamma}{1-\bar{\rho}(A_*)^{k_{i+1}}} +  \sqrt{ 2 Tr \left (\sigma^2 \Gamma_T + \frac{\gamma^2}{2p} \Gamma_T^{B_*} \right ) \left ( 1 + \frac{1}{c} \log \frac{4}{\delta} \right )}
\end{align*}
so:
\begin{equation*}
\| x_T - x_0^{ss,i+1} \|_2 \leq \frac{4 \beta(A_*) \| B_* \|_2 k_{i+1} \gamma}{1-\bar{\rho}(A_*)^{k_{i+1}}} +  \sqrt{ 2 Tr \left (\sigma^2 \Gamma_T + \frac{\gamma^2}{2p} \Gamma_T^{B_*} \right ) \left ( 1 + \frac{1}{c} \log \frac{4}{\delta} \right )}
\end{equation*}
it follows then that:
\begin{align*}
T_{ss}(\zeta,k_{i+1},x_{\bar{T}_i}) & \leq \max \Bigg \{ \frac{1}{2 \log \frac{1}{\bar{\rho}(A_*)}} \left ( 2 \log \left ( \frac{4 \beta(A_*) \| B_* \|_2 k_{i+1} \gamma}{1-\bar{\rho}(A_*)^{k_{i+1}}} + \sqrt{ 2 Tr \left (\sigma^2 \Gamma_T + \frac{\gamma^2}{2p} \Gamma_T^{B_*} \right ) \left ( 1 + \frac{1}{c} \log \frac{4}{\delta} \right )} \right ) \right . \\
& \ \ \ \ \ \ \ \ \ \ \ \ \ \ \ \ \ \  + \left .  \log \left ( \frac{2  \beta(A_*)^2}{k_{i+1} \zeta(1-\bar{\rho}(A_*)^2)}  \right ) \right ), \\
& \ \ \ \ \ \ \ \ \ \ \ \ \ \ \frac{1}{\log \frac{1}{\bar{\rho}(A_*)}} \left ( \log \left ( \frac{4 \beta(A_*) \| B_* \|_2 k_{i+1} \gamma}{1-\bar{\rho}(A_*)^{k_{i+1}}} +  \sqrt{ 2 Tr \left (\sigma^2 \Gamma_T + \frac{\gamma^2}{2p} \Gamma_T^{B_*} \right ) \left ( 1 + \frac{1}{c} \log \frac{4}{\delta} \right )} \right ) \right . \\
& \ \ \ \ \ \ \ \ \ \ \ \ \ \ \ \ \ \  + \left . \log \left ( \frac{4  \beta(A_*) \sqrt{k_{i+1} w^\top \tilde{\Gamma}_{k_{i+1}}^{u_{i+1}} w}}{k_{i+1} \zeta \sqrt{1-\bar{\rho}(A_*)^2}}  \right ) \right ) \Bigg \}
\end{align*}
Finally, we must upper bound $k_{i+1} w^\top \tilde{\Gamma}_{k_{i+1}}^{u_{i+1}} w$. Upper bounding this over all $w \in \mathcal{S}^{d-1}$ is equivalent to bounding:
\begin{align*}
\left \| \sum_{t=1}^k x_t^{u_{i+1},ss} {x_t^{u_{i+1},ss}}^\top \right \|_2 & = \frac{1}{k} \left \| \sum_{\ell =1}^{k_{i+1}} (e^{j \frac{2 \pi \ell}{k_{i+1}}} I - A_*)^{-1} B_* U(e^{j \frac{2 \pi \ell}{k_{i+1}}}) U(e^{j \frac{2 \pi \ell}{k_{i+1}}})^H B_*^H (e^{j \frac{2 \pi \ell}{k_{i+1}}} I - A_*)^{-H} \right \|_2 \\
& \leq \frac{1}{k_{i+1}} \left ( \max_{\ell = 1,...,k_{i+1}} \| (e^{j \frac{2 \pi \ell}{k_{i+1}}} I - A_*)^{-1} B_* \|_2^2 \right ) \sum_{\ell=1}^{k_{i+1}} \| U(e^{j \frac{2 \pi \ell}{k_{i+1}}}) \|_2^2 \\
& \leq k_{i+1} \gamma^2 \left ( \max_{\ell = 1,...,k_{i+1}} \| (e^{j \frac{2 \pi \ell}{k_{i+1}}} I - A_*)^{-1} B_* \|_2^2 \right )
\end{align*}
\end{proof}

\begin{lemma}\label{lem:k_Tbound}
After $i$ epochs, we will have that:
\begin{equation*}
 k_i \geq  \frac{\sqrt{2}}{2} \frac{k_0}{T_0} \sqrt{T} 
\end{equation*}
\end{lemma}
\begin{proof}
After the $i$th epoch, we will have that:
\begin{equation*}
T = \sum_{j = 0}^i 3^j T_0 = \frac{1}{2} ( 3^{i+1} - 1) T_0
\end{equation*}
Solving this for $i$ gives:
\begin{equation*}
i = \frac{\log \left ( \frac{2T + T_0}{T_0} \right )}{\log 3} - 1
\end{equation*}
Thus:
\begin{equation*}
k_i = 2^i k_0 = \frac{1}{2} 2^{\frac{\log \left ( \frac{2T + T_0}{T_0} \right )}{\log 3}} k_0 = \frac{k_0}{2} \left ( \frac{2T + T_0}{T_0} \right )^{\log 2 / \log 3}
\end{equation*}
Noting that $\log 2 / \log 3 \approx 0.63$, we can lower bound this as:
\begin{equation*}
k_i \geq \frac{1}{2} \frac{k_0}{T_0} \sqrt{2T + T_0} \geq \frac{\sqrt{2}}{2} \frac{k_0}{T_0} \sqrt{T}
\end{equation*}
\end{proof}


\section{Estimation of Linear Dynamical Systems with Periodic Inputs}\label{sec:concentration}
\begin{theorem}\label{thm:concentration2_2}
\textbf{(Full version of Theorem \ref{thm:concentration2})} Assume that we start from some initial state $x_0$ and we are playing some input $u_t = \tilde{u}_t + \eta_t^u$ where $\tilde{u}_t$ is deterministic with period $k$ and $\eta_t^u \sim \mathcal{N}(0,\sigma_u^2 I)$. Then as long as:
\begin{equation}\label{eq:thm_error_burnin1}
T \geq c k \left ( d + \log \det (\bar{\Gamma}_T {\Gamma_k^\eta}^{-1}) + \log \frac{1}{\delta} \right )
\end{equation}
we will have that:
\begin{equation}\label{eq:thm_error1}
\mathbb{P} \left [ \| \hat{A} - A_* \|_2 > C \sigma \sqrt{\frac{16 \log \frac{1}{\delta} + 8 \log \det (\bar{\Gamma}_T {\Gamma_k^\eta}^{-1} + I) + 16 d \log 5}{T \lambda_{\min} ( \Gamma_k^\eta)}} \right ] \leq 3 \delta
\end{equation}
and if:
\begin{equation}\label{eq:thm_error_burnin2}
T \geq 2T_{ss} \left ( \frac{1}{10} \lambda_{\min} (\tilde{\Gamma}_k^u), k, x_0 \right ) + c' k \left (  d + \max \{ \log \det ( \bar{\Gamma}_T (\tilde{\Gamma}_{k}^{u})^{-1} ), \log \det (\bar{\Gamma}_T {\Gamma_k^\eta}^{-1}) \} + \log \frac{1}{\delta} \right )
\end{equation}
then:
\begin{equation}\label{eq:thm_error2}
\mathbb{P} \left [ \| \hat{A} - A_* \|_2 > C' \sigma \sqrt{\frac{16 \log \frac{4}{3\delta} + 8 \log \det (\bar{\Gamma}_T (\Gamma_k^\eta + \tilde{\Gamma}_k^u)^{-1} + I) + 16 d \log 5}{T \lambda_{\min} ( \Gamma_k^\eta + \tilde{\Gamma}_k^u)}} \right ] \leq 3 \delta
\end{equation}
where $\bar{\Gamma}_T = 4 \left ( \tilde{\Gamma}_{T,0}^u + Tr(\Gamma_T^\eta) ( 1 + \log \frac{2}{\delta} ) I \right )$ and $c,c',C,C'$ are universal constants.
\end{theorem}

Note that $T_{ss} \left ( \frac{1}{10} \lambda_{\min} (\tilde{\Gamma}_k^u), k, x_0 \right )$ in (\ref{eq:thm_error_burnin2}) can be replaced with $T_{ss} \left ( c'' \lambda_{\min} (\Gamma_k^\eta), k, x_0 \right )$, which may be helpful if our system is not controllable, in which case it's possible $\lambda_{\min} (\tilde{\Gamma}_k^u) = 0$. An example of this argument can be found in the proof of Theorem \ref{cor:largek_alg_opt_informal}.

\subsection{Proof of Theorem \ref{thm:concentration2} and Theorem \ref{thm:concentration2_2}}
\begin{proof}
Define the following events:
\begin{align*}
& \mathcal{A}_1 := \left \{  \| \hat{A} - A_* \|_2 \leq C' \sigma \sqrt{\frac{16 \log \frac{1}{\delta} + 8 \log \det (\bar{\Gamma}_T (\Gamma_k^\eta )^{-1} + I) + 16 d \log 5}{T \lambda_{\min} ( \Gamma_k^\eta )}} \right \} \\
& \mathcal{E}_1 := \left \{  \sum_{t=1}^T x_t x_t^\top \preceq T \bar{\Gamma}_T \right \} \\
& \mathcal{E}_2 := \left \{   \sum_{t=1}^{T} x_t x_t^\top \succeq c_1 T \Gamma_{k}^\eta \right \} \\
& \mathcal{E}_3 := \left \{ \left \| \left ( \sum_{t=1}^{T} x_t x_t^\top \right )^{-1/2} \sum_{t=1}^{T} x_t \eta_t^\top \right \|_2 \leq c_3 \sigma \sqrt{\log \frac{1}{\delta} + d + \log \det ( \bar{\Gamma}_{T } (\Gamma_{k}^\eta )^{-1} + I)} \right \} 
\end{align*}

(\ref{eq:thm_error1}) follows directly from bounding $\mathbb{P}[\mathcal{A}_1^c]$. The following clearly holds:
\begin{align*}
\mathbb{P}[\mathcal{A}_1^c] & \leq \mathbb{P}[\mathcal{A}_1^c \cap \mathcal{E}_1] + \mathbb{P}[\mathcal{E}_1^c] \\
& \leq \mathbb{P}[\mathcal{A}_1^c \cap \mathcal{E}_1 \cap \mathcal{E}_2] + \mathbb{P}[\mathcal{E}_1 \cap \mathcal{E}_2^c] + \mathbb{P}[\mathcal{E}_1^c] \\
& \leq \mathbb{P}[\mathcal{A}_1^c \cap \mathcal{E}_1 \cap \mathcal{E}_2 \cap \mathcal{E}_3] + \mathbb{P}[\mathcal{E}_1 \cap \mathcal{E}_2 \cap \mathcal{E}_3^c] + \mathbb{P}[\mathcal{E}_1 \cap \mathcal{E}_2^c] + \mathbb{P}[\mathcal{E}_1^c] 
\end{align*} 
By Lemma \ref{lem:upper_bound}, we will have that $\mathbb{P}[\mathcal{E}_1^c] \leq \delta$. If (\ref{eq:thm_error_burnin1}) holds the burn in time required by Lemma \ref{lem:cov_lb_noise} will be met, so by Lemma \ref{lem:cov_lb_noise}, $\mathbb{P}[\mathcal{E}_2^c \cap \mathcal{E}_1] \leq \delta$. Similarly,  by Lemma \ref{lem:self_normalized_past_data2}, $\mathbb{P}[\mathcal{E}_3^c \cap \mathcal{E}_1 \cap \mathcal{E}_2] \leq \delta$. To bound $\mathbb{P}[\mathcal{A}_1^c \cap \mathcal{E}_1 \cap \mathcal{E}_2 \cap \mathcal{E}_3]$, note that we can decompose the error of the least squares estimate as:
\begin{align*}
\| \hat{A} - A_* \|_2 & = \| (X X^\top)^{-1} X^\top E \|_2 \\
& \leq \| (X X^\top)^{-1/2} \|_2 \| (X X^\top)^{-1/2} X^\top E \|_2 \\
& = \lambda_{\min}(X X^\top)^{-1/2} \| (X X^\top)^{-1/2} X^\top E \|_2
\end{align*}
On the event $ \mathcal{E}_1 \cap \mathcal{E}_2 \cap \mathcal{E}_3$, we will have that:
$$  \lambda_{\min}(X X^\top)^{-1/2} \leq  \sqrt{\frac{1}{c_1 T \lambda_{\min}(\Gamma_k^\eta)}}$$
$$ \| (X X^\top)^{-1/2} X^\top E \|_2 \leq c_3 \sigma \sqrt{\log \frac{1}{\delta} + d + \log \det ( \bar{\Gamma}_{T } (\Gamma_{k}^\eta )^{-1} + I)} $$
Combining these it follows that on this event:
$$\| \hat{A} - A_* \|_2 \leq C' \sigma \sqrt{\frac{16 \log \frac{1}{\delta} + 8 \log \det (\bar{\Gamma}_T (\Gamma_k^\eta )^{-1} + I) + 16 d \log 5}{T \lambda_{\min} ( \Gamma_k^\eta )}} $$
so $\mathbb{P}[\mathcal{A}_1^c \cap \mathcal{E}_1 \cap \mathcal{E}_2 \cap \mathcal{E}_3] = 0$. It follows then that $\mathbb{P}[\mathcal{A}_1^c] \leq 3 \delta$ which proves (\ref{eq:thm_error1}).

To show (\ref{eq:thm_error2}), define the following events:
\begin{align*}
& \mathcal{A}_2 := \left \{  \| \hat{A} - A_* \|_2 \leq C' \sigma \sqrt{\frac{16 \log \frac{1}{\delta} + 8 \log \det (\bar{\Gamma}_T (\Gamma_k^\eta + \tilde{\Gamma}_k^u)^{-1} + I) + 16 d \log 5}{T \lambda_{\min} ( \Gamma_k^\eta + \tilde{\Gamma}_k^u)}} \right \} \\
& \mathcal{E}_4 :=  \left \{ \sum_{t=1}^T x_t x_t^\top \succeq c_2 T \tilde{\Gamma}_{k}^{u} \right \} \\
& \mathcal{E}_5 := \left \{ \left \| \left ( \sum_{t=1}^{T} x_t x_t^\top \right )^{-1/2} \sum_{t=1}^{T} x_t \eta_t^\top \right \|_2 \leq c_3 \sigma \sqrt{\log \frac{1}{\delta} + d + \log \det ( \bar{\Gamma}_{T } (\Gamma_{k}^\eta + \tilde{\Gamma}_{k}^{u})^{-1} + I)} \right \} 
\end{align*}
Our goal now is to bound $\mathbb{P}[\mathcal{A}_2^c]$. Similar to the above, we have:
\begin{align*}
\mathbb{P}[\mathcal{A}_2^c] & \leq \mathbb{P}[\mathcal{A}_2^c \cap \mathcal{E}_1] + \mathbb{P}[\mathcal{E}_1^c] \\
& \leq \mathbb{P}[\mathcal{A}_2^c \cap \mathcal{E}_1 \cap \mathcal{E}_2 \cap \mathcal{E}_4] + \mathbb{P}[\mathcal{E}_1 \cap \mathcal{E}_2^c] + \mathbb{P}[\mathcal{E}_1 \cap \mathcal{E}_4^c] + \mathbb{P}[\mathcal{E}_1^c]  \\
& \leq \mathbb{P}[\mathcal{A}_2^c \cap \mathcal{E}_1 \cap \mathcal{E}_2 \cap \mathcal{E}_4 \cap \mathcal{E}_5] + \mathbb{P}[\mathcal{E}_1 \cap \mathcal{E}_2 \cap \mathcal{E}_4 \cap \mathcal{E}_5^c] + \mathbb{P}[\mathcal{E}_1 \cap \mathcal{E}_2^c] + \mathbb{P}[\mathcal{E}_1 \cap \mathcal{E}_4^c] + \mathbb{P}[\mathcal{E}_1^c] 
\end{align*}
As before, we have that $\mathbb{P}[\mathcal{E}_1^c] \leq \delta$ and, assuming (\ref{eq:thm_error_burnin2}) holds, $\mathbb{P}[\mathcal{E}_1 \cap \mathcal{E}_2^c] \leq \delta$. If (\ref{eq:thm_error_burnin2}) holds, by Corollary \ref{cor:burn_in} the burn in condition required by Lemma \ref{lem:cov_lb_inputs} will be met so we will also have that $\mathbb{P}[\mathcal{E}_1 \cap \mathcal{E}_4^c] \leq \delta$. By Lemma \ref{lem:self_normalized_past_data2} and the error decomposition of $\| \hat{A} - A_* \|_2$ used above, we have that $\mathbb{P}[\mathcal{E}_1 \cap \mathcal{E}_2 \cap \mathcal{E}_4 \cap \mathcal{E}_5^c] \leq \delta$ and $\mathbb{P}[\mathcal{A}_2^c \cap \mathcal{E}_1 \cap \mathcal{E}_2 \cap \mathcal{E}_4 \cap \mathcal{E}_5] = 0$. Thus, $\mathbb{P}[\mathcal{A}_2^c] \leq 4 \delta$ from which (\ref{eq:thm_error2}) follows directly.

\end{proof}

\subsection{Lower Bounds on Covariates and Self-Normalized Bounds}
The following proposition is crucial to proving a high probability bound on the error in the presence of non-random inputs.

\begin{proposition}\label{prop:covbound}\com{[Done]}
\textbf{(Full version of Proposition \ref{prob:covbound_informal})} Consider any $w \in \mathcal{S}^{d-1}$ and let $x_t$ evolve according to the dynamical system (\ref{eq:lds}). Let $u_t$ be a deterministic periodic signal and $k$ be an integer multiple of its period. Let $x_t^{u,ss}$ denote the steady state response of the system to this input and let $\alpha := \sum_{t=0}^{k-1} (w^\top x_t^{u,ss})^2$. Assume that $T_{ss}$ is chosen large enough so that, for any $T \geq 0$:
\begin{equation}\label{eq:prop_power}
\left | \sum_{t=T_{ss} + T+1}^{T_{ss} + T + k } (w^\top x_t^u - w^\top \bar{x}_{T_{ss}+T:k}^u)^2 - \alpha \right | \leq \frac{\alpha}{10} 
\end{equation}
where:
\begin{equation*}
\bar{x}_{T:k}^u := \frac{1}{k} \sum_{t= T+1}^{ T + k } x_t^u
\end{equation*}
Then we will have that:
\begin{equation}\label{eq:prop_cov_bound}
\mathbb{P} \left [ \sum_{t=T_{ss}+1}^{T_{ss}+T} (w^\top x_t)^2 \leq \frac{2}{81} k \lfloor T / k \rfloor w^\top \tilde{\Gamma}_{k}^u w  \right ] \leq e^{-\frac{2}{81} \lfloor T / k \rfloor }
\end{equation}
\end{proposition}

\begin{proof}
We first note that, since our system is linear, the output of the system due to the input,  $x_t^u$, will contain only the frequencies present in the input, $u_t$, with possibly some phase shift. Thus, the period of the periodic part of our output will be identical to that of the input once the system is in steady state.

Let:
\begin{equation*}
z_t := w^\top x_{T_{ss}+t} = w^\top (x_{T_{ss}+t}^u + x_{T_{ss}+t}^\eta), \ \ \ \ \mu_t := w^\top x_{T_{ss}+t}^u - w^\top \bar{x}_{T_{ss} + \lfloor (t - 1) / k  \rfloor k : k}^u
\end{equation*}
Note that $\sum_{t=1}^{k} \mu_{jk+t} = 0$ for any $j=0,1,2,...$.

Let:
\begin{equation*}
B_j := \mathbb{I} \left [ \sum_{i=1}^k z_{jk+i}^2 \geq c_1 \sum_{i=1}^k \mu_{jk + i}^2 \right ]
\end{equation*}
for some $c_1$ to be specified, where $\mathbb{I}$ is the indicator function. Then $\sum_{i=1}^k z_{jk+i}^2 \geq \left ( c_1 \sum_{i=1}^k \mu_{jk + i}^2 \right ) B_j$. Let $S = \lfloor T / k \rfloor$ and $c_2$ be some constant to be specified. Then:
\begin{align}\label{eq:cov_cher_bound}
\begin{split}
\mathbb{P} \left [ \sum_{t=1}^{T} z_t^2 \leq c_2 \sum_{t=1}^{T} \mu_t^2 \right ] & \leq \mathbb{P} \left [ \sum_{j=0}^{S-1} \left ( c_1 \sum_{i=1}^k \mu_{jk + i}^2 \right ) B_j \leq c_2 \sum_{t=1}^{T} \mu_t^2 \right ] \\
& \leq \inf_{\lambda < 0} \exp \left \{ -\lambda c_2 \sum_{t=1}^{T} \mu_t^2 \right \} \mathbb{E} \left [ \exp \left \{ \lambda \sum_{j=0}^{S-1} \left ( c_1 \sum_{i=1}^k \mu_{jk + i}^2 \right ) B_j \right \} \right ] 
\end{split}
\end{align}
where the last inequality is simply Chernoff's bound. To compute the expectation, we will use the tower property. To do so, it will be convenient to first calculate the conditional expectation of $B_j$. Letting $\mathcal{F}_j$ denote the $\sigma$-field generated by $\eta_0,...,\eta_{T_{ss}+jk}$, we have that: 
\begin{align*}
\mathbb{E}[B_j | \mathcal{F}_j] & = \mathbb{P}  \left [ \sum_{i=1}^k z_{jk+i}^2 \geq c_1 \sum_{i=1}^k \mu_{jk + i}^2 | \mathcal{F}_j \right ] \\
& = \mathbb{P}  \left [ \sum_{i=1}^k \left ( \mu_{jk+i} + w^\top x^\eta_{T_{ss}+jk + i} + w^\top \bar{x}_{T_{ss} + jk}^u  \right ) ^2 \geq c_1 \sum_{i=1}^k \mu_{jk + i}^2 | \mathcal{F}_j \right ] \\
& = \mathbb{P}  \left [ \sum_{i=1}^k \left ( \mu_{jk+i} + w^\top \sum_{s=0}^{i-1} A_*^{i-s-1} \eta_{T_{ss}+jk+s} + w^\top A_*^i x^\eta_{T_{ss}+jk} + w^\top \bar{x}_{T_{ss} + jk}^u \right ) ^2 \geq c_1 \sum_{i=1}^k \mu_{jk + i}^2 | \mathcal{F}_j \right ] \\
\end{align*}
where the last equality follows since:
\begin{equation*}
x^\eta_{T_{ss}+jk + i} = A_*^i x^{\eta}_{T_{ss}+jk} + \sum_{s=0}^{i-1} A_*^{i-s-1} \eta_{T_{ss}+jk+s}
\end{equation*}
Note that, conditioned on the $\mathcal{F}_j$, $w^\top A_*^i x^{\eta}_{T_{ss}+jk}$ and $w^\top \bar{x}_{T_{ss} + jk}^u$ are deterministic. Further, since $\eta_t$ is mean 0, $w^\top \sum_{s=0}^{i-1} A_*^{i-s-1} \eta_{T_{ss}+jk+i}$ will simply be a linear combination of mean 0 Gaussians and so will itself be a mean 0 Gaussian. This implies that $\mathbb{P}[ w^\top \sum_{s=0}^{i-1} A_*^{i-s-1} \eta_{T_{ss}+jk+i} \geq 0] = 1/2$. 

Since we have constructed $\mu_t$ in such a way as to be mean zero over a block of length $k$, for any fixed $a$:
$$ \sum_{i=1}^k (\mu_{jk + i}+a)^2 = \sum_{i=1}^k \mu_{jk+i}^2 + a \sum_{i=1}^k \mu_{jk+i} + a^2 = \sum_{i=1}^k \mu_{jk+i}^2 + a^2 \geq \sum_{i=1}^k \mu_{jk+i}^2 $$
In particular then:
\begin{equation}\label{eq:offset_sines}
\sum_{i=1}^k (\mu_{jk + i} + w^\top A_*^i x^\eta_{T_{ss}+jk} + w^\top \bar{x}_{T_{ss} + jk}^u ) ^2 | \mathcal{F}_j \geq \sum_{i=1}^k \mu_{jk + i}^2 | \mathcal{F}_j 
\end{equation}
which implies:
\begin{align*}
& \mathbb{P}  \left [ \sum_{i=1}^k \left ( \mu_{jk+i} + w^\top \sum_{s=0}^{i-1} A_*^{i-s-1} \eta_{T_{ss}+jk+s} + w^\top A_*^i x^\eta_{T_{ss}+jk} + w^\top \bar{x}_{T_{ss} + jk}^u \right ) ^2 \geq c_1 \sum_{i=1}^k \mu_{jk + i}^2 | \mathcal{F}_j \right ] \\
\geq \ & \mathbb{P}  \left [ \sum_{i=1}^k \left ( \mu_{jk+i} + w^\top \sum_{s=0}^{i-1} A_*^{i-s-1} \eta_{T_{ss}+jk+s} + w^\top A_*^i x^\eta_{T_{ss}+jk} + w^\top \bar{x}_{T_{ss} + jk}^u  \right ) ^2 \right . \\
&\hspace{6.5cm} \geq \left . c_1 \sum_{i=1}^k (\mu_{jk + i} + w^\top A_*^i x^\eta_{T_{ss}+jk} + w^\top \bar{x}_{T_{ss} + jk}^u)^2 | \mathcal{F}_j \right ]  \\
\geq \ & \mathbb{P}  \left [ \sum_{i=1}^k \left ( \mu_{jk+i} + w^\top A_*^i x^\eta_{T_{ss}+jk} + w^\top \bar{x}_{T_{ss} + jk}^u \right ) ^2 \mathbb{I} \left [ \left | \mu_{jk+i} + w^\top \sum_{s=0}^{i-1} A_*^{i-s-1} \eta_{T_{ss}+jk+s}  \right .  \right . \right . \\
& \hspace{2cm} \left .  \left . + w^\top A_*^i x^\eta_{T_{ss}+jk} + w^\top \bar{x}_{T_{ss} + jk}^u \right | \geq \left | \mu_{jk+i} + w^\top A_*^i x^\eta_{T_{ss}+jk} + w^\top \bar{x}_{T_{ss} + jk}^u\right | \Bigg ] \right .  \\ 
& \hspace{6.5cm}  \left .  \geq c_1 \sum_{i=1}^k (\mu_{jk + i} + w^\top A_*^i x^\eta_{T_{ss}+jk} + w^\top \bar{x}_{T_{ss} + jk}^u)^2 | \mathcal{F}_j \right ] \\
\overset{(a)}{\geq} & \frac{1/2 - c_1}{1 - c_1}
\end{align*}
where the last inequality follows by a reverse Markov inequality which states that, for any random variable $Z$ supported in $[0,1]$ almost surely and with $\mathbb{E}[Z] \geq p \in (0,1)$, for all $t \in [0,p]$, $\mathbb{P}[Z \geq t] \geq \frac{p - t}{1 - t}$ \cite{simchowitz2018learning}. Noting that, since the noise is 0 mean Gaussian, we have:
\begin{align*}
& \mathbb{E}  \left [ \sum_{i=1}^k \left ( \mu_{jk+i} + w^\top A_*^i x^\eta_{T_{ss}+jk} + w^\top \bar{x}_{T_{ss} + jk}^u \right ) ^2 \mathbb{I} \left [ \left | \mu_{jk+i} + w^\top \sum_{s=0}^{i-1} A_*^{i-s-1} \eta_{T_{ss}+jk+s} + w^\top A_*^i x^\eta_{T_{ss}+jk}  \right . \right . \right . \\
& \ \ \ \ \ \ \ \ \ \ \ \left . + w^\top \bar{x}_{T_{ss} + jk}^u \right | \geq   \left | \mu_{jk+i} + w^\top A_*^i x^\eta_{T_{ss}+jk} + w^\top \bar{x}_{T_{ss} + jk}^u \right | \Bigg ] | \mathcal{F}_j \Bigg ] \\
= \ & \sum_{i=1}^k \left ( \mu_{jk+i} + w^\top A_*^i x^\eta_{T_{ss}+jk} + w^\top \bar{x}_{T_{ss} + jk}^u \right ) ^2 \mathbb{P} \left [ \left | \mu_{jk+i} + w^\top \sum_{s=0}^{i-1} A_*^{i-s-1} \eta_{T_{ss}+jk+s} + w^\top A_*^i x^\eta_{T_{ss}+jk}  \right . \right .  \\
& \ \ \ \ \ \ \ \ \ \ \ \left . + w^\top \bar{x}_{T_{ss} + jk}^u \right | \geq \left | \mu_{jk+i} + w^\top A_*^i x^\eta_{T_{ss}+jk}  + w^\top \bar{x}_{T_{ss} + jk}^u \right | | \mathcal{F}_j  \Bigg ] \\
= \ & \frac{1}{2} \sum_{i=1}^k \left ( \mu_{jk+i} + w^\top A_*^i x^\eta_{T_{ss}+jk} + w^\top \bar{x}_{T_{ss} + jk}^u \right ) ^2
\end{align*}
From this $(a)$ follows by simple manipulations. Since we can choose $c_1$ as we wish, we set it equal to $c_1 = 1/4$ and conclude that:
\begin{equation*}
\mathbb{E}[B_j | \mathcal{F}_j ] \geq \frac{1}{3}
\end{equation*}
Returning to (\ref{eq:cov_cher_bound}), we can now use this result to bound the expectation. Note that:
\begin{align*}
& \mathbb{E} \left [ \exp \left \{ \lambda \sum_{j=0}^{S-1} \left ( c_1 \sum_{i=1}^k \mu_{jk + i}^2 \right ) B_j \right \} \right ]  =  \mathbb{E} \left [ \mathbb{E} \left [ \exp \left \{ \lambda \sum_{j=0}^{S-1} \left ( c_1 \sum_{i=1}^k \mu_{jk + i}^2 \right ) B_j \right \} | \mathcal{F}_{S-1} \right ]  \right ] \\
& \ \ \ \ \ \ \ \ \ \ \ =  \mathbb{E} \left [ \exp \left \{ \lambda \sum_{j=0}^{S-2} \left ( c_1 \sum_{i=1}^k \mu_{jk + i}^2 \right ) B_j \right \}  \mathbb{E} \left [ \exp \left \{ \lambda \left ( c_1 \sum_{i=1}^k \mu_{(S-1)k + i}^2 \right ) B_{S-1} \right \} | \mathcal{F}_{S-1} \right ]  \right ] \\
\end{align*}
Then by what we just proved and applying Hoeffding's Lemma, since $\lambda < 0$, we have:
\begin{equation*}
\mathbb{E} \left [ \exp \left \{ \lambda \left ( c_1 \sum_{i=1}^k \mu_{(S-1)k + i}^2 \right ) B_{S-1} \right \} | \mathcal{F}_{S-1} \right ]  \leq \exp \left \{ \frac{\lambda}{3} \left ( c_1 \sum_{i=1}^k \mu_{(S-1)k + i}^2 \right ) + \frac{\lambda^2}{8} \left ( c_1 \sum_{i=1}^k \mu_{(S-1)k + i}^2 \right )^2 \right \}
\end{equation*}
Repeating this procedure condition on each $\mathcal{F}_i$, we get:
\begin{align*}
\mathbb{E} \left [ \exp \left \{ \lambda \sum_{j=0}^{S-1} \left ( c_1 \sum_{i=1}^k \mu_{jk + i}^2 \right ) B_j \right \} \right ]  \leq \exp \left \{ \frac{\lambda}{3} \sum_{j=0}^{S-1} \left ( c_1 \sum_{i=1}^k \mu_{jk + i}^2 \right ) + \frac{\lambda^2}{8} \sum_{j=0}^{S-1} \left ( c_1 \sum_{i=1}^k \mu_{jk + i}^2 \right )^2 \right \}
\end{align*}
and so:
\begin{align*}
\begin{split}
\mathbb{P} \left [ \sum_{t=1}^T z_t^2 \leq c_2 \sum_{t=1}^T \mu_t^2 \right ] & \leq \inf_{\lambda < 0} \exp \left \{ -\lambda c_2 \sum_{t=1}^T \mu_t^2 \right \} \exp \left \{ \frac{\lambda}{3} \sum_{j=0}^{S-1} \left ( c_1 \sum_{i=1}^k \mu_{jk + i}^2 \right ) + \frac{\lambda^2}{8} \sum_{j=0}^{S-1} \left ( c_1 \sum_{i=1}^k \mu_{jk + i}^2 \right )^2 \right \} \\
& = \inf_{\lambda < 0}  \exp \left \{ \lambda \left ( \frac{c_1}{3} - c_2 \right ) \sum_{j=0}^{S-1} \left ( \sum_{i=1}^k \mu_{jk + i}^2 \right ) + \frac{\lambda^2}{8} \sum_{j=0}^{S-1} \left ( c_1 \sum_{i=1}^k \mu_{jk + i}^2 \right )^2 \right \} \\
& \leq \exp \left \{ - \frac{2 \left ( \left ( \frac{c_1}{3} - c_2 \right ) \sum_{j=0}^{S-1}  \sum_{i=1}^k \mu_{jk + i}^2 \right )^2}{\sum_{j=0}^{S-1} \left ( c_1 \sum_{i=1}^k \mu_{jk + i}^2 \right )^2} \right \} \\
& =  \exp \left \{ - C \frac{ \left ( \sum_{j=0}^{S-1}  \sum_{i=1}^k \mu_{jk + i}^2 \right )^2}{\sum_{j=0}^{S-1} \left (\sum_{i=1}^k \mu_{jk + i}^2 \right )^2} \right \}
\end{split}
\end{align*}
where the final inequality follows from choosing the optimal $\lambda < 0$ (and assuming $c_2$ chosen such that $c_1 / 3 - c_2$ is positive) and the final equality uses $C = 2 (c_1 / 3 - c_2)^2 / c_1^2$. By our assumption on the power (\ref{eq:prop_power}), we will have that $\sum_{i=1}^k \mu_{jk + i}^2 = \alpha + \alpha_j$ for some $|\alpha_j| \leq \alpha / 10$. Thus:
\begin{align*}
\begin{split}
\exp \left \{ - C \frac{ \left ( \sum_{j=0}^{S-1}  \sum_{i=1}^k \mu_{jk + i}^2 \right )^2}{\sum_{j=0}^{S-1} \left (\sum_{i=1}^k \mu_{jk + i}^2 \right )^2} \right \} & = \exp \left \{ - C \frac{ \left ( \sum_{j=0}^{S-1}  \alpha + \alpha_j \right )^2}{\sum_{j=0}^{S-1} (\alpha+\alpha_j)^2 } \right \} \\
& \leq \exp \left \{ - C \frac{ \left ( \sum_{j=0}^{S-1}  9/10 \alpha \right )^2}{\sum_{j=0}^{S-1} (11/10)^2 \alpha^2 } \right \} \\
& = \exp \left \{ - C \frac{81}{121} S \right \}
\end{split}
\end{align*}
where the inequality holds from maximizing this expression over $\alpha_j$.

Recalling that $S = \lfloor T / k \rfloor$, we conclude that:
\begin{equation*}
\mathbb{P} \left [ \sum_{t=1}^{T  } z_t^2 \leq c_2 \sum_{t=1}^{ T} \mu_t^2 \right ] \leq \exp \left \{ -C \frac{81}{121} \lfloor T / k \rfloor \right \}
\end{equation*}
It remains then to write this in form of (\ref{eq:prop_cov_bound}). Plugging in our definitions of $\mu_t$ and $z_t$, we have that the above is equivalent to:
\begin{align*}
& \mathbb{P} \left [ \sum_{t=T_{ss}+1}^{T_{ss}+T} (w^\top x_{t})^2 \leq c_2 \sum_{t=1}^T \left (w^\top x_{T_{ss}+t}^u - w^\top \bar{x}_{T_{ss} + \lfloor t / k - 1 \rfloor k : k}^u \right )^2 \right ] \leq e^{ -C \frac{81}{121} \lfloor T / k \rfloor } \\
\stackrel{(a)}{\implies} \ & \mathbb{P} \left [ \sum_{t=T_{ss}+1}^{T_{ss}+T} (w^\top x_{t})^2 \leq \frac{c_2}{2} \sum_{t=1}^{\lfloor T / k \rfloor k} \left (w^\top x_{T_{ss}+t}^{u,ss} \right )^2 \right ] \leq e^{ -C \frac{81}{121} \lfloor T / k \rfloor } \\
\stackrel{(b)}{\iff} \ & \mathbb{P} \left [ \sum_{t=1}^T (w^\top x_t)^2 \leq \frac{c_2}{2} \lfloor T / k \rfloor k w^\top \tilde{\Gamma}_k^u w \right ] \leq e^{ -C \frac{81}{121} \lfloor T / k \rfloor }
\end{align*}
where $(a)$ holds by our assumption on the power (\ref{eq:prop_power}) and $(b)$ follows by Parseval's Theorem. Choosing $c_2$ to balance the constants, we get that:
\begin{align*}
& \mathbb{P} \left [ \sum_{t=T_{ss} + 1}^{T_{ss} + T} (w^\top x_t)^2 \leq \frac{2}{81} k \lfloor T / k \rfloor w^\top \tilde{\Gamma}_k^u w \right ] \leq e^{ -\frac{2}{81} \lfloor T / k \rfloor }
\end{align*}
which completes the proof.
\end{proof}

\begin{lemma}\label{lem:cov_lb_noise}
Assume that our system is driven by some input $u_t = \tilde{u}_t + \eta_t^u$ where $\tilde{u}_t$ is deterministic and $\eta_t^u \sim \mathcal{N}(0, \sigma_u^2 I)$. Then on the event that:
$$ \sum_{t=1}^T x_t x_t^\top \preceq T \bar{\Gamma}_T$$
for some $\bar{\Gamma}_T$, choosing $k$ so that:
\begin{equation}\label{eq:cov_lb_noise_burnin}
T \geq \frac{25600}{27} k \left (  2 d \log ( 200/3) + \log \det ( \bar{\Gamma}_T{ \Gamma_{k}^{\eta}}^{-1} ) + \log \frac{1}{\delta} \right )
\end{equation}
we will have with probability less than $\delta$:
$$ \sum_{t=1}^T x_t x_t^\top \not\succeq \frac{27}{25600} T \Gamma^\eta_{k} $$
\end{lemma}
\begin{proof}
Take some $s \geq 0$, then:
\begin{equation*}
w^\top x_{s+t} | \mathcal{F}_s \sim \mathcal{N} \left ( w^\top A_*^t x_s + w^\top x_{s+t}^{u}, \sigma^2 w^\top \Gamma_{t-s} w + \sigma_u^2 w^\top \Gamma_{t-s}^{B_*} w \right )
\end{equation*}
where $x_{s+t}^{u}$ is the state obtained by driving the system with the input in the absence of noise, which is deterministic conditioned on $\mathcal{F}_s$. Given this, we have that $x_{s+t}$ satisfies the $(2k, \sigma^2 \Gamma_{k} + \sigma_u^2 \Gamma_{k}^{B_*}, 3/20)$-BMSB condition, as defined in \cite{simchowitz2018learning}. The proof of this closely mirrors the proof of Proposition 3.1 of \cite{simchowitz2018learning}. The primary difference is that the mean of $w^\top x_{s+t} | \mathcal{F}_s$ differs from that of the signal considered in \cite{simchowitz2018learning}, but this does not affect the argument and, as such, we omit it here. We can then apply Proposition 2.5 of \cite{simchowitz2018learning}  to get that:
$$ \mathbb{P} \left [ \sum_{t=1}^T (w^\top x_t)^2 \leq \frac{k \lfloor T / k \rfloor p^2 w^\top \Gamma^\eta_{k} w}{8} \right ] \leq e^{- \frac{\lfloor T / k \rfloor p^2}{16} }$$
where here $p = 3/20$. Following the proof of Theorem 2.4 of \cite{simchowitz2018learning}, let $\mathcal{T}$ be a 1/4-net in the norm $T \bar{\Gamma}_T$ of $ \left \{ w \ : \  k \lfloor T / k \rfloor  p^2 w^\top \Gamma_k^\eta w / 8 = 1\right \}$. By Lemma 4.1 of \cite{simchowitz2018learning}, $| \mathcal{T} | \leq 2d \log (10/p) + \log \det (\bar{\Gamma}_T {\Gamma^\eta_{k}}^{-1})$. Then by Lemma 4.1 of \cite{simchowitz2018learning} we have:
\begin{align*}
& \mathbb{P} \left [ \sum_{t=1}^T x_t x_t^\top \not \succeq \frac{k \lfloor T / k \rfloor p^2  \Gamma^\eta_{k} }{16},   \sum_{t=1}^T x_t x_t^\top \preceq T \bar{\Gamma}_T \right ] \\
\leq \ & \mathbb{P} \left [ \exists w \in \mathcal{T} \ : \ \sum_{t=1}^T (w^\top x_t)^2 < \frac{k \lfloor T / k \rfloor p^2 w^\top \Gamma^\eta_{k} w}{8},  \sum_{t=1}^T x_t x_t^\top \preceq T \bar{\Gamma}_T \right ] \\
\leq \ & \exp \left ( -\frac{\lfloor T / k \rfloor p^2}{16} + 2d \log (10/p) + \log \det (\bar{\Gamma}_T {\Gamma^\eta_{k}}^{-1}) \right ) \\
\stackrel{(a)}{\leq} \ & \exp \left ( -\frac{ 3 T  p^2}{64k} + 2d \log (10/p) + \log \det (\bar{\Gamma}_T {\Gamma^\eta_{k}}^{-1}) \right ) \\
\stackrel{(b)}{\leq} \ & \delta
\end{align*}
where $(a)$ holds if $T \geq 4k$, which is true by (\ref{eq:cov_lb_noise_burnin}), and $(b)$ holds by (\ref{eq:cov_lb_noise_burnin}). Lower bounding $\frac{k \lfloor T / k \rfloor p^2  \Gamma^\eta_{k} }{16}  \succeq \frac{ 3 p^2 T  \Gamma^\eta_{k} }{64}$ and plugging in $p=3/20$ completes the result.
\end{proof}

\begin{lemma}\label{lem:cov_lb_inputs}
Let $u_t$ be a deterministic input with period $k$ and let $T_{ss}$ be the time such that condition (\ref{eq:prop_power}) in Proposition \ref{prop:covbound} is met for all $w \in \mathcal{S}^{d-1}$. On the event that:
$$ \sum_{t=1}^T x_t x_t^\top \preceq T \bar{\Gamma}_T$$
for some $\bar{\Gamma}_T$, then as long as:
\begin{equation}\label{eq:cov_lb_input_burnin}
T \geq 2T_{ss} + 54 k \left ( 2 d \log ( 45 / 2) + \log \det ( \bar{\Gamma}_T (\tilde{\Gamma}_{k}^{u})^{-1} ) + \log \frac{1}{\delta} \right )
\end{equation}
with probability less than $\delta$:
$$ \sum_{t=1}^T x_t x_t^\top \not\succeq \frac{1}{108}  T  \tilde{\Gamma}_{k}^{u}$$
\end{lemma}
\begin{proof}
The proof of this follows \cite{simchowitz2018learning} closely but replacing Proposition 2.5 of \cite{simchowitz2018learning} with our Proposition \ref{prop:covbound}.

By Proposition \ref{prop:covbound} we will have that:
\begin{align*}
\mathbb{P} \left [ \sum_{t=1}^T (w^\top x_t)^2 \leq \frac{2}{81} k \lfloor (T - T_{ss}) / k \rfloor w^\top \tilde{\Gamma}_{k}^{u} w \right ] &  \leq  \mathbb{P} \left [ \sum_{t=T_{ss}}^T (w^\top x_t)^2 \leq \frac{2}{81} k \lfloor (T - T_{ss}) / k \rfloor w^\top \tilde{\Gamma}_{k}^{u} w \right ] \\
& \leq e^{-\frac{2}{81} \lfloor (T - T_{ss}) / k \rfloor}
\end{align*}
Following the proof of Theorem 2.4 of \cite{simchowitz2018learning}, let $\mathcal{T}$ be a 1/4-net in the norm $T \bar{\Gamma}_T$ of $ \left \{ w \ : \ 2 k \lfloor T / k \rfloor w^\top \tilde{\Gamma}_k^u w / 81 = 1\right \}$. By Lemma D.1 of \cite{simchowitz2018learning}, we have that $| \mathcal{T} | \leq 2d \log (45/2) + \log \det(\bar{\Gamma}_T (\tilde{\Gamma}_k^u)^{-1})$. Then by Lemma 4.1 of \cite{simchowitz2018learning} we have:
\begin{align*}
& \mathbb{P} \left [ \sum_{t=1}^T x_t x_t^\top \not \succeq \frac{k \lfloor (T - T_{ss}) / k \rfloor \tilde{\Gamma}_k^u}{81}, \sum_{t=1}^T x_t x_t^\top \preceq T \bar{\Gamma}_T \right ] \\
\leq \ & \mathbb{P} \left [ \exists w \in \mathcal{T} \ : \ \sum_{t=1}^T (w^\top x_t)^2 < \frac{2k \lfloor (T - T_{ss}) / k \rfloor w^\top \tilde{\Gamma}_k^u w}{81}, \sum_{t=1}^T x_t x_t^\top \preceq T \bar{\Gamma}_T  \right ] \\
\leq \ & \mathbb{P} \left [ \exists w \in \mathcal{T} \ : \ \sum_{t=T_{ss}}^T (w^\top x_t)^2 < \frac{2 k \lfloor (T - T_{ss}) / k \rfloor w^\top \tilde{\Gamma}_k^u w}{81}, \sum_{t=1}^T x_t x_t^\top \preceq T \bar{\Gamma}_T  \right ] \\
\leq \ & \exp \left ( - \frac{2 \lfloor (T - T_{ss}) / k \rfloor}{81} + 2d \log (45/2) + \log \det(\bar{\Gamma}_T (\tilde{\Gamma}_k^u)^{-1}) \right ) \\
\stackrel{(a)}{\leq} \ & \exp \left ( - \frac{ (T - T_{ss})}{54k} + 2d \log (45/2) + \log \det(\bar{\Gamma}_T (\tilde{\Gamma}_k^u)^{-1}) \right ) \\
\stackrel{(b)}{\leq} \ & \delta
\end{align*}
where $(a)$ holds so long as $T \geq T_{ss} + 4k$, which will be true by (\ref{eq:cov_lb_input_burnin}),  and $(b)$ holds by (\ref{eq:cov_lb_input_burnin}). The following holds by $T \geq T_{ss} + 4k$ and by (\ref{eq:cov_lb_input_burnin}):
\begin{align*}
 \frac{k \lfloor (T - T_{ss}) / k \rfloor \tilde{\Gamma}_{k}^u}{81}  \succeq \frac{ (T - T_{ss}) \tilde{\Gamma}_{k}^u}{54}  \succeq \frac{T \tilde{\Gamma}_k^u}{108}
\end{align*}
which completes the result.
\end{proof}

\begin{corollary}\label{cor:cov_lb_inputs_burnin}
Let:
$$ \mathcal{W} := \left \{ w \in \mathcal{S}^{d-1} \ : \  \frac{1}{108} T w^\top \tilde{\Gamma}_k^u w  \geq \frac{1}{2} w^\top M w \right \} $$
where $M \succeq 0$. Let $u_t$ be a deterministic input with period $k$ and let $T_{ss}$ be the time such that condition (\ref{eq:prop_power}) in Proposition \ref{prop:covbound} is met for all $w \in \mathcal{W}$. On the event that:
$$ \sum_{t=1}^T x_t x_t^\top \preceq T \bar{\Gamma}_T$$
for some $\bar{\Gamma}_T$, then as long as:
\begin{equation}\label{eq:cov_lb_input_burnin2}
T \geq 2T_{ss} + 54 k \left ( 2 d \log ( 45 / 2) + \log \det ( (\bar{\Gamma}_T + \frac{1}{T} M) (\tilde{\Gamma}_{k}^{u})^{-1} ) + \log \frac{1}{\delta} \right )
\end{equation}
with probability less than $\delta$:
$$ \sum_{t=1}^T x_t x_t^\top + M \not\succeq \frac{1}{108}  T  \tilde{\Gamma}_{k}^{u} + \frac{1}{2}  M$$
\end{corollary}
\begin{proof}
The proof of this result is very similar to that of Lemma \ref{lem:cov_lb_inputs}. For any $w \in \mathcal{S}^{d-1} \cap \mathcal{W}^c$:
\begin{equation}\label{eq:cov_lb_w_burnin}
w^\top \left ( \sum_{t=1}^T x_t x_t^\top + M \right ) w \geq w^\top M w \geq  \frac{1}{108} T w^\top \tilde{\Gamma}_k^u w  + \frac{1}{2} w^\top M w
\end{equation}
For any $w \in \mathcal{W}$, by Proposition \ref{prop:covbound}, given the definition of $T_{ss}$, we will have that:
\begin{align*}
\mathbb{P} \left [ \sum_{t=1}^T (w^\top x_t)^2 \leq \frac{2}{81} k \lfloor (T - T_{ss}) / k \rfloor w^\top \tilde{\Gamma}_{k}^{u} w \right ] &  \leq  \mathbb{P} \left [ \sum_{t=T_{ss}}^T (w^\top x_t)^2 \leq \frac{2}{81} k \lfloor (T - T_{ss}) / k \rfloor w^\top \tilde{\Gamma}_{k}^{u} w \right ] \\
& \leq e^{-\frac{2}{81} \lfloor T / k \rfloor}
\end{align*}
Following the proof of Theorem 2.4 of \cite{simchowitz2018learning}, let $\mathcal{T}$ be a 1/4-net in the norm $T \bar{\Gamma}_T + M$ of $ \left \{ w \ : \ 2 k \lfloor T / k \rfloor w^\top \tilde{\Gamma}_k^u w / 81 = 1\right \}$. By Lemma D.1 of \cite{simchowitz2018learning}, $| \mathcal{T} | \leq 2d \log (45/2) + \log \det((\bar{\Gamma}_T  + \frac{1}{T} M) (\tilde{\Gamma}_k^u)^{-1})$. Then by Lemma 4.1of \cite{simchowitz2018learning} we have:
\begin{align*}
& \mathbb{P} \left [ \sum_{t=1}^T x_t x_t^\top + M \not \succeq \frac{1}{108} T \tilde{\Gamma}_k^u + \frac{1}{2} M,  \sum_{t=1}^T x_t x_t^\top \preceq T \bar{\Gamma}_T  \right ] \\
\stackrel{(a)}{\leq} \ & \mathbb{P} \left [ \exists w \in \mathcal{T} \cap \mathcal{W} \ : \ \sum_{t=1}^T (w^\top x_t)^2 < \frac{2}{108} T w^\top \tilde{\Gamma}_k^u w  - w^\top M w,  \sum_{t=1}^T x_t x_t^\top \preceq T \bar{\Gamma}_T \right ] \\
\stackrel{(b)}{\leq} \ & \mathbb{P} \left [ \exists w \in \mathcal{T} \cap \mathcal{W} \ : \ \sum_{t=1}^T (w^\top x_t)^2 < \frac{2 k \lfloor (T - T_{ss}) / k \rfloor w^\top \tilde{\Gamma}_k^u w}{81} - w^\top M w,  \sum_{t=1}^T x_t x_t^\top \preceq T \bar{\Gamma}_T \right ] \\
\leq \ & \mathbb{P} \left [ \exists w \in \mathcal{T} \cap \mathcal{W} \ : \ \sum_{t=T_{ss}}^T (w^\top x_t)^2 < \frac{2 k \lfloor (T - T_{ss}) / k \rfloor w^\top \tilde{\Gamma}_k^u w}{81},  \sum_{t=1}^T x_t x_t^\top \preceq T \bar{\Gamma}_T  \right ] \\
\leq \ & \exp \left ( - \frac{2 \lfloor (T - T_{ss}) / k \rfloor}{81} + 2d \log (45/2) + \log \det((\bar{\Gamma}_T  + \frac{1}{T} M) (\tilde{\Gamma}_k^u)^{-1}) \right ) \\
\leq \ & \exp \left ( - \frac{ (T - T_{ss})}{54k} + 2d \log (45/2) + \log \det((\bar{\Gamma}_T + \frac{1}{T} M) (\tilde{\Gamma}_k^u)^{-1}) \right ) \\
\leq \ & \delta
\end{align*}
where $(a)$ holds by (\ref{eq:cov_lb_w_burnin}) and $(b)$ and the final inequalities hold so long as $T \geq T_{ss} + 4k$ and (\ref{eq:cov_lb_input_burnin2}) holds, since in that case we will have that $k \lfloor (T - T_{ss}) / k \rfloor / 81 \geq (T-T_{ss})/54 \geq T / 108$. 
\end{proof}

\begin{lemma}\label{lem:self_normalized_past_data2}\com{[Done]}
Assume that $x_t$ is generated from some input $u_t = \tilde{u}_t + \eta_t^u$ where $\tilde{u}_t$ is $\mathcal{F}_{t-1}$ measurable and $\eta_t^u \sim \mathcal{N}(0, \sigma_u^2 I)$. On the event that $V_+ \succeq \sum_{t=1}^T x_t x_t^\top \succeq V_-$, we will have that, with probability less than $\delta$:
\begin{equation*}
\left \| \left ( \sum_{t=1}^T x_t x_t^\top \right )^{-1/2} \sum_{t=1}^T x_t \eta_t^\top \right \|_2 > \sigma \sqrt{16 \log \frac{1}{\delta} + 8 \log \det (V_+ V_{-}^{-1} + I) + 16 d \log 5} 
\end{equation*}
\end{lemma}
\begin{proof}
Note that Proposition 8.2 of \cite{sarkar2018how} applies even when $x_t$ is driven by an input $\tilde{u}_t$ which is changing over time, since we choose $\tilde{u}_t$ to be $\mathcal{F}_{t-1}$ measurable, so $x_t$ is still $\mathcal{F}_{t-1}$ measurable. Therefore,  for any deterministic $V \succ 0$:
$$ \left \| \left ( \sum_{t=1}^T x_t x_t^\top + V \right )^{-1/2} \sum_{t=1}^T x_t \eta_t^\top \right \|_2 > \sigma \sqrt{8 \log \frac{1}{\delta} + 4 \log \det \left ( \left (\sum_{t=1}^T x_t x_t^\top \right ) V^{-1} + I \right ) + 8 d \log 5 }$$
with probability less than $\delta$.

On the event that $V_+ \succeq \sum_{t=1}^T x_t x_t^\top \succeq V_-$, we will have:
$$ \sum_{t=1}^T x_t x_t^\top + V_- \preceq 2  \sum_{t=1}^T x_t x_t^\top \implies \left ( \sum_{t=1}^T x_t x_t^\top + V_-  \right )^{-1} \succeq \frac{1}{2} \left ( \sum_{t=1}^T x_t x_t^\top \right )^{-1}$$
Choosing $V = V_-$ and using this inequality gives:
\begin{equation*}
\left \| \left ( \sum_{t=1}^T x_t x_t^\top \right )^{-1/2} \sum_{t=1}^T x_t \eta_t^\top \right \|_2 > \sigma \sqrt{16 \log \frac{1}{\delta} + 8 \log \det (V_+ V_{-}^{-1} + I) + 16 d \log 5} 
\end{equation*}
with probability less than $\delta$. 
\end{proof}

\subsection{Upper Bounds on Covariates}

\begin{lemma}\label{lem:upper_bound}\com{[Done]}
Assume $u_t = \tilde{u}_t + \eta_t^u$ for some deterministic $\tilde{u}_t$, $\eta_t^u \sim \mathcal{N}(0,\sigma_u^2 I)$, and for any initial state, then with probability at least $1-\delta$:
\begin{equation*}
\sum_{t=1}^T x_t x_t^\top \preceq 4 T \left ( \tilde{\Gamma}_{T,0}^u +  tr \left (\sigma^2 \Gamma_T + \sigma_u^2 \Gamma_T^{B_*} \right ) \left ( 1 +  \log \frac{2}{\delta} \right ) I  \right )
\end{equation*}
and for any $w$, with probability at least $1 - \delta$:
\begin{equation*}
\sum_{t=1}^T (w^\top x_t)^2 \leq 4 T w^\top \left ( \tilde{\Gamma}_{T,0}^u  +   \left (\sigma^2  \Gamma_T  + \sigma_u^2  \Gamma_T^{B_*}  \right )  \left ( 1 +  \log \frac{2}{\delta} \right ) \right ) w
\end{equation*}
\end{lemma}
\begin{proof}
We note that:
\begin{align*}
\sum_{t=1}^T x_t x_t^\top & = \sum_{t=1}^T \left ( x_t^{\tilde{u}} + x_t^{\eta^u} + x_t^\eta \right )  \left ( x_t^{\tilde{u}} + x_t^{\eta^u} + x_t^\eta \right )^\top   \\
& \overset{(a.s.)}{\preceq} 4 \sum_{t=1}^T \left [ x_t^{\tilde{u}} {x_t^{\tilde{u}}}^\top + x_t^{\eta^u} {x_t^{\eta^u}}^\top + x_t^\eta {x_t^\eta}^\top  \right ]
\end{align*}
Where here we let $x_t^{\tilde{u}}$ denote the response of the system to the deterministic part of the input and $x_t^{\eta^u}$ the response due to the random part of the input. The term $\sum_{t=1}^T x_t^{\tilde{u}} {x_t^{\tilde{u}}}^\top$ is then deterministic. Following Proposition 8.4 of \cite{sarkar2018how}, we can bound the second and third terms each with probability $1-\delta/2$ as:
\begin{equation*}
\left \| \sum_{t=1}^T x_t^{\eta^u} {x_t^{\eta^u}}^\top \right \|_2 \leq \sigma_u^2 tr \left (\sum_{t=0}^{T-1} \Gamma_t^{B_*} \right ) \left ( 1 + \log \frac{2}{\delta} \right ) \leq T \sigma_u^2 tr \left (\Gamma_T^{B_*} \right ) \left ( 1 +  \log \frac{2}{\delta} \right ) 
\end{equation*}
\begin{equation*}
\left \| \sum_{t=1}^T x_t^\eta {x_t^\eta}^\top \right \|_2 \leq \sigma^2 tr \left (\sum_{t=0}^{T-1} \Gamma_t \right ) \left ( 1 + \log \frac{2}{\delta} \right ) \leq T \sigma^2 tr \left (\Gamma_T \right ) \left ( 1 +  \log \frac{2}{\delta} \right ) 
\end{equation*}
Combining these bounds gives the result.

For the second inequality, following the same argument as in the proof of Proposition 8.4 of \cite{sarkar2018how}, we obtain:
\begin{equation*}
\sum_{t=1}^T (w^\top x_t^\eta)^2 \leq \left ( 1 + \log \frac{1}{\delta} \right ) \sum_{t=1}^T w^\top \Gamma_t w
\end{equation*}
combining this with the above gives the result.
\end{proof}

\begin{lemma}\label{lem:upper_bound3}\com{[Done]}
Assume that the input $u_t$ satisfies, for some $k$ and any $s \geq 0$:
\begin{equation*}
\frac{1}{k} \sum_{t=1}^k u_{s+t}^\top u_{s+t} \leq \gamma^2
\end{equation*}
then:
\begin{align*}
\sum_{t=1}^T x_t^u {x_t^u}^\top & \preceq \frac{1}{T} \sum_{t=1}^T G(e^{j\theta_t}) U(e^{j\theta_t}) U(e^{j\theta_t})^H G(e^{j\theta_t})^H + \frac{ 4 \beta(A_*)^2 k^2 \gamma^2}{(1 - \bar{\rho}(A_*)^k)^2} \left ( \max_{\theta \in [0, 2\pi]} \| G(e^{j\theta}) \|_2^2 \right ) I \\
& \ \ \ \ \ \ \ \ \ \ \ \  + \frac{ 4 \beta(A_*) k \gamma^2 \sqrt{T}}{1 - \bar{\rho}(A_*)^k} \left ( \max_{\theta \in [0, 2\pi]} \| G(e^{j\theta}) \|_2^2 \right ) I \\
\end{align*}
\end{lemma}
\begin{proof}
Denote $\theta_t = \frac{2 \pi t}{T}$. Then:
\begin{align*}
& \left \| \sum_{t=1}^T x_t^u {x_t^u}^\top - \frac{1}{T} \sum_{t=1}^T G(e^{j\theta_t}) U(e^{j\theta_t}) U(e^{j\theta_t})^H G(e^{j\theta_t})^H \right \|_2 \\
= \ & \left \| \frac{1}{T} \sum_{t=1}^T X(e^{j\theta_t}) X(e^{j\theta_t})^H - \frac{1}{T} \sum_{t=1}^T G(e^{j\theta_t}) U(e^{j\theta_t}) U(e^{j\theta_t})^H G(e^{j\theta_t})^H \right \|_2 \\
= \ & \frac{1}{T} \left \| \sum_{t=1}^T \left [ \Big (X(e^{j\theta_t}) - G(e^{j\theta_t}) U(e^{j\theta_t}) \Big )  U(e^{j\theta_t})^H G(e^{j\theta_t})^H + G(e^{j\theta_t}) U(e^{j\theta_t}) \Big (X(e^{j\theta_t}) - G(e^{j\theta_t}) U(e^{j\theta_t}) \Big )^H \right . \right . \\
& \ \ \ \ \ \ \ \ \ \ \left . \left . + \Big (X(e^{j\theta_t}) - G(e^{j\theta_t}) U(e^{j\theta_t}) \Big ) \Big (X(e^{j\theta_t}) - G(e^{j\theta_t}) U(e^{j\theta_t}) \Big )^H  \right ] \right \|_2 \\
\leq \ & \frac{1}{T} \sum_{t=1}^T \| X(e^{j\theta_t}) - G(e^{j\theta_t}) U(e^{j\theta_t}) \|_2^2 + \frac{2}{T} \sum_{t=1}^T \| X(e^{j\theta_t}) - G(e^{j\theta_t}) U(e^{j\theta_t}) \|_2 \| G(e^{j\theta_t}) U(e^{j\theta_t}) \|_2 \\
\overset{(a)}{\leq} \ & \frac{4}{T} \sum_{t=1}^T \frac{\| G(e^{j\theta_t}) \|_2^2  \beta(A_*)^2 k^2 \gamma^2}{(1 - \bar{\rho}(A_*)^k)^2} + \frac{4}{T} \sum_{t=1}^T \frac{\| G(e^{j\theta_t}) \|_2 \beta(A_*) k \gamma}{1 - \bar{\rho}(A_*)^k} \| G(e^{j\theta_t}) U(e^{j\theta_t}) \|_2 \\
\leq \ & \frac{ 4 \beta(A_*)^2 k^2 \gamma^2}{(1 - \bar{\rho}(A_*)^k)^2} \left ( \max_{\theta \in [0, 2\pi]} \| G(e^{j\theta}) \|_2^2 \right ) + \frac{ 4 \beta(A_*) k \gamma}{1 - \bar{\rho}(A_*)^k} \left ( \max_{\theta \in [0, 2\pi]} \| G(e^{j\theta}) \|_2^2 \right ) \frac{1}{T} \sum_{t=1}^T \| U(e^{j\theta_t}) \|_2
\end{align*}
where $(a)$ uses Lemma \ref{lem:freq_approx}. Since $\| x \|_1 \leq \sqrt{n} \| x \|_2$ for any $x \in \mathbb{R}^n$, we will have, by Parseval's Theorem and our assumption on $u_t$:
\begin{equation*}
\sum_{t=1}^T \| U(e^{j\theta_t}) \|_2 \leq \sqrt{T} \sqrt{\sum_{t=1}^T \| U(e^{j\theta_t}) \|_2^2} = \sqrt{T} \sqrt{T \sum_{t=1}^T u_t^\top u_t} \leq \sqrt{T} \sqrt{T^2 \gamma^2} = T^{3/2} \gamma
\end{equation*}
so:
\begin{align*}
& \frac{  \beta(A_*)^2 k^2 \gamma^2}{(1 - \bar{\rho}(A_*)^k)^2} \left ( \max_{\theta \in [0, 2\pi]} \| G(e^{j\theta}) \|_2^2 \right ) + \frac{ 4 \beta(A_*) k \gamma}{1 - \bar{\rho}(A_*)^k} \left ( \max_{\theta \in [0, 2\pi]} \| G(e^{j\theta}) \|_2^2 \right ) \frac{1}{T} \sum_{t=1}^T \| U(e^{j\theta_t}) \|_2 \\
\leq \ & \frac{4 \beta(A_*)^2 k^2 \gamma^2}{(1 - \bar{\rho}(A_*)^k)^2} \left ( \max_{\theta \in [0, 2\pi]} \| G(e^{j\theta}) \|_2^2 \right ) + \frac{4 \beta(A_*) k \gamma^2 \sqrt{T}}{1 - \bar{\rho}(A_*)^k} \left ( \max_{\theta \in [0, 2\pi]} \| G(e^{j\theta}) \|_2^2 \right )
\end{align*}
Thus:
\begin{align*}
\sum_{t=1}^T x_t^u {x_t^u}^\top & \preceq \frac{1}{T} \sum_{t=1}^T G(e^{j\theta_t}) U(e^{j\theta_t}) U(e^{j\theta_t})^H G(e^{j\theta_t})^H \\
& \ \ \ \ \ \ \ \ \ \ \ \  + \left \| \sum_{t=1}^T x_t^u {x_t^u}^\top - \frac{1}{T} \sum_{t=1}^T G(e^{j\theta_t}) U(e^{j\theta_t}) U(e^{j\theta_t})^H G(e^{j\theta_t})^H \right \|_2 I \\
& \preceq \frac{1}{T} \sum_{t=1}^T G(e^{j\theta_t}) U(e^{j\theta_t}) U(e^{j\theta_t})^H G(e^{j\theta_t})^H + \frac{4 \beta(A_*)^2 k^2 \gamma^2}{(1 - \bar{\rho}(A_*)^k)^2} \left ( \max_{\theta \in [0, 2\pi]} \| G(e^{j\theta}) \|_2^2 \right ) I \\
& \ \ \ \ \ \ \ \ \ \ \ \  + \frac{4 \beta(A_*) k \gamma^2 \sqrt{T}}{1 - \bar{\rho}(A_*)^k} \left ( \max_{\theta \in [0, 2\pi]} \| G(e^{j\theta}) \|_2^2 \right ) I
\end{align*}
\end{proof}

\begin{lemma}\label{lem:xtu_upper_bound}\com{[Done]}
Assume that we are running Algorithm \ref{alg:active_lds_noise} and that we started from initial condition $x_0 = 0$. Let $A_* = PJP^{-1}$ be the Jordan decomposition of $A_*$ and consider some $w \in \mathcal{S}^{d-1}$ such that $ \| w^\top P_{\underline{n}(j): \overline{n}(j)} \|_2 = 0$ except for $j = \ell$. Here $\underline{n}(j)$ and $\overline{n}(j)$ denote the start and stop indices of the $j$th Jordan block (so in particular, if $J_j$ is the $j$th Jordan block, we have that $J_j = [J]_{\underline{n}(j):\overline{n}(j),\underline{n}(j):\overline{n}(j)}$). Assume that $T$ is chosen to be within epoch $i$. Then, after $T$ steps:
\begin{equation*}
\sum_{t=1}^T (w^\top x_t^u)^2 \leq 3 T_i  \gamma^2 \max_{\theta \in [0,2\pi]} \| w^\top (e^{j \theta} I - A_*)^{-1} \|_2^2 \| B_* \|_2^2 + 16 \frac{\| P^{-1} \|_2^2 \| P \|_2^2 \| B_* \|_2^2 \beta(J_\ell)^2 \gamma^2 k_i}{(1 - \bar{\rho}(J_\ell)^{2k_0})(1 - \bar{\rho}(J_\ell)^{2})} 
\end{equation*}
\end{lemma}
\begin{proof}
Adopting the notation used in Algorithm \ref{alg:active_lds_noise}, let $T_i$ denote the length of the $i$th epoch. Denote $\bar{T}_i = \sum_{j=0}^{t-1} T_j$ be the start time of the $i$th epoch.

Following the analysis used in Section \ref{sec:transients}, we can break up the response into its steady state and transient components and write:
\begin{equation*}
x_{\bar{T}_i + t}^u = x_t^{ss_i} + A_*^t (x_{\bar{T}_i}^u - x_0^{ss_i}) 
\end{equation*}
for $t \in [\bar{T}_i + 1, \bar{T}_i + T_i]$, where $x_t^{ss_i}$ denotes the steady state response of the system at time $t$ to the inputs used at epoch $i$. We then have:
\begin{align*}
\sum_{t=1}^{T_i} (w^\top x_{\bar{T}_i + t}^u)^2 & = \sum_{t=1}^{T_i} (w^\top x_t^{ss_i} + w^\top A_*^t (x_{\bar{T}_i}^u - x_0^{ss_i}))^2 \\
& \leq 2 \sum_{t=1}^{T_i} (w^\top x_t^{ss_i})^2 + 2 \sum_{t=1}^{T_i} ( w^\top A_*^t (x_{\bar{T}_i}^u - x_0^{ss_i}))^2 \\
& \leq 2 \sum_{t=1}^{T_i} (w^\top x_t^{ss_i})^2 + 4 \sum_{t=1}^{T_i} ( w^\top A_*^t x_{\bar{T}_i}^u )^2 + 4 \sum_{t=1}^{T_i} ( w^\top A_*^t x_0^{ss_i} )^2
\end{align*}
Note that:
\begin{equation*}
x_{\bar{T}_i}^u = \sum_{s=1}^{\bar{T}_i - 1} A_*^{\bar{T}_i - s -1} B_* u_s, \ \ \ \ x_0^{ss_i} = \sum_{s = -\infty}^{-1} A_*^{-s-1} B_* u_s
\end{equation*}
where, relying on the periodicity of $u_s$, we let $u_s = u_{s \% k_i + k_i}$ for negative $s$. So:
\begin{align*}
( w^\top A_*^t x_0^{ss_i} )^2 & = \left ( \sum_{s = -\infty}^{-1} w^\top A_*^{t -s-1} B_* u_s \right )^2  \leq \sum_{s=-\infty}^{-1} \| w^\top A_*^{t-s-1} B_* u_s \|_2^2 \\
& \leq \sum_{s=-\infty}^{-1} \| w^\top P J^{t-s-1} \|_2^2 \| P^{-1} B_*\|_2^2 \| u_s \|_2^2 \\
& \leq \sum_{s=-\infty}^{-1} \| w^\top P_{\underline{n}(\ell): \overline{n}(\ell)} \|_2^2 \| J_{\ell}^{t-s-1} \|_2^2 \| P^{-1} B_*\|_2^2 \| u_s \|_2^2 \\
& \leq \| P \|_2^2 \| P^{-1} \|_2^2 \| B_* \|_2^2 \beta(J_\ell)^2 \sum_{s=-\infty}^{-1} \bar{\rho}(J_\ell)^{2t-2s-2} \| u_s \|_2^2 \\
& \leq \| P \|_2^2 \| P^{-1} \|_2^2 \| B_* \|_2^2 \beta(J_\ell)^2 \bar{\rho}(J_\ell)^{2t} \sum_{j=0}^\infty \bar{\rho}(J_\ell)^{2k_i j} \sum_{s = 0}^{k_i - 1} \| u_s \|_2^2 \\
& \leq \| P \|_2^2 \| P^{-1} \|_2^2 \| B_* \|_2^2 \beta(J_\ell)^2 k_i \gamma^2 \bar{\rho}(J_\ell)^{2t} \sum_{j=0}^\infty \bar{\rho}(J_\ell)^{2k_i j} \\
& = \| P \|_2^2 \| P^{-1} \|_2^2 \| B_* \|_2^2 \beta(J_\ell)^2 k_i \gamma^2 \frac{\bar{\rho}(J_\ell)^{2t}}{1 - \bar{\rho}(J_\ell)^{2k_i}}
\end{align*}
Repeating this calculation:
\begin{align*}
( w^\top A_*^t x_{\bar{T}_i}^u )^2 & = \left ( \sum_{s=1}^{\bar{T}_i - 1} w^\top A_*^{\bar{T}_i + t - s -1} B_* u_s \right )^2 \\
& \leq \|P^{-1}\|_2^2 \| B_*\|_2^2 \sum_{s=1}^{\bar{T}_i - 1} \| w^\top P J^{\bar{T}_i + t - s -1}\|_2^2  \|u_s \|_2^2 \\
& \leq \|P^{-1}\|_2^2 \| B_*\|_2^2 \sum_{s=1}^{\bar{T}_i - 1} \| w^\top P_{\underline{n}(\ell): \overline{n}(\ell)} \|_2^2 \| J_{\ell}^{\bar{T}_i + t - s -1} \|_2^2 \| u_s \|_2^2 \\
& \leq \|P^{-1}\|_2^2 \| P \|_2^2 \| B_*\|_2^2 \beta(J_\ell)^2 \sum_{s=1}^{\bar{T}_i - 1}  \bar{\rho}(J_\ell)^{2\bar{T}_i + 2t - 2s -2} \| u_s \|_2^2 \\
& \leq \|P^{-1}\|_2^2 \| P \|_2^2 \| B_*\|_2^2 \beta(J_\ell)^2 \sum_{j = 0}^{i-1} \sum_{z=0}^{\frac{T_j}{k_j}-1}  \bar{\rho}(J_\ell)^{2\bar{T}_i + 2t - 2 \bar{T}_j - 2 k_j z -2} \sum_{s=0}^{k_j}   \| u_s \|_2^2 \\
& \leq \|P^{-1}\|_2^2 \| P \|_2^2 \| B_*\|_2^2 \beta(J_\ell)^2 \gamma^2 \sum_{j = 0}^{i-1} \sum_{z=0}^{\frac{T_j}{k_j}-1} k_j \bar{\rho}(J_\ell)^{2\bar{T}_i + 2t - 2 \bar{T}_j - 2 k_j z -2} \\
& = \|P^{-1}\|_2^2 \| P \|_2^2 \| B_*\|_2^2 \beta(J_\ell)^2 \gamma^2 \bar{\rho}(J_\ell)^{2t} \sum_{j = 0}^{i-1} k_j \bar{\rho}(J_\ell)^{2\bar{T}_i - 2 \bar{T}_j + 2 k_j - 2 T_j -2} \frac{1 - \bar{\rho}(J_\ell)^{2 T_j}}{1 - \bar{\rho}(J_\ell)^{2 k_j}} \\
& = \|P^{-1}\|_2^2 \| P \|_2^2 \| B_*\|_2^2 \beta(J_\ell)^2 \gamma^2 \frac{\bar{\rho}(J_\ell)^{2t}}{1- \bar{\rho}(J_\ell)^{2 k_0}} \sum_{j = 0}^{i-1} k_j \\
& \leq \|P^{-1}\|_2^2 \| P \|_2^2 \| B_*\|_2^2 \beta(J_\ell)^2 \gamma^2 \frac{\bar{\rho}(J_\ell)^{2t}}{1- \bar{\rho}(J_\ell)^{2 k_0}} k_i
\end{align*}
where the last inequality follows since $k_j = 2 k_{j-1}$. Therefore:
\begin{align*}
\sum_{t=1}^{T_i} ( w^\top A_*^t x_0^{ss_i} )^2 & \leq \sum_{t=1}^{T_i} \| P \|_2^2 \| P^{-1} \|_2^2 \| B_* \|_2^2 \beta(J_\ell)^2 k_i \gamma^2 \frac{\bar{\rho}(J_\ell)^{2t}}{1 - \bar{\rho}(J_\ell)^{2k_i}} \\
& \leq  \frac{\| P \|_2^2 \| P^{-1} \|_2^2 \| B_* \|_2^2 \beta(J_\ell)^2 \gamma^2 k_i}{(1 - \bar{\rho}(J_\ell)^{2k_i})(1 - \bar{\rho}(J_\ell)^{2})}
\end{align*}
and:
\begin{align*}
 \sum_{t=1}^{T_i} ( w^\top A_*^t x_{\bar{T}_i}^u )^2 & \leq  \sum_{t=1}^{T_i} \|P^{-1}\|_2^2 \| P \|_2^2 \| B_*\|_2^2 \beta(J_\ell)^2 \gamma^2 \frac{\bar{\rho}(J_\ell)^{2t}}{1- \bar{\rho}(J_\ell)^{2 k_0}} k_i \\
 & \leq \frac{\|P^{-1}\|_2^2 \| P \|_2^2 \| B_*\|_2^2 \beta(J_\ell)^2 \gamma^2 k_i}{(1- \bar{\rho}(J_\ell)^{2 k_0})(1 - \bar{\rho}(J_\ell)^2)}
\end{align*}
Finally, by Parseval's Theorem:
\begin{align*}
\sum_{t=1}^{T_i} (w^\top x_t^{ss_i})^2 & = \frac{T_i}{k_i^2} \sum_{\theta \in \mathcal{I}_i} \| w^\top (e^{j \theta} I - A_*)^{-1} B_* U(e^{j \theta}) \|_2^2 \\
& \leq T_i \gamma^2 \max_{\ell \in \mathcal{I}_i} \| w^\top (e^{j \theta_\ell} I - A_*)^{-1} \|_2^2 \| B_* \|_2^2
\end{align*}
Combining this, we have:
\begin{align*}
\sum_{t=1}^{T} (w^\top x_{t}^u)^2 & \leq \sum_{j=0}^{i} \bigg ( 2 T_j \gamma^2 \max_{\theta \in \mathcal{I}_j} \| w^\top (e^{j \theta} I - A_*)^{-1} \|_2^2 \| B_* \|_2^2 + 4 \frac{\| P^{-1} \|_2^2 \| P \|_2^2 \| B_* \|_2^2 \beta(J_\ell)^2 \gamma^2 k_j}{(1 - \bar{\rho}(J_\ell)^{2k_j})(1 - \bar{\rho}(J_\ell)^{2})} \\
& \ \ \ \ \ \ \ \ \ \ \ \ \ \ \ \ \ + 4 \frac{\|P^{-1}\|_2^2 \| P \|_2^2 \| B_*\|_2^2 \beta(J_\ell)^2 \gamma^2 k_j}{(1- \bar{\rho}(J_\ell)^{2 k_0})(1 - \bar{\rho}(J_\ell)^2)} \bigg ) \\
& \leq 2 \max_{\theta \in [0,2\pi]} \gamma^2 \| w^\top (e^{j \theta} I - A_*)^{-1} \|_2^2 \| B_* \|_2^2  \sum_{j=0}^{i}  T_j   +  \sum_{j=0}^{i} 8 \frac{\| P^{-1} \|_2^2 \| P \|_2^2 \| B_* \|_2^2 \beta(J_\ell)^2 \gamma^2 k_j}{(1 - \bar{\rho}(J_\ell)^{2k_0})(1 - \bar{\rho}(J_\ell)^{2})}  \\
& \leq   3 T_i  \gamma^2 \max_{\theta \in [0,2\pi]} \| w^\top (e^{j \theta} I - A_*)^{-1} \|_2^2 \| B_* \|_2^2 + 16 \frac{\| P^{-1} \|_2^2 \| P \|_2^2 \| B_* \|_2^2 \beta(J_\ell)^2 \gamma^2 k_i}{(1 - \bar{\rho}(J_\ell)^{2k_0})(1 - \bar{\rho}(J_\ell)^{2})} 
\end{align*} 
\end{proof}

\subsection{Transients}\label{sec:transients}
Consider the response of a system to a deterministic, periodic, zero-mean input $u_t$ starting from some initial state $x_0^u$ at $t=0$ (here the mean is taken over a full period). We can break up the response into the steady state response, $x_t^{ss}$, and the transient response, $x_t^{tr}$:  $x_t^u = x_t^{ss} + x_t^{tr}$. Precisely, $x_t^{ss}$ is the response of the system if the input $u_t$ has been on for all time in the past and, to attain the desired response, we can set:
\begin{equation*}
x_t^{tr} = \left \{ \begin{matrix} x_0^u - x_t^{ss} & \text{for } t \leq 0 \\
A_* x_{t-1}^{tr} & \text{for } t > 0 \end{matrix} \right . = \left \{ \begin{matrix} x_0^u - x_t^{ss} & \text{for } t \leq 0 \\
A_*^t (x_0^u - x_0^{ss}) & \text{for } t > 0 \end{matrix} \right .
\end{equation*}
With these definitions, we will have:
\begin{equation*}
x_t^u = x_t^{tr} + x_t^{ss} = \left \{ \begin{matrix} x_0^u & \text{for } t \leq 0 \\
x_t^{ss} + A_*^t (x_0^u - x_0^{ss}) & \text{for } t > 0 \end{matrix} \right .
\end{equation*}
Assume that $\rho(A_*) < 1$, we will have that $\lim_{t \rightarrow \infty} \| x_t^u - x_t^{ss} \|_2 = 0$.

Take $k$ to be an integer multiple of the period of the input and note that, by linearity, $k$ will also be an integer multiple of the period of $x_t^{ss}$.

\begin{lemma}\label{lem:burn_in}\com{[Done]}
Using the definitions above, let: 
\begin{align*}
T_{ss}(\zeta,k,x_0^u) := & \max \left \{ \frac{1}{2 \log \bar{\rho}(A_*)} \log \left ( \frac{k \zeta(1-\bar{\rho}(A_*)^2)}{2 \| x_0^u - x_0^{ss} \|_2^2 \beta(A_*)^2}  \right ) , \right . \\
& \hspace{2cm} \left . \frac{1}{\log \bar{\rho}(A_*)} \log \left ( \frac{k \zeta \sqrt{1-\bar{\rho}(A_*)^2}}{4 \| x_0^u - x_0^{ss} \|_2 \beta(A_*) \sqrt{k w^\top \tilde{\Gamma}_k^u w}}  \right )  \right \}
\end{align*}
Then if $T' \geq T_{ss}(\zeta,k,x_0^u)$, we will have that:
\begin{equation*}
\frac{1}{k} \left | \sum_{t=T'}^{T' + k -1 } (w^\top x_t^u)^2 - k w^\top \tilde{\Gamma}_k^u w \right | \leq \zeta
\end{equation*}
\end{lemma}
\begin{proof}
Note first that $k w^\top \tilde{\Gamma}_k^u w = \sum_{t=1}^k (w^\top x_t^{ss})^2$. By what we have above:
\begin{align*}
& \left | \sum_{t=T'}^{T' + k -1 } (w^\top x_t^u)^2 - k w^\top \tilde{\Gamma}_k^u w \right | \\
= \ & \left | \sum_{t=T'}^{T' + k -1 } \left ( w^\top x_t^{ss} + w^\top A_*^t (x_0^u - x_0^{ss}) \right )^2 - k w^\top \tilde{\Gamma}_k^u w \right | \\
= \ & \left | \sum_{t=T'}^{T' + k -1 } \left ( w^\top x_t^{ss} \right)^2 + \sum_{t=T'}^{T' + k -1} \left ( w^\top A_*^t (x_0^u - x_0^{ss}) \right )^2 + 2 \sum_{t=T'}^{T' + k -1 } \left ( w^\top x_t^{ss} \right) \left ( w^\top A_*^t (x_0^u - x_0^{ss}) \right ) - k w^\top \tilde{\Gamma}_k^u w \right | \\
= \ & \left |  \sum_{t=T'}^{T' + k -1} \left ( w^\top A_*^t (x_0^u - x_0^{ss}) \right )^2 + 2 \sum_{t=T'}^{T' + k -1 } \left ( w^\top x_t^{ss} \right) \left ( w^\top A_*^t (x_0^u - x_0^{ss}) \right ) \right | \\
\leq \ &   \sum_{t=T'}^{T' + k -1} \left ( w^\top A_*^t (x_0^u - x_0^{ss}) \right )^2 + 2  \sqrt{ \sum_{t=T'}^{T' + k -1 } \left ( w^\top x_t^{ss} \right)^2 } \sqrt{  \sum_{t=T'}^{T' + k -1 } \left ( w^\top A_*^t (x_0^u - x_0^{ss}) \right )^2 } \\
= \ &    \sum_{t=T'}^{T' + k -1} \left ( w^\top A_*^t (x_0^u - x_0^{ss}) \right )^2 + 2  \sqrt{ k w^\top \tilde{\Gamma}_k^u w } \sqrt{  \sum_{t=T'}^{T' + k -1 } \left ( w^\top A_*^t (x_0^u - x_0^{ss}) \right )^2 } \\
\leq \ & \| x_0^u - x_0^{ss} \|_2^2 \sum_{t=T'}^{T' + k -1}  \| A_*^t \|_2^2  + 2 \| x_0^u - x_0^{ss} \|_2 \sqrt{ k w^\top \tilde{\Gamma}_k^u w } \sqrt{  \sum_{t=T'}^{T' + k -1}  \| A_*^t \|_2^2 } \\
\leq \ & \| x_0^u - x_0^{ss} \|_2^2 \beta(A_*)^2 \bar{\rho}(A_*)^{2 T'} \sum_{t=0}^{k -1}   \bar{\rho}(A_*)^{2t}  + 2 \| x_0^u - x_0^{ss} \|_2 \beta(A_*) \bar{\rho}(A_*)^{T'} \sqrt{ k w^\top \tilde{\Gamma}_k^u w } \sqrt{  \sum_{t=0}^{k -1}   \bar{\rho}(A_*)^{2t} } \\
\leq \ & \frac{\| x_0^u - x_0^{ss} \|_2^2 \beta(A_*)^2 \bar{\rho}(A_*)^{2 T'}}{1 - \bar{\rho}(A_*)^2} + \frac{2 \| x_0^u - x_0^{ss} \|_2 \beta(A_*) \sqrt{ k w^\top \tilde{\Gamma}_k^u w } \bar{\rho}(A_*)^{T'}  }{\sqrt{1- \bar{\rho}(A_*)^2}} \\
\stackrel{(a)}{\leq} \ & \frac{k \zeta}{2} + \frac{k \zeta}{2} \\
= \ & k \zeta
\end{align*}
where $(a)$ holds by our assumption on $T'$.
\end{proof}

\begin{corollary}\label{cor:burn_in}\com{[Done]}
Under the same assumptions as Lemma \ref{lem:burn_in}, we will have that:
\begin{equation*}
\frac{1}{k} \left | \sum_{t=T'}^{T' + k -1 } (w^\top x_t^u - w^\top \bar{x}^u)^2 - k w^\top \tilde{\Gamma}_k^u w \right | \leq \zeta
\end{equation*}
where:
\begin{equation*}
\bar{x}^u = \frac{1}{k} \sum_{t=T'}^{T' + k -1 } x_t^u
\end{equation*}
\end{corollary}
\begin{proof}
As before, we have:
\begin{align*}
& \left | \sum_{t=T'}^{T' + k -1 } (w^\top x_t^u - w^\top \bar{x}^u)^2 - k w^\top \tilde{\Gamma}_k^u w \right | \\
\leq \ & \sum_{t=T'}^{T' + k -1} \left ( w^\top A_*^t (x_0^u - x_0^{ss}) - w^\top \bar{x}^u \right )^2 + 2  \sqrt{ k w^\top \tilde{\Gamma}_k^u w } \sqrt{  \sum_{t=T'}^{T' + k -1 } \left ( w^\top A_*^t (x_0^u - x_0^{ss}) - w^\top \bar{x}^u \right )^2 } 
\end{align*}
Since, by assumption $u_t$ is zero-mean, it follows that $x_t^{ss}$ is zero-mean. Thus, the only non-zero mean component of $x_t^u$ is that due to the transient so:
\begin{equation*}
\frac{1}{k} \sum_{t=T'}^{T' + k - 1} w^\top A_*^t (x_0^u - x_0^{ss}) = w^\top \bar{x}^u
\end{equation*}
from which it follows that $w^\top A_*^t (x_0^u - x_0^{ss}) - w^\top \bar{x}^u$ is a zero-mean signal. Denoting $X^{tr}(e^{j\theta})$ the DFT of $x_t^{tr}$ over $t = T',...,T'+k-1$, by Parseval's Theorem, we will have that:
\begin{equation*}
\sum_{t=T'}^{T' + k -1} \left ( w^\top A_*^t (x_0^u - x_0^{ss}) - w^\top \bar{x}^u \right )^2 = \sum_{\ell=1}^{k-1} w^H X^{tr}(e^{j\frac{2\pi \ell}{k}}) {X^{tr}(e^{j\frac{2\pi \ell}{k}})}^H w
\end{equation*}
where, crucially, since $w^\top A_*^t (x_0^u - x_0^{ss}) - w^\top \bar{x}^u$ is zero-mean, we only sum over frequencies starting at $\theta = \frac{2 \pi}{k}$ (that is, we do not sum over the DC component). Thus:
\begin{align*}
\sum_{t=T'}^{T' + k -1} \left ( w^\top A_*^t (x_0^u - x_0^{ss}) - w^\top \bar{x}^u \right )^2 & = \sum_{\ell=1}^{k-1} w^H X^{tr}(e^{j\frac{2\pi \ell}{k}}) {X^{tr}(e^{j\frac{2\pi \ell}{k}})}^H w \\
& \leq \sum_{\ell=0}^{k-1} w^H X^{tr}(e^{j\frac{2\pi \ell}{k}}) {X^{tr}(e^{j\frac{2\pi \ell}{k}})}^H w \\
& = \sum_{t=T'}^{T' + k -1} \left ( w^\top A_*^t (x_0^u - x_0^{ss}) \right )^2 
\end{align*}
Thus:
\begin{align*}
& \sum_{t=T'}^{T' + k -1} \left ( w^\top A_*^t (x_0^u - x_0^{ss}) - w^\top \bar{x}^u \right )^2 + 2  \sqrt{ k w^\top \tilde{\Gamma}_k^u w } \sqrt{  \sum_{t=T'}^{T' + k -1 } \left ( w^\top A_*^t (x_0^u - x_0^{ss}) - w^\top \bar{x}^u \right )^2 }  \\
\leq \ & \sum_{t=T'}^{T' + k -1} \left ( w^\top A_*^t (x_0^u - x_0^{ss})  \right )^2 + 2  \sqrt{ k w^\top \tilde{\Gamma}_k^u w } \sqrt{  \sum_{t=T'}^{T' + k -1 } \left ( w^\top A_*^t (x_0^u - x_0^{ss}) \right )^2 }  \\
\leq \ & k \zeta
\end{align*}
where the last inequality follows since we have assumed Lemma \ref{lem:burn_in} holds.
\end{proof}

\begin{lemma}\label{lem:freq_approx}\com{[Done]}
Assume that the input $u_t$ satisfies, for some $k$ and any $s \geq 0$:
\begin{equation*}
\frac{1}{k} \sum_{t=1}^k u_{s+t}^\top u_{s+t} \leq \gamma^2
\end{equation*}
then:
\begin{equation*}
\| X(e^{j\theta}) - G(e^{j\theta}) U(e^{j\theta}) \|_2 \leq \frac{2 \| G(e^{j\theta}) \| \| B_* \|_2 \beta(A_*) k \gamma}{1 - \bar{\rho}(A_*)^k}
\end{equation*}
where $X(e^{j\theta})$ denotes the response of the noiseless system running for $T$ steps when the input $U(e^{j\theta})$ is applied.
\end{lemma}
\begin{proof}
Note that:
\begin{align*}
X(e^{j \theta}) & =  \sum_{t=0}^{T-1} e^{-j \theta t} x_{t}  =  \sum_{t=0}^{T-1} \sum_{s=0}^{t-1} A_*^{t-s-1} e^{-j \theta t} B_* u_s \\
& =  \sum_{s=0}^{T-1} \sum_{t=0}^{T-s-1} e^{-j \theta (t+s+1)} A_*^t B_* u_s  =  \sum_{s=0}^{T-1} \left ( \sum_{t=0}^{T-s-1} e^{-j \theta (t+1)} A_*^t  \right ) e^{-j \theta s} B_* u_s
\end{align*}
and:
\begin{equation*}
G(e^{j\theta}) = (e^{j\theta} I - A_*)^{-1} B_* = \sum_{s=0}^\infty e^{-j \theta (s+1)} A_*^s B_*, \ \ \ \ \ U(e^{j \theta}) = \sum_{t=0}^{T-1} e^{-j \theta t} u_t
\end{equation*}
Thus:
\begin{align*}
\| X(e^{j\theta}) - G(e^{j\theta}) U(e^{j\theta}) \|_2 & = \left \|  \sum_{s=0}^{T-1} \left ( \sum_{t=0}^{T-s-1} e^{-j \theta (t+1)} A_*^t  \right ) e^{-j \theta s} B_* u_s  - \sum_{s=0}^{T-1} \left ( \sum_{t=0}^{\infty} e^{-j \theta (t+1)} A_*^t  \right ) e^{-j \theta s} B_* u_s  \right \|_2 \\
& = \left \|  \sum_{s=0}^{T-1} \left ( \sum_{t=T-s}^{\infty} e^{-j \theta (t+1)} A_*^t  \right ) e^{-j \theta s} B_* u_s    \right \|_2 \\
& = \left \|  \sum_{s=0}^{T-1} \left ( e^{-j \theta (T-s) } A_*^{T-s} \sum_{t=0}^{\infty} e^{-j \theta (t+1)} A_*^t  \right ) e^{-j \theta s} B_* u_s    \right \|_2 \\
& = \left \|  \sum_{s=0}^{T-1} \left ( e^{-j \theta (T-s) } A_*^{T-s} G(e^{j\theta})  \right ) e^{-j \theta s}  u_s    \right \|_2 \\
& = \left \|  e^{-j \theta T } \sum_{s=0}^{T-1}  A_*^{T-s} G(e^{j\theta})   u_s    \right \|_2 \\
& \leq  \| G(e^{j\theta}) \|_2 \sum_{s=0}^{T-1} \| A_*^{T-s} \|_2 \|  u_s \|_2 \\
& \leq  \| G(e^{j\theta}) \|_2 \beta(A_*)  \sum_{s=0}^{T-1} \bar{\rho}(A_*)^{T-s} \| u_s \|_2 \\
& \leq  \frac{2 \| G(e^{j\theta}) \|_2 \beta(A_*) k \gamma}{1 - \bar{\rho}(A_*)^k}
\end{align*}
where the last inequality follows from the proof of Lemma \ref{lem:alg_xtnorm_bound}.

\end{proof}


\section{Optimal Design Perturbation Bounds}\label{sec:design_perturbation}

Throughout this section we assume we are running Algorithm \ref{alg:active_lds_noise} and that $T$ is the elapsed time after $i$ epochs. We will let $k = k_i$ to simplify expressions. We will also often simplify notation by writing $\theta_i := \frac{2 \pi i}{k}$ and $U_i := U(e^{j \frac{2 \pi i}{k}})$. 

Let:
\begin{equation*}
H_k(A,B,U,\mathcal{I}) :=  \sum_{i \in \mathcal{I}}  (e^{j 2 \pi i / k} I - A)^{-1} B U(e^{j 2 \pi i / k}) {U(e^{j 2 \pi i / k})}^H B^H (e^{j 2 \pi i / k} I - A )^{-H}
\end{equation*}
where $\mathcal{I} \subseteq [k]$.

Formally, for some $k$, the optimization problem we wish to solve is:
\begin{align*}
\texttt{OptInput}_k(A,B,\gamma^2,\mathcal{I},\{ x_t \}_{t=1}^T) := \begin{matrix*}[l]  \max_{u_1,...,u_k \in \mathbb{R}^p} \ \lambda_{\min} \left ( \frac{2T + T_0}{k^2} H_k(A,B,U,\mathcal{I}) +  \sum_{t=1}^T x_t x_t^\top \right ) \\
 \text{s.t.} \ \ \sum_{\ell=1}^k U(e^{j 2 \pi \ell / k})^H U(e^{j 2 \pi \ell / k}) \leq k^2 \gamma^2, \\
\ \ \ \ \ \ \ U(e^{j 2 \pi \ell / k}) = 0, \forall \ell \not\in \mathcal{I},  \ \sum_{t=1}^k u_t = 0 
 \end{matrix*}
\end{align*}
where $\gamma^2$ is simply some value constraining the power of our input signal and $U(e^{j 2 \pi \ell / k})$ denotes the DFT of $u_1,...,u_k$, the time domain signal. Note that the normalization $\frac{2T+T_0}{k^2}$ of $\sum_{t=1}^T x_t x_t^\top$ is due to the fact that, by Parseval's Theorem:
\begin{equation*}
\sum_{t=1}^{T_i} x_t^u {x_t^u}^\top = \frac{T_i}{k} \frac{1}{k} \sum_{i=1}^k X^u(e^{j 2 \pi i / k}) X^u(e^{j 2 \pi i / k})^H
\end{equation*}
assuming that $u_t$ has period $k$ and that we are in steady state. Further, by the update rule of Algorithm \ref{alg:active_lds_noise}, $T = \sum_{\ell = 0}^{i-1} 3^i T_0 = \frac{1}{2} ( 3^i - 1) T_0 = \frac{1}{2} T_i - \frac{1}{2} T_0$ so $T_i = 2T + T_0$, which is the expected amount of time we will play these inputs for.

It is worth noting that the constraint $\sum_{\ell=1}^k U(e^{j 2 \pi \ell / k})^H U(e^{j 2 \pi \ell / k}) \leq k^2 \gamma^2$ is equivalent, by Parseval's Theorem, to the constraint:
\begin{equation*}
\frac{1}{k} \sum_{t=1}^k u_t^\top u_t \leq \gamma^2
\end{equation*}

We will denote the optimal set of inputs on the true system as $u^*$ and the optimal set of inputs on the estimated system as $\hat{u}$ (that is, $\hat{u}$ is the solution to $\texttt{OptInput}_k(\hat{A},B_*,\gamma^2,\mathcal{I},\{x_t\}_{t=1}^T)$). 

Our main perturbation result is as follows.

\begin{theorem}\label{thm:opt_sln_perturbation}\com{[Done]}
(\textbf{Full version of Theorem \ref{thm:opt_sln_perturbation_informal}}) Assuming that $\| A_* - \hat{A} \|_2 \leq \epsilon$, then we will have that:
\begin{align*}
& \left | \lambda_{\min} \left (\frac{1}{k^2} H_k(A_*,B_*,U^*,\mathcal{I}) + \frac{1}{2T + T_0} \sum_{t=1}^T x_t x_t^{\top} \right )  - \lambda_{\min} \left ( \frac{1}{k^2} H_k(A_*,B_*,\hat{U},\mathcal{I}) +   \frac{1}{2T + T_0} \sum_{t=1}^T x_t x_t^{\top} \right ) \right | \\
\leq \ &  \max_{\substack{U \in \mathcal{U}_{\gamma^2} \\ w \in \mathcal{M}(A_*,\hat{A}, \{ x_t \}_{t=1}^T,\mathcal{I})}}    \frac{2\epsilon}{k^2}  L(A_*,B_*,U,\epsilon, \mathcal{I}, w) 
\end{align*}
where $\{ x_t \}_{t=1}^T$ is generated from a system with parameter $A_*$, $U^*$ is the solution to ${\normalfont \texttt{OptInput}}_k(A_*,B_*,\gamma^2,\mathcal{I},\{ x_t \}_{t=1}^T)$, $\hat{U}$ is the solution to ${\normalfont \texttt{OptInput}}_k(\hat{A},B_*,\gamma^2,\mathcal{I},\{ x_t \}_{t=1}^T)$, and: 
\begin{align*}
& \mathcal{M}(A_*,\hat{A}, \{ x_t \}_{t=1}^T,\mathcal{I})  := \bigg \{ w \in \mathcal{S}^{d-1} \ : \ \frac{k^2}{2T + T_0} \sum_{t=1}^T (w^\top x_t)^2 \\
& \ \ \ \ \ \ \ \leq \min_{w' \in \mathcal{S}^{d-1}} \gamma^2 \max_{i \in \mathcal{I}} \max \{ \| {w'}^\top (e^{j \theta_i} I - A_*)^{-1} B_* \|_2^2, \| {w'}^\top (e^{j \theta_i} I - \hat{A})^{-1} B_* \|_2^2 \}  + \frac{k^2}{2T + T_0} \sum_{t=1}^T ({w'}^\top x_t)^2 \bigg \}
\end{align*}
\begin{align*}
& L(A_*,B_*,U,\epsilon,\mathcal{I},w) \\
& \ \ \ := \max_{\substack{\delta \in [0,\epsilon], \Delta \in \mathbb{R}^{d \times d} \\ \| \Delta \|_2 = 1} } \ 2 \left | \sum_{i \in \mathcal{I}} w^\top (e^{j \theta_i} I - A_* - \delta \Delta)^{-1} \Delta (e^{j \theta_i} I - A_* - \delta \Delta)^{-1} B_* U_i U_i^H B_*^H (e^{j \theta_i} I - A_* - \delta \Delta)^{-H} w \right |
\end{align*}
\end{theorem}

\begin{remark}
As we will show in the proof of Theorem \ref{thm:opt_sln_perturbation}, the set $ \mathcal{M}(A_*,\hat{A}, \{ x_t \}_{t=1}^T,\mathcal{I})$ is guaranteed to contain the eigenvectors of $\frac{2T + T_0}{k^2} H_k(\hat{A},B_*,U^*,\mathcal{I}) +  \sum_{t=1}^T x_t x_t^{\top}$ and $\frac{2T + T_0}{k^2} H_k(A_*,B_*,\hat{U},\mathcal{I}) +  \sum_{t=1}^T x_t x_t^{\top}$ corresponding to their minimum eigenvalues. Restricting to a max over this set is sufficient to bound the difference in the minimum eigenvalues and avoids computing $L(A_*,B_*,U,\epsilon,\mathcal{I},w)$ for the worst case $w$---the $w$ corresponding to the most easily excited directions. 

The max of $L(A_*,B_*,U,\epsilon,\mathcal{I},w)$ over all $w \in \mathcal{S}^{d-1}$ will scale roughly as $\max_{i \in \mathcal{I}} \| (e^{j \theta_i} I - A_*) \|_2^3$. However, in some situations, as we show in Corollary \ref{cor:symmetric_a_informal}, the max over $ \mathcal{M}(A_*,\hat{A}, \{ x_t \}_{t=1}^T,\mathcal{I})$ will scale only as $\max_{i \in \mathcal{I}} \| (e^{j \theta_i} I - A_*) \|_2^2$. The reason for this is that, assuming a large enough gap between the largest and smallest eigenvalues of $A_*$, $ \mathcal{M}(A_*,\hat{A}, \{ x_t \}_{t=1}^T,\mathcal{I})$ will not include vectors corresponding to the subspace spanned by the eigenvectors corresponding to the largest eigenvalues, as these will be sufficiently excited by noise to make $\sum_{t=1}^T (w^\top x_t)^2$ large. In that case one can show that for all $w \in \mathcal{M}(A_*,\hat{A}, \{ x_t \}_{t=1}^T,\mathcal{I})$, $\| (e^{j \theta_i} I - A_* )^{-H} w \|_2 = \mathcal{O}( \| (e^{j \theta_i} I - A_* )^{-1} \|_2^{1/2})$. 
\end{remark}

\subsection{Proof of Theorem \ref{thm:opt_sln_perturbation_informal} and Theorem \ref{thm:opt_sln_perturbation}}
\begin{proof}
Throughout, to shorten notation, let $\xi = \frac{2T + T_0}{k^2}$. Note that:
\begin{align*}
& \left | \lambda_{\min} \left (\xi H_k(A_*,B_*,U^*,\mathcal{I}) +  \sum_{t=1}^T x_t x_t^{\top} \right )  - \lambda_{\min} \left (\xi H_k(A_*,B_*,\hat{U},\mathcal{I}) +   \sum_{t=1}^T x_t x_t^{\top} \right ) \right |  \\
& = \left | {w_{A_*,U^*}}^\top \left (\xi H_k(A_*,B_*,U^*,\mathcal{I}) +  \sum_{t=1}^T x_t x_t^{\top} \right ) w_{A_*,U^*}  - {w_{A_*,\hat{U}}}^\top \left (\xi H_k(A_*,B_*,\hat{U},\mathcal{I}) +   \sum_{t=1}^T x_t x_t^{\top} \right ) {w_{A_*,\hat{U}}} \right | 
\end{align*}
where $w_{A_*,U^*}, {w_{A_*,\hat{U}}}$ are the eigenvectors corresponding to the minimum eigenvalues of the matrices $\xi H_k(A_*,B_*,U^*,\mathcal{I}) +  \sum_{t=1}^T x_t x_t^{\top}$ and $\xi H_k(A_*,B_*,\hat{U},\mathcal{I}) +   \sum_{t=1}^T x_t x_t^{\top}$, respectively. We wish to show that:
\begin{equation*}
\left | {w_{A_*,U^*}}^\top \left ( \xi H_k(A_*,B_*,U^*,\mathcal{I}) + \sum_{t=1}^T x_t x_t^{\top} \right ) w_{A_*,U^*}  - {w_{A_*,\hat{U}}}^\top \left (\xi H_k(A_*,B_*,\hat{U},\mathcal{I}) +   \sum_{t=1}^T x_t x_t^{\top} \right ) {w_{A_*,\hat{U}}} \right | \leq \delta
\end{equation*}
for some choice of $\delta$. To show this, we will first show that:
\begin{equation*}
\left | {w_{A_*,U^*}}^\top \left (\xi H_k(A_*,B_*,U^*,\mathcal{I}) +  \sum_{t=1}^T x_t x_t^{\top} \right ) w_{A_*,U^*}  - \min_{w\in \mathcal{S}^{d-1}} w^\top \left (\xi H_k(\hat{A},B_*,\hat{U},\mathcal{I}) +  \sum_{t=1}^T x_t x_t^{\top} \right ) w \right | \leq \delta'
\end{equation*}
Denote $w_{\hat{A},\hat{U}}$ the solution of the above minimization. Denote also $w_{\hat{A},U^*}$ the eigenvector corresponding to the minimum eigenvalue of $\xi H_k(\hat{A},B_*,U^*,\mathcal{I}) +  \sum_{t=1}^T x_t x_t^{\top}$. Then if for all $U \in \mathcal{U}_{\gamma^2}$:
\begin{equation}\label{eq:per1}
\left | {w_{A_*,\hat{U}}}^\top \left ( \xi H_k(\hat{A},B_*,U,\mathcal{I}) +  \sum_{t=1}^T x_t x_t^{\top} \right ) w_{A_*,\hat{U}} - {w_{A_*,\hat{U}}}^\top \left ( \xi H_k(A_*,B_*,U,\mathcal{I}) +  \sum_{t=1}^T x_t x_t^{\top} \right ) w_{A_*,\hat{U}} \right | \leq \delta'
\end{equation}
and:
\begin{equation}\label{eq:per2}
\left | {w_{\hat{A},U^*}}^\top \left ( \xi H_k(\hat{A},B_*,U,\mathcal{I}) +  \sum_{t=1}^T x_t x_t^{\top} \right ) {w_{\hat{A},U^*}} - {w_{\hat{A},U^*}}^\top \left ( \xi H_k(A_*,B_*,U,\mathcal{I}) +  \sum_{t=1}^T x_t x_t^{\top} \right ) {w_{\hat{A},U^*}} \right | \leq \delta'
\end{equation}
the above will follow. To see this, assume that:
\begin{equation*}
{w_{\hat{A},\hat{U}}}^\top \left (\xi H_k(\hat{A},B_*,\hat{U},\mathcal{I}) +   \sum_{t=1}^T x_t x_t^{\top} \right ) {w_{\hat{A},\hat{U}}} - \delta' > {w_{A_*,U^*}}^\top \left (\xi H_k(A_*,B_*,U^*,\mathcal{I}) +  \sum_{t=1}^T x_t x_t^{\top} \right ) w_{A_*,U^*}  
\end{equation*}
then:
\begin{align*}
& {w_{\hat{A},\hat{U}}}^\top \left (\xi H_k(\hat{A},B_*,\hat{U},\mathcal{I}) +   \sum_{t=1}^T x_t x_t^{\top} \right ) {w_{\hat{A},\hat{U}}} - \delta' \\
& > {w_{A_*,U^*}}^\top \left (\xi H_k(A_*,B_*,U^*,\mathcal{I}) +  \sum_{t=1}^T x_t x_t^{\top} \right ) w_{A_*,U^*}  \\
& \stackrel{(a)}{\geq} {{w_{A_*,\hat{U}}}}^\top \left (\xi H_k(A_*,B_*,\hat{U},\mathcal{I}) +  \sum_{t=1}^T x_t x_t^{\top} \right ) {w_{A_*,\hat{U}}} \\
& \stackrel{(b)}{\geq}{{w_{A_*,\hat{U}}}}^\top \left (\xi H_k(\hat{A},B_*,\hat{U},\mathcal{I}) +  \sum_{t=1}^T x_t x_t^{\top} \right ) {w_{A_*,\hat{U}}} - \delta' \\
& \stackrel{(c)}{\geq}{{w_{\hat{A},\hat{U}}}}^\top \left (\xi H_k(\hat{A},B_*,\hat{U},\mathcal{I}) +  \sum_{t=1}^T x_t x_t^{\top} \right ) {w_{\hat{A},\hat{U}}} - \delta' 
\end{align*}
where $(a)$ follows by optimality of $U^*$, $(b)$ follows by our assumption (\ref{eq:per1}), and $(c)$ follows since ${w_{\hat{A},\hat{U}}}$ corresponds to the minimum eigenvalue of $\xi H_k(\hat{A},B_*,\hat{U},\mathcal{I}) +  \sum_{t=1}^T x_t x_t^{\top}$. This is clearly a contradiction, which implies that:
\begin{equation*}
{w_{\hat{A},\hat{U}}}^\top \left (\xi H_k(\hat{A},B_*,\hat{U},\mathcal{I}) +   \sum_{t=1}^T x_t x_t^{\top} \right ) {w_{\hat{A},\hat{U}}} - \delta' \leq {w_{A_*,U^*}}^\top \left (\xi H_k(A_*,B_*,U^*,\mathcal{I}) +  \sum_{t=1}^T x_t x_t^{\top} \right ) w_{A_*,U^*}  
\end{equation*}
We can repeat this argument identically in the opposite direction:
\begin{align*}
& {w_{A_*,U^*}}^\top \left (\xi H_k(A_*,B_*,U^*,\mathcal{I}) +   \sum_{t=1}^T x_t x_t^{\top} \right ) w_{A_*,U^*} - \delta' \\
& > {w_{\hat{A},\hat{U}}}^\top \left (\xi H_k(\hat{A},B_*,\hat{U},\mathcal{I}) +  \sum_{t=1}^T x_t x_t^{\top} \right ) {w_{\hat{A},\hat{U}}}  \\
& \geq {{w_{\hat{A},U^*}}}^\top \left (\xi H_k(\hat{A},B_*,U^*,\mathcal{I}) +  \sum_{t=1}^T x_t x_t^{\top} \right ) {w_{\hat{A},U^*}} \\
& \geq {{w_{\hat{A},U^*}}}^\top \left (\xi H_k(A_*,B_*,U^*,\mathcal{I}) +  \sum_{t=1}^T x_t x_t^{\top} \right ) {w_{\hat{A},U^*}} - \delta' \\
& \geq {w_{A_*,U^*}}^\top \left (\xi H_k(A_*,B_*,U^*,\mathcal{I}) +  \sum_{t=1}^T x_t x_t^{\top} \right ) w_{A_*,U^*} - \delta' 
\end{align*}
which is another contradiction. Combining these, it follows then that:
\begin{equation}\label{eq:per3}
\left | {w_{A_*,U^*}}^\top \left (\xi H_k(A_*,B_*,U^*,\mathcal{I}) +  \sum_{t=1}^T x_t x_t^{\top} \right ) w_{A_*,U^*}  -  {w_{\hat{A},\hat{U}}}^\top \left (\xi H_k(\hat{A},B_*,\hat{U},\mathcal{I}) +   \sum_{t=1}^T x_t x_t^{\top} \right ) {w_{\hat{A},\hat{U}}} \right | \leq \delta'
\end{equation}
We now return to bounding the difference assuming  (\ref{eq:per1}) and (\ref{eq:per2}) hold:
\begin{align*}
& \left | \lambda_{\min} \left (\xi H_k(A_*,B_*,U^*,\mathcal{I}) +  \sum_{t=1}^T x_t x_t^{\top} \right )  - \lambda_{\min} \left (\xi H_k(A_*,B_*,\hat{U},\mathcal{I}) +   \sum_{t=1}^T x_t x_t^{\top} \right ) \right |  \\
& = \left | {w_{A_*,U^*}}^\top \left (\xi H_k(A_*,B_*,U^*,\mathcal{I}) +  \sum_{t=1}^T x_t x_t^{\top} \right ) w_{A_*,U^*}  - {w_{A_*,\hat{U}}}^\top \left (\xi H_k(A_*,B_*,\hat{U},\mathcal{I}) +   \sum_{t=1}^T x_t x_t^{\top} \right ) {w_{A_*,\hat{U}}} \right | 
\end{align*}
First, assume that:
\begin{equation*}
{w_{A_*,U^*}}^\top \left (\xi H_k(A_*,B_*,U^*,\mathcal{I}) +  \sum_{t=1}^T x_t x_t^{\top} \right ) w_{A_*,U^*} \geq {w_{A_*,\hat{U}}}^\top \left ( \xi H_k(\hat{A},B_*,\hat{U},\mathcal{I}) +  \sum_{t=1}^T x_t x_t^{\top} \right ) {w_{A_*,\hat{U}}}
\end{equation*}
then:
\begin{align*}
& \left | {w_{A_*,U^*}}^\top \left (\xi H_k(A_*,B_*,U^*,\mathcal{I}) +  \sum_{t=1}^T x_t x_t^{\top} \right ) w_{A_*,U^*}  - {w_{A_*,\hat{U}}}^\top \left (\xi H_k(A_*,B_*,\hat{U},\mathcal{I}) +   \sum_{t=1}^T x_t x_t^{\top} \right ) {w_{A_*,\hat{U}}} \right |  \\
\leq \ & \left | {w_{A_*,U^*}}^\top \left (\xi H_k(A_*,B_*,U^*,\mathcal{I}) +  \sum_{t=1}^T x_t x_t^{\top} \right ) w_{A_*,U^*}  - {w_{A_*,\hat{U}}}^\top \left (\xi H_k(\hat{A},B_*,\hat{U},\mathcal{I}) +   \sum_{t=1}^T x_t x_t^{\top} \right ) {w_{A_*,\hat{U}}} \right | \\
& \ \ \ \ + \left | {w_{A_*,\hat{U}}}^\top \left (\xi H_k(\hat{A},B_*,\hat{U},\mathcal{I}) +   \sum_{t=1}^T x_t x_t^{\top} \right ) {w_{A_*,\hat{U}}}  - {w_{A_*,\hat{U}}}^\top \left (\xi H_k(A_*,B_*,\hat{U},\mathcal{I}) +   \sum_{t=1}^T x_t x_t^{\top} \right ) {w_{A_*,\hat{U}}} \right | \\
\leq \ & \left | {w_{A_*,U^*}}^\top \left (\xi H_k(A_*,B_*,U^*,\mathcal{I}) +  \sum_{t=1}^T x_t x_t^{\top} \right ) w_{A_*,U^*}  - {w_{\hat{A},\hat{U}}}^\top \left (\xi H_k(\hat{A},B_*,\hat{U},\mathcal{I}) +   \sum_{t=1}^T x_t x_t^{\top} \right ) {w_{\hat{A},\hat{U}}} \right | \\
& \ \ \ \ + \left | {w_{A_*,\hat{U}}}^\top \left (\xi H_k(\hat{A},B_*,\hat{U},\mathcal{I}) +   \sum_{t=1}^T x_t x_t^{\top} \right ) {w_{A_*,\hat{U}}}  - {w_{A_*,\hat{U}}}^\top \left (\xi H_k(A_*,B_*,\hat{U},\mathcal{I}) +   \sum_{t=1}^T x_t x_t^{\top} \right ) {w_{A_*,\hat{U}}} \right | \\
\leq \ & 2 \delta'
\end{align*}
where the final inequality follows by (\ref{eq:per1}) and (\ref{eq:per3}). Assume instead that:
\begin{equation*}
{w_{A_*,U^*}}^\top \left (\xi H_k(A_*,B_*,U^*,\mathcal{I}) +  \sum_{t=1}^T x_t x_t^{\top} \right ) w_{A_*,U^*} < {w_{A_*,\hat{U}}}^\top \left ( \xi H_k(\hat{A},B_*,\hat{U},\mathcal{I}) +  \sum_{t=1}^T x_t x_t^{\top} \right ) {w_{A_*,\hat{U}}}
\end{equation*}
then:
\begin{align*}
& \left | {w_{A_*,U^*}}^\top \left (\xi H_k(A_*,B_*,U^*,\mathcal{I}) +  \sum_{t=1}^T x_t x_t^{\top} \right ) w_{A_*,U^*}  - {w_{A_*,\hat{U}}}^\top \left (\xi H_k(A_*,B_*,\hat{U},\mathcal{I}) +   \sum_{t=1}^T x_t x_t^{\top} \right ) {w_{A_*,\hat{U}}} \right |  \\
\leq \ & \left | {w_{A_*,U^*}}^\top \left (\xi H_k(A_*,B_*,U^*,\mathcal{I}) +  \sum_{t=1}^T x_t x_t^{\top} \right ) w_{A_*,U^*}  - {w_{\hat{A},\hat{U}}}^\top \left (\xi H_k(\hat{A},B_*,\hat{U},\mathcal{I}) +   \sum_{t=1}^T x_t x_t^{\top} \right ) {w_{\hat{A},\hat{U}}} \right | \\
& \ \ \ \ + \left | {w_{\hat{A},\hat{U}}}^\top \left (\xi H_k(\hat{A},B_*,\hat{U},\mathcal{I}) +   \sum_{t=1}^T x_t x_t^{\top} \right ) {w_{\hat{A},\hat{U}}}  - {w_{A_*,\hat{U}}}^\top \left (\xi H_k(A_*,B_*,\hat{U},\mathcal{I}) +   \sum_{t=1}^T x_t x_t^{\top} \right ) {w_{A_*,\hat{U}}} \right | \\
\leq \ & \delta' + \left | {w_{\hat{A},\hat{U}}}^\top \left (\xi H_k(\hat{A},B_*,\hat{U},\mathcal{I}) +   \sum_{t=1}^T x_t x_t^{\top} \right ) {w_{\hat{A},\hat{U}}}  - {w_{A_*,\hat{U}}}^\top \left (\xi H_k(A_*,B_*,\hat{U},\mathcal{I}) +   \sum_{t=1}^T x_t x_t^{\top} \right ) {w_{A_*,\hat{U}}} \right |
\end{align*}
where the final equality follows by (\ref{eq:per3}). If we assume that:
\begin{equation*}
{w_{\hat{A},\hat{U}}}^\top \left (\xi H_k(\hat{A},B_*,\hat{U},\mathcal{I}) +   \sum_{t=1}^T x_t x_t^{\top} \right ) {w_{\hat{A},\hat{U}}} \geq {w_{A_*,\hat{U}}}^\top \left (\xi H_k(A_*,B_*,\hat{U},\mathcal{I}) +   \sum_{t=1}^T x_t x_t^{\top} \right ) {w_{A_*,\hat{U}}}
\end{equation*}
then:
\begin{align*}
& \left | {w_{\hat{A},\hat{U}}}^\top \left (\xi H_k(\hat{A},B_*,\hat{U},\mathcal{I}) +   \sum_{t=1}^T x_t x_t^{\top} \right ) {w_{\hat{A},\hat{U}}}  - {w_{A_*,\hat{U}}}^\top \left (\xi H_k(A_*,B_*,\hat{U},\mathcal{I}) +   \sum_{t=1}^T x_t x_t^{\top} \right ) {w_{A_*,\hat{U}}} \right | \\
\leq \ & \left | {w_{A_*,\hat{U}}}^\top \left (\xi H_k(\hat{A},B_*,\hat{U},\mathcal{I}) +   \sum_{t=1}^T x_t x_t^{\top} \right ) {w_{A_*,\hat{U}}}  - {w_{A_*,\hat{U}}}^\top \left (\xi H_k(A_*,B_*,\hat{U},\mathcal{I}) +   \sum_{t=1}^T x_t x_t^{\top} \right ) {w_{A_*,\hat{U}}} \right | \\
\leq \ & \delta'
\end{align*}
where the final equality follows by (\ref{eq:per1}). Otherwise:
\begin{align*}
& \left | {w_{\hat{A},\hat{U}}}^\top \left (\xi H_k(\hat{A},B_*,\hat{U},\mathcal{I}) +   \sum_{t=1}^T x_t x_t^{\top} \right ) {w_{\hat{A},\hat{U}}}  - {w_{A_*,\hat{U}}}^\top \left (\xi H_k(A_*,B_*,\hat{U},\mathcal{I}) +   \sum_{t=1}^T x_t x_t^{\top} \right ) {w_{A_*,\hat{U}}} \right | \\
\leq \ & \left | {w_{\hat{A},\hat{U}}}^\top \left (\xi H_k(\hat{A},B_*,\hat{U},\mathcal{I}) +   \sum_{t=1}^T x_t x_t^{\top} \right ) {w_{\hat{A},\hat{U}}}  - {w_{A_*,U^*}}^\top \left (\xi H_k(A_*,B_*,U^*,\mathcal{I}) +   \sum_{t=1}^T x_t x_t^{\top} \right ) w_{A_*,U^*} \right | \\
\leq \ & \delta'
\end{align*}
where the first inequality holds since $U^*$ are the optimal inputs and the final equality follows by (\ref{eq:per3}). Combining these, we conclude that:
\begin{equation*}
 \left | \lambda_{\min} \left (\xi H_k(A_*,B_*,U^*,\mathcal{I}) +  \sum_{t=1}^T x_t x_t^{\top} \right )  - \lambda_{\min} \left (\xi H_k(A_*,B_*,\hat{U},\mathcal{I}) +   \sum_{t=1}^T x_t x_t^{\top} \right ) \right | \leq 2 \delta'
\end{equation*}

To get a bound of the form:
\begin{equation*}
\left | w^\top \left ( \xi H_k(A_*,B_*,U,\mathcal{I}) +  \sum_{t=1}^T x_t x_t^\top \right ) w - w^\top \left ( \xi H_k(\hat{A},B_*,U,\mathcal{I}) +  \sum_{t=1}^T x_t x_t^\top \right ) w \right | \leq \delta' 
\end{equation*}
and guarantee (\ref{eq:per1}) and (\ref{eq:per2}) hold we can apply Lemma \ref{lem:directional_perturbation} which states that:
\begin{equation*}
\left | w^\top \left ( \xi H_k(A_*,B_*,U,\mathcal{I}) +  \sum_{t=1}^T x_t x_t^\top \right ) w - w^\top \left ( \xi H_k(\hat{A},B_*,U,\mathcal{I}) +  \sum_{t=1}^T x_t x_t^\top \right ) w \right | \leq  \xi  \epsilon L(A_*,B_*,U,\epsilon, \mathcal{I}, w)
\end{equation*}
We want to guarantee that such a condition holds for ${w_{\hat{A},U^*}}$ and ${w_{A_*,\hat{U}}}$. In practice we cannot determine what these are exactly since this requires knowledge of $A_*$. Thus, instead, we will find a set $\mathcal{M}(A_*,\hat{A}, \{ x_t \}_{t=1}^T,\mathcal{I})$ which is guaranteed to contain them. Setting:
\begin{align*}
& \mathcal{M}(A_*,\hat{A}, \{ x_t \}_{t=1}^T,\mathcal{I})  := \bigg \{ w \in \mathcal{S}^{d-1} \ : \   \sum_{t=1}^T (w^\top x_t)^2 \leq \min_{w' \in \mathcal{S}^{d-1}} \sum_{t=1}^T ({w'}^\top x_t)^2  \\
& \ \ \ \ \ \ \ \ \ \ \ \ \ \ \ \ \ \ +   (2T + T_0) \gamma^2 \max_{i \in \mathcal{I}} \max \{ \| {w'}^\top (e^{j \theta_i} I - A_*)^{-1} B_* \|_2^2, \| {w'}^\top (e^{j \theta_i} I - \hat{A})^{-1} B_* \|_2^2 \} \bigg \}
\end{align*}
this will be satisfied. To see why, note that $$\min_{w' \in \mathcal{S}^{d-1}} (2T + T_0) \gamma^2 \max_{i \in \mathcal{I}} \max \{ \| (e^{j \theta_i} I - A_*)^{-1} B_* \|_2^2, \| (e^{j \theta_i} I - \hat{A})^{-1} B_* \|_2^2 \} +   \sum_{t=1}^T ({w'}^\top x_t)^2$$ upper bounds $$\lambda_{\min} \left ( \xi H_k(A_*,B_*,U,\mathcal{I}) +  \sum_{t=1}^T x_t x_t^\top \right )$$ and $$\lambda_{\min} \left ( \xi H_k(\hat{A},B_*,U,\mathcal{I}) +  \sum_{t=1}^T x_t x_t^\top \right )$$ for all $U \in \mathcal{U}_{\gamma^2}$, so if $$ \sum_{t=1}^T (w^\top x_t)^2 > \min_{w' \in \mathcal{S}^{d-1}} (2T + T_0) \gamma^2 \max_{i \in \mathcal{I}} \max \{ \| (e^{j \theta_i} I - A_*)^{-1} B_* \|_2^2, \| (e^{j \theta_i} I - \hat{A})^{-1} B_* \|_2^2 \} +  \sum_{t=1}^T ({w'}^\top x_t)^2$$ then $w$ cannot possibly correspond to the minimum eigenvalue of either $$\xi H_k(A_*,B_*,U,\mathcal{I}) +  \sum_{t=1}^T x_t x_t^\top$$ or $$\xi H_k(\hat{A},B_*,U,\mathcal{I}) +  \sum_{t=1}^T x_t x_t^\top$$

Thus, to conclude, we will have that:
\begin{align*}
& \left | \lambda_{\min} \left (\xi H_k(A_*,B_*,U^*,\mathcal{I}) +  \sum_{t=1}^T x_t x_t^{\top} \right )  - \lambda_{\min} \left (\xi H_k(A_*,B_*,\hat{U},\mathcal{I}) +   \sum_{t=1}^T x_t x_t^{\top} \right ) \right | \\
\leq \ &  \max_{\substack{U \in \mathcal{U}_{\gamma^2} \\ w \in \mathcal{M}(A_*,\hat{A}, \{ x_t \}_{t=1}^T,\mathcal{I})}} 2 \xi \epsilon L(A_*,B_*,U,\epsilon, \mathcal{I}, w)
\end{align*}
\end{proof}

\subsection{Perturbation Lemmas}
\begin{corollary}\label{cor:noise_planning_per} \com{[Not done]}
Assuming that $\| A_* - \hat{A} \|_2 \leq \epsilon$ and that the largest Jordan block of $A_*$ has dimension $q$, we will have that, for small enough $\epsilon$:
\begin{align*}
& \left | \lambda_{\min} \left ( \frac{2T + T_0}{k^2} H_k(A_*,B_*,U^*,\mathcal{I}) +   M \right )  - \lambda_{\min} \left ( \frac{2T + T_0}{k^2} H_k(A_*,B_*,\hat{U},\mathcal{I}) +    M \right ) \right | \\
\leq \ &  \max_{\substack{U \in \mathcal{U}_{\gamma^2}, w \in \mathcal{S}^{d-1}}} 2 \left ( \frac{2T + T_0}{k^2} \right ) \epsilon L(A_*,B_*,U,\epsilon, \mathcal{I}, w) + \| M - \hat{M} \|_2
\end{align*}
where here $U^*$ is the solution to ${\normalfont \texttt{OptInput}}_k \left (A_*,B_*, \gamma^2, \mathcal{I},M \right )$ and $\hat{U}$ is the solution to

 ${\normalfont \texttt{OptInput}}_k \left (\hat{A},B_*, \gamma^2, \mathcal{I}, \hat{M} \right )$.
\end{corollary}
\begin{proof}
The proof of this result follows identically the proof of Theorem \ref{thm:opt_sln_perturbation} except now instead of showing:
\begin{equation*}
\left | w^\top \left ( \frac{2T + T_0}{k^2} H_k(A_*,B_*,U,\mathcal{I}) +  \sum_{t=1}^T x_t x_t^\top \right ) w - w^\top \left ( \frac{2T + T_0}{k^2} H_k(\hat{A},B_*,U,\mathcal{I}) +  \sum_{t=1}^T x_t x_t^\top \right ) w \right | \leq \delta' 
\end{equation*}
we must show:
\begin{align*}
& \left | w^\top \left ( \frac{2T + T_0}{k^2} H_k(A_*,B_*,U,\mathcal{I})  + M \right ) w  - w^\top \left ( \frac{2T + T_0}{k^2} H_k(\hat{A},B_*,U,\mathcal{I})  + \hat{M} \right ) w \right | \leq \delta' 
\end{align*}
for some $\delta'$. Note that:
\begin{align*}
& \left | w^\top \left ( \frac{2T + T_0}{k^2} H_k(A_*,B_*,U,\mathcal{I})  + M \right ) w  - w^\top \left ( \frac{2T + T_0}{k^2} H_k(\hat{A},B_*,U,\mathcal{I}) + \hat{M} \right ) w \right | \\
\leq \ & \frac{2T + T_0}{k^2} \left | w^\top H_k(A_*,B_*,U,\mathcal{I}) w - w^\top H_k(\hat{A},B_*,U,\mathcal{I}) w \right | + \left \|  M - \hat{M} \right \|_2
\end{align*}
By Lemma \ref{lem:directional_perturbation} we can upper bound:
\begin{equation*}
\left | w^\top H_k(A_*,B_*,U,\mathcal{I}) w - w^\top H_k(\hat{A},B_*,U,\mathcal{I}) w \right | \leq \epsilon L(A_*,B_*,U,\epsilon,\mathcal{I},w)
\end{equation*}
Given this, the rest of the proof of Theorem \ref{thm:opt_sln_perturbation} follows identically now.
\end{proof}

It is not clear in general how large $L(A_*,B_*,U,\epsilon,\mathcal{I},w)$ is and how it scales with $\epsilon$. The following lemma provides an interpretable upper bound on $L(A_*,B_*,U,\epsilon,\mathcal{I},w)$ when $\epsilon$ is small enough.

\begin{lemma}\label{lem:perturbation_sufficient}\com{[Done]}
Assume that $U$ has period $k$. Then as long as:
\begin{equation*}
\epsilon \leq \frac{1}{\max_{i \in \mathcal{I}} a \|  (e^{j\theta_i} I - A_*)^{-1} \|_2}
\end{equation*}
for some $a > 1$, then:
\begin{align*}
&  \max_{w \in \mathcal{M}, U \in \mathcal{U}_{\gamma^2}}  L(A_*,B_*,U,\epsilon,\mathcal{I},w) \\
= \ & \max_{\substack{w \in \mathcal{M}, U \in \mathcal{U}_{\gamma^2} \\ \delta' \in [0, \epsilon], \| \Delta \|_2 = 1 }} 2 \left | \sum_{i \in \mathcal{I}} w^\top (e^{j \theta_i} I - A_* - \delta' \Delta)^{-1} \Delta (e^{j \theta_i} I - A_* - \delta' \Delta)^{-1} B_* U_i U_i^H B_*^H  (e^{j \theta_i} I - A_* - \delta' \Delta)^{-H} w   \right | \\
\leq \ & \max_{\substack{w \in \mathcal{M}, i \in \mathcal{I}}} 2 \left ( \frac{a}{a - 1} \right)^3  k^2 \gamma^2 \| w^\top (e^{j\theta_i} I - A_*)^{-1} \|_2^2  \frac{\| (e^{j\theta_i} I - A_*)^{-1} B_* \|_2^2}{\|  (e^{j\theta_i} I - A_*)^{-1} \|_2}
\end{align*}
\end{lemma}
\begin{proof}
\begin{align*}
& \max_{\substack{w \in \mathcal{M}, U \in \mathcal{U}_{\gamma^2} \\ \delta' \in [0, \epsilon], \| \Delta \|_2 = 1}}  \left | \sum_{i \in \mathcal{I}} w^\top (e^{j \theta_i} I - A_* - \delta' \Delta)^{-1} \Delta (e^{j \theta_i} I - A_* - \delta' \Delta)^{-1} B_* U_i U_i^H B_*^H  (e^{j \theta_i} I - A_* - \delta' \Delta)^{-H} w   \right | \\
\leq \ & \max_{\substack{w \in \mathcal{M}, U \in \mathcal{U}_{\gamma^2} \\ \delta' \in [0, \epsilon], \| \Delta \|_2 = 1}}   \sum_{i \in \mathcal{I}}  \| w^\top (e^{j \theta_i} I - A_* - \delta' \Delta)^{-1} \|_2 \| (e^{j \theta_i} I - A_* - \delta' \Delta)^{-1} B_* U_i U_i^H B_*^H  (e^{j \theta_i} I - A_* - \delta' \Delta)^{-H} w  \|_2 \\
\leq \ & \max_{\substack{w \in \mathcal{M}, U \in \mathcal{U}_{\gamma^2} \\ \delta' \in [0, \epsilon], \| \Delta \|_2 = 1}}   \left ( \max_{i \in \mathcal{I}} \| w^\top (e^{j \theta_i} I - A_* - \delta' \Delta)^{-1} \|_2 \| (e^{j \theta_i} I - A_* - \delta' \Delta)^{-1} B_* \|_2 \| B_*^H  (e^{j \theta_i} I - A_* - \delta' \Delta)^{-H} w  \|_2 \right ) \\
& \hspace{4cm} \cdot \sum_{i \in \mathcal{I}} U_i^H U_i \\
\leq \ & \max_{\substack{w \in \mathcal{M}, i \in \mathcal{I} \\ \delta' \in [0, \epsilon], \| \Delta \|_2 = 1}}    k^2 \gamma^2   \| w^\top (e^{j \theta_i} I - A_* - \delta' \Delta)^{-1} \|_2 \| (e^{j \theta_i} I - A_* - \delta' \Delta)^{-1} B_* \|_2 \| B_*^H  (e^{j \theta_i} I - A_* - \delta' \Delta)^{-H} w  \|_2 
\end{align*}
where the final inequality holds since, by Parseval's Theorem:
\begin{equation*}
\frac{1}{T} \sum_{t=1}^T u_t^\top u_t = \frac{1}{T} \frac{T}{k} \frac{1}{k} \sum_{i=1}^k U_i^H U_i
\end{equation*} 
so:
\begin{equation*}
\frac{1}{T} \sum_{t=1}^T u_t^\top u_t \leq \gamma^2 \implies \sum_{i=1}^k U_i^H U_i \leq k^2 \gamma^2
\end{equation*}

By Lemma \ref{lem:tf_expansion} and our condition on $\epsilon$ we have that:
\begin{equation*}
(e^{j \theta_i} I - A_* - \delta \Delta)^{-1}  = \sum_{s=0}^\infty (e^{j\theta_i} I - A_*)^{-1} (\delta \Delta (e^{j\theta_i} I - A_*)^{-1})^s
\end{equation*}
Thus:
\begin{align*}
& \max_{\substack{w \in \mathcal{M}, \delta' \in [0, \epsilon] \\ \| \Delta \|_2 = 1, i \in \mathcal{I}}}   k^2 \gamma^2   \| w^\top (e^{j \theta_i} I - A_* - \delta' \Delta)^{-1} \|_2 \| (e^{j \theta_i} I - A_* - \delta' \Delta)^{-1} B_* \|_2 \| B_*^H  (e^{j \theta_i} I - A_* - \delta' \Delta)^{-H} w  \|_2 \\
= &  \max_{\substack{w \in \mathcal{M}, \delta' \in [0, \epsilon] \\ \| \Delta \|_2 = 1, i \in \mathcal{I}}}   k^2 \gamma^2  \left \| w^\top (e^{j\theta_i} I - A_*)^{-1} \sum_{s=0}^\infty  (\delta' \Delta (e^{j\theta_i} I - A_*)^{-1})^s \right \|_2 \\
& \ \ \ \ \ \ \ \ \ \ \ \ \ \ \ \ \ \ \ \ \cdot \left \| (e^{j\theta_i} I - A_*)^{-1} \sum_{s=0}^\infty  (\delta' \Delta (e^{j\theta_i} I - A_*)^{-1})^s B_* \right \|_2 \\
& \ \ \ \ \ \ \ \ \ \ \ \ \ \ \ \ \ \ \ \ \cdot \left \| w^\top (e^{j\theta_i} I - A_*)^{-1} \sum_{s=0}^\infty  (\delta' \Delta (e^{j\theta_i} I - A_*)^{-1})^s B_*  \right \|_2 \\
\leq & \max_{\substack{w \in \mathcal{M}, \delta' \in [0, \epsilon] \\ \| \Delta \|_2 = 1, i \in \mathcal{I}}} k^2 \gamma^2 \left ( \| w^\top (e^{j\theta_i} I - A_*)^{-1} \|_2 \sum_{s=0}^\infty \| (\delta' \Delta (e^{j\theta_i} I - A_*)^{-1})^s \|_2 \right ) \\
& \ \ \ \ \ \ \ \ \ \ \ \ \ \ \ \ \ \ \ \ \cdot \left ( \| (e^{j\theta_i} I - A_*)^{-1}\|_2 \sum_{s=0}^\infty  \| (\delta' \Delta (e^{j\theta_i} I - A_*)^{-1})^s B_* \|_2 \right )  \\
& \ \ \ \ \ \ \ \ \ \ \ \ \ \ \ \ \ \ \ \ \cdot \left ( \| w^\top (e^{j\theta_i} I - A_*)^{-1} \|_2 \sum_{s=0}^\infty  \| (\delta' \Delta (e^{j\theta_i} I - A_*)^{-1})^s B_* \|_2 \right ) \\
\leq & \max_{\substack{w \in \mathcal{M}, i \in \mathcal{I}}} k^2 \gamma^2 \left ( \| w^\top (e^{j\theta_i} I - A_*)^{-1} \|_2 \sum_{s=0}^\infty \epsilon^s \| (e^{j\theta_i} I - A_*)^{-1} \|_2^s \right ) \\
& \ \ \ \ \ \ \ \ \ \ \ \ \ \ \ \ \ \ \ \ \cdot \left ( \| (e^{j\theta_i} I - A_*)^{-1}\|_2 \sum_{s=0}^\infty  \epsilon^s \|  (e^{j\theta_i} I - A_*)^{-1} \|_2^{s-1} \| (e^{j\theta_i} I - A_*)^{-1} B_* \|_2 \right )  \\
& \ \ \ \ \ \ \ \ \ \ \ \ \ \ \ \ \ \ \ \ \cdot \left ( \| w^\top (e^{j\theta_i} I - A_*)^{-1} \|_2 \sum_{s=0}^\infty  \epsilon^s \|  (e^{j\theta_i} I - A_*)^{-1} \|_2^{s-1} \| (e^{j\theta_i} I - A_*)^{-1} B_* \|_2 \right ) 
\end{align*}
If:
\begin{equation*}
\epsilon \leq \frac{1}{\max_{i \in \mathcal{I}} a \|  (e^{j\theta_i} I - A_*)^{-1} \|_2}
\end{equation*}
then this can be upper bounded as:
\begin{align*}
\leq & \max_{\substack{w \in \mathcal{M}, i \in \mathcal{I}}} k^2 \gamma^2 \left ( \| w^\top (e^{j\theta_i} I - A_*)^{-1} \|_2 \sum_{s=0}^\infty \frac{1}{a^s} \right )  \cdot \left ( \| (e^{j\theta_i} I - A_*)^{-1} B_* \|_2 \sum_{s=0}^\infty \frac{1}{a^s}  \right )  \\
& \ \ \ \ \ \ \ \ \ \ \ \ \ \ \ \ \ \ \ \ \cdot \left ( \| w^\top (e^{j\theta_i} I - A_*)^{-1} \|_2 \frac{\| (e^{j\theta_i} I - A_*)^{-1} B_* \|_2}{\|  (e^{j\theta_i} I - A_*)^{-1} \|_2} \sum_{s=0}^\infty  \frac{1}{a^s}  \right ) \\
& \leq \max_{\substack{w \in \mathcal{M}, i \in \mathcal{I}}} \left ( \frac{a}{a - 1} \right)^3 k^2 \gamma^2 \| w^\top (e^{j\theta_i} I - A_*)^{-1} \|_2^2  \frac{\| (e^{j\theta_i} I - A_*)^{-1} B_* \|_2^2}{\|  (e^{j\theta_i} I - A_*)^{-1} \|_2}
\end{align*}
\end{proof}

To get deterministic bounds on the algorithm performance, it is helpful to deterministically upper bound $\mathcal{M}(A_*,\hat{A}, \{ x_t \}_{t=1}^T)$. The following lemma provides such a bound.

\begin{lemma}\label{lem:M_superset}\com{[Done]}
Assume that $A_* = P J P^{-1}$ is the Jordan decomposition of $A_*$, let $J_\ell$ denote the $\ell$th Jordan block, and assume $A_*$ has $r$ Jordan blocks. On the event that:
$$ \sum_{t=1}^T x_t x_t^\top \succeq c T \Gamma_k^\eta $$
and:
\begin{align*}
\sum_{t=1}^T ({w'}^\top x_t)^2 & \leq 4 \sum_{t=1}^T ({w'}^\top x_t^u)^2 + 4 T \left ( 1 +  \log \frac{2}{\delta} \right ) {w'}^\top  \left (\sigma^2  \Gamma_T  + \sigma_u^2  \Gamma_T^{B_*}  \right )  w'  
\end{align*}
for some $w'$ to be specified, and if:
\begin{equation*}
\epsilon \leq \frac{1}{\max_{i \in \mathcal{I}} 2 \| (e^{j \theta_i} I - A_*)^{-1} \|_2}
\end{equation*}
then:
\begin{align*}
& \mathcal{M}(A_*,\hat{A}, \{ x_t \}_{t=1}^T,\mathcal{I}) \subseteq \bar{\mathcal{M}}_k(A_*,B_*,\delta,\gamma^2)
\end{align*}
and:
\begin{align*}
& \mathcal{M}(\hat{A}, \{ x_t \}_{t=1}^T) \subseteq \bar{\mathcal{M}}_k(A_*,B_*,\delta,\gamma^2)
\end{align*}
where:
\begin{align*}
& \bar{\mathcal{M}}_k (A_*,B_*,\delta,\gamma^2)  := \bigg \{ w \in \mathcal{S}^{d-1} \ : \  \frac{T}{2T + T_0} w^\top \Gamma_k^\eta w \leq \min_{\ell \in [r] } \max_{\theta \in [0,2\pi]} 6 \gamma^2 \| P^{-1} \|_2^2 \| P \|_2^2 \| (e^{j \theta} I - J_\ell)^{-1} \|_2^2 \| B_* \|_2^2 \\
& \hspace{2cm} + 2  \left ( 1 +  \log \frac{2}{\delta} \right )  \left (\frac{\| P^{-1} \|_2^2 \| P \|_2^2 \beta(J_\ell)^2 (\sigma^2  + \sigma_u^2 \| B_* \|_2^2)}{1 - \bar{\rho}(J_\ell)^2}   \right )  + 16 \frac{\| P^{-1} \|_2^2 \| P \|_2^2 \| B_* \|_2^2 \beta(J_\ell)^2 \gamma^2 }{(1 - \bar{\rho}(J_\ell)^{2k_0})(1 - \bar{\rho}(J_\ell)^{2})}  \bigg \}
\end{align*}
and here $k$ is the frequency discretization at the epoch with end-time $T$.
\end{lemma}  
\begin{proof}
By definition:
\begin{align*}
& \mathcal{M}(A_*,\hat{A}, \{ x_t \}_{t=1}^T,\mathcal{I}) \\
& := \bigg \{ w \in \mathcal{S}^{d-1} \ : \ \frac{k^2}{2T + T_0} \sum_{t=1}^T (w^\top x_t)^2 \leq \min_{w' \in \mathcal{S}^{d-1}} \frac{k^2}{2T + T_0} \sum_{t=1}^T ({w'}^\top x_t)^2  \\
& \hspace{3cm} + k^2 \gamma^2 \max_{i \in \mathcal{I}} \max \{ \| {w'}^\top (e^{j \theta_i} I - A_*)^{-1} B_* \|_2^2, \| {w'}^\top (e^{j \theta_i} I - \hat{A})^{-1} B_* \|_2^2 \} \bigg \} \\
& = \bigg \{ w \in \mathcal{S}^{d-1} \ : \ \frac{1}{2T + T_0} \sum_{t=1}^T (w^\top x_t)^2 \leq \min_{w' \in \mathcal{S}^{d-1}}  \frac{1}{2T + T_0} \sum_{t=1}^T ({w'}^\top x_t)^2  \\
& \hspace{3cm} +  \gamma^2 \max_{i \in \mathcal{I}} \max \{ \| {w'}^\top (e^{j \theta_i} I - A_*)^{-1} B_* \|_2^2, \| {w'}^\top (e^{j \theta_i} I - \hat{A})^{-1} B_* \|_2^2 \} \bigg \}
\end{align*}
By Lemma \ref{lem:tf_expansion} and our condition on $\epsilon$, we have that:
\begin{align*}
\max_{i \in \mathcal{I}}  \| {w'}^\top (e^{j \theta_i} I - \hat{A})^{-1} B_* \|_2 & = \max_{i \in \mathcal{I}} \left \| {w'}^\top \sum_{s=0}^\infty (e^{j \theta_i} I - A_*)^{-1} (\delta \Delta (e^{j\theta_i} I - A_*)^{-1})^s B_* \right \|_2 \\
& \leq \max_{i \in \mathcal{I}}  \| {w'}^\top (e^{j \theta_i} I - A_*)^{-1} \|_2 \sum_{s=0}^\infty  \epsilon^s \|   (e^{j\theta_i} I - A_*)^{-1} \|_2^s \| B_*  \|_2 \\
& \leq \max_{i \in \mathcal{I}}  \| {w'}^\top (e^{j \theta_i} I - A_*)^{-1} \|_2 \| B_* \|_2 \sum_{s=0}^\infty \frac{1}{2^s} \\
& = \max_{i \in \mathcal{I}}  2 \| {w'}^\top (e^{j \theta_i} I - A_*)^{-1} \|_2 \| B_* \|_2 
\end{align*}
By assumption:
\begin{align*}
\sum_{t=1}^T ({w'}^\top x_t)^2 & \leq 4 \sum_{t=1}^T ({w'}^\top x_t^u)^2 + 4 T \left ( 1 +  \log \frac{2}{\delta} \right ) {w'}^\top  \left (\sigma^2  \Gamma_T  + \sigma_u^2  \Gamma_T^{B_*}  \right )  w'  
\end{align*}
Lemma \ref{lem:xtu_upper_bound} implies that, assuming we choose $w'$ such that $\left \| {w'}^\top P_{\underline{n}(j): \overline{n}(j)} \right \|_2 = 0$ for $j \neq \ell$ and that $T$ chosen such that it is within epoch $i$:
\begin{equation*}
\sum_{t=1}^T ({w'}^\top x_t^u)^2 \leq 3 T_i  \gamma^2 \max_{\theta \in [0,2\pi]} \| {w'}^\top (e^{j \theta} I - A_*)^{-1} \|_2^2 \| B_* \|_2^2 + 16 \frac{\| P^{-1} \|_2^2 \| P \|_2^2 \| B_* \|_2^2 \beta(J_\ell)^2 \gamma^2 k_i}{(1 - \bar{\rho}(J_\ell)^{2k_0})(1 - \bar{\rho}(J_\ell)^{2})} 
\end{equation*}
Following the computation from Lemma \ref{lem:beta_w_bound} and noting that:
\begin{equation*}
\| {w'}^\top (e^{j \theta} I - A_*)^{-1} \|_2^2 \leq \| P^{-1} \|_2^2 \| {w'}^\top P (e^{j \theta} I - J)^{-1} \|_2^2
\end{equation*}
and that the inverse of a block diagonal matrix is equal to the matrix formed from each of the blocks inverted individually, we then have that:
\begin{equation*}
\max_{\theta \in [0,2\pi]} \| {w'}^\top (e^{j \theta} I - A_*)^{-1} \|_2^2 \leq \| P^{-1} \|_2^2 \| P \|_2^2 \max_{\theta \in [0,2\pi]} \| (e^{j \theta} I - J_\ell)^{-1} \|_2^2
\end{equation*}
In addition, Lemma \ref{lem:beta_w_bound}, gives that:
\begin{align*}
{w'}^\top \Gamma_T w' & = \sum_{s=0}^{T-1} \| {w'}^\top A_*^s \|_2^2 \\
& \leq \| P^{-1} \|_2^2 \| P \|_2^2 \beta(J_\ell)^2 \sum_{s=0}^{T-1} \bar{\rho}(J_\ell)^{2s} \\
& \leq \frac{\| P^{-1} \|_2^2 \| P \|_2^2 \beta(J_\ell)^2}{1 - \bar{\rho}(J_\ell)^2}
\end{align*}
and a similar calculation holds for ${w'}^\top \Gamma_T^{B_*} w'$.

Combining everything gives:
\begin{align*}
& \min_{w' \in \mathcal{S}^{d-1}} \gamma^2 \max_{i \in \mathcal{I}} \max \{ \| {w'}^\top (e^{j \theta_i} I - A_*)^{-1} B_* \|_2^2, \| {w'}^\top (e^{j \theta_i} I - \hat{A})^{-1} B_* \|_2^2 \}  + \frac{1}{2T + T_0} \sum_{t=1}^T ({w'}^\top x_t)^2 \\
\overset{(a)}{\leq} \ & \min_{\ell \in [r]}  \max_{i \in \mathcal{I}} 3 \gamma^2 \| P^{-1} \|_2^2 \| {w'}^\top P (e^{j \theta} I - J)^{-1} \|_2^2 \| B_* \|_2^2 \\
& \ \ \ \ \ \ \ \ \  + 4 \frac{1}{2T + T_0} T \left ( 1 +  \log \frac{2}{\delta} \right )  \left (\frac{\| P^{-1} \|_2^2 \| P \|_2^2 \beta(J_\ell)^2 (\sigma^2  + \sigma_u^2 \| B_* \|_2^2)}{1 - \bar{\rho}(J_\ell)^2}   \right )   \\
& \ \ \ \ \ \ \ \ \ +   3\gamma^2 \max_{\theta \in [0,2\pi]} \| {w'}^\top (e^{j \theta} I - A_*)^{-1} \|_2^2 \| B_* \|_2^2 + 16 \frac{1}{T_i} \frac{\| P^{-1} \|_2^2 \| P \|_2^2 \| B_* \|_2^2 \beta(J_\ell)^2 \gamma^2 k_i}{(1 - \bar{\rho}(J_\ell)^{2k_0})(1 - \bar{\rho}(J_\ell)^{2})} \\
\leq \ & \min_{\ell \in [r]}  \max_{\theta \in [0,2\pi]} 6 \gamma^2 \| P^{-1} \|_2^2 \| {w'}^\top P (e^{j \theta} I - J)^{-1} \|_2^2 \| B_* \|_2^2    \\
& \ \ \ \ \ \ \ \ \  + 2 \left ( 1 +  \log \frac{2}{\delta} \right )  \left (\frac{\| P^{-1} \|_2^2 \| P \|_2^2 \beta(J_\ell)^2 (\sigma^2  + \sigma_u^2 \| B_* \|_2^2)}{1 - \bar{\rho}(J_\ell)^2}   \right )  + 16 \frac{1}{T_i} \frac{\| P^{-1} \|_2^2 \| P \|_2^2 \| B_* \|_2^2 \beta(J_\ell)^2 \gamma^2 k_i}{(1 - \bar{\rho}(J_\ell)^{2k_0})(1 - \bar{\rho}(J_\ell)^{2})} \\
\leq \ & \min_{\ell \in [r]}  \max_{\theta \in [0,2\pi]} 6 \gamma^2 \| P^{-1} \|_2^2 \| P \|_2^2 \| (e^{j \theta} I - J_\ell)^{-1} \|_2^2 \| B_* \|_2^2    \\
& \ \ \ \ \ \ \ \ \  + 2 \left ( 1 +  \log \frac{2}{\delta} \right )  \left (\frac{\| P^{-1} \|_2^2 \| P \|_2^2 \beta(J_\ell)^2 (\sigma^2  + \sigma_u^2 \| B_* \|_2^2)}{1 - \bar{\rho}(J_\ell)^2}   \right ) + 16  \frac{\| P^{-1} \|_2^2 \| P \|_2^2 \| B_* \|_2^2 \beta(J_\ell)^2 \gamma^2 }{(1 - \bar{\rho}(J_\ell)^{2k_0})(1 - \bar{\rho}(J_\ell)^{2})} 
\end{align*}
Assume that $\beta(J_k) \bar{\rho}(J_k) \leq \beta(J_i) \bar{\rho}(J_i)$ for all $i \neq k$, and let $w'$ be some vector such that $\left \| {w'}^\top P_{\underline{n}(i): \overline{n}(i)} \right \|_2 = 0$ for $i \neq k$. Note that $(a)$ will also upper bound:
$$ \min_{w' \in \mathcal{S}^{d-1}} \frac{4}{3} \gamma^2 \max_{i \in \mathcal{I}}  \| {w'}^\top (e^{j \theta_i} I - \hat{A})^{-1} B_* \|_2^2  + \frac{1}{2T + T_0} \sum_{t=1}^T ({w'}^\top x_t)^2$$
the upper bound in the membership condition of $\mathcal{M}(\hat{A},\{x_t\}_{t=1}^T)$.

Finally, given our assumption that $\sum_{t=1}^T x_t x_t^\top \succeq c T \Gamma_k^\eta$, we will have that:
\begin{equation*}
\sum_{t=1}^T (w^\top x_t)^2 \geq c T w^\top \Gamma_k^\eta w
\end{equation*}
and the result follows.
\end{proof}

Finally, for Theorem \ref{thm:asymp}, it is necessary to quantify how close $\Gamma_t(A_*)$ is to $\Gamma_t (\hat{A})$. This is quantified below.

\begin{lemma}\label{lem:gamma_perturbation}\com{[Done]}
Let $q$ be the dimension of the largest Jordan block of $A_*$. Then if $\| \hat{A} - A_* \|_2 \leq \epsilon$, for small enough $\epsilon$, where at least $\rho(A_*) + \sqrt[q]{ 2 \kappa(A_*) \epsilon} < 1$, we have:
\begin{equation*}
\left \|  \sum_{s=0}^{t-1} A_*^s (A_*^s)^\top -  \sum_{s=0}^{t-1} \hat{A}^s (\hat{A}^s)^\top \right \|_2  \leq \left (\max_{\theta \in [0,2\pi]} \frac{ 128 \| (e^{j \theta} I - A_*)^{-1} \|_2^3  }{\left (1 - \left (1/2 + \rho(A_*)/2 + \sqrt[q]{ 2 \kappa(A_*) \epsilon}/2 \right )^2 \right )^2} \right ) \epsilon
\end{equation*}
\end{lemma}
\begin{proof}
We first compute the directional derivate of $\sum_{s=0}^{t-1}  A_*^s (A_*^s)^\top $ with respect to $A_*$ in direction $\Delta$:
\begin{align*}
& D\left [ \sum_{s=0}^{t-1}  A_*^s (A_*^s)^\top  \right ] [\Delta] \\
= \ & \lim_{\delta \rightarrow 0} \frac{\sum_{s=0}^{t-1}  (A_*+\delta \Delta)^s ((A_*+\delta \Delta)^s)^\top  - \sum_{s=0}^{t-1}  A_*^s (A_*^s)^\top }{|\delta|} \\
= \ & \lim_{\delta \rightarrow 0} \frac{ \delta \sum_{s=1}^{t-1} A_*^s \sum_{\ell=0}^{s-1} (A_*^\ell)^\top \Delta^\top (A_*^{s-\ell-1})^\top  + \delta \sum_{s=1}^{t-1} \left ( \sum_{\ell=0}^{s-1} A_*^\ell \Delta A_*^{s-\ell-1} \right ) (A_*^s)^\top  + O(\delta^2)}{|\delta|} \\
= \ &  \sum_{s=1}^{t-1}  A_*^s \sum_{\ell=0}^{s-1} (A_*^\ell)^\top \Delta^\top (A_*^{s-\ell-1})^\top +  \sum_{s=1}^{t-1} \left ( \sum_{\ell=0}^{s-1} A_*^\ell \Delta A_*^{s-\ell-1} \right ) (A_*^s)^\top
\end{align*}
Thus:
\begin{align*}
& \left \| \sum_{s=0}^{t-1} A_*^s (A_*^s)^\top -  \sum_{s=0}^{t-1} \hat{A}^s (\hat{A}^s)^\top \right \|_2 \\
\leq \ & \left ( \max_{\substack{\Delta : \| \Delta \|_2 = 1 \\ A' : \| A_* - A' \|_2 \leq \epsilon}} \left \| \sum_{s=1}^{t-1}  {A'}^s \sum_{\ell=0}^{s-1} ({A'}^\ell)^\top \Delta^\top ({A'}^{s-\ell-1})^\top +  \sum_{s=1}^{t-1} \left ( \sum_{\ell=0}^{s-1} {A'}^\ell \Delta {A'}^{s-\ell-1} \right ) ({A'}^s)^\top  \right \|_2 \right ) \| A_* - \hat{A} \|_2 \\
\leq \ & \left ( \max_{A' : \| A_* - A' \|_2 \leq \epsilon} 2  \sum_{s=1}^{t-1} \sum_{\ell=0}^{s-1}  \|  {A'}^s \|_2   \| {A'}^\ell \|_2   \|  ( {A'}^{s-\ell-1}) \|_2 \right ) \epsilon \\
\leq \ & \left ( \max_{A' : \| A_* - A' \|_2 \leq \epsilon} 2  \sum_{s=1}^{t-1} \sum_{\ell=0}^{s-1}  \beta(A')^3 \bar{\rho}(A')^{2s - \ell - 1} \bar{\rho}(A')^{\ell}  \right ) \epsilon \\
\leq \ & \left ( \max_{A' : \| A_* - A' \|_2 \leq \epsilon} \frac{ 2  \beta(A')^3  \bar{\rho}(A') }{(1 - \bar{\rho}(A')^2)^2} \right ) \epsilon
\end{align*}
We can upper bound $\beta(A')$ as:
\begin{equation*}
\beta(A') \leq \max_{\theta \in [0,2\pi]} 2 \| (e^{j\theta} I - A')^{-1} \|_2
\end{equation*}
Writing $A' = A_* + \delta \Delta$ for $\delta \in [0, \epsilon]$ and $\| \Delta \|_2 = 1$, by Lemma \ref{lem:tf_expansion}, if $\epsilon \leq \frac{1}{\max_{\theta \in [0,2\pi]} 2 \| (e^{j \theta} I - A_*)^{-1} \|_2}$:
\begin{align*}
\max_{\theta \in [0,2\pi]} & \| (e^{j\theta} I - A_* - \delta \Delta)^{-1} \|_2  = \max_{\theta \in [0,2\pi]} \left \| (e^{j \theta} I - A_*)^{-1} \sum_{s=0}^\infty (\delta \Delta (e^{j \theta} I - A_*)^{-1})^s \right \|_2 \\
& \leq \max_{\theta \in [0,2\pi]} \| (e^{j \theta} I - A_*)^{-1} \|_2 \sum_{s=0}^\infty \epsilon^s \| (e^{j \theta} I - A_*)^{-1} \|_2^s  \leq \max_{\theta \in [0,2\pi]} \| (e^{j \theta} I - A_*)^{-1} \|_2 \sum_{s=0}^\infty \frac{1}{2^s} \\
& = \max_{\theta \in [0,2\pi]} 2 \| (e^{j \theta} I - A_*)^{-1} \|_2
\end{align*}
By Lemma \ref{lem:rho_perturbation} we will have that:
\begin{equation*}
 \max_{A' : \| A_* - A' \|_2 \leq \epsilon} \rho(A') \leq \rho(A_*) + \sqrt[q]{ 2 \kappa(A_*) \epsilon}
\end{equation*}
Combining these we have that:
\begin{align*}
\left ( \max_{A' : \| A_* - A' \|_2 \leq \epsilon} \frac{ 2  \beta(A')^3 \bar{\rho}(A') }{(1 - \bar{\rho}(A')^2)^2} \right ) \epsilon \leq \left (\max_{\theta \in [0,2\pi]} \frac{ 128 \| (e^{j \theta} I - A_*)^{-1} \|_2^3  }{\left (1 - \left (1/2 + \rho(A_*)/2 + \sqrt[q]{ 2 \kappa(A_*) \epsilon}/2 \right )^2 \right )^2} \right ) \epsilon
\end{align*}
\end{proof}

\subsection{Additional Lemmas}
\begin{lemma}\label{lem:directional_perturbation}\com{[Done]}
If $\| A - \hat{A} \|_2 \leq \epsilon$, then for any $w \in \mathcal{S}^{d-1}$:
\begin{equation*}
\left | w^\top H_k(A,B,U,\mathcal{I}) w - w^\top H_k(\hat{A},B,U,\mathcal{I}) w \right | \leq \epsilon L(A,B,U,\epsilon,\mathcal{I},w) 
\end{equation*}
where:
\begin{align*}
& L(A,B,U,\epsilon,\mathcal{I},w) \\
& \ \ \ \ := \max_{\substack{\delta \in [0,\epsilon], \Delta \in \mathbb{R}^{d \times d} \\ \| \Delta \|_2 = 1} } \ 2 \left | \sum_{i \in \mathcal{I}} w^\top (e^{j \theta_i} I - A - \delta \Delta)^{-1} \Delta (e^{j \theta_i} I - A - \delta \Delta)^{-1} B U_i U_i^H B^H (e^{j \theta_i} I - A - \delta \Delta)^{-H} w \right |
\end{align*}
\end{lemma}
\begin{proof}
To bound $\left | w^\top H_k(A,B,U,\mathcal{I}) w - w^\top H_k(\hat{A},B,U,\mathcal{I}) w \right |$, we calculate the directional derivative of $w^\top H_k(A,B,U,\mathcal{I}) w$ with respect to $A$ and use this to bound the Lipschitz constant of the function $w^\top H_k(A,B,U,\mathcal{I}) w$. The directional derivative is given by:
\begin{align*}
D[w^\top H_k(A,B,U,\mathcal{I}) w ][\Delta] & = \lim_{\delta \rightarrow 0} \frac{w^\top H_k(A + \delta \Delta,B,U,\mathcal{I}) w - w^\top H_k(A,B,U,\mathcal{I}) w}{|\delta|} 
\end{align*}
Lemma \ref{lem:tf_expansion} gives that, for small enough $\delta$:
\begin{align*}
(e^{j \theta} I - A - \delta \Delta)^{-1} & = \sum_{s=0}^\infty (e^{j\theta} I - A)^{-1} (\delta \Delta (e^{j\theta} I - A)^{-1})^s 
\end{align*}
so:
\begin{align*}
w^\top H_k(A + \delta \Delta,B,U,\mathcal{I}) w & = w^\top \left ( \sum_{i \in \mathcal{I}} (e^{j \theta_i} I - A - \delta \Delta)^{-1} B U_i U_i^H B^H (e^{j\theta_i} I - A - \delta \Delta)^{-H} \right ) w \\
& = w^\top \left ( \sum_{i \in \mathcal{I}} (e^{j \theta_i} I - A )^{-1} B U_i U_i^H B^H (e^{j\theta_i} I - A )^{-H} \right ) w \\
& \ \ \ \ + 2 \delta w^\top \left ( \sum_{i \in \mathcal{I}} (e^{j \theta_i} I - A )^{-1} \Delta (e^{j \theta_i} I - A )^{-1} B U_i U_i^H B^H (e^{j\theta_i} I - A )^{-H} \right ) w \\
& \ \ \ \ + O(\delta^2)
\end{align*}
and thus:
\begin{align*}
& \lim_{\delta \rightarrow 0} \frac{w^\top H_k(A + \delta \Delta,B,U,\mathcal{I}) w - w^\top H_k(A,B,U,\mathcal{I}) w}{|\delta|}  \\
& \ \ \ \ =  2 w^\top \left ( \sum_{i \in \mathcal{I}} (e^{j \theta_i} I - A )^{-1} \Delta (e^{j \theta_i} I - A )^{-1} B U_i U_i^H B^H (e^{j\theta_i} I - A )^{-H} \right ) w
\end{align*}
Given our assumption that $\| A - \hat{A} \|_2 \leq \epsilon$, we can bound the difference $\left | w^\top H_k(A,B,U,\mathcal{I}) w - w^\top H_k(\hat{A},B,U,\mathcal{I}) w \right |$ by bounding the Lipschitz constant of $w^\top H_k(A,B,U,\mathcal{I}) w$ over the domain $\{ A + \delta \Delta  \ : \ \delta \in [0,\epsilon], \Delta \in \mathbb{R}^{d \times d}, \| \Delta \|_2 = 1 \}$. Since a Lipschitz function is upper bounded by the derivative, this then gives that:
\begin{align*}
& \left | w^\top H_k(A,B,U,\mathcal{I}) w - w^\top H_k(\hat{A},B,U,\mathcal{I}) w \right | \\
 \leq \ & \left ( \max_{\substack{\delta \in [0,\epsilon], \Delta \in \mathbb{R}^{d \times d} \\ \| \Delta \|_2 = 1} } \ 2 \left | \sum_{i \in \mathcal{I}} w^\top (e^{j \theta_i} I - A - \delta \Delta)^{-1} \Delta (e^{j \theta_i} I - A - \delta \Delta)^{-1} B U_i U_i^H B^H (e^{j \theta_i} I - A - \delta \Delta)^{-H} w \right | \right ) \\
 & \hspace{3cm} \cdot \| A - \hat{A} \|_2 \\
 \leq \ & \epsilon L(A,B,U,\epsilon, \mathcal{I}, w)
\end{align*}
\end{proof}

\begin{lemma}\label{lem:tf_expansion}\com{[Done]}
For $\delta < \frac{1}{\| (e^{j \theta} I - A)^{-1} \|_2}$:
\begin{align*}
(e^{j \theta} I - A - \delta \Delta)^{-1} & = \sum_{s=0}^\infty (e^{j\theta} I - A)^{-1} (\delta \Delta (e^{j\theta} I - A)^{-1})^s 
\end{align*}
\end{lemma}
\begin{proof}
To see that this is true, we can simply multiply the right hand side above by $(e^{j \theta} I - A - \delta \Delta)$ and observe that the result is $I$. We wish to show that:
$$ \left ( \sum_{s=0}^\infty (e^{j\theta} I - A)^{-1} (\delta \Delta (e^{j\theta} I - A)^{-1})^s \right ) (e^{j \theta} I - A - \delta \Delta) = I$$
Consider, for fixed $n$:
\begin{align*}
& \left \| \left ( \sum_{s=0}^n (e^{j\theta} I - A)^{-1} (\delta \Delta (e^{j\theta} I - A)^{-1})^s \right ) (e^{j \theta} I - A - \delta \Delta) - I \right \|_2 \\
= \ & \left \| I - \delta (e^{j\theta} I - A)^{-1} \Delta + \sum_{s=1}^n (e^{j\theta} I - A)^{-1} \left (\delta \Delta (e^{j\theta} I - A)^{-1} \right )^{s-1} \left ( \delta \Delta - \delta^2 \Delta (e^{j\theta} I - A)^{-1} \Delta \right ) - I \right \|_2 \\
= \ & \left \| I - \delta (e^{j\theta} I - A)^{-1} \Delta + \sum_{s=1}^n \delta (e^{j\theta} I - A)^{-1} \Big (  \left (\delta \Delta (e^{j\theta} I - A)^{-1} \right )^{s-1}  -  \left (\delta \Delta (e^{j\theta} I - A)^{-1} \right )^{s} \Big ) \Delta - I \right \|_2 \\
= \ & \left \| I - \delta (e^{j\theta} I - A)^{-1} \left (\delta \Delta (e^{j\theta} I - A)^{-1} \right )^{n} \Delta - I \right \|_2 \\ 
\leq \ & \delta^{n+1} \| (e^{j\theta} I - A)^{-1} \|_2^{n+1}
\end{align*}
Since $\delta < \frac{1}{\| (e^{j \theta} I - A)^{-1} \|_2}$, we can make $\delta^{n+1} \| (e^{j\theta} I - A)^{-1} \|_2^{n+1}$ arbitrarily small by making $n$ large. Thus, for any $\epsilon > 0$, we can find an $N$ such that for all $n \geq N$:
$$\left \| \left ( \sum_{s=0}^n (e^{j\theta} I - A)^{-1} (\delta \Delta (e^{j\theta} I - A)^{-1})^s \right ) (e^{j \theta} I - A - \delta \Delta) - I \right \|_2 \leq \epsilon$$
This implies that:
\begin{align*}
 \lim_{n \rightarrow \infty} & \left ( \sum_{s=0}^n  (e^{j\theta} I - A)^{-1} (\delta \Delta (e^{j\theta} I - A)^{-1})^s \right ) (e^{j \theta} I - A - \delta \Delta) \\
& =  \ \left ( \sum_{s=0}^\infty (e^{j\theta} I - A)^{-1} (\delta \Delta (e^{j\theta} I - A)^{-1})^s \right ) (e^{j \theta} I - A - \delta \Delta) \\
& = \ I 
\end{align*}
\end{proof}

\begin{lemma}\label{lem:tf_approx_bound}
If: 
\begin{equation*}
\| A - \hat{A} \|_2 \leq \frac{1}{a \| (e^{j \theta} I - A)^{-1} \|_2}
\end{equation*}
for some $a > 1$, then:
\begin{equation*}
\| (e^{j \theta} I - \hat{A})^{-1} \|_2 \leq \frac{a}{a - 1} \| (e^{j \theta} I - A)^{-1} \|_2
\end{equation*}
and:
\begin{equation*}
\| w^\top (e^{j \theta} I - \hat{A})^{-1} \|_2 \leq \frac{a}{a - 1} \| w^\top (e^{j \theta} I - A)^{-1} \|_2
\end{equation*}
\end{lemma}
\begin{proof}
Denote $\hat{A} = A + \delta \Delta$ for some $\| \Delta \|_2 = 1$ and $\delta \leq \frac{1}{a \| (e^{j \theta} I - A)^{-1} \|_2}$. By Lemma \ref{lem:tf_expansion}:
\begin{align*}
\| (e^{j \theta} I - & A - \delta \Delta)^{-1} \|_2  = \left \| \sum_{s=0}^\infty (e^{j\theta} I - A)^{-1} (\delta \Delta (e^{j\theta} I - A)^{-1})^s \right \|_2 \\
& \leq \| (e^{j\theta} I - A)^{-1} \|_2 \sum_{s=0}^\infty \| (\delta \Delta (e^{j\theta} I - A)^{-1})^s \|_2  \leq \| (e^{j\theta} I - A)^{-1} \|_2 \sum_{s=0}^\infty \delta^s \| (e^{j\theta} I - A)^{-1} \|_2^s \\
& \leq \| (e^{j\theta} I - A)^{-1} \|_2 \sum_{s=0}^\infty \frac{1}{a^s}  = \frac{a}{a - 1} \| (e^{j \theta} I - A)^{-1} \|_2
\end{align*}
For the second inequality, we can simply multiply the first term in the expression by $w^\top$ and we see that the result holds. 
\end{proof}

\begin{lemma}\label{lem:rho_perturbation}\com{[Done]}
Assume that $\| A - \hat{A} \|_2 \leq \epsilon$ for some small enough $\epsilon$. Denote by $\rho(\hat{A})$ the spectral radius of $\hat{A}$. Let $A = P J P^{-1}$ be the Jordan decomposition of $A$. Then if $J$ is diagonal, we will have that:
\begin{equation*}
\rho(\hat{A}) \leq \rho(A) + \kappa(A) \epsilon
\end{equation*}
where $\kappa(A) = \| P \|_2 \| P^{-1} \|_2$. If $J$ is not diagonal then, letting $n$ be the dimension of its largest Jordan block:
\begin{equation*}
\rho(\hat{A}) \leq \rho(A) + \sqrt[n]{ 2 \kappa(A) \epsilon}
\end{equation*}
\end{lemma}
\begin{proof}
Let $A = P J P^{-1}$ be the Jordan decomposition of $A$. Assume that $\hat{A} = A + \delta \Delta$ where $\delta \in [0,\epsilon]$ and $\| \Delta \|_2 = 1$. Let $\mu$ be the eigenvalue of $\hat{A}$ with largest magnitude and assume that $\mu$ is not an eigenvalue of $A$ (otherwise we are trivially done). Since $\mu$ is an eigenvalue of $\hat{A}$, following a standard proof of the Bauer-Fike Theorem we have:
\begin{align*}
0 & = \det ( A + \delta \Delta - \mu I )  = \det( P^{-1}) \det ( A + \delta \Delta - \mu I ) \det(P) \\
& = \det ( P^{-1}( A + \delta \Delta - \mu I ) P )  = \det(J + \delta P^{-1} \Delta P - \mu I) \\
& = \det(J - \mu I) \det( \delta (J - \mu I)^{-1} P^{-1} \Delta P + I)
\end{align*}
Since by assumption $\mu$ is not an eigenvalue of $A$, $\det(J - \mu I) \neq 0$, which implies that $-1$ is an eigenvalue of $\delta (J - \mu I)^{-1} P^{-1} \Delta P$. Since the spectral norm upper bounds all eigenvalues:
\begin{align*}
1 & \leq \| \delta (J - \mu I)^{-1} P^{-1} \Delta P \|_2 \\
& \leq \delta \| (J - \mu I)^{-1} \|_2 \| P^{-1} \|_2 \| P \|_2 \\
\end{align*}
so:
\begin{equation}\label{eq:jordan_norm1}
\frac{1}{\| (J - \mu I)^{-1} \|_2} \leq \kappa(A) \delta
\end{equation}
If $J$ is diagonal, then $\| (J - \mu I)^{-1} \|_2 = \frac{1}{\min_{i} | \lambda_{i}(A) - \mu |}$. Denoting $i^* = \argmin_{i} | \lambda_{i}(A) - \mu |$, we then have:
\begin{equation*}
| \lambda_{i^*}(A) - \mu | \leq \kappa(A) \delta \implies | \mu | \leq | \lambda_{i^*}(A) | + \kappa(A) \delta \leq \rho(A) + \kappa(A) \epsilon
\end{equation*}
where the implication follows by the reverse triangle inequality. 

If $J$ is not diagonal, then $J - \mu I$ will be a Jordan form with eigenvalues $\lambda_i - \mu$. In particular then we have:
\begin{equation*}
J - \mu I = \begin{bmatrix} \tilde{J}_1 & 0 & \ldots & 0 \\
0 & \tilde{J}_2 & \ldots & 0 \\
\vdots & \vdots & \ddots & \vdots \\
0 & 0 & \ldots & \tilde{J}_k
\end{bmatrix} = \begin{bmatrix} (\lambda_1 - \mu) I_{n(1)} + D_{n(1)} & 0 & \ldots & 0 \\
0 & (\lambda_2 - \mu) I_{n(2)} + D_{n(2)} & \ldots & 0 \\
\vdots & \vdots & \ddots & \vdots \\
0 & 0 & \ldots & (\lambda_k - \mu) I_{n(k)} + D_{n(k)}
\end{bmatrix}
\end{equation*}
where $\tilde{J}_i$ is the $i$th Jordan block of $J-\mu I$, $n(i)$ is the dimension of the $i$th Jordan block, and:
\begin{equation*}
D_{n} = \begin{bmatrix} 0 & 1 & 0 & \ldots & 0 \\
0 & 0 & 1 & \ldots & 0 \\
\vdots & \vdots & \vdots & \ddots & \vdots \\
0 & 0 & 0 & \ldots & 1 \\
0 & 0 & 0 & \ldots & 0
\end{bmatrix} \in \mathbb{R}^{n \times n}
\end{equation*}
Since the inverse of a block diagonal matrix is simply formed by inverting each block, we can calculate $(J - \mu I)^{-1}$ by calculating the inverse of each block $(\lambda_i - \mu) I_{n(i)} + D_{n(i)}$ individually. Note that each block is invertible since we have assumed that $\mu$ is not an eigenvalue of $A$. By Taylor expanding, and the fact that $D_{n(i)}$ is nilpotent, we have:
\begin{equation*}
( (\lambda_i - \mu) I_{n(i)} + D_{n(i)})^{-1} = \sum_{\ell = 1}^{n(i)} \frac{1}{(\lambda_i - \mu)^\ell} D_{n(i)}^{\ell - 1} 
\end{equation*}
so:
\begin{equation*}
(J - \mu I)^{-1} = \begin{bmatrix} \sum_{\ell = 1}^{n(1)} \frac{1}{(\lambda_1 - \mu)^\ell} D_{n(1)}^{\ell - 1}  & 0 & \ldots & 0 \\
0 & \sum_{\ell = 1}^{n(2)} \frac{1}{(\lambda_2 - \mu)^\ell} D_{n(2)}^{\ell - 1}  & \ldots & 0 \\
\vdots & \vdots & \ddots & \vdots \\
0 & 0 & \ldots & \sum_{\ell = 1}^{n(k)} \frac{1}{(\lambda_k - \mu)^\ell} D_{n(k)}^{\ell - 1} 
\end{bmatrix} 
\end{equation*}
Since eigenvalues are continuous functions of the entries of a matrix \cite{horn2012matrix}, for small enough $\delta$, we will have that $|\mu - \lambda_i| \leq 1/2$ for some $i$. If this holds then:
\begin{align*}
\sum_{\ell = 1}^{n(i)-1}  |\lambda_i - \mu |^\ell & = \frac{1 - |\lambda_i - \mu |^{n(i)-1}}{\frac{1}{|\lambda_i - \mu |} - 1}  \leq \frac{1 }{\frac{1}{|\lambda_i - \mu |} - 1}  \leq 1
\end{align*}
Since:
\begin{equation*}
\sum_{\ell = 1}^{n(i)-1}  |\lambda_i - \mu |^\ell = | \lambda_i - \mu |^{n(i)} \sum_{\ell = 1}^{n(i)-1}  \frac{1}{|\lambda_i - \mu |^\ell }
\end{equation*}
it follows that:
\begin{equation*}
 \sum_{\ell = 1}^{n(i)}  \frac{1}{|\lambda_i - \mu |^\ell} \leq 2 \frac{1}{|\lambda_i - \mu |^{n(i)}} 
\end{equation*}
Then:
\begin{align*}
\| (J - \mu I)^{-1} \|_2 & = \max_{i = 1,...,k} \left \| \sum_{\ell = 1}^{n(i)} \frac{1}{(\lambda_i - \mu)^\ell} D_{n(i)}^{\ell - 1} \right \|_2 \\
& \leq \max_{i=1,...,k}  \sum_{\ell = 1}^{n(i)}  \frac{1}{|\lambda_i - \mu |^\ell} \left \| D_{n(i)}^{\ell - 1} \right \|_2 \\
& = \max_{i=1,...,k}  \sum_{\ell = 1}^{n(i)}  \frac{1}{|\lambda_i - \mu |^\ell} \\
& \leq \max_{i,j=1,...,k}  2 \frac{1}{|\lambda_i - \mu |^{n(j)}} 
\end{align*}
Combining this with (\ref{eq:jordan_norm1}) and denoting $i^*,j^*$ the indices at which the above maximum is achieved, we get that:
\begin{align*}
|\lambda_{i^*} - \mu |^{n(j^*)} \leq 2 \kappa(A) \delta & \implies  |\lambda_{i^*} - \mu | \leq \sqrt[n(j^*)]{ 2 \kappa(A) \delta} \\
& \implies  |\mu| \leq |\lambda_{i^*} | + \sqrt[n(j^*)]{ 2 \kappa(A) \delta} \\
& \implies  |\mu| \leq \rho (A)+ \sqrt[n(j^*)]{ 2 \kappa(A) \delta}
\end{align*}
\end{proof}

\begin{lemma}\label{lem:beta_w_bound}\com{[Done]}
Let $A = PJP^{-1}$ be the Jordan decomposition of $A$. Assume that $A$ has $r$ Jordan blocks and denote by $\underline{n}(i)$ and $\overline{n}(i)$ the start and stop indices of the $i$th Jordan block (so in particular, if $J_i$ is the $i$th Jordan block, we have that $J_i = [J]_{\underline{n}(i):\overline{n}(i),\underline{n}(i):\overline{n}(i)}$). Let $P_{i:j}$ to denote $[p_i,...,p_j]$, the matrix with columns equal to the $i$th to $j$th columns of $P$. Then:
\begin{equation*}
\| w^\top A^{\ell} \|_2 \leq \left \| P^{-1} \right \|_2 \sum_{i=1}^r \left \| w^\top P_{\underline{n}(i): \overline{n}(i)} \right \| \beta(J_i) \bar{\rho}(J_i)^{\ell}
\end{equation*}
\end{lemma}
\begin{proof}
We have:
\begin{align*}
\| w^\top A^{\ell} \|_2 & = \| w^\top P J^{\ell} P^{-1} \|_2  \leq \| w^\top P J^{\ell} \|_2 \| P^{-1} \|_2  = \left \| [w^\top p_1, \ldots, w^\top p_d ] J^{\ell} \right \|_2  \left \| P^{-1} \right \|_2 \\
& = \left \| \left [ [w^\top p_1, \ldots, w^\top p_{\overline{n}(1)}] J_1^{\ell}, \ldots, [w^\top p_{\underline{n}(r)}, \ldots, w^\top p_{\overline{n}(r)}] J_r^{\ell} \right ] \right \|_2  \left \| P^{-1} \right \|_2
\end{align*}
Since, for nonnegative $a,b$, $\sqrt{a + b} \leq \sqrt{a} + \sqrt{b}$ (by virtue of the fact that $a+b \leq (\sqrt{a} + \sqrt{b})^2 = a + b + 2 \sqrt{a} \sqrt{b}$), it then follows that:
\begin{align*}
\left \| \left [ [w^\top p_1, \ldots, w^\top p_{\overline{n}(1)}] J_1^{\ell}, \ldots, [w^\top p_{\underline{n}(r)}, \ldots, w^\top p_{\overline{n}(r)}] J_r^{\ell} \right ] \right \|_2 & \leq \sum_{i=1}^r \left \| [w^\top p_{\underline{n}(i)}, \ldots, w^\top p_{\overline{n}(i)}] J_i^{\ell} \right \|_2 \\
& =  \sum_{i=1}^r \left \| w^\top P_{\underline{n}(i): \overline{n}(i)} J_i^{\ell} \right \|_2 \\
& \leq \sum_{i=1}^r \left \| w^\top P_{\underline{n}(i): \overline{n}(i)} \right \| \left \| J_i^{\ell} \right \|_2 \\
& \leq \sum_{i=1}^r \left \| w^\top P_{\underline{n}(i): \overline{n}(i)} \right \| \beta(J_i) \bar{\rho}(J_i)^{\ell}
\end{align*}
\end{proof}


\section{Lower Bound}\label{sec:lower_bound}
We base our analysis off the lower bound presented in \cite{jedra2019sample}. A slight modification of their analysis to our situation yields the following result.

\begin{theorem}\label{thm:lb1}\com{[Done]}
For any matrix $A_*$, for all $\epsilon > 0, \delta \in (0,1)$, the sample complexity $\tau_{\epsilon \delta}$ of any $(\epsilon,\delta)$-locally-stable algorithm in $A_*$ satisfies:
\begin{equation*}
\lambda_{\min} \left ( \mathbb{E} \left [ \sum_{t=1}^{\tau_{\epsilon \delta}} x_t x_t^\top \right ] \right ) \geq \frac{\sigma^2}{2 \epsilon^2} \log \frac{1}{2.4 \delta}
\end{equation*}
\end{theorem}
\begin{proof}
The proof of this result is essentially identical to the proof of Theorem 1 in \cite{jedra2019sample} and we omit it here.
\end{proof}

Denoting $x_t^u$ the response of the system due to the input and $x_t^\eta$ the response due to the noise, we can write:
\begin{align*}
\sum_{t=1}^{T } x_t x_t^\top & = \sum_{t=1}^{T } \left [ x_t^u {x_t^u}^\top + x_t^{\eta} {x_t^\eta}^\top + x_t^u {x_t^\eta}^\top + x_t^\eta {x_t^u}^\top \right ] \overset{a.s.}{\preceq} 2 \sum_{t=1}^{T } \left [ x_t^u {x_t^u}^\top + x_t^{\eta} {x_t^\eta}^\top  \right ]
\end{align*}
Thus:
\begin{equation*}
2 \lambda_{\min} \left ( \mathbb{E} \left [ \sum_{t=1}^{T } x_t^u {x_t^u}^\top + x_t^\eta {x_t^\eta}^\top \right ] \right ) \geq \lambda_{\min} \left ( \mathbb{E} \left [ \sum_{t=1}^{T } x_t x_t^\top \right ] \right )
\end{equation*}
so, Theorem \ref{thm:lb1} gives that:
\begin{equation}\label{eq:lb_separate}
  \lambda_{\min} \left ( \mathbb{E} \left [ \sum_{t=1}^{\tau_{\epsilon \delta}} x_t^u {x_t^u}^\top + x_t^\eta {x_t^\eta}^\top \right ] \right ) = \lambda_{\min} \left ( \mathbb{E} \left [ \sum_{t=1}^{\tau_{\epsilon \delta}} x_t^u {x_t^u}^\top \right ] + \sum_{t=1}^{\tau_{\epsilon \delta}} \sigma^2 \Gamma_t  \right ) \geq \frac{\sigma^2}{4 \epsilon^2} \log \frac{1}{2.4 \delta}
\end{equation}

\subsection{Proof of Theorem \ref{thm:lb}}
\begin{proof}
Since (\ref{eq:lb_separate}) holds for all input sequences $u_t$, and since we wish to minimize the lower bound, we will have in particular:
\begin{equation*}
\max_{u \in \mathcal{U}_{\gamma^2}} \ \lambda_{\min} \left ( \mathbb{E} \left [ \sum_{t=1}^{\tau_{\epsilon \delta}} x_t^u {x_t^u}^\top \right ] + \sum_{t=1}^{\tau_{\epsilon \delta}} \sigma^2 \Gamma_t  \right ) \geq \frac{\sigma^2}{4 \epsilon^2} \log \frac{1}{2.4 \delta}
\end{equation*}
Since $x_t^u$ is deterministic conditioned on $u_t$, maximizing $\lambda_{\min} \left ( \mathbb{E} \left [ \sum_{t=1}^{\tau_{\epsilon \delta}} x_t^u {x_t^u}^\top \right ] + \sum_{t=1}^{\tau_{\epsilon \delta}} \sigma^2 \Gamma_t  \right )$ is equivalent to maximizing $\lambda_{\min} \left (  \sum_{t=1}^{\tau_{\epsilon \delta}} x_t^u {x_t^u}^\top  + \sum_{t=1}^{\tau_{\epsilon \delta}} \sigma^2 \Gamma_t  \right )$.
For any input $u$ satisfying the power constraint given in the statement of Theorem \ref{thm:lb}, by Lemma \ref{lem:upper_bound3}:
\begin{align*}
& \lambda_{\min} \left ( \sum_{t=1}^{\tau_{\epsilon \delta}} x_t^u {x_t^u}^\top + \sum_{t=1}^{\tau_{\epsilon \delta}} \sigma^2 \Gamma_t  \right ) \\
& \leq \lambda_{\min} \Bigg ( \frac{1}{\tau_{\epsilon \delta}} \sum_{t=1}^{\tau_{\epsilon \delta}} G(e^{j\theta_t}) U(e^{j\theta_t}) U(e^{j\theta_t})^H G(e^{j\theta_t})^H + \frac{ 4 \beta(A_*)^2 k^2 \gamma^2}{(1 - \bar{\rho}(A_*)^k)^2} \left ( \max_{\theta \in [0, 2\pi]} \| G(e^{j\theta}) \|_2^2 \right ) \\
& \ \ \ \ \ \ \ \ \ \ \ \ \ \ \ \ \ \ \  + \frac{ 4 \beta(A_*) k \gamma^2 \sqrt{\tau_{\epsilon \delta}}}{1 - \bar{\rho}(A_*)^k} \left ( \max_{\theta \in [0, 2\pi]} \| G(e^{j\theta}) \|_2^2 \right )  + \sum_{t=1}^{\tau_{\epsilon \delta}} \sigma^2 \Gamma_t \Bigg ) 
\end{align*} 
Note that the term $\frac{1}{\tau_{\epsilon \delta}} \sum_{t=1}^{\tau_{\epsilon \delta}} G(e^{j\theta_t}) U(e^{j\theta_t}) U(e^{j\theta_t})^H G(e^{j\theta_t})^H + \sum_{t=1}^{\tau_{\epsilon \delta}} \sigma^2 \Gamma_t$ is scaling as $\tau_{\epsilon \delta}$ since $\Gamma_t \succeq I$. Thus, for large enough $\tau_{\epsilon \delta}$, since the left hand side is only scaling as $\sqrt{\tau_{\epsilon \delta}}$:
\begin{align*}
&  \frac{4 \beta(A_*)^2 k^2 \gamma^2}{(1 - \bar{\rho}(A_*)^k)^2} \left ( \max_{\theta \in [0, 2\pi]} \| G(e^{j\theta}) \|_2^2 \right ) I + \frac{ 4 \beta(A_*) k \gamma^2 \sqrt{\tau_{\epsilon \delta}}}{1 - \bar{\rho}(A_*)^k} \left ( \max_{\theta \in [0, 2\pi]} \| G(e^{j\theta}) \|_2^2 \right ) I \\
& \ \ \ \ \ \ \  \preceq  \frac{1}{\tau_{\epsilon \delta}} \sum_{t=1}^{\tau_{\epsilon \delta}} G(e^{j\theta_t}) U(e^{j\theta_t}) U(e^{j\theta_t})^H G(e^{j\theta_t})^H + \sum_{t=1}^{\tau_{\epsilon \delta}} \sigma^2 \Gamma_t
\end{align*}
so, for large enough $\tau_{\epsilon \delta}$: 
\begin{align*}
&  \lambda_{\min} \Bigg ( \frac{1}{\tau_{\epsilon \delta}} \sum_{t=1}^{\tau_{\epsilon \delta}} G(e^{j\theta_t}) U^*(e^{j\theta_t}) U^*(e^{j\theta_t})^H G(e^{j\theta_t})^H + \frac{ 4 \beta(A_*)^2 k^2 \gamma^2}{(1 - \bar{\rho}(A_*)^k)^2} \left ( \max_{\theta \in [0, 2\pi]} \| G(e^{j\theta}) \|_2^2 \right ) I \\
& \ \ \ \ \ \ \ \ \ \ \ \ \ \ \ \ \ \ \  + \frac{4 \beta(A_*) k \gamma^2 \sqrt{\tau_{\epsilon \delta}}}{1 - \bar{\rho}(A_*)^k} \left ( \max_{\theta \in [0, 2\pi]} \| G(e^{j\theta}) \|_2^2 \right ) I  + \sum_{t=1}^{\tau_{\epsilon \delta}} \sigma^2 \Gamma_t \Bigg ) \\
& \leq 2 \lambda_{\min} \Bigg ( \frac{1}{\tau_{\epsilon \delta}} \sum_{t=1}^{\tau_{\epsilon \delta}} G(e^{j\theta_t}) U(e^{j\theta_t}) U(e^{j\theta_t})^H G(e^{j\theta_t})^H   + \sum_{t=1}^{\tau_{\epsilon \delta}} \sigma^2 \Gamma_t \Bigg ) 
\end{align*}
 For small enough $\epsilon$, $\tau_{\epsilon \delta}$ will be sufficiently large for this to hold. We have then that:
 \begin{align*}
& \max_{U \in \mathcal{U}_{\gamma^2}}  2 \lambda_{\min} \Bigg ( \frac{1}{\tau_{\epsilon \delta}} \sum_{t=1}^{\tau_{\epsilon \delta}} G(e^{j\theta_t}) U(e^{j\theta_t}) U(e^{j\theta_t})^H G(e^{j\theta_t})^H   + \sum_{t=1}^{\tau_{\epsilon \delta}} \sigma^2 \Gamma_t \Bigg ) \\
& \ \ \ \ \ \geq \max_{u \in \mathcal{U}_{\gamma^2}} \ \lambda_{\min} \left ( \mathbb{E} \left [ \sum_{t=1}^{\tau_{\epsilon \delta}} x_t^u {x_t^u}^\top \right ] + \sum_{t=1}^{\tau_{\epsilon \delta}} \sigma^2 \Gamma_t  \right ) \geq \frac{\sigma^2}{4 \epsilon^2} \log \frac{1}{2.4 \delta}
\end{align*}

 By Lemma \ref{lem:sln_limit}, we know that:
 \begin{equation*}
\lim_{i \rightarrow \infty} \max_{u \in \mathcal{U}_{\gamma^2}} \lambda_{\min} ( \sigma^2 \Gamma_{2^i} + \tilde{\Gamma}_{2^i}^{u} ) = \max_{u \in \mathcal{U}_{\gamma^2}} \lambda_{\min} (\sigma^2 \Gamma_{\infty} + \tilde{\Gamma}_{\infty}^{u})
\end{equation*}
exists and, further, that:
\begin{equation*}
\max_{U \in \mathcal{U}_{\gamma^2}} \lambda_{\min} \Bigg ( \frac{1}{\tau_{\epsilon \delta}} \sum_{t=1}^{\tau_{\epsilon \delta}} G(e^{j\theta_t}) U(e^{j\theta_t}) U(e^{j\theta_t})^H G(e^{j\theta_t})^H   + \sum_{t=1}^{\tau_{\epsilon \delta}} \sigma^2 \Gamma_t \Bigg )  \leq \max_{u \in \mathcal{U}_{\gamma^2}} \lambda_{\min} (\tau_{\epsilon \delta} \sigma^2 \Gamma_{\infty} + \tau_{\epsilon \delta} \tilde{\Gamma}_{\infty}^{u})
\end{equation*}
for all $\tau_{\epsilon \delta}$. Thus, for small enough $\epsilon$, we will have that:
\begin{equation*}
\tau_{\epsilon \delta} \geq \frac{\sigma^2}{\max_{u \in \mathcal{U}_{\gamma^2}} 8 \epsilon^2  \lambda_{\min} \left (  \sigma^2 \Gamma_{\infty} +  \tilde{\Gamma}_{\infty}^{u}  \right )  } \log \frac{1}{2.4 \delta}
\end{equation*}
\end{proof}


\section{Additional Lemmas}
\begin{lemma}\label{lem:discrete1}
Assume that $\rho(A) < 1$. Then for any $\theta_1, \theta_2$, we will have that:
\begin{equation*}
\| (e^{j\theta_1} I - A)^{-1} - (e^{j\theta_2} I - A)^{-1} \|_2 \leq \left ( \max_{\theta \in [0, 2\pi]} \|  (e^{j \theta} I -  A)^{-2} \|_2 \right ) | \theta_1 - \theta_2 | 
\end{equation*}
so it follows that $(e^{j\theta} I - A)^{-1}$ is Lipschitz continuous in $\theta$.
\end{lemma}
\begin{proof}
Noting that, since we assume $\rho(A) < 1$, using the identity that $(I + A)^{-1} = I - A + A^2 - A^3 + ...$, we have:
\begin{equation*}
 (e^{j \theta} I - A)^{-1}  = \left ( e^{-j\theta}  I + e^{-j 2 \theta} A + e^{-j 3 \theta} A^2 + ... \right ) 
\end{equation*}
Thus:
\begin{align*}
\frac{d}{d \theta}  (e^{j \theta} I - A)^{-1} =\sum_{\ell=0}^\infty -j (\ell + 1) e^{-j (\ell +1) \theta } A^\ell 
\end{align*}
For any matrix $A$ with $\rho(A) < 1$ we have:
\begin{align*}
& (I + 2 A + 3 A^2 + 4 A^3 + ...)(I - A)^2  = (I + A + A^2 + A^3 + ...)(I - A)  = I \\
\implies \ & (I + 2 A + 3 A^2 + 4 A^3 + ...)^{-1} = (I - A)^{-2}
\end{align*}
which implies:
\begin{equation*}
\sum_{\ell=0}^\infty -j (\ell + 1) e^{-j (\ell +1) \theta } A^\ell  = -j e^{-j \theta } \sum_{\ell=0}^\infty  (\ell + 1) (e^{-j \theta } A)^\ell  = - j e^{-j \theta} (I - e^{- j \theta} A)^{-2} 
\end{equation*}
So the Lipschitz constant of $(e^{j \theta} I - A)^{-1}$ is bounded by:
\begin{equation*}
\max_{\theta \in [0, 2\pi]} \| - j e^{-j \theta} (I - e^{- j \theta} A)^{-2}  \|_2 \leq \max_{\theta \in [0, 2\pi]} \|  (e^{j \theta} I -  A)^{-2}  \|_2
\end{equation*}
from which the result follows directly. 
\end{proof}

\begin{lemma}\label{lem:sln_limit}
For any sequences of integers $n_i, m_i$ such that $\lim_{i \rightarrow \infty} n_i = \lim_{i \rightarrow \infty} m_i = \infty$, we will have that:
\begin{equation*}
\lim_{i \rightarrow \infty} \lambda_{\min}(\tilde{\Gamma}_{n_i}^{u^*}) = \lim_{j \rightarrow \infty} \lambda_{\min}(\tilde{\Gamma}_{m_j}^{u^*})
\end{equation*}
assuming the limit of each exists. Further, for any finite $j$, we will have:
\begin{equation*}
\lambda_{\min}(\tilde{\Gamma}_{m_j}^{u^*}) \leq \lim_{i \rightarrow \infty} \lambda_{\min}(\tilde{\Gamma}_{n_i}^{u^*})
\end{equation*}
\end{lemma}
\begin{proof} Assume the opposite, that there exists some sequence of integers $n_i, m_i$ satisfying the above condition such that $\lim_{i \rightarrow \infty} \lambda_{\min}(\tilde{\Gamma}_{n_i}^{u^*}) > \lim_{j \rightarrow \infty} \lambda_{\min}(\tilde{\Gamma}_{m_j}^{u^*})$. By the definition of a limit, this implies that there exists some finite $i_0$ such that for any $i \geq i_0$, we will have that $\lambda_{\min}(\tilde{\Gamma}_{n_i}^{u^*}) > \lambda_{\min}(\tilde{\Gamma}_{m_j}^{u^*})$ for all $j$. For any $\ell \in [n_{i_0}]$, note that we can make:
\begin{equation*}
\left | \frac{\ell}{n_{i_0}} - \frac{\ell(j)}{m_j} \right |
\end{equation*}
arbitrarily small for large enough $j$ (since $m_j \rightarrow \infty$ and by proper choice of $\ell(j)$). By Lemma \ref{lem:discrete1}, this implies that we can make:
\begin{equation*}
\left \| (e^{j  \frac{2 \pi \ell}{n_{i_0}}} I - A)^{-1 } - (e^{j \frac{2 \pi \ell(j)}{m_j} } I - A)^{-1} \right \|_2
\end{equation*}
arbitrarily small. Thus, for large enough $j$, we can simply set the inputs at positions $\frac{\ell(j)}{m_j}$ identical to those at positions $\frac{\ell}{n_{i_0}}$ for each $\ell$, and make $\lambda_{\min}(\tilde{\Gamma}_{m_j}^{u^*})$ arbitrarily close to $\lambda_{\min}(\tilde{\Gamma}_{n_{i_0}}^{u^*})$ while still meeting the feasibility constraint on the input. This contradicts the fact that $\lim_{i \rightarrow \infty} \lambda_{\min}(\tilde{\Gamma}_{n_i}^{u^*}) > \lim_{j \rightarrow \infty} \lambda_{\min}(\tilde{\Gamma}_{m_j}^{u^*})$, which implies that $\lim_{i \rightarrow \infty} \lambda_{\min}(\tilde{\Gamma}_{n_i}^{u^*}) = \lim_{j \rightarrow \infty} \lambda_{\min}(\tilde{\Gamma}_{m_j}^{u^*})$.

To see that:
\begin{equation*}
\lambda_{\min}(\tilde{\Gamma}_{m_j}^{u^*}) \leq \lim_{i \rightarrow \infty} \lambda_{\min}(\tilde{\Gamma}_{n_i}^{u^*})
\end{equation*}
assume that this is not the case, that there exists some finite $j$ such that $\lambda_{\min}(\tilde{\Gamma}_{m_j}^{u^*}) > \lim_{i \rightarrow \infty} \lambda_{\min}(\tilde{\Gamma}_{n_i}^{u^*})$. Then using the same argument as above, we can make $\lambda_{\min}(\tilde{\Gamma}_{n_i}^{u^*})$ arbitrarily close to $\lambda_{\min}(\tilde{\Gamma}_{m_j}^{u^*})$ for large enough $i$, which contradicts the fact that $\lambda_{\min}(\tilde{\Gamma}_{m_j}^{u^*}) > \lim_{i \rightarrow \infty} \lambda_{\min}(\tilde{\Gamma}_{n_i}^{u^*})$.
\end{proof}

\begin{lemma}\label{lem:cov_limit_existence}
For any integer $k_0$ and finite input power budget $\gamma^2$,
$$\lim_{i \rightarrow \infty} \max_{u \in \mathcal{U}_{\gamma^2}} \lambda_{\min} ( \sigma^2\Gamma_{k_0 2^i} + \tilde{\Gamma}_{k_0 2^i}^{u} )$$
exists and is finite.
\end{lemma}
\begin{proof}
Note that $\max_{u \in \mathcal{U}_{\gamma^2}} \lambda_{\min} (\sigma^2 \Gamma_{k_0 2^i} + \tilde{\Gamma}_{k_0 2^i}^{u} )$ will be bounded for all $i$ assuming our system is stable and the power of the inputs is constrained. Further, note that $\max_{u \in \mathcal{U}_{\gamma^2}} \lambda_{\min} ( \sigma^2 \Gamma_{k_0 2^i} + \tilde{\Gamma}_{k_0 2^i}^{u} ) \leq \max_{u \in \mathcal{U}_{\gamma^2}} \lambda_{\min} (\sigma^2 \Gamma_{k_0 2^j} + \tilde{\Gamma}_{k_0 2^j}^{u} )$ for $i \leq j$ since the frequencies optimized over to obtain $\max_{u \in \mathcal{U}_{\gamma^2}} \lambda_{\min} (\sigma^2 \Gamma_{k_0 2^j} + \tilde{\Gamma}_{k_0 2^j}^{u} )$ are a superset of those optimized over to obtain $\max_{u \in \mathcal{U}_{\gamma^2}} \lambda_{\min} (\sigma^2 \Gamma_{k_0 2^i} + \tilde{\Gamma}_{k_0 2^i}^{u} )$, and since $\Gamma_{k_0 2^i} \preceq \tilde{\Gamma}_{k_0 2^j}$. By the monotone convergence theorem, this implies that:
\begin{equation*}
\lim_{i \rightarrow \infty} \max_{u \in \mathcal{U}_{\gamma^2}} \lambda_{\min} ( \sigma^2\Gamma_{k_0 2^i} + \tilde{\Gamma}_{k_0 2^i}^{u} ) = c^*
\end{equation*}
exists and is finite. 
\end{proof}

\section{Suboptimality of Colored Noise}\label{sec:subopt_noise_proof}
First, note that satisfying the power constraint in this setting is equivalent to $Tr(\Sigma) \leq \gamma^2$. Under this constraint, the optimal noise covariance can be obtained by solving:
\begin{align*}
& \max_{\Sigma \succeq 0} \ \lambda_{\min} \left ( \sigma^2 \sum_{t=0}^k A_*^t (A_*^t)^\top + \sum_{t=0}^k A_*^t B_* \Sigma B_*^\top (A_*^t)^\top \right ) \\
& \ \ \text{s.t.} \ \  Tr(\Sigma) \leq \gamma^2
\end{align*}
In our setting, with $\gamma^2 \gg \sigma^2$, solving this is approximately equivalent to solving:
\begin{align*}
& \max_{\tilde{\Sigma} \succeq 0} \ \lambda_{\min} \left ( \sum_{t=0}^k \Lambda^t \tilde{\Sigma} \Lambda^t \right ) \\
& \ \ \text{s.t.} \ \  Tr(\tilde{\Sigma}) \leq \gamma^2
\end{align*}
where $\tilde{\Sigma} = V^\top \Sigma V$.  Let $\tilde{\Sigma}^*$ be the optimal diagonal solution, and note that, in this case, we will have:
$$ \sum_{t=0}^k \Lambda^t \tilde{\Sigma}^* \Lambda^t = \frac{\gamma^2}{\sum_{i=1}^d \frac{1 - \lambda_i^2}{1 - \lambda_i^{2k}}} I $$
To see this, note that for any diagonal $\tilde{\Sigma}$ with $i$th element $\gamma_i^2$:
$$ \left [ \sum_{t=0}^k \Lambda^t \tilde{\Sigma} \Lambda^t \right ]_{ii} = \frac{\gamma_i^2 (1 - \lambda_i^{2k})}{1 - \lambda_i^2} $$
The optimal solution will clearly be the solution that balances the energy in every diagonal element, that is:
$$ \frac{\gamma_i^2 (1 - \lambda_i^{2k})}{1 - \lambda_i^2} = \frac{\gamma_j^2 (1 - \lambda_j^{2k})}{1 - \lambda_j^2} $$
for all $i,j \in [d]$, so combining this constraint with the trace constraint yields:
$$ \frac{\gamma_j^2 (1 - \lambda_j^{2k})}{1 - \lambda_j^2} \sum_{i=1}^d \frac{1 - \lambda_i^2}{1 - \lambda_i^{2k}} = \gamma^2 \implies \gamma_j^2 = \frac{1 - \lambda_j^2}{1 - \lambda_j^{2k}} \frac{\gamma^2}{\sum_{i=1}^d \frac{1 - \lambda_i^2}{1 - \lambda_i^{2k}}}$$
and thus the $j$th diagonal element will be:
$$\frac{\gamma^2}{\sum_{i=1}^d \frac{1 - \lambda_i^2}{1 - \lambda_i^{2k}}}$$
Consider now some other matrix $\Delta$ that is not necessarily diagonal. Note then that:
\begin{align*}
\lambda_{\min} \left ( \sum_{t=0}^k \Lambda^t (\tilde{\Sigma}^* + \Delta) \Lambda^t \right ) & = \lambda_{\min} \left ( \sum_{t=0}^k \Lambda^t \tilde{\Sigma}^*  \Lambda^t + \sum_{t=0}^k \Lambda^t \Delta  \Lambda^t \right ) \\
& = \lambda_{\min} \left ( \frac{\gamma^2}{\sum_{i=1}^d \frac{1 - \lambda_i^2}{1 - \lambda_i^{2k}}} I + \sum_{t=0}^k \Lambda^t \Delta  \Lambda^t \right ) \\
& = \frac{\gamma^2}{\sum_{i=1}^d \frac{1 - \lambda_i^2}{1 - \lambda_i^{2k}}} + \lambda_{\min} \left ( \sum_{t=0}^k \Lambda^t \Delta  \Lambda^t \right ) 
\end{align*}
For $\tilde{\Sigma}^* + \Delta$ to be in the constraint set, we must have that $Tr(\tilde{\Sigma}^* + \Delta) = \gamma^2 + Tr(\Delta) \leq \gamma^2 \implies Tr(\Delta) \leq 0$. To have that:
$$ \frac{\gamma^2}{\sum_{i=1}^d \frac{1 - \lambda_i^2}{1 - \lambda_i^{2k}}} + \lambda_{\min} \left ( \sum_{t=0}^k \Lambda^t \Delta  \Lambda^t \right ) \geq \frac{\gamma^2}{\sum_{i=1}^d \frac{1 - \lambda_i^2}{1 - \lambda_i^{2k}}}$$
we must have that $ \sum_{t=0}^k \Lambda^t \Delta  \Lambda^t$ is positive definite. However, this is not possible since the diagonal elements of $ \sum_{t=0}^k \Lambda^t \Delta  \Lambda^t$ are the sum of non-negative scalings of the diagonal elements of $\Delta$, and since $\Delta$ must have at least one non-positive element on the diagonal to meet the constraint $Tr(\Delta) \leq 0$, it follows that $ \sum_{t=0}^k \Lambda^t \Delta  \Lambda^t$ has at least one non-positive diagonal element. Since the diagonal elements of every positive definite matrix are positive, $ \sum_{t=0}^k \Lambda^t \Delta  \Lambda^t$ cannot be positive definite, so we cannot increase the value of $\lambda_{\min} \left ( \sum_{t=0}^k \Lambda^t (\tilde{\Sigma}^* + \Delta) \Lambda^t \right )$. By convexity of the constraint set, it follows that the directional derivative in the direction of any other point in our constraint set is negative. Since this is a concave function, it follows that $\tilde{\Sigma}^*$ is optimal.

Thus, the optimal noise will yield a covariance with minimum eigenvalue $\frac{\gamma^2}{\sum_{i=1}^d \frac{1 - \lambda_i^2}{1 - \lambda_i^{2k}}}$. For $k$ sufficiently large, we have that:
$$ \frac{\gamma^2}{\sum_{i=1}^d \frac{1 - \lambda_i^2}{1 - \lambda_i^{2k}}} = \Theta \left ( \frac{\gamma^2}{\| \mathbf{1} - \lambda \|_1} \right )$$

\section{Additional Experimental Results}\label{sec:experiment_additional}

\begin{figure}[H]
\centering
\begin{minipage}{.5\textwidth}
  \centering
  \captionsetup{justification=centering}
  \includegraphics[width=\linewidth]{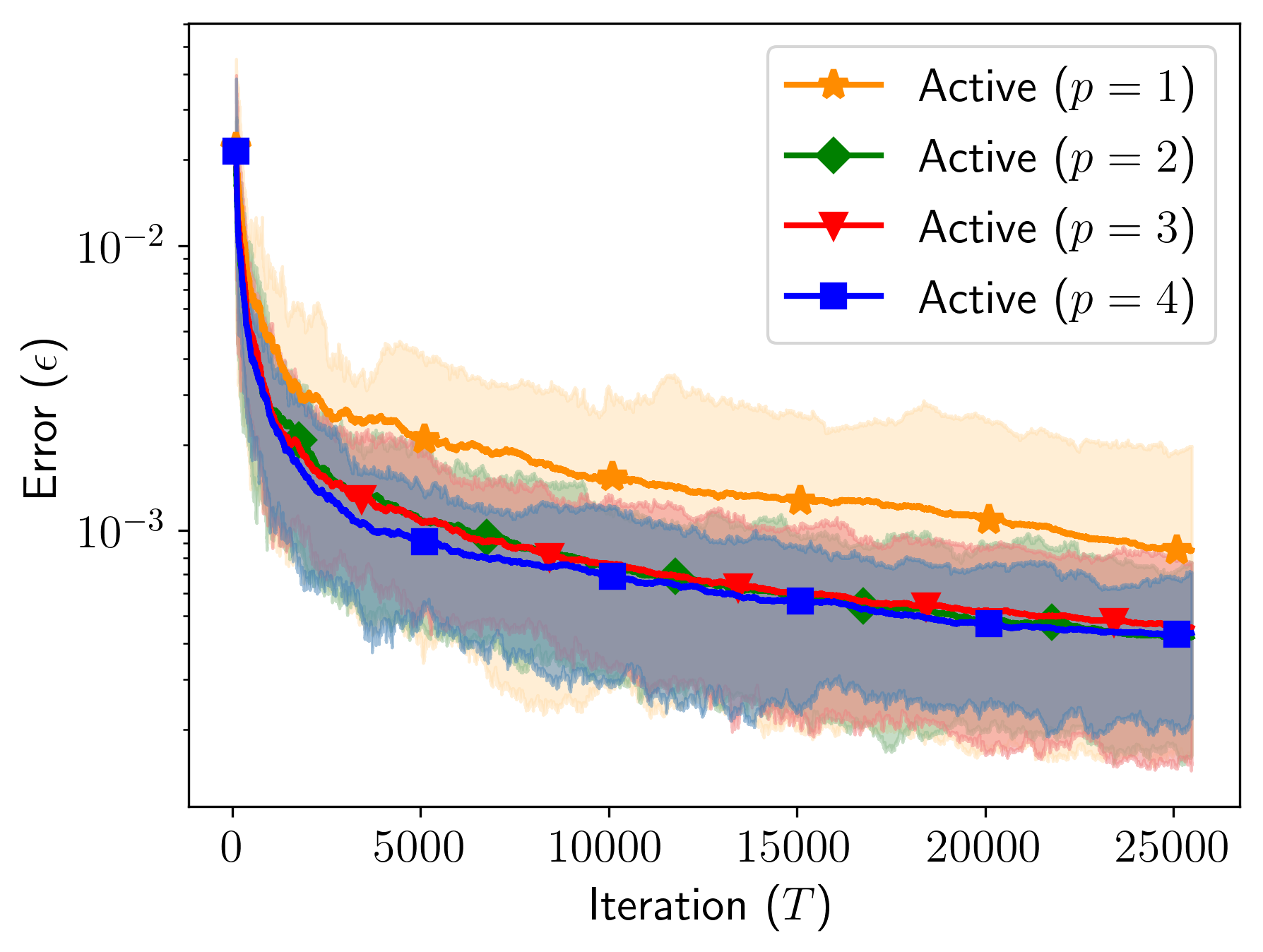}
  \captionof{figure}{$A_*$ Jordan block with $d = 4$, $\rho(A_*) = 0.9$, $B_*$ randomly generated with specified value of $p$}
  \label{fig:varyingp}
\end{minipage}%
\begin{minipage}{.5\textwidth}
  \centering
  \captionsetup{justification=centering}
  \includegraphics[width=\linewidth]{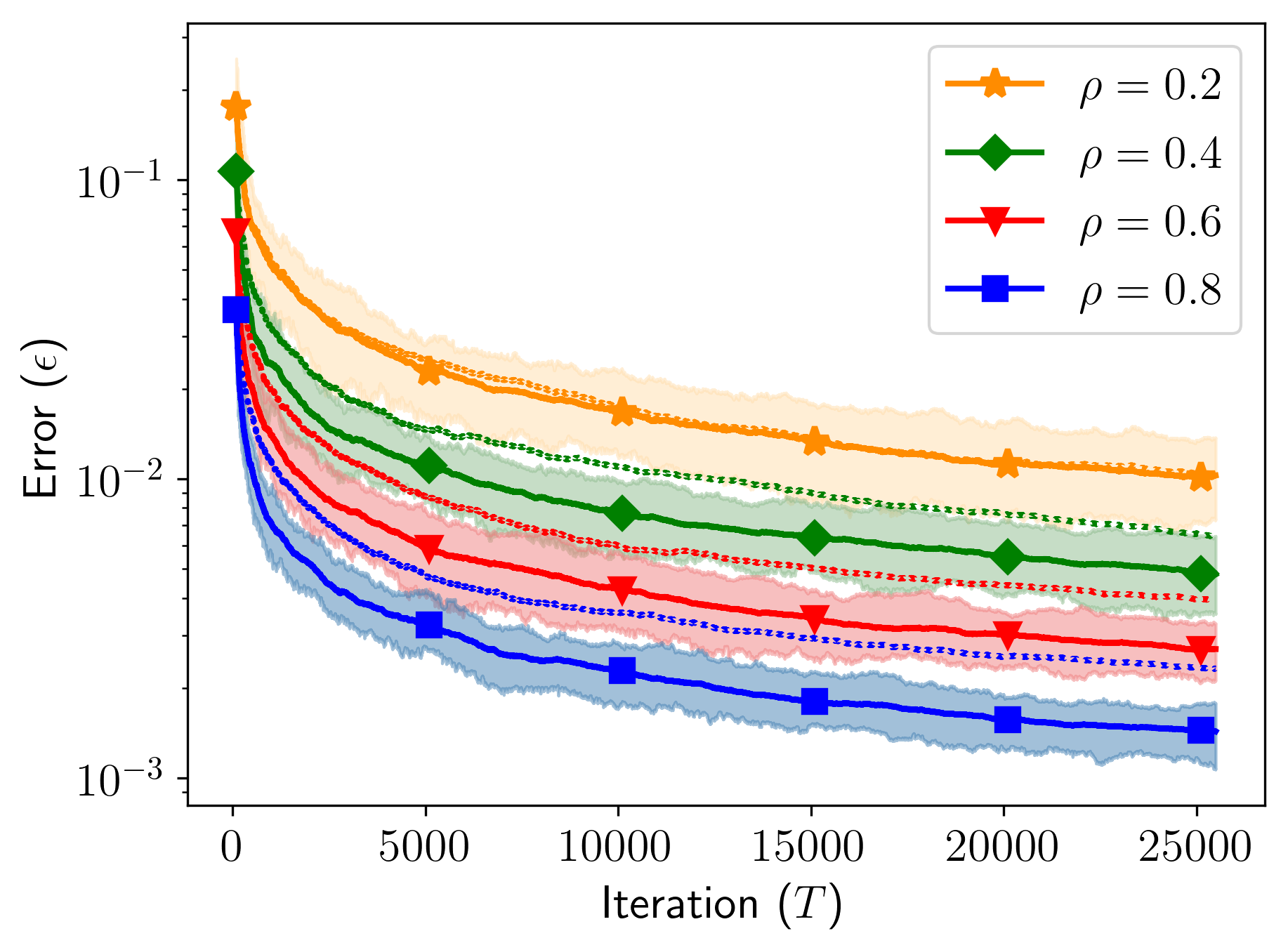}
  \captionof{figure}{$A_*$ diagonalizable by a unitary matrix and has given spectral radius, $p = 4$ and $B_*$ randomly generated. Dotted lines illustrate the performance of $u_t \sim \mathcal{N}(0,\gamma^2 I/p)$ for each value of $\rho$}
  \label{fig:varyingrho}
\end{minipage}
\end{figure}

Figure \ref{fig:varyingp} illustrates how the shape of $B$ can influence the effectiveness of active system identification. With $p = 1$, it is not possible to control the direction of the input, which can greatly reduce the effectiveness of input design. Interestingly, for all $p > 1$, the performance is roughly the same---increasing $p$ beyond 2 does not provide a large gain in the effectiveness of input design.

Figure \ref{fig:varyingrho} plots how the estimation rate depends on the spectral radius. Here the performance of our algorithm is plotted as the solid line and the performance of of isotropic noise as the dotted line. As our theory predicts, systems with a larger spectral radius are easier to estimate. Further, as Corollary \ref{cor:symmetric_a_informal} states, the gap between our algorithm and isotropic noise increases as $\rho$ increases---for $\rho = 0.2$ there is almost no gain in designing inputs actively but as $\rho$ increases the gains of active input design also increase.

\end{document}